\documentclass[english]{article}
\usepackage[T1]{fontenc}
\usepackage[latin9]{inputenc}
\usepackage{geometry}
\usepackage{babel}
\usepackage{verbatim}
\usepackage{float}
\usepackage{bm}
\usepackage{amsmath}
\usepackage{amssymb}
\usepackage{graphicx}
\usepackage{hyperref}
\hypersetup{
 colorlinks,linkcolor=red,anchorcolor=blue,citecolor=blue}

\makeatletter

\floatstyle{ruled}
\newfloat{algorithm}{tbp}{loa}
\providecommand{\algorithmname}{Algorithm}
\floatname{algorithm}{\protect\algorithmname}

\usepackage{cite}\usepackage{amsthm}\usepackage{dsfont}\usepackage{array}\usepackage{mathrsfs}\usepackage{comment}\onecolumn

\usepackage{color}\usepackage{babel}

\newcommand{\R}{\mathbb{R}}

\allowdisplaybreaks

\usepackage{xfrac}
\usepackage{enumitem}
\setlist[itemize]{leftmargin=1em}
\setlist[enumerate]{leftmargin=1em}

\usepackage{babel}
\usepackage{algorithm}% http://ctan.org/pkg/algorithms
\usepackage{algorithmic}% http://ctan.org/pkg/algorithmicx
\usepackage{arydshln}

\newcommand{\ba}{\bm{a}}
\newcommand{\bb}{\bm{b}}

\newcommand{\be}{\bm{e}}

\newcommand{\bu}{\bm{u}}

\newcommand{\bA}{\bm{A}}
\newcommand{\bB}{\bm{B}}
\newcommand{\bC}{\bm{C}}
\newcommand{\bD}{\bm{D}}
\newcommand{\bE}{\bm{E}}
\newcommand{\bF}{\bm{F}}

\newcommand{\bH}{\bm{H}}
\newcommand{\bI}{\bm{I}}

\newcommand{\bM}{\bm{M}}

\newcommand{\bQ}{\bm{Q}}
\newcommand{\bR}{\bm{R}}

\newcommand{\bU}{\bm{U}}
\newcommand{\bV}{\bm{V}}
\newcommand{\bW}{\bm{W}}
\newcommand{\bX}{\bm{X}}
\newcommand{\bY}{\bm{Y}}

\newcommand{\cO}{\mathcal{O}}
\newcommand{\cP}{\mathcal{P}}

\newcommand{\EE}{\mathbb{E}}

\newcommand{\bDelta}{\bm{\Delta}}

\newcommand{\bSigma}{\bm{\Sigma}}

\DeclareMathOperator{\ind}{\mathds{1}}  % Indicator

\newcommand{\norm}[1]{\left\|#1\right\|}

\newcommand{\inner}[2]{\left\langle #1,#2 \right\rangle}

\newcommand{\twonorm}[1]{\left\Vert#1\right\Vert_2}
\newcommand{\twotoinftynorm}[1]{\left\Vert#1\right\Vert_{2,\infty}}
\newcommand{\inftynorm}[1]{\left\Vert#1\right\Vert_\infty}
\newcommand{\Fnorm}[1]{\left\Vert#1\right\Vert_\mathrm{F}}
\newcommand{\nuclearnorm}[1]{\left\Vert#1\right\Vert_\ast}
\newcommand{\trace}[1]{\mathrm{Tr}\left(#1\right)}
\newcommand{\abs}[1]{\left|#1\right|}

\definecolor{yxc}{RGB}{255,0,0}
\definecolor{yjc}{RGB}{125,0,0}
\definecolor{cm}{RGB}{0,0,200}
\definecolor{yly}{RGB}{0,150,0}
\newcommand{\cm}[1]{\textcolor{cm}{[CM: #1]}}
\newcommand{\yxc}[1]{\textcolor{yxc}{[YXC: #1]}}

\newcommand{\yjc}[1]{\textcolor{yjc}{[YJC: #1]}}

\usepackage{babel}

\makeatother

\begin{document}
\theoremstyle{plain} \newtheorem{lemma}{\textbf{Lemma}} \newtheorem{prop}{\textbf{Proposition}}\newtheorem{theorem}{\textbf{Theorem}}\setcounter{theorem}{0}
\newtheorem{corollary}{\textbf{Corollary}} \newtheorem{example}{\textbf{Example}}
\newtheorem{definition}{\textbf{Definition}} \newtheorem{fact}{\textbf{Fact}}
\newtheorem{claim}{\textbf{Claim}}

\theoremstyle{definition}

\theoremstyle{remark}\newtheorem{remark}{\textbf{Remark}}\newtheorem{conjecture}{Conjecture}\newtheorem{condition}{\textbf{Condition}}\newtheorem{assumption}{\textbf{Assumption}}
\title{Noisy Matrix Completion: Understanding Statistical  Guarantees for  Convex Relaxation via Nonconvex Optimization\footnotetext{Author
names are sorted alphabetically.}}

\author{Yuxin Chen\thanks{Department of Electrical Engineering, Princeton University, Princeton,
NJ 08544, USA; Email: \texttt{yuxin.chen@princeton.edu}. } \and Yuejie Chi\thanks{Department of Electrical and Computer Engineering, Carnegie Mellon
University, Pittsburgh, PA 15213, USA; Email: \texttt{yuejiechi@cmu.edu}. } \and Jianqing Fan\thanks{Department of Operations Research and Financial Engineering, Princeton
University, Princeton, NJ 08544, USA; Email: \texttt{\{jqfan, congm,
yulingy\}@princeton.edu}.} \and Cong Ma\footnotemark[3] \and Yuling Yan\footnotemark[3] }

%\date{February 2019; \quad Revised September 2019}
%\date{}
\maketitle
\begin{abstract}
This paper studies noisy low-rank matrix completion: given partial
and noisy entries of a large low-rank matrix, the goal is to estimate
the underlying matrix faithfully and efficiently. Arguably one of
the most popular paradigms to tackle this problem is convex relaxation,
which achieves remarkable efficacy in practice. However, the theoretical
support of this approach is still far from optimal in the noisy setting,
falling short of explaining its empirical success.

We make progress towards demystifying the practical efficacy of convex
relaxation {vis-à-vis} random noise. When the rank and the condition number of the unknown matrix
are bounded by a constant, we demonstrate that the convex programming approach
achieves near-optimal estimation errors --- in terms of the Euclidean
loss, the entrywise loss, and the spectral norm loss --- for a wide
range of noise levels.  All of this is enabled by bridging convex
relaxation with the nonconvex Burer--Monteiro approach, a seemingly
distinct algorithmic paradigm that is provably robust against noise.
More specifically, we show that an approximate critical point of the
nonconvex formulation serves as an extremely tight approximation of
the convex solution, thus allowing us to transfer the desired statistical
guarantees of the nonconvex approach to its convex counterpart.

\end{abstract}


\tableofcontents{}

\section{Introduction}

Suppose we are interested in a large low-rank data matrix, but only get to
observe a highly incomplete subset of its entries. Can we hope to
estimate the underlying data matrix in a reliable manner? This problem,
often dubbed as \emph{low-rank matrix completion}, spans a diverse
array of science and engineering applications (e.g.~collaborative filtering~\cite{rennie2005fast},
localization~\cite{so2007theory}, system identification~\cite{liu2009interior},
magnetic resonance parameter mapping~\cite{zhang2015accelerating}, joint alignment
\cite{chen2016projected}), and has inspired a flurry of research
activities in the past decade. In the statistics literature, matrix completion also falls under the category of  factor models with a large amount of missing data, which finds numerous statistical applications such as controlling false discovery rates for dependence data~\cite{efron2007correlation,efron2010correlated,fan2012estimating,fan2019farmtest}, factor-adjusted variable selection~\cite{kneip2011factor,fan2018factor}, principal component regression
\cite{jolliffe1982note,bai2006confidence,paul2008preconditioning, fan2017sufficient}, and large covariance matrix estimation\cite{fan2013large,fan2019robust}.  
Recent years have witnessed the development of 
many tractable algorithms that come with statistical guarantees, with convex relaxation being one of the most popular paradigms~\cite{fazel2004rank,ExactMC09,CanTao10}.
See~\cite{davenport2016overview,chen2018harnessing} for an overview
of this topic.

This paper focuses on noisy low-rank matrix completion, assuming that the revealed
entries are corrupted by a certain amount of noise. Setting the stage,
consider the task of estimating a rank-$r$ data matrix ${\bm{M}^{\star}=[M_{ij}^{\star}]_{1\leq i,j\leq n}\in\mathbb{R}^{n\times n}}$,\footnote{It is straightforward to rephrase our discussions to a general rectangular
matrix of size $n_{1}\times n_{2}$. The current paper sets $n=n_{1}=n_{2}$
throughout for simplicity of presentation.  
%Additionally, in the context of statistical factor models,  one can think of $\bM^*$ as the factor loading matrix multiplied by $r$ latent factors.
} and suppose that this needs to be performed on the basis of a subset
of noisy entries
\begin{equation}
M_{ij}=M_{ij}^{\star}+E_{ij},\qquad(i,j)\in\Omega,\label{eq:noise-model}
\end{equation}
where $\Omega\subseteq\{1,\cdots,n\}\times\{1,\cdots,n\}$ denotes
a set of indices, and $E_{ij}$ stands for the additive noise at the
location $(i,j)$. As we shall elaborate shortly, solving noisy matrix
completion via convex relaxation, while practically exhibiting excellent stability 
(in terms of the estimation errors against noise), is far less understood theoretically compared to the
noiseless setting.

\begin{comment}
Mathematically, one can pose the problem as follows. Let $\bm{M}^{\star}\in\mathbb{R}^{n_{1}\times n_{2}}$
be a rank-$r$ matrix. Suppose one observes the noisy entries $\{M_{ij}^{\star}+E_{ij}\}_{(i,j)\in\Omega}$,
where $\Omega\subset[n_{1}]\times[n_{2}]$ denotes the index set and
$\bm{E}=[E_{ij}]$ represents the additive noise (note that the values
of $\bm{E}$ outside of $\Omega$ do not matter). Throughout this
paper, we will focus on the i.i.d. Bernoulli model, i.e.~each $(i,j)\in[n_{1}]\times[n_{2}]$
belongs to $\Omega$ with probability $p$ independently. Then the
goal of matrix completion is to estimate the underlying matrix $\bm{M}^{\star}$
as accurately as possible given its partial entries $\{M_{ij}^{\star}+E_{ij}\}_{(i,j)\in\Omega}$.
include both convex relaxation (in particular, nuclear norm minimization)
\cite{ExactMC09,CanTao10,Gross2011recovering,recht2011simpler,Negahban2012restricted,MR2906869}
and nonconvex optimization~\cite{KesMonSew2010,Se2010Noisy,sun2016guaranteed,chen2015fast,ma2017implicit,zheng2016convergence}.
\end{comment}

\subsection{Convex relaxation: limitations of prior results\label{subsec:Nuclear-norm-minimization}}

Naturally, one would search for a low-rank solution that best fits
the observed entries. One choice is the regularized least-squares
formulation given by
\begin{equation}
\underset{\bm{Z}\in\mathbb{R}^{n\times n}}{\text{minimize}}\qquad\frac{1}{2}\sum_{(i,j)\in\Omega}\big(Z_{ij}-M_{ij}\big)^{2}+\lambda\,\mathsf{rank}(\bm{Z}),\label{eq:LS-rank}
\end{equation}
where $\lambda>0$ is some regularization parameter. In words,
this approach optimizes certain trade-off between the goodness of fit
(through the squared loss expressed in the first term of~\eqref{eq:LS-rank})
and the low-rank structure (through the rank function in the second
term of~\eqref{eq:LS-rank}). Due to computational intractability
of rank minimization, we often resort to convex relaxation in order
to obtain computationally feasible solutions. One notable example
is the following convex program:
\begin{equation}
\underset{\bm{Z}\in\mathbb{R}^{n\times n}}{\text{minimize}}\qquad g(\bm{Z})\triangleq\frac{1}{2}\sum_{(i,j)\in\Omega}\big(Z_{ij}-M_{ij}\big)^{2}+\lambda\left\Vert \bm{Z}\right\Vert _{*},\label{eq:convex-LS}
\end{equation}
where $\|\bm{Z}\|_{*}$ denotes the nuclear norm (i.e.~the sum of
singular values) of $\bm{Z}$ --- a convex surrogate for the rank
function. A significant portion of existing theory supports the use
of this paradigm in the noiseless setting: when $E_{ij}$ vanishes
for all $(i,j)\in\Omega$, the solution to~\eqref{eq:convex-LS} is
known to be faithful (i.e.~the estimation error becomes zero)
even under near-minimal sample complexity~\cite{ExactMC09,CanPla10,CanTao10,Gross2011recovering,recht2011simpler,chen2015incoherence}.

%By contrast, the performance of convex relaxation remains largely
%unclear when it comes to more practically relevant noisy scenarios. To begin with, the stability of an equivalent
%variant of~\eqref{eq:convex-LS} against noise was first studied by Cand\`es and
%Plan~\cite{CanPla10}.\footnote{Technically,~\cite{CanPla10} deals with the constrained version of~(\ref{eq:convex-LS}), which is equivalent to the Lagrangian form
%as in~(\ref{eq:convex-LS}) with a proper choice of the regularization
%parameter.} The estimation error $\|\bm{Z}_{\mathsf{cvx}}-\bm{M}^{\star}\|_{\mathrm{F}}$
%derived therein, of the solution $\bm{Z}_{\mathsf{cvx}}$ to~\eqref{eq:convex-LS},
%is significantly larger than the oracle lower bound. This
%does not explain well the effectiveness of~\eqref{eq:convex-LS} in
%practice.
%In fact, the numerical experiments reported in~\cite{CanPla10} already indicated that the performance of convex relaxation is far better than their theoretical bounds.
%In order to improve the statistical guarantees,
%several variants of~\eqref{eq:convex-LS} have been put forward, most
%notably by Negahban and Wainwright~\cite{Negahban2012restricted}
%and Koltchinskii et al.~\cite{MR2906869}; see Section~\ref{subsec:Models-and-main}
%for more details. Nevertheless, the stability analysis in~\cite{Negahban2012restricted,MR2906869}
%could often be suboptimal when the magnitudes of the noise are not sufficiently
%large; in fact, their estimation error bounds do not vanish as the
%size of the noise approaches zero.

By contrast, the performance of convex relaxation remains largely
unclear when it comes to noisy settings (which are often more practically relevant). Cand\`es and
Plan~\cite{CanPla10} first studied the stability of an equivalent
variant\footnote{Technically,~\cite{CanPla10} deals with the constrained version of~(\ref{eq:convex-LS}), which is equivalent to the Lagrangian form
as in~(\ref{eq:convex-LS}) with a proper choice of the regularization
parameter.} of~\eqref{eq:convex-LS} against noise. The estimation error $\|\bm{Z}_{\mathsf{cvx}}-\bm{M}^{\star}\|_{\mathrm{F}}$
derived therein, of the solution $\bm{Z}_{\mathsf{cvx}}$ to~\eqref{eq:convex-LS},
is significantly larger than the oracle lower bound. This
does not explain well the effectiveness of~\eqref{eq:convex-LS} in
practice. In fact, the numerical experiments reported in~\cite{CanPla10} already indicated that the performance of convex relaxation is far better than their theoretical bounds. This discrepancy between numerical performance and existing theoretical bounds gives rise to the following natural yet challenging questions: \emph{Where
does the convex program~(\ref{eq:convex-LS}) stand in terms of its
stability vis-à-vis additive noise? Can we establish 
statistical performance guarantees that match its practical effectiveness?}

We note in passing that several other convex relaxation formulations have been thoroughly analyzed for noisy matrix completion, most notably by Negahban and Wainwright~\cite{Negahban2012restricted} and by Koltchinskii et~al.~\cite{MR2906869}. These works have significantly advanced our understanding of the power of convex relaxation. However, the estimators studied therein, particularly the one in \cite{MR2906869}, are quite different from the one (\ref{eq:convex-LS}) considered here; as a consequence, the analysis therein does not lead to improved statistical guarantees of (\ref{eq:convex-LS}). Moreover, the performance guarantees provided for these variants are also suboptimal when restricted to the class of ``incoherent'' or ``de-localized'' matrices, unless the magnitudes of the noise are fairly large. See Section~\ref{subsec:Models-and-main} for more detailed discussions as well as numerical comparisons of these algorithms. 

%In order to improve the ,
%several variants of~\eqref{eq:convex-LS} ; see Section~\ref{subsec:Models-and-main}
%for more details. While these works significantly improve our understanding about the power of convex relaxation, the stability analysis presented therein (e.g.~\cite{Negahban2012restricted}) remains suboptimal when the magnitudes of the noise are not sufficiently
%large; in fact, their estimation error bounds often do not vanish even when the
%magnitudes of the noise approach zero. 

%Is
%it capable of accommodating a wider (and more practical) range of 
%noise levels?
\subsection{A detour: nonconvex optimization}

While the focus of the current paper is convex relaxation, we take
a moment to discuss a seemingly distinct algorithmic paradigm: nonconvex
optimization, which turns out to be remarkably helpful in understanding
convex relaxation. Inspired by the Burer--Monteiro approach~\cite{burer2003nonlinear},
the nonconvex scheme starts by representing the rank-$r$ decision matrix (or parameters) $\bm{Z}$ as $\bm{Z}=\bm{X}\bm{Y}^{\top}$ via  low-rank
factors $\bm{X},\bm{Y}\in\mathbb{R}^{n\times r}$, and proceeds by
solving the  following nonconvex (regularized) least-squares problem~\cite{KesMonSew2010}
\begin{equation}
\underset{\bm{X},\bm{Y}\in\mathbb{R}^{n\times r}}{\text{minimize}}\qquad\frac{1}{2}\sum_{(i,j)\in\Omega}\big[\big(\bm{X}\bm{Y}^{\top}\big)_{ij}-M_{ij}\big]^{2}+\mathsf{reg}(\bm{X},\bm{Y}).\label{eq:nonconvex_mc_noisy-general}
\end{equation}
Here, $\mathsf{reg}(\cdot,\cdot)$ denotes a certain regularization
term that promotes additional structural properties. 

To see its intimate connection with
the convex program~(\ref{eq:convex-LS}), we make the following observation: if the solution
to~(\ref{eq:convex-LS}) has rank $r$, then it must coincide with
the solution to %the following auxiliary nonconvex optimization problem.
\begin{equation}
\underset{\bm{X},\bm{Y}\in\mathbb{R}^{n\times r}}{\text{minimize}}\qquad\frac{1}{2}\sum_{(i,j)\in\Omega}
\big[\big(\bm{X}\bm{Y}^{\top}\big)_{ij}-M_{ij}\big]^{2}+
\underset{\mathsf{reg}(\bm{X},\bm{Y})}{\underbrace{\frac{\lambda}{2}
\|\bm{X}\|_{\mathrm{F}}^{2}+\frac{\lambda}{2}\|\bm{Y}\|_{\mathrm{F}}^{2}}}.
\label{eq:nonconvex_mc_noisy-lambda}
\end{equation}
This can be easily verified by recognizing the elementary fact that
\begin{equation}
\|\bm{Z}\|_{*}=\inf_{\bm{X},\bm{Y}\in\mathbb{R}^{n\times r}:\bm{X}\bm{Y}^{\top}=\bm{Z}}\left\{ \tfrac{1}{2}\|\bm{X}\|_{\mathrm{F}}^{2}+\tfrac{1}{2}\|\bm{Y}\|_{\mathrm{F}}^{2}\right\} \label{eq:nuclear-fro-relation}
\end{equation}
for any rank-$r$ matrix $\bm{Z}$~\cite{srebro2005rank,mazumder2010spectral}. Note, however, that it is very challenging to predict when the key
assumption in establishing this connection --- namely, the rank-$r$
assumption of the solution to the convex program (\ref{eq:convex-LS}) --- can possibly
hold (and in particular, whether it can hold under minimal sample complexity requirement). 

%The auxiliary nonconvex optimization problem~\eqref{eq:nonconvex_mc_noisy-lambda} does not mean to be solved in practice.  It is created to help us to construct a random quantity (can depend on both the data and unknown parameters) whose statistical errors can more easily be derived and is close to the unique optimizer of~\eqref{eq:convex-LS}.  Indeed, we can even run a gradient descend algorithm starting from the true values $\bX^\star$ and $\bY^\star$ to construct such a random quantity, as long as this sequence can help us establish statistical properties of the estimator~\eqref{eq:convex-LS}.

Despite the nonconvexity of~(\ref{eq:nonconvex_mc_noisy-general}),
simple first-order optimization methods, in conjunction with proper
initialization, are often effective in solving
(\ref{eq:nonconvex_mc_noisy-general}). Partial examples include gradient
descent on manifold~\cite{KesMonSew2010,Se2010Noisy,wei2016guarantees},
gradient descent~\cite{sun2016guaranteed,ma2017implicit}, and projected
gradient descent~\cite{chen2015fast,zheng2016convergence}. Apart
from their practical efficiency, the nonconvex optimization approach
is also appealing in theory. To begin with, algorithms tailored to
\eqref{eq:nonconvex_mc_noisy-general} often enable exact recovery
in the noiseless setting. Perhaps more importantly, for a wide range
of noise settings, the nonconvex approach achieves appealing
estimation accuracy~\cite{chen2015fast,ma2017implicit}, which could be significantly
better than those bounds derived for convex relaxation discussed earlier.
See~\cite{chi2018nonconvex,chen2018harnessing} for a summary of recent
results. Such intriguing statistical guarantees  motivate us to take
a closer inspection of the underlying connection between the two contrasting
algorithmic frameworks.

\subsection{Empirical evidence: convex and nonconvex solutions are often close\label{subsec:Numerics-cvx-ncvx}}

In order to obtain a better sense of the relationships between convex
and nonconvex approaches, we begin by comparing the estimates returned
by the two approaches via numerical experiments. Fix $n=1000$ and
$r=5$. We generate $\bm{M}^{\star}=\bm{X}^{\star}\bm{Y}^{\star\top}$,
where $\bm{X}^{\star},\bm{Y}^{\star}\in\mathbb{R}^{n\times r}$ are
random orthonormal matrices. Each entry $M_{ij}^{\star}$ of $\bm{M}^{\star}$
is observed with probability $p=0.2$ independently, and then corrupted
by an independent Gaussian noise $E_{ij}\sim\mathcal{N}(0,\sigma^{2})$.
Throughout the experiments, we set $\lambda=5\sigma\sqrt{np}$. The
convex program~(\ref{eq:convex-LS}) is solved by the proximal gradient
method~\cite{parikh2014proximal}, whereas we attempt solving the
nonconvex formulation~(\ref{eq:nonconvex_mc_noisy-lambda}) by gradient
descent with spectral initialization (see~\cite{chi2018nonconvex}
for details). Let $\bm{Z}_{\mathsf{cvx}}$ (resp.~$\bm{Z}_{\mathsf{ncvx}}=\bm{X}_{\mathsf{ncvx}}\bm{Y}_{\mathsf{ncvx}}^{\top}$)
be the solution returned by the convex program~(\ref{eq:convex-LS})
(resp.~the nonconvex program~(\ref{eq:nonconvex_mc_noisy-lambda})).
Figure~\ref{fig:nnm_error} displays the relative estimation errors
of both methods ($\|\bm{Z}_{\mathsf{cvx}}-\bm{M}^{\star}\|_{\mathrm{F}}/\|\bm{M}^{\star}\|_{\mathrm{F}}$
and $\|\bm{Z}_{\mathsf{ncvx}}-\bm{M}^{\star}\|_{\mathrm{F}}/\|\bm{M}^{\star}\|_{\mathrm{F}}$)
as well as the relative distance $\|\bm{Z}_{\mathsf{cvx}}-\bm{Z}_{\mathsf{ncvx}}\|_{\mathrm{F}}/\|\bm{M}^{\star}\|_{\mathrm{F}}$
between the two estimates. The results are averaged over 20 independent
trials.

\begin{figure}
\center

\includegraphics[scale=0.4]{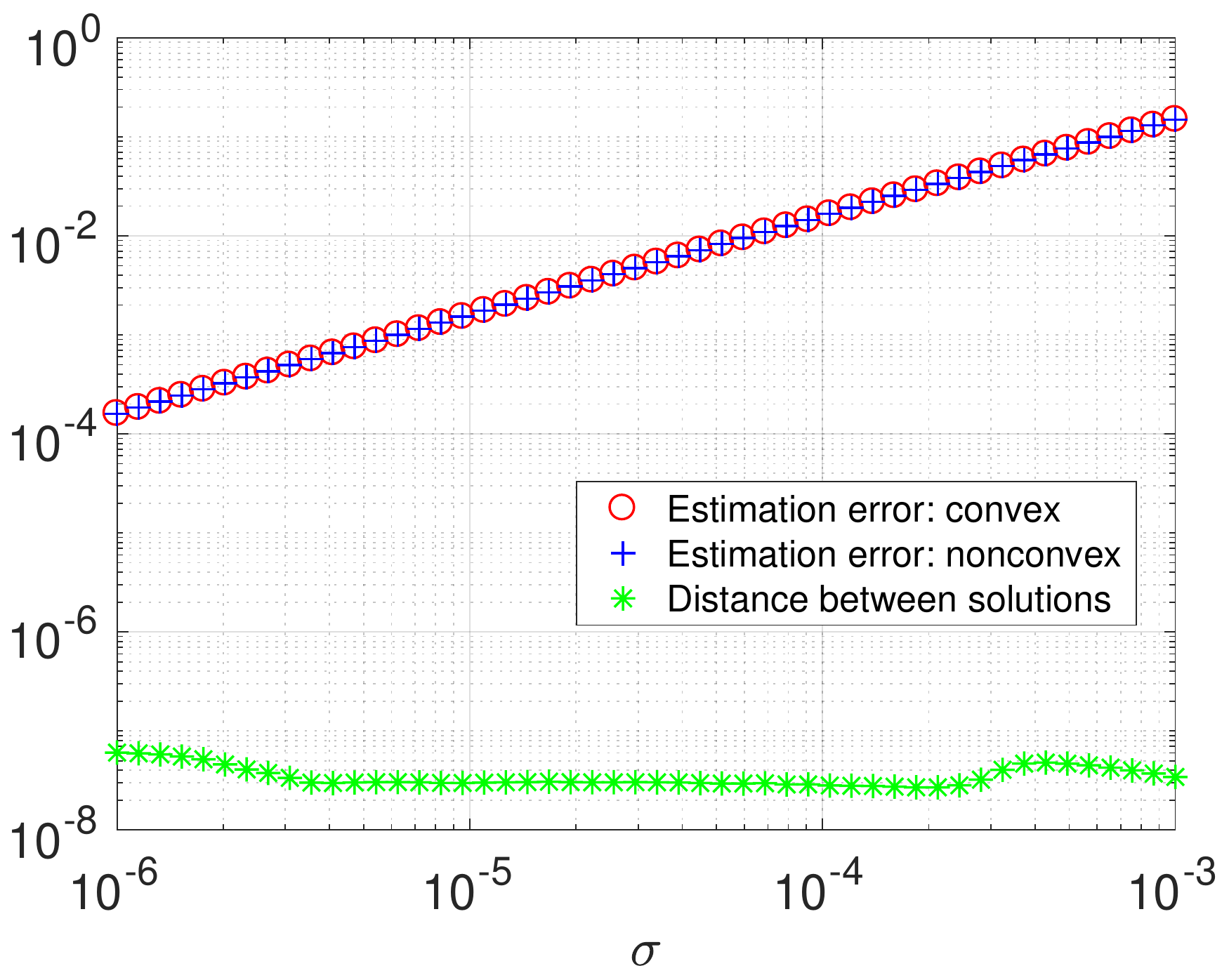}\caption{The relative estimation errors of both $\bm{Z}_{\mathsf{cvx}}$ (the
estimate of the convex program~(\ref{eq:convex-LS})) and $\bm{Z}_{\mathsf{ncvx}}$
(the estimate returned by the nonconvex approach tailored to~(\ref{eq:nonconvex_mc_noisy-lambda}))
and the relative distance between them vs.~the standard deviation
$\sigma$ of the noise. The results are reported for $n=1000$, $r=5$,
$p=0.2$, $\lambda=5\sigma\sqrt{np}$ and are averaged over 20 independent
trials. \label{fig:nnm_error}}
\end{figure}

Interestingly, the distance between the convex and the nonconvex solutions
seems extremely small (e.g.~$\|\bm{Z}_{\mathsf{cvx}}-\bm{Z}_{\mathsf{ncvx}}\|_{\mathrm{F}}/\|\bm{M}^{\star}\|_{\mathrm{F}}$
is typically below $10^{-7}$); in comparison, the relative estimation
errors of both $\bm{Z}_{\mathsf{cvx}}$ and $\bm{Z}_{\mathsf{ncvx}}$
are substantially larger. In other words, the estimate returned by
the nonconvex approach serves as a remarkably accurate approximation
of the convex solution. Given that the nonconvex approach is often
guaranteed to achieve intriguing statistical guarantees vis-à-vis random noise
\cite{ma2017implicit}, this suggests that the convex program is equally
stable --- a phenomenon that was not captured by  prior theory
\cite{CanPla10}. \emph{Can we leverage existing theory for the nonconvex
scheme to improve the statistical analysis of the convex relaxation
approach?}

Before continuing, we remark that the above numerical connection between convex relaxation~(\ref{eq:convex-LS}) and nonconvex  optimization~(\ref{eq:nonconvex_mc_noisy-lambda}) has already been observed multiple times in prior literature~\cite{fazel2002matrix, srebro2005rank, RecFazPar07, mazumder2010spectral, Se2010Noisy}. Nevertheless, all prior observations on this connection were either completely empirical, or provided in a way that does not lead to improved statistical error bounds of the convex paradigm (\ref{eq:convex-LS}). In fact, the difficulty in rigorously justifying the above numerical observations has been noted in the literature; see e.g.~\cite{Se2010Noisy}.\footnote{The seminal work~\cite{Se2010Noisy} by Keshavan, Montanari and Oh stated that ``\emph{In view of the identity (\ref{eq:nuclear-fro-relation}) it might be possible to use the results in this paper to prove stronger guarantees on the nuclear norm minimization approach. Unfortunately this
implication is not immediate $\ldots$ Trying to establish such
an implication, and clarifying the relation between the two approaches is nevertheless a
promising research direction.}''}

\subsection{Models and main results \label{subsec:Models-and-main}}

The numerical experiments reported in Section~\ref{subsec:Numerics-cvx-ncvx}
suggest an alternative route for analyzing convex relaxation for noisy
matrix completion. If one can formally justify the proximity between
the convex and the nonconvex solutions, then it is possible to propagate
the appealing stability guarantees from the nonconvex scheme to the
convex approach. As it turns out, this simple idea leads to significantly
enhanced statistical guarantees for the convex program~(\ref{eq:convex-LS}),
which we formally present in this subsection.

\subsubsection{Models and assumptions}

Before proceeding, we introduce a few model assumptions that play
a crucial role in our theory.

\smallskip
\begin{assumption}\label{assumption:sampling-noise}\quad{}

\begin{itemize}[leftmargin=2.5em]\item[(a)] \textbf{(Random sampling)}
Each index $(i,j)$ belongs to the index set $\Omega$ independently with probability
$p$.

\item[(b)] \textbf{(Random noise)} The noise matrix $\bm{E}=[E_{ij}]_{1\leq i,j\leq n}$
is composed of i.i.d.~zero-mean sub-Gaussian random variables with sub-Gaussian
norm at most $\sigma>0$, i.e.~$\|E_{ij}\|_{\psi_{2}}\leq\sigma$
(see \cite[Definition 5.7]{Vershynin2012}).

\end{itemize}

\end{assumption}
\smallskip

In addition, let $\bm{M}^{\star}=\bm{U}^{\star}\bm{\Sigma}^{\star}\bm{V}^{\star\top}$
be the singular value decomposition (SVD) of $\bm{M}^{\star}$, where
${\bm{U}^{\star},\bm{V}^{\star}\in\mathbb{R}^{n\times r}}$ consist
of orthonormal columns and $\bm{\Sigma}^{\star}=\mathsf{diag}(\sigma_{1}^{\star},\sigma_{2}^{\star},\cdots,\sigma_{r}^{\star})\in\mathbb{R}^{r\times r}$
is a diagonal matrix obeying $\sigma_{\max}\triangleq\sigma_{1}^{\star}\geq\sigma_{2}^{\star}\geq\cdots\geq\sigma_{r}^{\star}\triangleq\sigma_{\min}$.
Denote by $\kappa\triangleq\sigma_{\max}/\sigma_{\min}$ the condition
number of $\bm{M}^{\star}$. We impose the following incoherence condition
on $\bm{M}^{\star}$, which is known to be crucial for reliable recovery
of $\bm{M}^{\star}$~\cite{ExactMC09,chen2015incoherence}.

\begin{definition}\label{def:incoherence}A rank-$r$ matrix $\bm{M}^{\star}\in\mathbb{R}^{n\times n}$
with SVD $\bm{M}^{\star}=\bm{U}^{\star}\bm{\Sigma}^{\star}\bm{V}^{\star\top}$
is said to be $\mu$-incoherent if
\[
\left\Vert \bm{U}^{\star}\right\Vert _{2,\infty}\leq\sqrt{\frac{\mu}{n}}\left\Vert \bm{U}^{\star}\right\Vert _{\mathrm{F}}=\sqrt{\frac{\mu r}{n}}\qquad\text{and}\qquad\left\Vert \bm{V}^{\star}\right\Vert _{2,\infty}\leq\sqrt{\frac{\mu}{n}}\left\Vert \bm{V}^{\star}\right\Vert _{\mathrm{F}}=\sqrt{\frac{\mu r}{n}}.
\]
Here, $\| \bm{U}\|_{2,\infty}$ denotes the largest
$\ell_{2}$ norm of all rows of a matrix $\bm{U}$.

\end{definition}

\begin{remark}
It is worth noting that several other conditions on the low-rank matrix have been proposed in the noisy setting. Examples include the spikiness condition~\cite{Negahban2012restricted} and the bounded $\ell_{\infty}$ norm condition~\cite{MR2906869}. However, these conditions alone are often unable to ensure identifiability of the true matrix even in the absence of noise. 
\end{remark}

\subsubsection{Theoretical guarantees: when both the rank and the condition number are constants}
With these in place, we are positioned to present our improved statistical
guarantees for convex relaxation. 
For convenience of presentation, 
we shall begin with a simple yet fundamentally important class of settings when the rank $r$ and the condition number $\kappa$ are both fixed constants. As it turns out, this class of problems arises in  a variety of engineering applications. For example, in a fundamental problem in cryo-EM called angular synchronization \cite{singer2011angular}, one needs to deal with rank-2 or rank-3 matrices with $\kappa = 1$; in a joint shape mapping problem that arises in computer graphics \cite{huang2013consistent,chen2014matching}, the matrix under consideration has low rank and a condition number equal to 1;  and in structure from motion in computer vision \cite{tomasi1992shape}, one often seeks to estimate a matrix with $r\leq 3$ and a small condition number.
Encouragingly, our theory delivers near-optimal statistical guarantees for such practically important scenarios.

%and in sensor network localization \cite{so2007theory}, one often needs to estimate a matrix with $r\leq 3$ and a small condition number. 

%Its connection to nonconvex optimization, which is at the core of the analysis and also more complicated to describe precisely, is deferred to Section~\ref{sec:Proof-achitecture}.

\begin{theorem}\label{thm:main-convex-simplified}Let $\bm{M}^{\star}$ be rank-$r$
and $\mu$-incoherent with a condition number $\kappa$, where the rank and the condition number satisfy $r,\kappa=O(1)$. Suppose that Assumption~\ref{assumption:sampling-noise}
holds and take $\lambda=C_{\lambda}\sigma\sqrt{np}$ in~(\ref{eq:convex-LS})
for some large enough constant $C_{\lambda}>0$. Assume the sample
size obeys $n^{2}p\geq C\mu^{2}n\log^{3}n$ for some
sufficiently large constant $C>0$, and the noise satisfies $\sigma\lesssim\sqrt{\frac{np}{\mu^3 \log n}}\left\Vert \bm{M}^{\star}\right\Vert _{\infty}$
for some sufficiently small constant $c>0$. Then with probability
exceeding $1-O(n^{-3})$: 

\begin{enumerate}
\item Any minimizer $\bm{Z}_{\mathsf{cvx}}$ of~(\ref{eq:convex-LS})
obeys \begin{subequations}\label{eq:Zcvx-error-simplified}
\begin{align}
\big\|\bm{Z}_{\mathsf{cvx}}-\bm{M}^{\star}\big\|_{\mathrm{F}}\,\, & \lesssim\frac{\sigma}{\sigma_{\min}}\sqrt{\frac{n}{p}}\,\big\|\bm{M}^{\star}\big\|_{\mathrm{F}};\quad\big\|\bm{Z}_{\mathsf{cvx}}-\bm{M}^{\star}\big\| \lesssim\frac{\sigma}{\sigma_{\min}}\sqrt{\frac{n}{p}}\,\big\|\bm{M}^{\star}\big\|;\label{eq:main-fro-norm-error-simplified}\\
\big\|\bm{Z}_{\mathsf{cvx}}-\bm{M}^{\star}\big\|_{\infty} & \lesssim\frac{\sigma}{\sigma_{\min}}\sqrt{\frac{\mu n\log n}{p}}\,\big\|\bm{M}^{\star}\big\|_{\infty}. 
\label{eq:main-inf-norm-error-simplified}
%\\
%\big\|\bm{Z}_{\mathsf{cvx}}-\bm{M}^{\star}\big\|\,\,\,\,\, & \lesssim\frac{\sigma}{\sigma_{\min}}\sqrt{\frac{n}{p}}\,\big\|\bm{M}^{\star}\big\|;\label{eq:main-spectral-norm-error-simplified}
\end{align}
\end{subequations}

\item Letting $\bm{Z}_{\mathsf{cvx},r}\triangleq\mathrm{arg}\min_{\bm{Z}:\mathsf{rank}(\bm{Z})\leq r}\|\bm{Z}-\bm{Z}_{\mathsf{cvx}}\|_{\mathrm{F}}$
be the best rank-$r$ approximation of $\bm{Z}_{\mathsf{cvx}}$, we
have
\begin{equation}
\|\bm{Z}_{\mathsf{cvx},r}-\bm{Z}_{\mathsf{cvx}}\|_{\mathrm{F}}\leq\frac{1}{n^{3}}\cdot\frac{\sigma}{\sigma_{\min}}\sqrt{\frac{n}{p}} \,\big\|\bm{M}^{\star}\big\|,
 \label{eq:Zcvx-r-bound-simplified}
\end{equation}
and the error bounds in~\eqref{eq:Zcvx-error-simplified} continue to hold if $\bm{Z}_{\mathsf{cvx}}$
is replaced by $\bm{Z}_{\mathsf{cvx},r}$.
\end{enumerate}\end{theorem}

\begin{remark}Here
and throughout, $f(n)\lesssim g(n)$ or $f(n)=O(g(n))$ means $|f(n)|/|g(n)|\leq C$
for some constant $C>0$ when $n$ is sufficiently large; $f(n)\gtrsim g(n)$ means $|f(n)|/|g(n)|\geq C$
for some constant $C>0$ when $n$ is sufficiently large; and $f(n)\asymp g(n)$ if and only if $f(n)\lesssim g(n)$ and $f(n)\gtrsim g(n)$. In addition, $\|\cdot\|_{\infty}$
denotes the entrywise $\ell_{\infty}$ norm, whereas $\|\cdot\|$
is the spectral norm. \end{remark}
\begin{remark}
	The factor $1/{n^3}$ in~\eqref{eq:Zcvx-r-bound-simplified} can be replaced by $1/{n^c}$ for an arbitrarily large fixed constant $c>0$ (e.g.~$c=100$).
\end{remark}

	To explain the applicability of the above theorem, we first remark on the conditions required for this theorem to hold; for simplicity, we assume that $\mu =O(1)$. 
\begin{itemize}
	\item {\em Sample complexity.} To begin with, the sample size needs to exceed the order of $ n \mathrm{poly}\log n$, which is information-theoretically optimal up to some logarithmic term \cite{CanTao10}. 
			
	\item {\em Noise size.} We then turn attention to the noise requirement, i.e.~$\sigma\lesssim\sqrt{\frac{np}{\log n}}\left\Vert \bm{M}^{\star}\right\Vert _{\infty}$. Note that under the sample size condition $n^{2}p\geq Cn\log^{3}n$, the size of the noise in each entry is allowed to be substantially larger than the maximum entry in the matrix. In other words, the signal-to-noise ratio w.r.t.~each observed entry could be very small. 
According to prior literature (e.g.~\cite[Theorem~1.1]{Se2010Noisy} and \cite[Theorem~2]{ma2017implicit}), such noise conditions are typically required for spectral methods to perform noticeably better than random guessing. 

%the incoherence condition (cf.~Definition~\ref{def:incoherence}) guarantees that the largest entry $\|\bm{M}^{\star}\|_{\infty}$ of the matrix $\bm{M}^{\star}$ is on the order of $\sigma_{\min}/ n$ (recall that  $r,\kappa \asymp 1$). As a result, the noise condition stated in Theorem~\ref{thm:main-convex-simplified} is equivalent to 
%%
%\begin{equation*}
%\sigma\lesssim\sqrt{\frac{np}{\log n}}\left\Vert \bm{M}^{\star}\right\Vert _{\infty}. 
%\end{equation*}
%%
\end{itemize}

Further, Theorem~\ref{thm:main-convex-simplified} has several important implications about the power of convex relaxation.
The discussions below again concentrate on the case where $\mu=O(1)$. 
\begin{itemize}
\item \emph{Near-optimal stability guarantees}. Our results reveal that the Euclidean error
of any convex optimizer $\bm{Z}_{\mathsf{cvx}}$ of~(\ref{eq:convex-LS})
obeys
\begin{equation}
\big\|\bm{Z}_{\mathsf{cvx}}-\bm{M}^{\star}\big\|_{\mathrm{F}}\lesssim\sigma\sqrt{n/p},\label{eq:our-rank1}
\end{equation}
implying that the performance of convex relaxation degrades gracefully as the signal-to-noise ratio decreases. This result matches the
oracle lower bound derived in~\cite[Eq.~(III.13)]{CanPla10},
which also improves upon their statistical guarantee. Specifically, Candès and Plan~\cite{CanPla10} provided a stability guarantee in
the presence of arbitrary bounded noise. When applied to the random
noise model assumed here, their results yield $\big\|\bm{Z}_{\mathsf{cvx}}-\bm{M}^{\star}\big\|_{\mathrm{F}}\lesssim\sigma n^{3/2}$,
which could be $O(\sqrt{n^{2}p})$ times more conservative than our
bound~(\ref{eq:our-rank1}).

\item \emph{Nearly low-rank structure of the convex solution}. In light
of~(\ref{eq:Zcvx-r-bound-simplified}), the optimizer of the convex program
(\ref{eq:convex-LS}) is almost, if not exactly, rank-$r$. When the
true rank $r$ is known \emph{a priori}, it is not uncommon for practitioners to return
the rank-$r$ approximation of $\bm{Z}_{\mathsf{cvx}}$. Our theorem
formally justifies that there is no loss of statistical accuracy --- measured
in terms of either $\|\cdot\|_{\mathrm{F}}$ or $\|\cdot\|_{\infty}$
--- when performing the rank-$r$ projection operation.

\item \emph{Entrywise and spectral norm error control. }Moving beyond the
Euclidean loss, our theory  uncovers that the estimation errors
of the convex optimizer are fairly spread out across all entries,
thus implying near-optimal entrywise error control. This is a stronger
form of error bounds, as an optimal Euclidean estimation accuracy
alone does not preclude the possibility of the estimation errors being
spiky and localized. Furthermore, the spectral norm error of the convex optimizer
is also well-controlled. Figure~\ref{fig:infty_op} displays the relative
estimation errors in both the $\ell_{\infty}$ norm and the spectral
norm, under the same setting as in Figure~\ref{fig:nnm_error}. As can be seen,
both forms of estimation errors scale linearly with the noise level, corroborating our theory.
\item \emph{Implicit regularization}. As a byproduct of the entrywise error
control, this result indicates that the additional constraint $\|\bm{Z}\|_{\infty}\leq\alpha$
suggested by~\cite{Negahban2012restricted} is automatically satisfied
and is hence unnecessary. In other words, the convex approach implicitly
controls the spikiness of its entries, without resorting to explicit
regularization. This is also confirmed by the numerical experiments reported in Figure \ref{fig:comparison}, where we see that the estimation error of (\ref{eq:convex-LS}) and that of the constrained version considered in \cite{Negahban2012restricted} are nearly identical. 

\begin{figure}
\center

\includegraphics[scale=0.4]{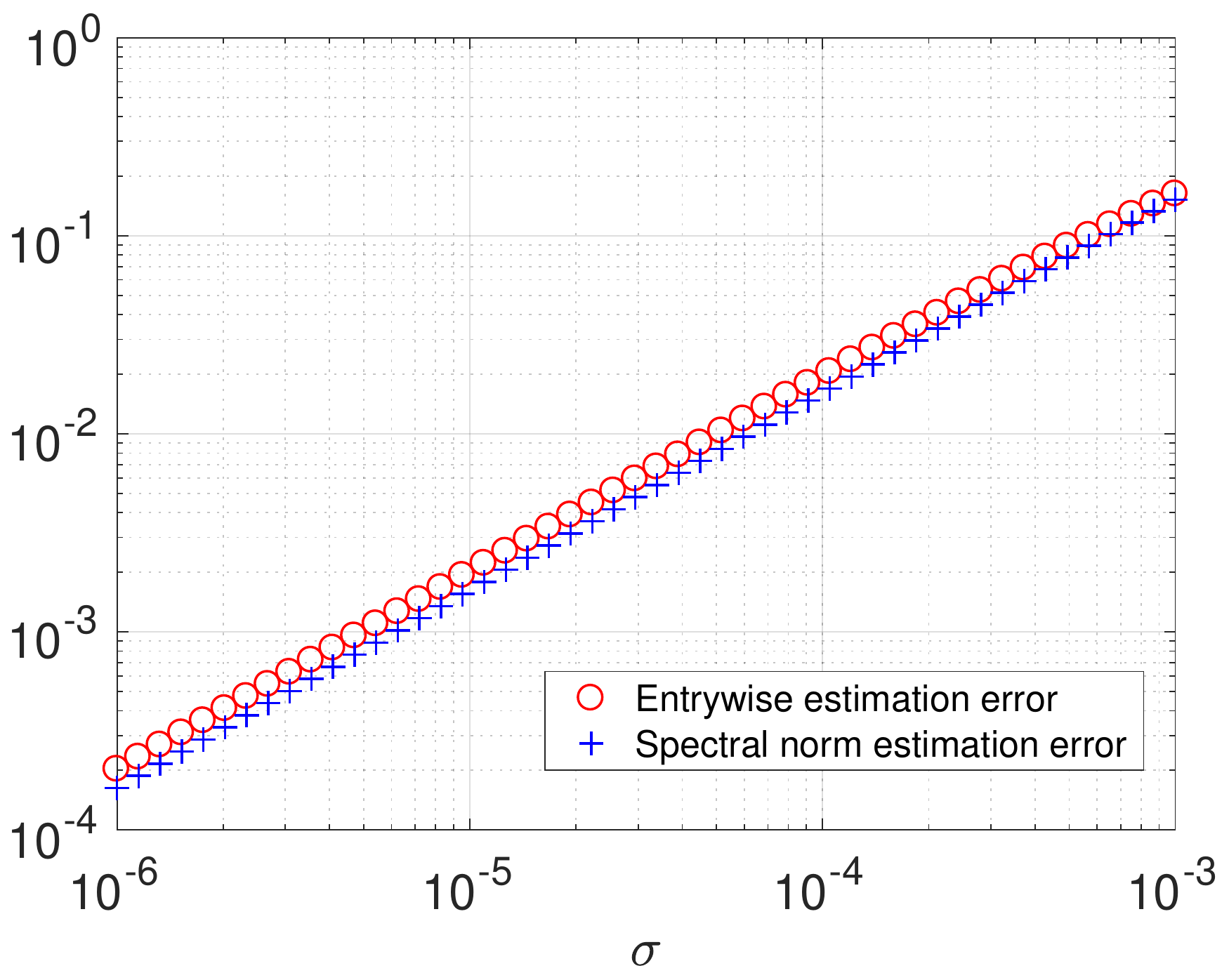}

\caption{The relative estimation error of $\bm{Z}_{\mathsf{cvx}}$ measured
by both $\|\cdot\|_{\infty}$ (i.e.~$\|\bm{Z}_{\mathsf{cvx}}-\bm{M}^{\star}\|_{\infty}/\|\bm{M}^{\star}\|_{\infty}$)
and $\|\cdot\|$ (i.e.~$\|\bm{Z}_{\mathsf{cvx}}-\bm{M}^{\star}\|/\|\bm{M}^{\star}\|$)
vs.~the standard deviation $\sigma$ of the noise. The results are
reported for $n=1000$, $r=5$, $p=0.2$, $\lambda=5\sigma\sqrt{np}$
and are averaged over 20 independent trials. \label{fig:infty_op}}
\end{figure}

\item \emph{Statistical guarantees for fast iterative optimization methods}.
Various iterative algorithms have been developed to solve the nuclear
norm regularized least-squares problem~(\ref{eq:convex-LS}) up to
an arbitrarily prescribed accuracy, examples including SVT (or proximal
gradient methods)~\cite{cai2010singular}, FPC~\cite{ma2011fixed},
SOFT--IMPUTE~\cite{mazumder2010spectral}, FISTA~\cite{beck2009fast,toh2010accelerated},
to name just a few. Our theory immediately provides statistical guarantees
for these algorithms. As we shall make precise in Section~\ref{sec:Proof-achitecture},
any point $\bm{Z}$ with $g(\bm{Z})\leq g(\bm{Z}_{\mathsf{cvx}})+\varepsilon$ (where $g(\cdot)$ is defined in~\eqref{eq:convex-LS})
enjoys the same error bounds as in~(\ref{eq:Zcvx-error-simplified}) (with $\bm{Z}_{\mathsf{cvx}}$
replaced by $\bm{Z}$ in~(\ref{eq:Zcvx-error-simplified})), provided that $\varepsilon>0$
is sufficiently small. In other words, when these convex optimization
algorithms converge w.r.t.~the objective value, they are guaranteed
to return a statistically reliable estimate.
\end{itemize}

\begin{figure}[t]
\centering

\begin{tabular}{cc}
\includegraphics[scale=0.3]{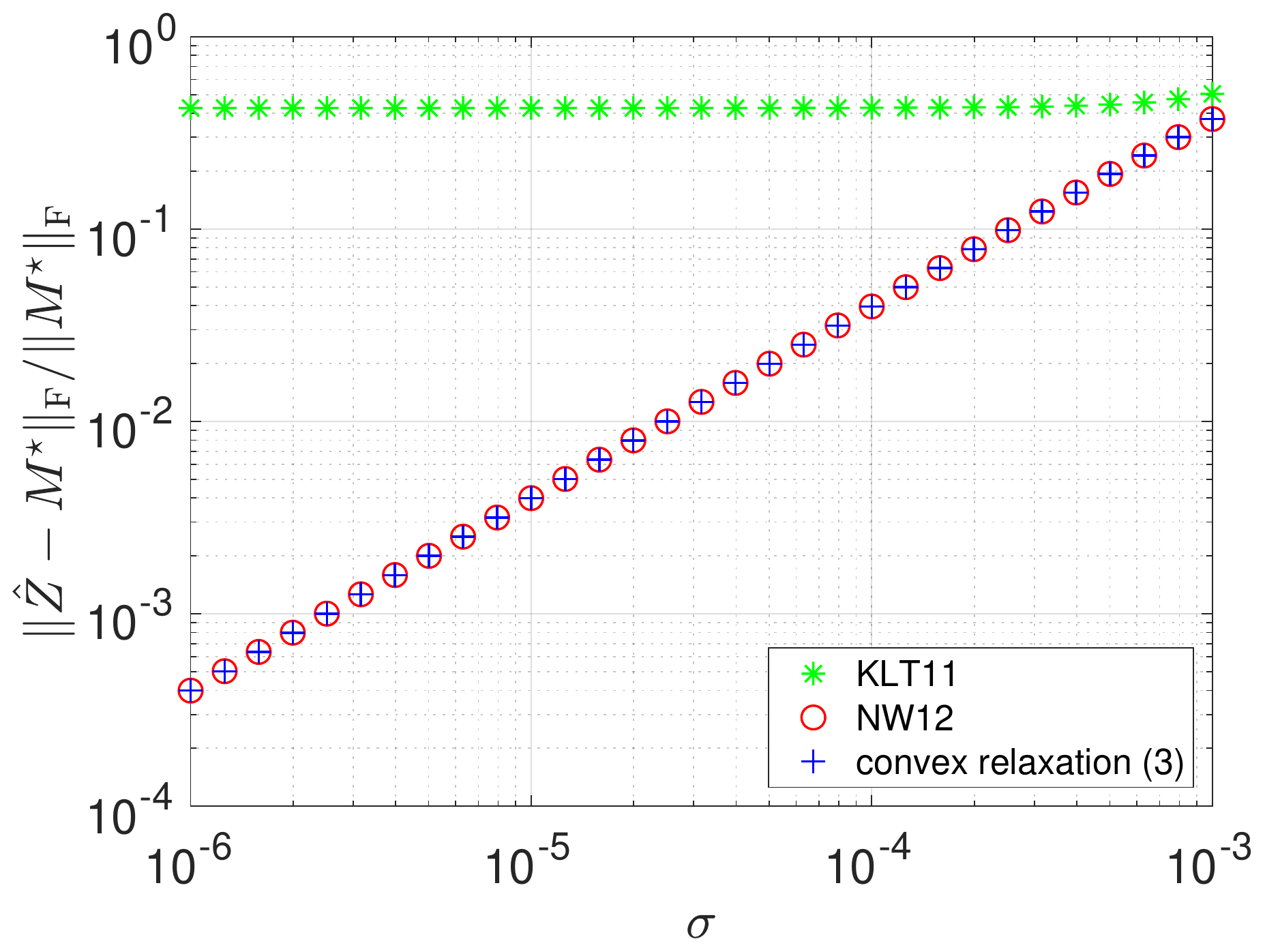} \qquad\quad & \qquad\quad \includegraphics[scale=0.3]{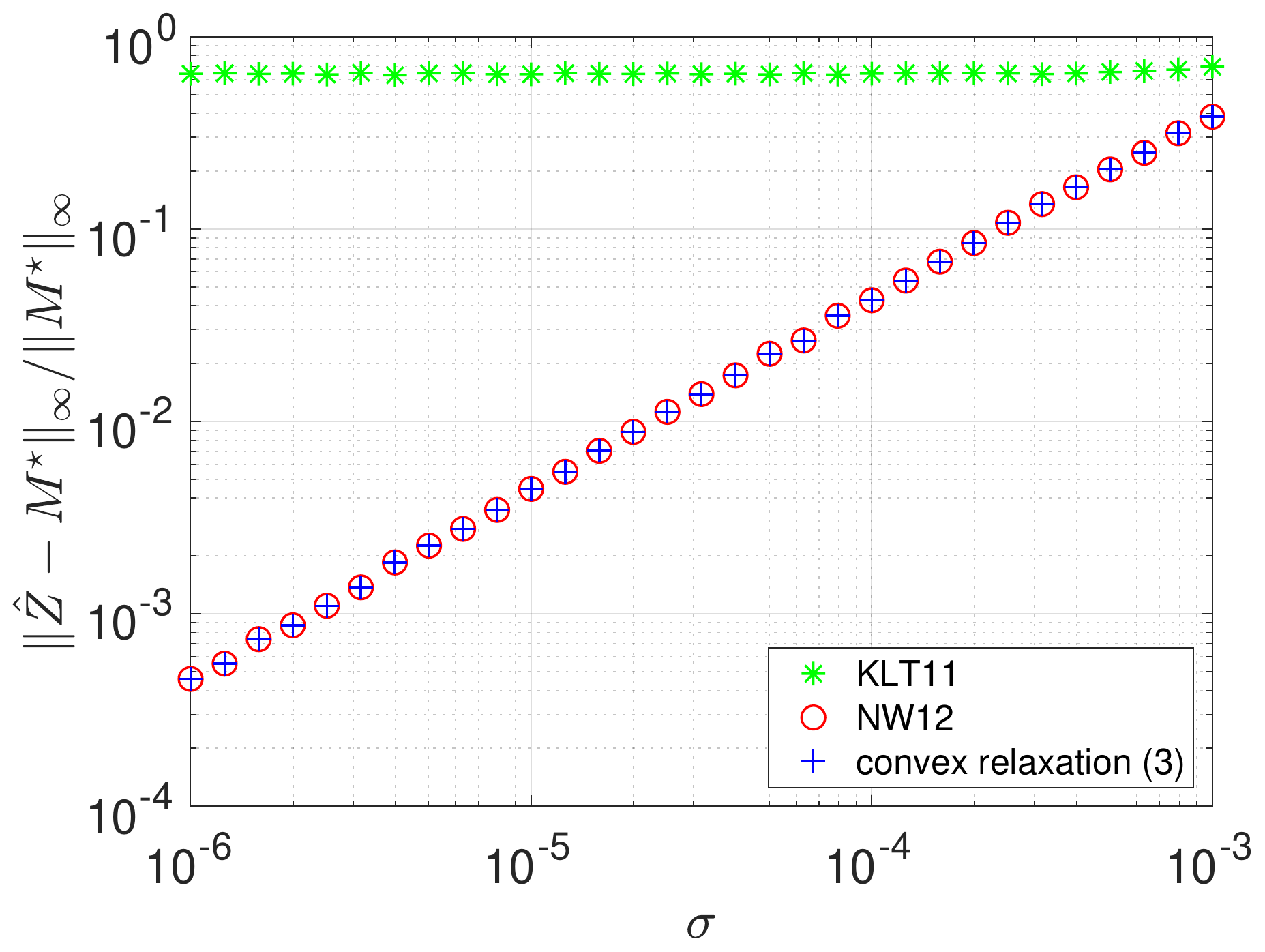} \tabularnewline
(a)  & \qquad\quad\quad(b) \tabularnewline
\end{tabular}

\caption{The relative estimation errors of $\hat{\bm{Z}}$, measured in terms of $\ell_{\mathrm{F}}$ and $\ell_{\infty}$, vs.~the standard deviation~$\sigma$ of the noise. Here $\hat{\bm{Z}}$ can be either the modified convex estimator in~\cite{MR2906869}, the constrained convex estimator in~\cite{Negahban2012restricted} or the vanilla convex estimator~(\ref{eq:convex-LS}). The results are reported for $n=1000$, $r=5$,
$p=0.2$, and are averaged over 20 Monte-Carlo
trials. For the modified convex estimator in~\cite{MR2906869}, we choose the regularization parameter $\lambda$ therein to be $1.5 \max\{ \sigma,\|\bm{M}^{\star}\|_{\infty}\} \sqrt{1/(n^3p)}$, as suggested by their theory. For the constrained one in~\cite{Negahban2012restricted}, the regularization parameter $\lambda$ is set to be $5\sigma\sqrt{np}$ and the constraint $\alpha$ is set to be $\|\bm{M}^\star\|_{\infty}$. Both choices  are recommended by \cite{Negahban2012restricted}. As for~(\ref{eq:convex-LS}), we set $\lambda=5\sigma\sqrt{np}$. 
\label{fig:comparison}}
\end{figure}

To better understand our contributions, we take a moment to discuss two important but different convex programs studied in \cite{Negahban2012restricted} and \cite{MR2906869}. To begin with, under a spikiness assumption on the low-rank matrix, Negahban and Wainwright~\cite{Negahban2012restricted} proposed to
enforce an extra entrywise constraint $\|\bm{Z}\|_{\infty}\leq\alpha$ when
solving~\eqref{eq:convex-LS}, in order to explicitly control the
spikiness of the estimate. When applied to our model with $r,\kappa,\mu\asymp 1$, their results
read (up to some logarithmic factor)
\begin{equation}
\big\|\hat{\bm{Z}}-\bm{M}^{\star}\big\|_{\mathrm{F}}\lesssim\max\left\{ \sigma,\|\bm{M}^{\star}\|_{\infty}\right\} \sqrt{n/p},\label{eq:our-rank1-1}
\end{equation}
where $\hat{\bm{Z}}$ is the estimate returned by their modified convex algorithm.
While this matches the optimal bound when $\sigma \gtrsim \|\bm{M}^{\star}\|_{\infty}$, 
it becomes suboptimal when $\sigma\ll\|\bm{M}^{\star}\|_{\infty}$ (under our models). Moreover, as we have already discussed, 
 the extra spikiness constraint becomes unnecessary in the regime considered herein. This also means that our result complements existing theory about the convex program in \cite{Negahban2012restricted} by demonstrating its minimaxity for an additional range of noise. 
Another work by Koltchinskii et~al.~\cite{MR2906869} investigated a completely different convex algorithm, which is effectively a spectral method (namely, one round of soft singular value thresholding on a rescaled zero-padded data matrix). The algorithm is shown to be minimax optimal over the class of low-rank matrices with bounded $\ell_{\infty}$ norm (note that this is very different from the set of incoherent matrices studied here). When specialized to our model, their error bound is the same as~\eqref{eq:our-rank1-1} (modulo some log factor), which also becomes suboptimal as $\sigma$ decreases.  As can be seen from the numerical experiments in Figure \ref{fig:comparison}, the estimation error of this thresholding-based spectral algorithm does not decrease as the noise shrinks, and its performance seems uniformly outperformed by that of convex relaxation (\ref{eq:convex-LS}) and the constrained estimator in \cite{Negahban2012restricted}.  In fact, this is part of our motivation to pursue an improved theoretical understanding of the formulation (\ref{eq:convex-LS}).

Finally, we make note of a connection between our result and prior theory developed for the noiseless case. Specifically, when the noise vanishes (i.e.~$\sigma \rightarrow 0$), one can take a diminishing sequence of regularization parameters $\{\lambda_{k}\}$ with $\lambda_k\rightarrow 0$, then the resulting estimation errors associated with this sequence should decrease to $0$ as $k\rightarrow \infty$ (which implies exact recovery in the limit of $k$). This parallels the connection between Lasso in sparse linear regression and basis pursuit in compressed sensing.

\subsubsection{Theoretical guarantees: extensions to more general settings}

So far we have presented results when the true matrix has bounded rank and condition number, i.e.~$r,\kappa=O(1)$. 
Our theory actually accommodates a significantly broader range of scenarios, where the rank and the condition number are both allowed to grow with the dimension $n$.

\begin{theorem}\label{thm:main-convex}Let $\bm{M}^{\star}$ be rank-$r$
and $\mu$-incoherent with a condition number $\kappa$. Suppose Assumption~\ref{assumption:sampling-noise}
holds and take $\lambda=C_{\lambda}\sigma\sqrt{np}$ in~(\ref{eq:convex-LS})
for some large enough constant $C_{\lambda}>0$. Assume the sample
size obeys $n^{2}p\geq C\kappa^{4}\mu^{2}r^{2}n\log^{3}n$ for some
sufficiently large constant $C>0$, and the noise satisfies $\sigma\sqrt{\frac{n}{p}}\leq c\frac{\sigma_{\min}}{\sqrt{\kappa^{4}\mu r\log n}}$
for some sufficiently small constant $c>0$. Then with probability
exceeding $1-O(n^{-3})$,

\begin{enumerate}
\item Any minimizer $\bm{Z}_{\mathsf{cvx}}$ of~(\ref{eq:convex-LS})
obeys \begin{subequations}\label{eq:Zcvx-error}
\begin{align}
\big\|\bm{Z}_{\mathsf{cvx}}-\bm{M}^{\star}\big\|_{\mathrm{F}}\,\, & \lesssim\kappa\frac{\sigma}{\sigma_{\min}}\sqrt{\frac{n}{p}}\,\big\|\bm{M}^{\star}\big\|_{\mathrm{F}},\label{eq:main-fro-norm-error}\\
\big\|\bm{Z}_{\mathsf{cvx}}-\bm{M}^{\star}\big\|_{\infty} & \lesssim\sqrt{\kappa^{3}\mu r}\cdot\frac{\sigma}{\sigma_{\min}}\sqrt{\frac{n\log n}{p}}\,\big\|\bm{M}^{\star}\big\|_{\infty},\label{eq:main-inf-norm-error}\\
\big\|\bm{Z}_{\mathsf{cvx}}-\bm{M}^{\star}\big\|\,\,\,\,\, & \lesssim\frac{\sigma}{\sigma_{\min}}\sqrt{\frac{n}{p}}\,\big\|\bm{M}^{\star}\big\|;\label{eq:main-spectral-norm-error}
\end{align}
\end{subequations}

\item Letting $\bm{Z}_{\mathsf{cvx},r}\triangleq\mathrm{arg}\min_{\bm{Z}:\mathsf{rank}(\bm{Z})\leq r}\|\bm{Z}-\bm{Z}_{\mathsf{cvx}}\|_{\mathrm{F}}$
be the best rank-$r$ approximation of $\bm{Z}_{\mathsf{cvx}}$, we
have
\begin{equation}
\|\bm{Z}_{\mathsf{cvx},r}-\bm{Z}_{\mathsf{cvx}}\|_{\mathrm{F}}\leq\frac{1}{n^{3}}\cdot\frac{\sigma}{\sigma_{\min}}\sqrt{\frac{n}{p}} \,\big\|\bm{M}^{\star}\big\|,
 \label{eq:Zcvx-r-bound}
\end{equation}
and the error bounds in~\eqref{eq:Zcvx-error} continue to hold if $\bm{Z}_{\mathsf{cvx}}$
is replaced by $\bm{Z}_{\mathsf{cvx},r}$.
\end{enumerate}\end{theorem}

\begin{remark}[The noise condition] The incoherence condition (cf.~Definition~\ref{def:incoherence}) guarantees that the largest entry $\|\bm{M}^{\star}\|_{\infty}$ of the matrix $\bm{M}^{\star}$ is no larger than $\kappa \mu r \sigma_{\min} / n$. As a result, 
the noise condition stated in Theorem~\ref{thm:main-convex} covers all scenarios obeying
\begin{equation*}
\sigma\lesssim\sqrt{\frac{np}{\kappa^{6}\mu^{3}r^{3}\log n}}\left\Vert \bm{M}^{\star}\right\Vert _{\infty}.
\end{equation*}
	Therefore, the typical size of the noise is allowed to be much larger than the size of the largest entry of $\bm{M}^\star$, provided that $p\gg \frac{\kappa^{6}\mu^{3}r^{3}\log n}{n}$. 
In particular, when $r, \kappa = O(1)$, this recovers the noise condition in Theorem~\ref{thm:main-convex-simplified}.
\end{remark}

Notably, the sample size condition for noisy matrix completion (i.e.~$n^{2}p\geq C\kappa^{4}\mu^{2}r^{2}n\log^{3}n$) is more stringent than that in the noiseless setting (i.e.~$n^2 p\asymp n r \log^2 n$), and our statistical guarantees are likely suboptimal with respect to the dependency on $r$
and $\kappa$. This sub-optimality is mainly due to the analysis of nonconvex optimization, a key ingredient of our analysis of convex relaxation. In fact, the state-of-the-art nonconvex analysis~\cite{Se2010Noisy, chen2015fast, ma2017implicit} requires the sample size to be much larger than the optimal one (e.g.~$n^2 p \gg n\mathsf{poly}(r) \mathsf{poly}(\kappa)$) even in the noiseless setting. It would  certainly be interesting, and in fact important, to see whether it is possible to develop a theory with optimal dependency on $r$ and $\kappa$. We leave this for future investigation.

%Interested
%readers are referred to Section~\ref{sec:Discussion} for more detailed
%discussions.

	Despite the above sub-optimality issue, implications similar to those of Theorem~\ref{thm:main-convex-simplified} hold for this general setting. To begin with, the nearly low-rank structure of the convex solution is preserved (cf.~(\ref{eq:Zcvx-r-bound})). In addition, the estimation error of the convex estimate is spread out across entries (cf.~(\ref{eq:main-inf-norm-error})), thus uncovering an implicit regularization phenomenon underlying convex relaxation (which implicitly regularizes the spikiness constraint on the solution). Last but not least, the upper bounds (\ref{eq:Zcvx-error}) and (\ref{eq:Zcvx-r-bound}) continue to hold for approximate minimizers of the convex program~(\ref{eq:convex-LS}), thus yielding statistical guarantees for numerous iterative algorithms aimed at minimizing~(\ref{eq:convex-LS}).

\section{Strategy and novelty\label{sec:Proof-achitecture}}

In this section, we introduce the strategy for proving our main theorem, i.e.~Theorem~\ref{thm:main-convex}. Theorem~\ref{thm:main-convex-simplified} follows immediately. 
Informally, the main technical difficulty stems from the lack of closed-form
expressions for the primal solution to~(\ref{eq:convex-LS}), which
in turn makes it difficult to construct a dual certificate. This is
in stark contrast to the noiseless setting, where one clearly anticipates
the ground truth $\bm{M}^{\star}$ to be the primal solution; in fact,
this is precisely why the analysis for the noisy case is significantly
more challenging. Our strategy, as we shall detail below, mainly entails invoking an iterative nonconvex algorithm to ``approximate''
such a primal solution.

Before continuing, we introduce a few more notations. Let $\mathcal{P}_{\Omega}(\cdot):\mathbb{R}^{n\times n}\mapsto\mathbb{R}^{n\times n}$
represent the projection onto the subspace of matrices supported on
$\Omega$, namely,
\begin{equation}
\left[\mathcal{P}_{\Omega}\left(\bm{Z}\right)\right]_{ij}=\begin{cases}
Z_{ij}, & \text{for }\left(i,j\right)\in\Omega\\
0, & \text{otherwise}
\end{cases}\label{eq:defn-Pomega}
\end{equation}
for any matrix $\bm{Z}\in\mathbb{R}^{n\times n}$. For a rank-$r$
matrix $\bm{M}$ with singular value decomposition $\bm{U}\bm{\Sigma}\bm{V}^{\top}$,
denote by $T$ its tangent space, i.e.
\begin{equation}
T=\left\{ \bm{U}\bm{A}^{\top}+\bm{B}\bm{V}^{\top}\mid\bm{A},\bm{B}\in\mathbb{R}^{n\times r}\right\} .\label{eq:defn-T}
\end{equation}
Correspondingly, let $\mathcal{P}_{T}(\cdot)$ be the orthogonal projection
onto the subspace $T$, that is,
\begin{equation}
\mathcal{P}_{T}\left(\bm{Z}\right)=\bm{U}\bm{U}^{\top}\bm{Z}+\bm{Z}\bm{V}\bm{V}^{\top}-\bm{U}\bm{U}^{\top}\bm{Z}\bm{V}\bm{V}^{\top}\label{eq:defn-PT}
\end{equation}
for any matrix $\bm{Z}\in\mathbb{R}^{n\times n}$. In addition, let
$T^{\perp}$ and $\mathcal{P}_{T^{\perp}}(\cdot)$ denote the orthogonal
complement of $T$ and the projection onto $T^{\perp}$, respectively.
With regards to the ground truth, we denote
\begin{equation}
\bm{X}^{\star}=\bm{U}^{\star}(\bm{\Sigma}^{\star})^{1/2}\qquad\text{and}\qquad\bm{Y}^{\star}=\bm{V}^{\star}(\bm{\Sigma}^{\star})^{1/2}.\label{eq:defn-Xstar-Ystar}
\end{equation}
The nonconvex problem~(\ref{eq:nonconvex_mc_noisy-lambda}) is equivalent
to
\begin{equation}
\underset{\bm{X},\bm{Y}\in\mathbb{R}^{n\times r}}{\text{minimize}}\qquad f(\bm{X},\bm{Y})\triangleq\frac{1}{2p}
\big\|\mathcal{P}_{\Omega}\left(\bm{X}\bm{Y}^{\top}-
\bm{M}\right)\big\|_{\mathrm{F}}^{2}+\frac{\lambda}{2p}
\|\bm{X}\|_{\mathrm{F}}^{2}+\frac{\lambda}{2p}\|\bm{Y}\|_{\mathrm{F}}^{2},
\label{eq:nonconvex_mc_noisy}
\end{equation}
where we have inserted an extra factor $1/p$ (compared to~(\ref{eq:nonconvex_mc_noisy-lambda}))
to simplify the presentation of the analysis later on.

\subsection{Exact duality}

In order to analyze the convex program~(\ref{eq:convex-LS}), it is
natural to start with the first-order optimality condition. Specifically,
suppose that $\bm{Z}\in\mathbb{R}^{n\times n}$ is a (primal) solution
to~(\ref{eq:convex-LS}) with SVD $\bm{Z}=\bm{U}\bm{\Sigma}\bm{V}^{\top}$.\footnote{Here and below, we use $\bm{Z}$ (rather than $\bm{Z}_{\mathsf{cvx}}$)
for notational simplicity, whenever it is clear from the context.} As before, let $T$ be the tangent space of $\bm{Z}$, and let $T^{\perp}$
be the orthogonal complement of $T$. Then the first-order optimality
condition for~(\ref{eq:convex-LS}) reads: there exists a matrix $\bm{W}\in T^{\perp}$
(called a dual certificate) such that \begin{subequations}\label{eq:KKT}
\begin{align}
\frac{1}{\lambda}\mathcal{P}_{\Omega}\big(\bm{M}-\bm{Z}\big) & =\bm{U}\bm{V}^{\top}+\bm{W};\label{eq:KKT-1}\\
\left\Vert \bm{W}\right\Vert  & \leq1.\label{eq:KKT-2}
\end{align}
\end{subequations} This condition is not only necessary to certify
the optimality of $\bm{Z}$, but also ``almost sufficient'' in guaranteeing
the uniqueness of the solution $\bm{Z}$; see Appendix~\ref{sec:Proof-of-Lemma-unique-minimizer}
for in-depth discussions.

The challenge then boils down to identifying such a primal-dual
pair $(\bm{Z},\bm{W})$ satisfying the optimality condition \eqref{eq:KKT}.
For the noise-free case, the primal solution is clearly $\bm{Z}=\bm{M}^{\star}$
if exact recovery is to be expected; the dual certificate can then
be either constructed exactly by the least-squares solution to a certain
underdetermined linear system \cite{ExactMC09,CanTao10}, or produced
approximately via a clever golfing scheme pioneered by Gross \cite{Gross2011recovering}.
For the noisy case, however, it is often difficult to hypothesize
on the primal solution $\bm{Z}$, as it depends on the random noise
in a complicated way. In fact, the lack of a suitable guess of $\bm{Z}$ (and hence $\bm{W}$)
was the major hurdle that prior works faced when carrying out the
duality analysis.

\begin{comment}
As we will elaborate in Section~\ref{sec:Proof-achitecture}, the
primary difficulty in obtaining sharp estimation error bound for~(\ref{eq:convex_mc_noisy})
lies in the fact that one has no prior information at all regarding
the solution to~(\ref{eq:convex_mc_noisy}) when the noise is present.
This is in sharp contrast to the noiseless setting~(\ref{eq:convex_mc_noiseless}),
where one clearly hopes for $\bm{M}^{\star}$, the ground truth, to
be the primal solution. With the help of $\bm{M}^{\star}$, one can
in turn build the dual vector via the clever golfing scheme \cite{Gross2011recovering}
to certify the optimality\emph{ }of $\bm{M}^{\star}$. Directly using
the same dual vector in the noiseless case would lead to a\emph{ }suboptimal
error bound, as has been done in \cite{CanPla10}. Therefore, answering
the above question affirmatively calls for a ``good'' guess of the
primal solution to~(\ref{eq:convex_mc_noisy}) and a novel way to
build its corresponding dual vector.
\end{comment}

\subsection{A candidate primal solution via nonconvex optimization\label{subsec:exact-nonconvex-solution}}

Motivated by the numerical experiment in Section~\ref{subsec:Numerics-cvx-ncvx},
we propose to examine whether the optimizer of the nonconvex problem
(\ref{eq:nonconvex_mc_noisy-lambda}) stays close to the solution
to the convex program~(\ref{eq:convex-LS}). Towards this, suppose
that $\bm{X},\bm{Y}\in\mathbb{R}^{n\times r}$ form a critical point
of~(\ref{eq:nonconvex_mc_noisy-lambda}) with $\mathsf{rank}(\bm{X})=\mathsf{rank}(\bm{Y})=r$.\footnote{Once again, we abuse the notation $(\bm{X},\bm{Y})$ (instead of using
$(\bm{X}_{\mathsf{ncvx}},\bm{Y}_{\mathsf{ncvx}})$) for notational
simplicity, whenever it is clear from the context.} Then the first-order condition reads \begin{subequations}\label{eq:nonconvex-1st-order}
\begin{align}
\frac{1}{\lambda}\mathcal{P}_{\Omega}\big(\bm{M}-\bm{X}\bm{Y}^{\top}\big)\bm{Y} & =\bm{X};\label{eq:nonconvex-1st-order-1}\\
\frac{1}{\lambda}\left[\mathcal{P}_{\Omega}\big(\bm{M}-\bm{X}\bm{Y}^{\top}\big)\right]^{\top}\bm{X} & =\bm{Y}.\label{eq:nonconvex-1st-order-2}
\end{align}
\end{subequations}

To develop some intuition about the connection between~(\ref{eq:KKT})
and~(\ref{eq:nonconvex-1st-order}), let us take a look at the case
with $r=1$. Denote $\bm{X}=\bm{x}$ and $\bm{Y}=\bm{y}$ and assume
that the two rank-1 factors are ``balanced'', namely, $\|\bm{x}\|_{2}=\|\bm{y}\|_{2}\neq0$.
It then follows from~(\ref{eq:nonconvex-1st-order}) that $\lambda^{-1}\mathcal{P}_{\Omega}(\bm{M}-\bm{x}\bm{y}^{\top})$
has a singular value $1$, whose corresponding left and right singular
vectors are $\bm{x}/\|\bm{x}\|_{2}$ and $\bm{y}/\|\bm{y}\|_{2}$,
respectively. In other words, one can express
\begin{equation}
\frac{1}{\lambda}\mathcal{P}_{\Omega}\big(\bm{M}-\bm{x}\bm{y}^{\top}\big)=\frac{1}{\|\bm{x}\|_{2}\|\bm{y}\|_{2}}\bm{x}\bm{y}^{\top}+\bm{W},
\end{equation}
where $\bm{W}$ is orthogonal to the tangent space of $\bm{x}\bm{y}^{\top}$; this is precisely the condition~(\ref{eq:KKT-1}).
It remains to argue that~(\ref{eq:KKT-2}) is valid as well. Towards
this end, the first-order condition~(\ref{eq:nonconvex-1st-order})
alone is insufficient, as there might be non-global critical points
(e.g.~saddle points) that are unable to approximate the convex
solution well. Fortunately, as long as the candidate $\bm{x}\bm{y}^{\top}$
is not far away from the ground truth $\bm{M}^{\star}$, one
can guarantee $\|\bm{W}\|<1$ as required in~(\ref{eq:KKT-2}).

The above informal argument about the link between the convex and
the nonconvex problems can be rigorized. To begin with, we introduce
the following conditions on the regularization parameter $\lambda$.

\smallskip
\begin{condition}[\textbf{Regularization parameter}]\label{assumption:link-cvx-ncvx-noise}
The regularization parameter $\lambda$ satisfies
\begin{enumerate}[leftmargin=2.5em]
\item[(a)] \textbf{(Relative to noise) $\|\mathcal{P}_{\Omega}\left(\bm{E}\right)\|<\lambda/8.$}
\item[(b)] \textbf{(Relative to nonconvex solution) }\label{assumption:link-cvx-ncvx-closeness} $\|\mathcal{P}_{\Omega}(\bm{X}\bm{Y}^{\top}-\bm{M}^{\star})-
    p(\bm{X}\bm{Y}^{\top}-\bm{M}^{\star})\|<\lambda/8$.
\end{enumerate}
\end{condition}
\smallskip

\begin{remark}Condition~\ref{assumption:link-cvx-ncvx-noise} requires
that the regularization parameter $\lambda$ should dominate a certain
norm of the noise, as well as of the deviation of $\bm{X}\bm{Y}^{\top}-\bm{M}^{\star}$ from its mean $p(\bm{X}\bm{Y}^{\top}-\bm{M}^{\star})$;
as will be seen shortly, the latter condition can be met when $(\bm{X},\bm{Y})$
is sufficiently close to $(\bm{X}^{\star},\bm{Y}^{\star})$. \end{remark}

With the above condition in place, the following result demonstrates
that a critical point $(\bm{X},\bm{Y})$ of the nonconvex problem
(\ref{eq:nonconvex_mc_noisy-lambda}) readily translates to the unique
minimizer of the convex program~(\ref{eq:convex-LS}). This lemma
is established in Appendix~\ref{sec:Proof-of-Lemma-sufficient-condition-for-minimizer}.

\begin{lemma}[\textbf{Exact nonconvex vs.~convex} \textbf{optimizers}]
\label{lemma:link-cvx-and-ncvx}Suppose that $(\bm{X},\bm{Y})$ is
a critical point of~(\ref{eq:nonconvex_mc_noisy-lambda}) satisfying
$\mathsf{rank}(\bm{X})=\mathsf{rank}(\bm{Y})=r$, and the sampling
operator $\mathcal{P}_{\Omega}$ is injective when restricted to the
elements of the tangent space $T$ of $\bm{X}\bm{Y}^{\top}$, namely,
\begin{equation}
\mathcal{P}_{\Omega}(\bm{H})=\bm{0} \quad \Longleftrightarrow \quad \bm{H}=\bm{0,}\quad\text{for all }\bm{H}\in T.\label{eq:injectivity_weak}
\end{equation}
Under Condition~\ref{assumption:link-cvx-ncvx-noise}, the point $\bm{Z}\triangleq\bm{X}\bm{Y}^{\top}$
is the unique minimizer of~(\ref{eq:convex-LS}).  
\end{lemma}

\begin{comment}
In words, a critical point $(\bm{X},\bm{Y})$ of the nonconvex problem
(\ref{eq:nonconvex_mc_noisy-lambda}) translates to the unique minimizer
of the convex program~(\ref{eq:convex-LS}), as long as (1) the gap
between $(\bm{X},\bm{Y})$ and the truth is well-controlled in some
sense, (2) the noise size is not too large, and (3) $\mathcal{P}_{\Omega}$
is injective when restricted to the tangent space of $\bm{X}\bm{Y}^{\top}$.
\end{comment}

In order to apply Lemma~\ref{lemma:link-cvx-and-ncvx}, one needs
to locate a critical point  of~(\ref{eq:nonconvex_mc_noisy-lambda})
that is sufficiently close to the truth, for which one natural candidate
is the global optimizer of~(\ref{eq:nonconvex_mc_noisy-lambda}).
The caveat, however, is the lack of theory characterizing directly
the properties of the optimizer of~(\ref{eq:nonconvex_mc_noisy-lambda}).
Instead, what is available in prior theory is the characterization
of some iterative sequence (e.g.~gradient descent iterates)
aimed at solving~(\ref{eq:nonconvex_mc_noisy-lambda}). It is unclear
from prior theory whether the iterative algorithm under study (e.g.~gradient
descent) converges to the global optimizer in the presence of noise.
This leads to technical difficulty in justifying the proximity between
the nonconvex optimizer and the convex solution via Lemma~\ref{lemma:link-cvx-and-ncvx}.

\subsection{Approximate nonconvex optimizers\label{subsec:Approximate-nonconvex-optimizers}}

Fortunately, perfect knowledge of the nonconvex optimizer is not pivotal.
Instead, an approximate solution to the nonconvex problem~(\ref{eq:nonconvex_mc_noisy-lambda}) (or equivalently~\eqref{eq:nonconvex_mc_noisy})
suffices to serve as a reasonably tight approximation of the convex
solution. More precisely, we desire two factors $(\bm{X},\bm{Y})$
that result in nearly zero (rather than exactly zero) gradients:
\[
\nabla_{\bm{X}}f(\bm{X},\bm{Y})\approx\bm{0}\qquad\text{and}\qquad\nabla_{\bm{Y}}f(\bm{X},\bm{Y})\approx\bm{0},
\]
where $f(\cdot, \cdot)$ is the nonconvex objective function as defined in \eqref{eq:nonconvex_mc_noisy}. 
This relaxes the condition discussed in Lemma~\ref{lemma:link-cvx-and-ncvx}
(which only applies to critical points of~(\ref{eq:nonconvex_mc_noisy-lambda})
as opposed to approximate critical points). As it turns out, such
points can be found via gradient descent tailored to~(\ref{eq:nonconvex_mc_noisy-lambda}).
The sufficiency of the near-zero gradient condition is made possible
by slightly strengthening the injectivity assumption \eqref{eq:injectivity_weak},
which is stated below.

\begin{condition}[\textbf{Injectivity}]\label{assumption:link-cvx-ncvx-injectivity} Let $T$ be the tangent space
of $\bm{X}\bm{Y}^{\top}$. There is a quantity $c_{\mathrm{inj}}>0$
such that
\begin{equation}
p^{-1}\left\Vert \mathcal{P}_{\Omega}\left(\bm{H}\right)\right\Vert _{\mathrm{F}}^{2}\geq c_{\mathrm{inj}}\left\Vert \bm{H}\right\Vert _{\mathrm{F}}^{2},\qquad\text{for all }\bm{H}\in T.\label{eq:assumption-injectivity}
\end{equation}

\end{condition}

The following lemma states quantitatively how an approximate nonconvex
optimizer serves as an excellent proxy of the convex solution, which
we establish in Appendix~\ref{sec:Proof-of-Lemma-approx-link-cvx-ncvx-simple}.

\begin{comment}
or
\[
\begin{cases}
\mathcal{P}_{\Omega}\left(\bm{X}\bm{Y}^{\top}-\bm{M}\right)\bm{Y}+\lambda\bm{X}\text{ }\approx\bm{0};\\
\big(\mathcal{P}_{\Omega}(\bm{X}\bm{Y}^{\top}-\bm{M})\big)^{\top}\bm{X}+\lambda\bm{Y}\approx\bm{0}.
\end{cases}
\]
\end{comment}

\begin{lemma}[\textbf{Approximate nonconvex vs.~convex optimizers}]\label{lemma:approx-link-cvx-ncvx-simple}Suppose
that $(\bm{X},\bm{Y})$ obeys
\begin{equation}
\left\Vert \nabla f\left(\bm{X},\bm{Y}\right)\right\Vert _{\mathrm{F}}\leq c\frac{\sqrt{c_{\mathrm{inj}}p}}{\kappa}\cdot\frac{\lambda}{p}
\sqrt{\sigma_{\min}}\label{eq:small-gradient-f}
\end{equation}
for some sufficiently small constant $c>0$. Further assume that any
singular value of $\bm{X}$ and $\bm{Y}$ lies in $[\sqrt{\sigma_{\min}/2},\sqrt{2\sigma_{\max}}]$.
Then under Conditions~\ref{assumption:link-cvx-ncvx-noise} and \ref{assumption:link-cvx-ncvx-injectivity},
any minimizer $\bm{Z}_{\mathsf{cvx}}$ of~(\ref{eq:convex-LS}) satisfies
\begin{equation}
\big\|\bm{X}\bm{Y}^{\top}-\bm{Z}_{\mathsf{cvx}}\big\|_{\mathrm{F}}\lesssim\frac{\kappa}{c_{\mathrm{inj}}}\frac{1}{\sqrt{\sigma_{\min}}}\left\Vert \nabla f\left(\bm{X},\bm{Y}\right)\right\Vert _{\mathrm{F}}.\label{eq:closeness-cvx-ncvx}
\end{equation}
\end{lemma}\begin{remark}\label{remark:extension-nonoptimizer}In
fact, this lemma continues to hold if $\bm{Z}_{\mathsf{cvx}}$ is
replaced by any $\bm{Z}$ obeying $g(\bm{Z})\leq g(\bm{X}\bm{Y}^{\top})$,
where $g(\cdot)$ is the objective function defined in~(\ref{eq:convex-LS}) and
$\bX$ and $\bY$ are low-rank factors obeying conditions of Lemma~\ref{lemma:approx-link-cvx-ncvx-simple}.
This is important in providing statistical guarantees for iterative
methods like SVT \cite{cai2010singular}, FPC \cite{ma2011fixed},
SOFT--IMPUTE \cite{mazumder2010spectral}, FISTA \cite{beck2009fast},
etc. To be more specific, suppose that $(\bm{X},\bm{Y})$ results
in an approximate optimizer of~(\ref{eq:convex-LS}), namely, $g(\bm{X}\bm{Y}^{\top})=g(\bm{Z}_{\mathsf{cvx}})+\varepsilon$
for some sufficiently small $\varepsilon>0$. Then for any $\bm{Z}$
obeying $g(\bm{Z})\leq g(\bm{X}\bm{Y}^{\top})=g(\bm{Z}_{\mathsf{cvx}})+\varepsilon$,
one has
\begin{equation}
\big\|\bm{X}\bm{Y}^{\top}-\bm{Z}\big\|_{\mathrm{F}}\lesssim\frac{\kappa}{c_{\mathrm{inj}}}\frac{1}{\sqrt{\sigma_{\min}}}\left\Vert \nabla f\left(\bm{X},\bm{Y}\right)\right\Vert _{\mathrm{F}}.\label{eq:XY-Z-gap}
\end{equation}
As a result, as long as the above-mentioned algorithms converge in
terms of the objective value, they must return a solution obeying
(\ref{eq:XY-Z-gap}), which is exceedingly close to $\bm{X}\bm{Y}^{\top}$
if $\|\nabla f(\bm{X},\bm{Y})\|_{\mathrm{F}}$ is small.\end{remark}

\begin{comment}
$\bm{Z}:=\bm{X}\bm{Y}^{\top}$ is a good approximation of the minimizer
of the convex program in the sense that any minimizer $\hat{\bm{Z}}$
of~(\ref{eq:convex-LS})
\end{comment}

%It is worth noting that, due to the lack of strong convexity, small gradients alone do not imply closeness to a global minimizer of~(\ref{eq:nonconvex_mc_noisy-lambda}).

It is clear from Lemma~\ref{lemma:approx-link-cvx-ncvx-simple} that,
as the size of the gradient $\nabla f(\bm{X},\bm{Y})$ gets smaller,
the nonconvex estimate $\bm{X}\bm{Y}^{\top}$ becomes an increasingly
tighter approximation of any convex optimizer of~(\ref{eq:convex-LS}),
which is consistent with Lemma~\ref{lemma:link-cvx-and-ncvx}. In
contrast to Lemma~\ref{lemma:link-cvx-and-ncvx}, due to the lack of strong convexity, a nonconvex estimate with a near-zero gradient does not imply the uniqueness of the optimizer
of the convex program~(\ref{eq:convex-LS}); rather, it indicates
that any minimizer of~(\ref{eq:convex-LS}) lies within a sufficiently
small neighborhood surrounding $\bm{X}\bm{Y}^{\top}$ (cf.~(\ref{eq:closeness-cvx-ncvx})).

\subsection{Construction of an approximate nonconvex optimizer}

So far, Lemmas~\ref{lemma:link-cvx-and-ncvx}-\ref{lemma:approx-link-cvx-ncvx-simple}
are both deterministic results based on Condition~\ref{assumption:link-cvx-ncvx-noise}.
As we will soon see, under Assumption~\ref{assumption:sampling-noise},
we can derive simpler conditions that --- with high probability --- guarantee
Condition~\ref{assumption:link-cvx-ncvx-noise}.
We start with Condition~\ref{assumption:link-cvx-ncvx-noise}(a).

\begin{lemma}\label{lemma:noise-bound}Suppose $n^{2}p\geq Cn\log^{2}n$
for some sufficiently large constant $C>0$. Then with probability
at least $1-O(n^{-10})$, one has $\left\Vert \mathcal{P}_{\Omega}\left(\bm{E}\right)\right\Vert \lesssim\sigma\sqrt{np}.$
	As a result, Condition~\ref{assumption:link-cvx-ncvx-noise} holds (i.e.~$\|\mathcal{P}_{\Omega}(\bm{E})\|<\lambda/8$) as long as $\lambda=C_{\lambda}\sigma\sqrt{np}$
for some sufficiently large constant $C_{\lambda}>0$. \end{lemma}\begin{proof}This
follows from \cite[Lemma 11]{chen2015fast} with a slight and straightforward
modification to accommodate the asymmetric noise here. For brevity,
we omit the proof. \end{proof}

Turning attention to Condition~\ref{assumption:link-cvx-ncvx-noise}(b) and Condition~\ref{assumption:link-cvx-ncvx-injectivity}, we have
the following lemma, the proof of which is deferred to Appendix~\ref{subsec:Proof-of-Lemma-injectivity-main}.

\begin{lemma}\label{lemma:injectivity-main} Under the assumptions of Theorem~\ref{thm:main-convex},  with probability
exceeding $1-O(n^{-10})$
we have 
%, i.e.
	$$\|\mathcal{P}_{\Omega}(\bm{X}\bm{Y}^{\top}-\bm{M}^{\star})-p(\bm{X}\bm{Y}^{\top}-\bm{M}^{\star})\|<\lambda/8 \qquad \hfill \mathrm{(Condition~}\ref{assumption:link-cvx-ncvx-noise}\mathrm{(b))} $$
%and 
%, i.e.
\begin{align*}
\frac{1}{p}\left\Vert \mathcal{P}_{\Omega}\left(\bm{H}\right)\right\Vert _{\mathrm{F}}^{2} & \geq\frac{1}{32\kappa}\left\Vert \bm{H}\right\Vert _{\mathrm{F}}^{2},\quad\text{for all }\bm{H}\in T
\qquad
	\mathrm{(Condition~}\ref{assumption:link-cvx-ncvx-injectivity}\mathrm{~with~}c_{\mathrm{inj}}=(32\kappa)^{-1}\mathrm{)}
\end{align*}
hold simultaneously for all $(\bm{X},\bm{Y})$ obeying
\begin{align}\label{subeq:condition-inf}
& \max\left\{ \left\Vert \bm{X} -\bm{X}^{\star}\right\Vert _{\mathrm{2,\infty}},\left\Vert \bm{Y} -\bm{Y}^{\star}\right\Vert _{\mathrm{2,\infty}}\right\}  \nonumber\\
 & \qquad\qquad  \leq C_{\infty}\kappa\left(\frac{\sigma}{\sigma_{\min}}\sqrt{\frac{n\log n}{p}}+\frac{\lambda}{p\,\sigma_{\min}}\right)\max\left\{ \left\Vert \bm{X}^{\star}\right\Vert _{2,\infty},\left\Vert \bm{Y}^{\star}\right\Vert _{2,\infty}\right\} .
\end{align}
Here, $T$ denotes the tangent space of $\bm{X}\bm{Y}^{\top}$,
and $C_{\infty}>0$ is some absolute constant. \end{lemma}

This lemma is a uniform result, namely, the bounds hold irrespective
of the statistical dependency between $(\bm{X},\bm{Y})$ and $\Omega$.
As a consequence, to demonstrate the proximity between the convex
and the nonconvex solutions (cf.~(\ref{eq:closeness-cvx-ncvx})),
it remains to identify a point $(\bm{X},\bm{Y})$ with vanishingly
small gradient (cf.~\eqref{eq:small-gradient-f}) that is sufficiently
close to the truth (cf.~\eqref{subeq:condition-inf}).

As we already alluded to previously, a simple gradient descent algorithm
aimed at solving the nonconvex problem~(\ref{eq:nonconvex_mc_noisy-lambda})
might help us produce an approximate nonconvex optimizer. This procedure
is summarized in Algorithm~\ref{alg:gd-mc-primal}. Our hope is this: when initialized at the ground truth and run for
sufficiently many iterations, the GD trajectory produced by Algorithm
\ref{alg:gd-mc-primal} will contain at least one approximate stationary
point of~(\ref{eq:nonconvex_mc_noisy-lambda}) with the desired properties
\eqref{eq:small-gradient-f} and \eqref{subeq:condition-inf}. We
shall note that Algorithm~\ref{alg:gd-mc-primal} is {\em not practical}
since it starts from the ground truth $(\bm{X}^{\star},\bm{Y}^{\star})$; this is an auxiliary step mainly to simplify the theoretical analysis. 
While we can certainly make it practical by adopting spectral initialization as in \cite{ma2017implicit,chen2019nonconvex}, it requires more lengthy proofs without further improving our statistical guarantees.  
%Consequently, we decide to 

\begin{algorithm}[ht]
\caption{Construction of an approximate primal solution.}

\label{alg:gd-mc-primal}\begin{algorithmic}

\STATE \textbf{{Initialization}}: $\bm{X}^{0}=\bm{X}^{\star}$;
$\bm{Y}^{0}=\bm{Y}^{\star}$.

\STATE \textbf{{Gradient updates}}: \textbf{for }$t=0,1,\ldots,t_{0}-1$
\textbf{do}

\STATE \vspace{-1em}
 \begin{subequations} \label{subeq:GD-rules}
\begin{align}
\bm{X}^{t+1}= & \bm{X}^{t}-\eta\nabla_{\bm{X}}f(\bm{X}^{t},\bm{Y}^{t})=\bm{X}^{t}-\frac{\eta}{p}\Big(\mathcal{P}_{\Omega}\left(\bm{X}^{t}\bm{Y}^{t\top}-\bm{M}\right)\bm{Y}^{t}+\lambda\bm{X}^{t}\Big);\label{eq:gradient_update}\\
\bm{Y}^{t+1}= & \bm{Y}^{t}-\eta\nabla_{\bm{Y}}f(\bm{X}^{t},\bm{Y}^{t})=\bm{Y}^{t}-\frac{\eta}{p}\Big(\left[\mathcal{P}_{\Omega}\left(\bm{X}^{t}\bm{Y}^{t\top}-\bm{M}\right)\right]^{\top}\bm{X}^{t}+\lambda\bm{Y}^{t}\Big).
\end{align}
\end{subequations}Here, $\eta>0$ is the step size.

\begin{comment}
\STATE\textbf{{Output}}: \yxc{fill in missing details here}.
\end{comment}

\end{algorithmic}
\end{algorithm}

\begin{comment}
Inspired by the clever \emph{golfing scheme }\cite{Gross2011recovering}
in the noiseless case, we turn to finding an \emph{approximate} primal-dual
pair, instead of the exact one as in Lemmas~\ref{lemma:exact-primal-dual}
and~\ref{lemma:link-cvx-and-ncvx}. This is supplied in the following
lemma.

Lemma~\ref{lemma:approx-link-cvx-ncvx} can be viewed as an approximate
version of Lemma~\ref{lemma:link-cvx-and-ncvx}, which utilizes a
\emph{particular approximate }stationary point $(\bm{X},\bm{Y})$
to certify the performance of any minimizer of~(\ref{eq:convex_mc_noisy})
(cf.~(\ref{eq:approximate-fro-error})). Two notable differences are
worth mentioning. First, the small gradient condition~(\ref{subeq:approximate-stationary-ncvx})
in Lemma~\ref{lemma:approx-link-cvx-ncvx} relaxes the exact stationary
condition in Lemma~\ref{lemma:link-cvx-and-ncvx}. Second, Lemma~\ref{lemma:approx-link-cvx-ncvx}
requires an approximate balance between $\bm{X}$ and $\bm{Y}$; see
~(\ref{eq:approximate-balance}). This can be also viewed as a relaxed
condition of the stationary point since every stationary point of
~(\ref{eq:nonconvex_mc_noisy}) is balanced in the sense that $\bX^{\top}\bX-\bY^{\top}\bY=\bm{0}$.
As a result, the theoretical guarantee is not as strong as that provided
by an exact primal-dual pair in that we cannot guarantee the uniqueness
of the minimizer of~(\ref{eq:convex_mc_noisy}) any more. Nevertheless,
this is a minor issue from a practical point of view.
\end{comment}

\subsection{Properties of the nonconvex iterates}

In this subsection, we will build upon the literature on nonconvex
low-rank matrix completion to justify that the estimates returned
by Algorithm~\ref{alg:gd-mc-primal} satisfy the requirement stated
in~(\ref{subeq:condition-inf}). Our theory will be largely established
upon the leave-one-out strategy introduced by Ma et al.~\cite{ma2017implicit},
which is an effective analysis technique to control the $\ell_{2,\infty}$
error of the estimates. This strategy has recently been extended by
Chen~et~al.~\cite{chen2019nonconvex} to the more general rectangular
case with an improved sample complexity bound.

Before continuing, we introduce several useful notations. Notice that
the matrix product of $\bm{X}^{\star}$ and $\bm{Y}^{\star\top}$
is invariant under global orthonormal transformation, namely, for
any orthonormal matrix $\bm{R}\in\mathbb{R}^{r\times r}$ one has
$\bm{X}^{\star}\bm{R}(\bm{Y}^{\star}\bm{R})^{\top}=\bm{X}^{\star}\bm{Y}^{\star\top}$.
Viewed in this light, we shall consider distance metrics modulo global
rotation. In particular, the theory relies heavily on a specific global
rotation matrix defined as follows
\begin{equation}
\bm{H}^{t}\triangleq\arg\min_{\bm{R}\in\mathcal{O}^{r\times r}}\big(\left\Vert \bm{X}^{t}\bm{R}-\bm{X}^{\star}\right\Vert _{\mathrm{F}}^{2}+\left\Vert \bm{Y}^{t}\bm{R}-\bm{Y}^{\star}\right\Vert _{\mathrm{F}}^{2}\big)^{1/2},
\label{eq:defn-rotation-H}
\end{equation}
where $\mathcal{O}^{r\times r}$ is the set of $r\times r$ orthonormal
matrices.

We are now ready to present the performance guarantees for Algorithm
\ref{alg:gd-mc-primal}.

\begin{comment}
Set $\lambda=C_{\lambda}\sigma\sqrt{np}$ for some sufficiently large
constant. Suppose the sample size obeys $n^{2}p\geq C\kappa^{XXX}\mu^{2}r^{2}n\log^{2}n$
for some sufficiently large constant $C>0$ and the noise satisfies
$\sigma\sqrt{\frac{n}{p}}\leq c\frac{\sigma_{\min}}{\kappa^{XXX}\sqrt{\mu r\log n}}$
for some small enough constant $c>0$. \cm{After we finalize the
conditions, we can put back this sentence ``''}
\end{comment}

\begin{lemma}[\textbf{Quality of the nonconvex estimates}]\label{lemma:nonconvex-GD}Instate
the notation and hypotheses of Theorem~\ref{thm:main-convex}. With
probability at least $1-O\left(n^{-3}\right)$, the iterates $\{(\bm{X}^{t},\bm{Y}^{t})\}_{0\leq t\leq t_{0}}$
of Algorithm~\ref{alg:gd-mc-primal} satisfy \begin{subequations}
\label{subeq:induction_GD}
\begin{align}
\max\left\{ \left\Vert \bm{X}^{t}\bm{H}^{t}-\bm{X}^{\star}\right\Vert _{\mathrm{F}},\left\Vert \bm{Y}^{t}\bm{H}^{t}-\bm{Y}^{\star}\right\Vert _{\mathrm{F}}\right\}  & \leq C_{\mathrm{F}}\left(\frac{\sigma}{\sigma_{\min}}\sqrt{\frac{n}{p}}+\frac{\lambda}{p\,\sigma_{\min}}\right)\left\Vert \bm{X}^{\star}\right\Vert _{\mathrm{F}},\label{eq:nonconvex-fro-norm}\\
\max\left\{ \left\Vert \bm{X}^{t}\bm{H}^{t}-\bm{X}^{\star}\right\Vert ,\left\Vert \bm{Y}^{t}\bm{H}^{t}-\bm{Y}^{\star}\right\Vert \right\}  & \leq C_{\mathrm{op}}\left(\frac{\sigma}{\sigma_{\min}}\sqrt{\frac{n}{p}}+\frac{\lambda}{p\,\sigma_{\min}}\right)\left\Vert \bm{X}^{\star}\right\Vert ,\label{eq:nonconvex-spectral-norm}\\
\max\left\{ \left\Vert \bm{X}^{t}\bm{H}^{t}-\bm{X}^{\star}\right\Vert _{\mathrm{2,\infty}},\left\Vert \bm{Y}^{t}\bm{H}^{t}-\bm{Y}^{\star}\right\Vert _{\mathrm{2,\infty}}\right\}  \nonumber\\
  \leq C_{\infty}\kappa\left(\frac{\sigma}{\sigma_{\min}}\sqrt{\frac{n\log n}{p}}+\frac{\lambda}{p\,\sigma_{\min}}\right)&\max\left\{ \left\Vert \bm{X}^{\star}\right\Vert _{2,\infty},\left\Vert \bm{Y}^{\star}\right\Vert _{2,\infty}\right\}, \label{eq:nonconvex-2-infty-norm}
\end{align}
\end{subequations}
\begin{equation}
\min_{0\leq t<t_{0}}\left\Vert \nabla f\left(\bm{X}^{t},\bm{Y}^{t}\right)\right\Vert _{\mathrm{F}}\leq\frac{1}{n^{5}}\frac{\lambda}{p}\sqrt{\sigma_{\min}},\label{eq:nonconvex-small-gradient}
\end{equation}
 where $C_{\mathrm{F}},C_{\mathrm{op}}, C_{\infty}>0$ are some absolute
constants, provided that $\eta\asymp1/(n\kappa^{3}\sigma_{\max})$
and that $t_{0}=n^{18}$. \end{lemma}

This lemma, which we establish in Appendix~\ref{sec:Proof-of-Lemma-nonconvex-GD},
reveals that for a polynomially large number of iterations, all iterates
of the gradient descent sequence --- when initialized at the ground
truth --- remain fairly close to the true low-rank factors. This
holds in terms of the estimation errors measured by the Frobenius
norm, the spectral norm, and the $\ell_{2,\infty}$ norm. In
particular, the proximity in terms of the $\ell_{2,\infty}$
norm error plays a pivotal role in implementing our analysis strategy
(particularly Lemmas~\ref{lemma:approx-link-cvx-ncvx-simple}-\ref{lemma:injectivity-main})
described previously. In addition, this lemma (cf.~(\ref{eq:nonconvex-small-gradient}))
guarantees the existence of a small-gradient point within this sequence
$\{(\bm{X}^{t},\bm{Y}^{t})\}_{0\leq t\leq t_{0}}$, a somewhat straightforward
property of GD tailored to smooth problems~\cite{nesterov2012make}.  This in turn enables us to invoke Lemma~\ref{lemma:approx-link-cvx-ncvx-simple}.

As immediate consequences of Lemma~\ref{lemma:nonconvex-GD}, with
high probability we have \begin{subequations}\label{subeq:XY-quality}
\begin{align}
\left\Vert \bm{X}^{t}\bm{Y}^{t\top}-\bm{M}^{\star}\right\Vert _{\mathrm{F}}\,\, & \leq3\kappa C_{\mathrm{F}}\left(\frac{\sigma}{\sigma_{\min}}\sqrt{\frac{n}{p}}+\frac{\lambda}{p\,\sigma_{\min}}\right)\left\Vert \bm{M}^{\star}\right\Vert _{\mathrm{F}}\label{eq:induction_ell_F-1}\\
\left\Vert \bm{X}^{t}\bm{Y}^{t\top}-\bm{M}^{\star}\right\Vert _{\infty} & \leq3C_{\infty}\sqrt{\kappa^{3}\mu r}\left(\frac{\sigma}{\sigma_{\min}}\sqrt{\frac{n\log n}{p}}+\frac{\lambda}{p\,\sigma_{\min}}\right)\left\Vert \bm{M}^{\star}\right\Vert _{\infty}\label{eq:induction_original_ell_infty-MC_thm-1}\\
\left\Vert \bm{X}^{t}\bm{Y}^{t\top}-\bm{M}^{\star}\right\Vert \,\,\,\,\,  & \leq3C_{\mathrm{op}}\left(\frac{\sigma}{\sigma_{\min}}\sqrt{\frac{n}{p}}+\frac{\lambda}{p\,\sigma_{\min}}\right)\left\Vert \bm{M}^{\star}\right\Vert \label{eq:induction_original_operator-MC_thm-1}
\end{align}
\end{subequations}
for all $0\leq t\leq t_{0}$. The proof is
deferred to Appendix~\ref{subsec:Proof-of-XY-quality}.

\subsection{Proof of Theorem~\ref{thm:main-convex}}

Let $t_{*}\triangleq\arg\min_{0\leq t<t_{0}}\left\Vert \nabla f\left(\bm{X}^{t},\bm{Y}^{t}\right)\right\Vert _{\mathrm{F}}$,
and take $(\bm{X}_{\mathsf{ncvx}},\bm{Y}_{\mathsf{ncvx}})=\left(\bm{X}^{t_{*}}\bm{H}^{t_{*}},\bm{Y}^{t_{*}}\bm{H}^{t_{*}}\right)$ (cf.~\eqref{eq:defn-rotation-H}).
It is straightforward to verify that $(\bm{X}_{\mathsf{ncvx}},\bm{Y}_{\mathsf{ncvx}})$
obeys (i) the small-gradient condition~(\ref{eq:small-gradient-f}),
and (ii) the proximity condition~(\ref{subeq:condition-inf}). We are now positioned to invoke Lemma~\ref{lemma:approx-link-cvx-ncvx-simple}:  for any optimizer $\bm{Z}_{\mathsf{cvx}}$ of~(\ref{eq:convex-LS}), one has
\begin{align}
	\big\|\bm{Z}_{\mathsf{cvx}}-\bm{X}_{\mathsf{ncvx}}\bm{Y}_{\mathsf{ncvx}}^{\top}\|_{\mathrm{F}}
	& \lesssim\frac{\kappa}{c_{\mathrm{inj}}}\frac{1}{\sqrt{\sigma_{\min}}}\left\Vert \nabla f(\bm{X}_{\mathsf{ncvx}},\bm{Y}_{\mathsf{ncvx}})\right\Vert _{\mathrm{F}}\lesssim\frac{\kappa^{2}}{n^{5}}\frac{\lambda}{p} \nonumber\\
	& = \frac{\kappa}{n^{5}}\frac{\lambda}{p\,\sigma_{\min}} (\kappa\sigma_{\min}) ~{=}~ \frac{\kappa}{n^{5}}\frac{\lambda}{p\,\sigma_{\min}} \|\bm{M}^{\star}\| \nonumber\\
	& \lesssim \frac{1}{n^{4}}\frac{\lambda}{p\,\sigma_{\min}} \|\bm{M}^{\star}\|. \label{eq:proximity-UB1}
\end{align}
 The last line arises since $n \gg \kappa$ ---  a consequence of the sample complexity condition $n p \gtrsim \kappa^4 \mu^2 r^2 \log^3 n$ (and hence $n\geq np \gtrsim \kappa^4 \mu^2 r^2 \log^3 n \gg \kappa^4$).
This taken collectively with the property~(\ref{subeq:XY-quality})
implies that
\begin{align*}
\big\|\bm{Z}_{\mathsf{cvx}}-\bm{M}^{\star}\|_{\mathrm{F}} & \leq\big\|\bm{Z}_{\mathsf{cvx}}-\bm{X}_{\mathsf{ncvx}}\bm{Y}_{\mathsf{ncvx}}^{\top}\|_{\mathrm{F}}+\big\|\bm{X}_{\mathsf{ncvx}}\bm{Y}_{\mathsf{ncvx}}^{\top}-\bm{M}^{\star}\|_{\mathrm{F}}\\
 & \lesssim\frac{1}{n^{4}}\frac{\lambda}{p\,\sigma_{\min}} \| \bm{M}^{\star} \| + \kappa\left(\frac{\sigma}{\sigma_{\min}}\sqrt{\frac{n}{p}}+\frac{\lambda}{p\,\sigma_{\min}}\right)\left\Vert \bm{M}^{\star}\right\Vert _{\mathrm{F}}\\
 & \asymp\kappa\left(\frac{\sigma}{\sigma_{\min}}\sqrt{\frac{n}{p}}+\frac{\lambda}{p\,\sigma_{\min}}\right)\left\Vert \bm{M}^{\star}\right\Vert _{\mathrm{F}}.
\end{align*}
%
%where the last line follows from $\kappa\sigma_{\min}\leq\|\bm{M}^{\star}\|_{\mathrm{F}}$.
In other words, since $\bm{X}_{\mathsf{ncvx}}\bm{Y}_{\mathsf{ncvx}}^{\top}$
and $\bm{Z}_{\mathsf{ncvx}}$ are exceedingly close, the error $\bm{Z}_{\mathsf{cvx}}-\bm{M}^{\star}$ is mainly accredited to $\bm{X}_{\mathsf{ncvx}}\bm{Y}_{\mathsf{ncvx}}^{\top}-\bm{M}^{\star}$.
Similar arguments lead to
\begin{align*}
\big\|\bm{Z}_{\mathsf{cvx}}-\bm{M}^{\star}\| & \lesssim\left(\frac{\sigma}{\sigma_{\min}}\sqrt{\frac{n}{p}}+\frac{\lambda}{p\,\sigma_{\min}}\right)\left\Vert \bm{M}^{\star}\right\Vert ,\\
\big\|\bm{Z}_{\mathsf{cvx}}-\bm{M}^{\star}\|_{\infty} & \lesssim\sqrt{\kappa^{3}\mu r}\left(\frac{\sigma}{\sigma_{\min}}\sqrt{\frac{n\log n}{p}}+\frac{\lambda}{p\,\sigma_{\min}}\right)\left\Vert \bm{M}^{\star}\right\Vert _{\infty}.
\end{align*}

We are left with proving the properties of $\bm{Z}_{\mathsf{cvx},r}$. Since $\bm{Z}_{\mathsf{cvx},r}$ is defined to be the best rank-$r$ approximation of $\bm{Z}_{\mathsf{cvx}}$,
%Observe that
% %
% \begin{align*}
% \big\|\bm{Z}_{\mathsf{cvx}}-\bm{X}_{\mathsf{ncvx}}\bm{Y}_{\mathsf{ncvx}}^{\top}\|_{\mathrm{F}} & \lesssim\frac{\kappa^{2}}{n^{5}}\frac{\lambda}{p}=\frac{\kappa^{2}}{n^{5}}\cdot C_{\lambda}\sigma\sqrt{\frac{n}{p}}\frac{1}{\sqrt{r}\sigma_{\min}}\sqrt{r}\sigma_{\min}\\
%  & \leq\frac{1}{n^{4}}\cdot\frac{\sigma}{\sigma_{\min}}\sqrt{\frac{n}{p}}\|\bm{M}^{\star}\|_{\mathrm{F}},
% \end{align*}
%
% where the second line follows from the facts that $n\gg\kappa^{2}$
% and that $\sqrt{r}\sigma_{\min}\leq\|\bm{M}^{\star}\|_{\mathrm{F}}$.
%By the definition of $\bm{Z}_{\mathsf{cvx},r}$,
%
one can invoke \eqref{eq:proximity-UB1} to derive
\[
	\big\|\bm{Z}_{\mathsf{cvx}}-\bm{Z}_{\mathsf{cvx},r}\|_{\mathrm{F}}\leq\big\|\bm{Z}_{\mathsf{cvx}}-\bm{X}_{\mathsf{ncvx}}\bm{Y}_{\mathsf{ncvx}}^{\top}\|_{\mathrm{F}}
	\lesssim \frac{1}{n^{4}}\frac{\lambda}{p\,\sigma_{\min}} \|\bm{M}^{\star}\|,
\]
from which~(\ref{eq:Zcvx-r-bound}) follows. Repeating the above calculations
implies that~(\ref{eq:Zcvx-error}) holds if $\bm{Z}_{\mathsf{cvx}}$ is replaced by $\bm{Z}_{\mathsf{cvx},r}$,
thus concluding the proof.

\section{Prior art\label{sec:Prior-art}}

Nuclear norm minimization, pioneered by the seminal works \cite{RecFazPar07,ExactMC09,CanTao10,fazel2002matrix},\emph{
}has been a popular and principled approach to low-rank matrix recovery.
In the noiseless setting, i.e.~$\bm{E}=\bm{0}$, it amounts to solving
the following constrained convex program
\begin{equation}
\text{minimize}_{\bm{Z}\in\mathbb{R}^{n\times n}}\text{\ensuremath{\left\Vert \bm{Z}\right\Vert _{*}\quad\quad}subject to}\quad\mathcal{P}_{\Omega}\left(\bm{Z}\right)=\mathcal{P}_{\Omega}\left(\bm{M}^{\star}\right),\label{eq:SDP-noiseless}
\end{equation}
which enjoys great theoretical success. Informally, this approach
enables exact recovery of a rank-$r$ matrix $\bm{M}^{\star}\in\mathbb{R}^{n\times n}$
as soon as the sample size is about the order of $nr$ --- the intrinsic
degrees of freedom of a rank-$r$ matrix \cite{Gross2011recovering,recht2011simpler,chen2015incoherence}.
In particular, Gross \cite{Gross2011recovering} blazed a trail by developing an ingenious golfing scheme for dual construction --- an analysis technique that has found applications far beyond matrix completion. 
When it comes to the noisy case, Candès and Plan \cite{CanPla10}
first studied the stability of convex programming when the noise is
bounded and possibly adversarial, followed by \cite{Negahban2012restricted}
and \cite{MR2906869} using two modified convex programs. As we have
already discussed, none of these papers provide optimal statistical guarantees under our model
when $r=O(1)$. Other related papers such as \cite{klopp2014noisy,cai2016matrix}
include similar estimation error bounds and suffer from
similar sub-optimality issues.

Turning to nonconvex optimization, we note that this approach has
recently received much attention for various low-rank matrix factorization
problems, owing to its superior computational advantage compared to
convex programming (e.g.~\cite{KesMonSew2010,jain2013low,candes2014wirtinger,ChenCandes15solving,tu2016low,zhang2017nonconvex}).
The convergence guarantees for matrix completion have been established
for various algorithms such as gradient descent on manifold
\cite{KesMonSew2010,Se2010Noisy}, alternating minimization \cite{jain2013low,hardt2014understanding}, gradient descent \cite{sun2016guaranteed,ma2017implicit,wang2016unified,chen2019nonconvex},
and projected gradient descent \cite{chen2015fast}, provided that
a suitable initialization (like spectral initialization) is available
\cite{KesMonSew2010,jain2013low,sun2016guaranteed,ma2017implicit,chen2018asymmetry}.
Our work is mostly related to \cite{ma2017implicit,chen2019nonconvex},
which studied (vanilla) gradient descent for nonconvex matrix completion. This
algorithm was first analyzed by \cite{ma2017implicit} via a leave-one-out
argument --- a technique that proves useful in analyzing various
statistical algorithms \cite{el2015impact,sur2017likelihood,zhong2017near,chen2017spectral,abbe2017entrywise,li2018nonconvex,ding2018leave,chen2018gradient}.
In the absence of noise and omitting logarithmic factors, \cite{ma2017implicit}
showed that $O(nr^{3})$ samples are sufficient for vanilla GD to
yield $\varepsilon$ accuracy in $O(\log\frac{1}{\varepsilon})$ iterations (without the need of extra regularization procedures); the
sample complexity was further improved to $O(nr^{2})$ by \cite{chen2019nonconvex}.
Apart from gradient descent, other nonconvex methods (e.g.~\cite{rennie2005fast,jain2010guaranteed,wen2012solving,jain2013low,fornasier2011low,vandereycken2013low,lai2013improved,hardt2014understanding,jin2016provable,rohde2011estimation,wei2016guarantees,ding2018leave,gunasekar2013noisy,cao2016poisson,zhao2015nonconvex-estimation})
and landscape\,/\,geometry properties have been investigated \cite{ge2016matrix,chen2017memory,park2017non,ge2017no,shapiro2018matrix};
these are, however, beyond the scope of the current paper.

Another line of works asserted that a large family of SDPs admits
low-rank solutions \cite{barvinok1995problems}, which in turn motivates
the Burer-Monteiro approach \cite{burer2003nonlinear,boumal2016non}.
When applied to matrix completion, however, the generic theoretical
guarantees therein lead to conservative results. Take the noiseless
case (\ref{eq:SDP-noiseless}) for instance: these results revealed
the existence of a solution of rank at most $O(\sqrt{n^{2}p})$, which
however is often much larger the true rank (e.g.~when $r\asymp1$
and $p\asymp\mathrm{poly}\log(n)/n$, one has $\sqrt{n^{2}p}\gg\sqrt{n}\gg r$).
Moreover, this line of works does not imply that all solutions to
the SDP of interest are (approximately) low-rank.

Finally, the connection between convex and nonconvex optimization
has also been explored in line spectral estimation \cite{li2018approximate},
although the context therein is drastically different from ours.

\begin{comment}
In comparison, the focus of this work is to establish an intriguing
connection between convex and nonconvex optimization, thus tightening
stability analysis of convex relaxation.
\end{comment}

\section{Discussion\label{sec:Discussion}}

This paper provides an improved statistical analysis for the natural
convex program (\ref{eq:convex-LS}), without the need of enforcing
additional spikiness constraint. Our theoretical analysis uncovers
an intriguing connection between convex relaxation and nonconvex optimization,
which we believe is applicable to many other problems beyond matrix
completion. Having said that, our current theory leaves open a variety
of important directions for future exploration. Here we sample a few
interesting ones.
\begin{itemize}
\item \emph{Improving dependency on $r$ and $\kappa$. }While our theory
is optimal when $r$ and $\kappa$ are both constants, it becomes
increasingly looser as either $r$ or $\kappa$ grows. For instance,
in the noiseless setting, it has been shown that the sample complexity
for convex relaxation scales as $O(nr)$ --- linear in $r$ and independent
of $\kappa$ --- which is better than the current results.
It is worth noting that  existing theory for nonconvex matrix factorization
typically falls short of providing optimal scaling in $r$ and $\kappa$
\cite{KesMonSew2010,sun2016guaranteed,chen2015fast,ma2017implicit,chen2019nonconvex}.
Thus, tightening the dependency of sample complexity on $r$ and $\kappa$
might call for new analysis tools.
\item \emph{Approximate low-rank structure}. So far our theory is built
upon the assumption that the ground-truth matrix $\bm{M}^{\star}$
is exactly low-rank, which falls short of accommodating the more realistic
scenario where $\bm{M}^{\star}$ is only approximately low-rank. For
the approximate low-rank case, it is not yet clear whether the nonconvex
factorization approach can still serve as a tight proxy. In addition,
the landscape of nonconvex optimization for the approximately low-rank
case \cite{chen2017memory} might shed light on how to handle this
case.
\item \emph{Extension to deterministic noise. }Our current theory --- in
particular, the leave-one-out analysis for the nonconvex approach
--- relies heavily on the randomness assumption (i.e.~i.i.d.~sub-Gaussian)
of the noise. In order to justify the broad applicability of convex
relaxation, it would be interesting to see whether one can generalize
the theory to cover deterministic noise with bounded magnitudes.
\item \emph{Extension to structured matrix completion}. Many applications
involve low-rank matrices that exhibit additional structures, enabling
a further reduction of the sample complexity \cite{Fazel2003Hankel,chen2014robust,cai2019fast}.
For instance, if a matrix is Hankel and low-rank, then the sample
complexity can be $O(n)$ times smaller than the generic low-rank
case. The existing stability guarantee of Hankel matrix completion,
however, is overly pessimistic compared to practical performance \cite{chen2014robust}.
The analysis framework herein might be amenable to the study of Hankel
matrix completion and help close the theory-practice gap.
\item \emph{Extension to robust PCA and blind deconvolution. }Moving beyond
matrix completion, there are other problems that are concerned with
recovering low-rank matrices. Notable examples include robust principal
component analysis \cite{CanLiMaWri09,chandrasekaran2011rank,chen2013low},
blind deconvolution \cite{ahmed2014blind,ling2015self} and blind
demixing \cite{ling2017blind,jung2017blind}. The stability analyses of the convex
relaxation approaches for these problems \cite{zhou2010stable,ahmed2014blind,ling2017blind}
often adopt a similar approach as \cite{CanPla10}, and consequently
are sub-optimal. The insights from the present paper might promise
tighter statistical guarantees for such problems.
\end{itemize}

Finally, we remark that the intimate link between convex and nonconvex optimization enables statistically optimal  inference and uncertainty quantification for noisy matrix completion (e.g.~construction of optimal confidence intervals for each missing entry). The interested readers are referred to our companion paper \cite{chen2019inference} for in-depth discussions.

\begin{comment}
\begin{itemize}
\item Suppose that, in addition to random sub-Gaussian noise, a small fraction
of the observed entries are further corrupted by adversarial outliers
with arbitrary magnitudes. The convex relaxation approach for this
problem also enjoys optimal sample complexity in the noise-free case,
even when a constant fraction of the observed entries are corrupted
by outliers \cite{CanLiMaWri09,chandrasekaran2011rank,chen2013low}.
Our analysis in the current paper might be promising in enabling near-optimal
stability guarantees in the presence of noise, where existing analysis
is sub-optimal \cite{zhou2010stable}.
\item \emph{Extension to blind deconvolution and blind demixing. }Moving
beyond low-rank matrix completion, there are other problems that are
equivalent to recovering a low-rank matrix from its linear measurements.
Notable examples include blind deconvolution \cite{ahmed2014blind,ling2015self}
and blind demixing \cite{ling2017blind}, for which the measurement
mechanisms are often assumed to be semi-random. The stability analysis
of the convex relaxation algorithm in \cite{ahmed2014blind,ling2017blind}
adopts a similar approach as \cite{CanPla10}, and consequently is
sub-optimal. The insight from this paper together with the leave-one-out
analysis in \cite{ma2017implicit,dong2018nonconvex} might lead to
a tighter analysis for such problems.
\end{itemize}
\end{comment}

\section*{Acknowledgements}

Y.~Chen is supported in part by the AFOSR YIP award FA9550-19-1-0030,
by the ARO grant W911NF-18-1-0303, by the ONR grant N00014-19-1-2120, by the NSF grants CCF-1907661 and IIS-1900140, and by the Princeton SEAS innovation award. Y.~Chi is supported in part by ONR under the grants N00014-18-1-2142 and N00014-19-1-2404,
by ARO under the grant W911NF-18-1-0303, and by NSF under the grants
CAREER ECCS-1818571 and CCF-1806154. J.~Fan is supported in part by NSF Grants DMS-1662139 and DMS-1712591, ONR grant N00014-19-1-2120, and NIH Grant R01-GM072611-12. This work was done in part while Y.~Chen was visiting the Kavli Institute for Theoretical Physics (supported in part by NSF grant PHY-1748958). Y.~Chen thanks Emmanuel Candès for motivating discussions about noisy
matrix completion.

\bibliographystyle{alphaabbr}
\bibliography{bibfileNonconvex}

\appendix
%dummy comment inserted by tex2lyx to ensure that this paragraph is not empty%dummy comment inserted by tex2lyx to ensure that this paragraph is not empty%dummy comment inserted by tex2lyx to ensure that this paragraph is not empty

\newpage{}

\section{Preliminaries \label{sec:Preliminaries}}

In this section, we gather a few notations and preliminary facts that
are used throughout the proofs.

To begin with, in view of the incoherence assumption (cf.~Definition~\ref{def:incoherence}), one has 
\begin{equation}
\left\Vert \bm{X}^{\star}\right\Vert _{2,\infty}\leq\sqrt{\mu r/n}\left\Vert \bm{X}^{\star}\right\Vert \qquad\text{and}\qquad\left\Vert \bm{Y}^{\star}\right\Vert _{2,\infty}\leq\sqrt{\mu r/n}\left\Vert \bm{Y}^{\star}\right\Vert .\label{eq:X-Y-incoherence}
\end{equation}
This follows from 
\[
\left\Vert \bm{X}^{\star}\right\Vert _{2,\infty}=\big\|\bm{U}^{\star}\left(\bm{\Sigma}^{\star}\right)^{1/2}\big\|_{2,\infty}\leq\left\Vert \bm{U}^{\star}\right\Vert _{2,\infty}\big\|\left(\bm{\Sigma}^{\star}\right)^{1/2}\big\|\leq\sqrt{\mu r/n}\left\Vert \bm{X}^{\star}\right\Vert .
\]
The bound for $\bm{Y}^{\star}$ follows from the same argument. In addition, we write $A \ll B$ (resp.~$A \gg B$) if there exists a sufficiently small (resp.~large) constant $c$ such that $A \leq c B$ (resp.~$A \geq c B$).

Finally, for notational convenience, we shall often denote 
\begin{equation}
\mathcal{P}_{\Omega}^{\mathsf{debias}}\left(\bm{B}\right)\triangleq\mathcal{P}_{\Omega}\left(\bm{B}\right)-p\bm{B},\qquad\text{for all }\bm{B}\in\mathbb{R}^{n\times n}.\label{eq:defn-Pomega-tilde}
\end{equation}

\section{Exact duality analysis\label{sec:Proof-of-Lemma-unique-minimizer}}

We show in this section that why the first-order optimality condition
is almost sufficient in guaranteeing the uniqueness of the optimizer. The argument is standard, see e.g. \cite{ExactMC09}.

\begin{lemma}\label{lemma:exact-primal-dual}Let $\bm{Z}=\bm{U}\bm{\Sigma}\bm{V}^{\top}$
be the SVD of $\bm{Z}\in\mathbb{R}^{n\times n}$. Denote by $T$ be
the tangent space of $\bm{Z}$ and by $T^{\perp}$ its orthogonal
complement. Suppose that there exists $\bm{W}\in T^{\perp}$ such
that 
\begin{equation}
\frac{1}{\lambda}\mathcal{P}_{\Omega}\big(\bm{M}-\bm{Z}\big)=\bm{U}\bm{V}^{\top}+\bm{W}.\label{eq:dual-definition}
\end{equation}
Then $\bm{Z}$ is the unique minimizer of (\ref{eq:convex-LS}) if
\begin{enumerate}
\item $\left\Vert \bm{W}\right\Vert <1$;
\item The operator $\mathcal{P}_{\Omega}(\cdot)$ restricted to elements
in $T$ is injective, i.e.~$\mathcal{P}_{\Omega}\left(\bm{H}\right)=\bm{0}$
implies $\bm{H}=\bm{0}$ for any $\bm{H}\in T$. 
\end{enumerate}
\end{lemma}

\begin{comment}
\begin{remark}The injectivity of $\mathcal{P}_{\Omega}(\cdot)$ restricted
to elements in $T$ is guaranteed as long as $\|\frac{1}{p}\mathcal{P}_{T}\mathcal{P}_{\Omega}\mathcal{P}_{T}-\mathcal{P}_{T}\|<1$,
as shown in \cite{ExactMC09}.\end{remark} 
\end{comment}

\begin{proof}[Proof of Lemma \ref{lemma:exact-primal-dual}]

\begin{comment}
@@@@@@@@@

The proof relies heavily on the convexity of the nuclear norm $\|\cdot\|_{*}$
as well as the strict convexity of the squared $\ell_{\mathrm{F}}$
loss. In essence, the first two conditions together correspond to
the first-order optimality condition of (\ref{eq:convex_mc_noisy}),
which guarantees $\hat{\bm{Z}}$ to be a minimizer of (\ref{eq:convex_mc_noisy}).
In addition, with the help of the third condition, i.e. the injectivity
of $\mathcal{P}_{\Omega}(\cdot)$ on $\mathcal{T}$, one can strengthen
the guarantee and certify that $\hat{\bm{Z}}$ is the \emph{unique
}minimizer of (\ref{eq:convex_mc_noisy}). See Appendix \ref{sec:Proof-of-Lemma-unique-minimizer}.

The proof can be divided into two steps: (1) show that $\bm{Z}$ is
a minimizer of (\ref{eq:convex_mc_noisy}) by checking the first order
stationary condition; (2) show its uniqueness. 
\end{comment}

To begin with, the assumption of this lemma implies that 
\[
\bm{U}\bm{V}^{\top}+\bm{W}\in\partial\|\bm{Z}\|_{*},
\]
where $\partial\|\bm{Z}\|_{*}$ denotes the subdifferential of $\|\cdot\|_{*}$
at $\bm{Z}$. This combined with (\ref{eq:dual-definition}) reveals
that
\begin{equation}
\frac{1}{\lambda}\mathcal{P}_{\Omega}\left(\bm{M}-\bm{Z}\right)\in\partial\left\Vert \bm{Z}\right\Vert _{*},\label{eq:1st-order-stationary-condition-nnm-1}
\end{equation}
thus indicating that $\bm{Z}$ is a minimizer of the convex program
(\ref{eq:convex-LS}).

Next, we justify the uniqueness of $\bm{Z}$. Before continuing, we
record a fact regarding the minimizers of~(\ref{eq:convex-LS}).

\begin{claim}\label{claim:P-Omega-H=00003D00003D00003D00003D0}Suppose
that $\bm{Z}_{1}$ and $\bm{Z}_{2}$ are both minimizers of (\ref{eq:convex-LS}).
Then one has $\mathcal{P}_{\Omega}\left(\bm{Z}_{1}\right)=\mathcal{P}_{\Omega}\left(\bm{Z}_{2}\right)$.\end{claim}

With this claim at hand, every minimizer of (\ref{eq:convex-LS})
can be written as $\bm{Z}+\bm{H}$ for some $\bm{H}$ obeying $\mathcal{P}_{\Omega}(\bm{H})=\bm{0}$.
It then suffices to prove that for any $\bm{H}\neq\bm{0}$, one has
$g\left(\bm{Z}+\bm{H}\right)>g\left(\bm{Z}\right)$, where $g(\cdot)$
is the objective function in~(\ref{eq:convex-LS}). To this end, we
note that 
\begin{align}
g\left(\bm{Z}+\bm{H}\right) & =\tfrac{1}{2}\left\Vert \mathcal{P}_{\Omega}\left(\bm{Z}+\bm{H}-\bm{M}\right)\right\Vert _{\mathrm{F}}^{2}+\lambda\left\Vert \bm{Z}+\bm{H}\right\Vert _{*}\nonumber \\
 & =\tfrac{1}{2}\left\Vert \mathcal{P}_{\Omega}\left(\bm{Z}-\bm{M}\right)\right\Vert _{\mathrm{F}}^{2}+\lambda\left\Vert \bm{Z}+\bm{H}\right\Vert _{*},\label{eq:gzh-LB1}
\end{align}
where the last relation follows from Claim \ref{claim:P-Omega-H=00003D00003D00003D00003D0}
(i.e.~$\mathcal{P}_{\Omega}(\bm{H})=\bm{0}$). Let $\bm{S}$ be a
subgradient of $\|\cdot\|_{*}$ at point $\bm{Z}$ obeying 
\begin{equation}
\mathcal{P}_{T}\left(\bm{S}\right)=\bm{U}\bm{V}^{\top},\quad\left\Vert \mathcal{P}_{T^{\perp}}\left(\bm{S}\right)\right\Vert \leq1\quad\text{and}\quad\left\langle \mathcal{P}_{T^{\perp}}\left(\bm{S}\right),\mathcal{P}_{T^{\perp}}\left(\bm{H}\right)\right\rangle =\left\Vert \mathcal{P}_{T^{\perp}}\left(\bm{H}\right)\right\Vert _{*}.\label{eq:S-property-1}
\end{equation}
Using the convexity of $\|\cdot\|_{*}$, one can further lower bound
(\ref{eq:gzh-LB1}) by 
\begin{align*}
g\left(\bm{Z}+\bm{H}\right) & \geq\tfrac{1}{2}\left\Vert \mathcal{P}_{\Omega}\left(\bm{Z}-\bm{M}\right)\right\Vert _{\mathrm{F}}^{2}+\lambda\left(\left\Vert \bm{Z}\right\Vert _{*}+\left\langle \bm{S},\bm{H}\right\rangle \right)\\
 & =g\left(\bm{Z}\right)+\lambda\left\langle \bm{S},\bm{H}\right\rangle \\
 & =g\left(\bm{Z}\right)+\lambda\left\langle \bm{U}\bm{V}^{\top}+\bm{W},\bm{H}\right\rangle +\lambda\left\langle \bm{S}-\bm{U}\bm{V}^{\top}-\bm{W},\bm{H}\right\rangle \\
 & \overset{(\text{i})}{=}g\left(\bm{Z}\right)+\lambda\left\langle \bm{S}-\bm{U}\bm{V}^{\top}-\bm{W},\bm{H}\right\rangle \\
 & \overset{(\text{ii})}{=}g\left(\bm{Z}\right)+\lambda\left\langle \mathcal{P}_{T^{\perp}}\left(\bm{S}\right)-\bm{W},\bm{H}\right\rangle .
\end{align*}
Here, (i) follows from our assumption that $\bm{U}\bm{V}^{\top}+\bm{W}$
is supported on $\Omega$ (cf.~(\ref{eq:dual-definition})) and the
fact that $\mathcal{P}_{\Omega}(\bm{H})=\bm{0}$, and (ii) holds since
$\mathcal{P}_{T}(\bm{S})=\bm{U}\bm{V}^{\top}$ (cf.~(\ref{eq:S-property-1})).
We can now expand the above expression as 
\begin{align}
g\left(\bm{Z}+\bm{H}\right) & \geq g\left(\bm{Z}\right)+\lambda\left\langle \mathcal{P}_{T^{\perp}}\left(\bm{S}\right),\mathcal{P}_{T^{\perp}}\left(\bm{H}\right)\right\rangle -\lambda\left\langle \bm{W},\mathcal{P}_{T^{\perp}}\left(\bm{H}\right)\right\rangle \nonumber \\
 & \geq g\left(\bm{Z}\right)+\lambda\left(1-\left\Vert \bm{W}\right\Vert \right)\left\Vert \mathcal{P}_{T^{\perp}}\left(\bm{H}\right)\right\Vert _{*},\label{eq:last-ineq}
\end{align}
where the last inequality holds by using the last property of (\ref{eq:S-property-1})
and invoking the elementary inequality 
\[
\left\langle \bm{W},\mathcal{P}_{T^{\perp}}\left(\bm{H}\right)\right\rangle \leq\left\Vert \bm{W}\right\Vert \left\Vert \mathcal{P}_{T^{\perp}}\left(\bm{H}\right)\right\Vert _{*}.
\]
Given that $\bm{W}$ is assumed to obey $\|\bm{W}\|<1$, one has $g\left(\bm{Z}+\bm{H}\right)>g\left(\bm{Z}\right)$ unless $\mathcal{P}_{T^{\perp}}(\bm{H})=\bm{0}$. However, if $\mathcal{P}_{T^{\perp}}(\bm{H})=\bm{0}$
(and hence $\bm{H}\in T$), then the injectivity assumption together
with the fact that $\mathcal{P}_{\Omega}(\bm{H})=\bm{0}$ forces $\bm{H}=\bm{0}$.
Consequently, any minimizer $\bm{Z}+\bm{H}$ with $\bm{H}\neq\bm{0}$
must satisfy $g\left(\bm{Z}+\bm{H}\right)>g\left(\bm{Z}\right)$,
which results in contradiction. This concludes the proof. \end{proof}

\begin{proof}[Proof of Claim \ref{claim:P-Omega-H=00003D00003D00003D00003D0}]Consider
any minimizers $\bm{Z}_{1}\neq\bm{Z}_{2}$, and suppose instead that
$\mathcal{P}_{\Omega}\left(\bm{Z}_{1}-\bm{Z}_{2}\right)\neq\bm{0}$.
For any $0<\alpha<1$, define 
\[
\bm{Z}_{\alpha}\triangleq\alpha\bm{Z}_{1}+\left(1-\alpha\right)\bm{Z}_{2}.
\]
Since $\|\cdot\|_{*}$ is convex, we have 
\begin{align}
g\left(\bm{Z}_{\alpha}\right) & =\tfrac{1}{2}\left\Vert \mathcal{P}_{\Omega}\left(\alpha\bm{Z}_{1}+\left(1-\alpha\right)\bm{Z}_{2}-\bm{M}\right)\right\Vert _{\mathrm{F}}^{2}+\lambda\left\Vert \alpha\bm{Z}_{1}+\left(1-\alpha\right)\bm{Z}_{2}\right\Vert _{*}\nonumber \\
 & \leq\tfrac{1}{2}\left\Vert \mathcal{P}_{\Omega}\left(\alpha\bm{Z}_{1}+\left(1-\alpha\right)\bm{Z}_{2}-\bm{M}\right)\right\Vert _{\mathrm{F}}^{2}+\alpha\lambda\left\Vert \bm{Z}_{1}\right\Vert _{*}+\left(1-\alpha\right)\lambda\left\Vert \bm{Z}_{2}\right\Vert _{*}.\label{eq:g-alpha}
\end{align}
Furthermore, by the strong convexity of $\|\cdot\|_{\mathrm{{F}}}^{2}$
we have 
\begin{align*}
g\left(\bm{Z}_{\alpha}\right) & <\tfrac{1}{2}\big(\alpha\left\Vert \mathcal{P}_{\Omega}\left(\bm{Z}_{1}-\bm{M}\right)\right\Vert _{\mathrm{F}}^{2}+\left(1-\alpha\right)\left\Vert \mathcal{P}_{\Omega}\left(\bm{Z}_{2}-\bm{M}\right)\right\Vert _{\mathrm{F}}^{2}\big)+\alpha\lambda\left\Vert \bm{Z}_{1}\right\Vert _{*}+\left(1-\alpha\right)\lambda\left\Vert \bm{Z}_{2}\right\Vert _{*}\\
 & =\alpha g\left(\bm{Z}_{1}\right)+\left(1-\alpha\right)g\left(\bm{Z}_{2}\right)=g(\bm{Z}_{1}).
\end{align*}
This contradicts the fact that $\bm{Z}_{1}$ is a minimizer of (\ref{eq:convex-LS}),
thus completing the proof.\end{proof}

\begin{comment}
In addition, observe that

\begin{align}
 & \alpha\left\Vert \mathcal{P}_{\Omega}\left(\bm{Z}_{1}-\bm{M}\right)\right\Vert _{\mathrm{F}}^{2}+\left(1-\alpha\right)\left\Vert \mathcal{P}_{\Omega}\left(\bm{Z}_{2}-\bm{M}\right)\right\Vert _{\mathrm{F}}^{2}-\left\Vert \mathcal{P}_{\Omega}\left(\alpha\bm{Z}_{1}+\left(1-\alpha\right)\bm{Z}_{2}-\bm{M}\right)\right\Vert _{\mathrm{F}}^{2}\nonumber \\
 & \quad=\left(\alpha-\alpha^{2}\right)\left\Vert \mathcal{P}_{\Omega}\left(\bm{Z}_{1}-\bm{M}\right)\right\Vert _{\mathrm{F}}^{2}+\left[\left(1-\alpha\right)-\left(1-\alpha\right)^{2}\right]\left\Vert \mathcal{P}_{\Omega}\left(\bm{Z}_{2}-\bm{M}\right)\right\Vert _{\mathrm{F}}^{2}\nonumber \\
 & \quad\quad-2\alpha\left(1-\alpha\right)\left\langle \mathcal{P}_{\Omega}\left(\bm{Z}_{1}-\bm{M}\right),\mathcal{P}_{\Omega}\left(\bm{Z}_{2}-\bm{M}\right)\right\rangle \nonumber \\
 & \quad=\alpha\left(1-\alpha\right)\left[\left\Vert \mathcal{P}_{\Omega}\left(\bm{Z}_{1}-\bm{M}\right)\right\Vert _{\mathrm{F}}^{2}+\left\Vert \mathcal{P}_{\Omega}\left(\bm{Z}_{2}-\bm{M}\right)\right\Vert _{\mathrm{F}}^{2}-2\left\langle \mathcal{P}_{\Omega}\left(\bm{Z}_{1}-\bm{M}\right),\mathcal{P}_{\Omega}\left(\bm{Z}_{2}-\bm{M}\right)\right\rangle \right]\nonumber \\
 & \quad=\alpha\left(1-\alpha\right)\left\Vert \mathcal{P}_{\Omega}\left(\bm{Z}_{1}-\bm{Z}_{2}\right)\right\Vert _{\mathrm{F}}^{2}>0\label{eq:strict-convexity}
\end{align}
\end{comment}

\section{Connections between convex and nonconvex solutions}

\subsection{Proof of Lemma \ref{lemma:link-cvx-and-ncvx} \label{sec:Proof-of-Lemma-sufficient-condition-for-minimizer}}

First of all, since $(\bm{X},\bm{Y})$ is a stationary point of (\ref{eq:nonconvex_mc_noisy-lambda}),
we have the first-order optimality conditions\begin{subequations}\label{subeq:1st-order-stationary-condition-nonconvex}
\begin{align}
\mathcal{P}_{\Omega}\left(\bm{M}-\bm{X}\bm{Y}^{\top}\right)\bm{Y} & =\lambda\bm{X};\label{eq:1st-order-stationary-condition-nonconvex-X}\\
\left[\mathcal{P}_{\Omega}\left(\bm{M}-\bm{X}\bm{Y}^{\top}\right)\right]^{\top}\bm{X} & =\lambda\bm{Y}.\label{eq:1st-order-stationary-condition-nonconvex-Y}
\end{align}
\end{subequations} As an immediate consequence, one has
\begin{equation}
\bm{X}^{\top}\bm{X}=\lambda^{-1}\bm{X}^{\top}\mathcal{P}_{\Omega}\left(\bm{M}-\bm{X}\bm{Y}^{\top}\right)\bm{Y}=\bm{Y}^{\top}\bm{Y}.\label{eq:exact-balance}
\end{equation}
In words, any stationary point $(\bm{X},\bm{Y})$ has ``balanced''
scale.

Let $\bm{U}\bm{\Sigma}\bm{V}^{\top}$ be the singular value decomposition
of $\bm{X}\bm{Y}^{\top}$ with $\bm{U},\bm{V}\in\mathbb{R}^{n\times r}$
orthonormal and $\bm{\Sigma}\in\mathbb{R}^{r\times r}$ diagonal.
In view of the balanced scale of $(\bm{X},\bm{Y})$ (namely, (\ref{eq:exact-balance}))
and Lemma \ref{lemma:balance_determine}, we can write
\begin{equation}
\bm{X}=\bm{U}\bm{\Sigma}^{1/2}\bm{R}\qquad\text{and}\qquad\bm{Y}=\bm{V}\bm{\Sigma}^{1/2}\bm{R}\label{eq:X-Y-representation}
\end{equation}
for some orthonormal matrix $\bm{R}\in\mathbb{R}^{r\times r}$. Substitution
into (\ref{subeq:1st-order-stationary-condition-nonconvex}) results
in \begin{subequations}\label{subeq:rearranged-1st-order-ncvx}
\begin{align}
\mathcal{P}_{\Omega}\left(\bm{M}-\bm{X}\bm{Y}^{\top}\right)\bm{V} & =\lambda\bm{U};\\
\left[\mathcal{P}_{\Omega}\left(\bm{M}-\bm{X}\bm{Y}^{\top}\right)\right]^{\top}\bm{U} & =\lambda\bm{V},
\end{align}
\end{subequations}implying that the columns of $\bm{U}$ (resp.~$\bm{V}$)
are the left (resp.~right) singular vectors of the matrix $\mathcal{P}_{\Omega}(\bm{M}-\bm{X}\bm{Y}^{\top})$.
We can therefore write
\begin{equation}
\frac{1}{\lambda}\mathcal{P}_{\Omega}\left(\bm{M}-\bm{X}\bm{Y}^{\top}\right)=\bm{U}\bm{V}^{\top}+\bm{W},\label{eq:rank-r-decomposition}
\end{equation}
where $\bm{W}\in T^{\perp}$; recall that $T$ is the tangent space
of $\bm{X}\bm{Y}^{\top}$ and also $\bm{U}\bm{V}^{\top}$. In view of Lemma \ref{lemma:exact-primal-dual},
it suffices to show that $\|\bm{W}\|<1$, which is the content of
the rest of the proof.

One can rewrite $\mathcal{P}_{\Omega}(\bm{M}-\bm{X}\bm{Y}^{\top})$
as
\[
\mathcal{P}_{\Omega}\left(\bm{M}-\bm{X}\bm{Y}^{\top}\right)=p\left(\bm{M}^{\star}-\bm{X}\bm{Y}^{\top}\right)+\mathcal{P}_{\Omega}^{\mathsf{debias}}\left(\bm{M}^{\star}-\bm{X}\bm{Y}^{\top}\right)+\mathcal{P}_{\Omega}\left(\bm{E}\right).
\]
Substitute this identity into (\ref{subeq:rearranged-1st-order-ncvx})
and rearrange terms to obtain
\begin{align*}
\left[p\bm{M}^{\star}+\mathcal{P}_{\Omega}^{\mathsf{debias}}\left(\bm{M}^{\star}-\bm{X}\bm{Y}^{\top}\right)+\mathcal{P}_{\Omega}\left(\bm{E}\right)\right]\bm{V} & =\bm{U}\left(p\bm{\Sigma}+\lambda\bm{I}_{r}\right);\\
\left[p\bm{M}^{\star}+\mathcal{P}_{\Omega}^{\mathsf{debias}}\left(\bm{M}^{\star}-\bm{X}\bm{Y}^{\top}\right)+\mathcal{P}_{\Omega}\left(\bm{E}\right)\right]^{\top}\bm{U} & =\bm{V}\left(p\bm{\Sigma}+\lambda\bm{I}_{r}\right).
\end{align*}
These tell us that the columns of $\bm{U}$ (resp.~$\bm{V}$) are
the left (resp.~right) singular vectors of the matrix
\[
p\bm{M}^{\star}+\mathcal{P}_{\Omega}^{\mathsf{debias}}\left(\bm{M}^{\star}-\bm{X}\bm{Y}^{\top}\right)+\mathcal{P}_{\Omega}\left(\bm{E}\right),
\]
which is equivalent to saying that\footnote{Here, the pre-factor $\lambda$ is chosen to simplify the analysis
later on.}
\begin{equation}
p\bm{M}^{\star}+\mathcal{P}_{\Omega}^{\mathsf{debias}}\left(\bm{M}^{\star}-\bm{X}\bm{Y}^{\top}\right)+\mathcal{P}_{\Omega}\left(\bm{E}\right)=\bm{U}\left(p\bm{\Sigma}+\lambda\bm{I}_{r}\right)\bm{V}^{\top}+\lambda\bm{W}_{2},\label{eq:rank-r-decomposition-new}
\end{equation}
for some $\bm{W}_{2}\in T^{\perp}$. One can then derive from (\ref{eq:rank-r-decomposition}) that
\begin{align*}
\bm{W} & \overset{(\text{i})}{=}\tfrac{1}{\lambda} \mathcal{P}_{T^{\perp}}\left[\mathcal{P}_{\Omega}\left(\bm{M}-\bm{X}\bm{Y}^{\top}\right)\right] \\
 & =\tfrac{1}{\lambda} \mathcal{P}_{T^{\perp}}\left[p\bm{M}^{\star}-p\bm{X}\bm{Y}^{\top}+\mathcal{P}_{\Omega}^{\mathsf{debias}}\left(\bm{M}^{\star}-\bm{X}\bm{Y}^{\top}\right)+\mathcal{P}_{\Omega}\left(\bm{E}\right)\right] \\
 & \overset{(\text{ii})}{=}\tfrac{1}{\lambda} \mathcal{P}_{T^{\perp}}\left[p\bm{M}^{\star}+\mathcal{P}_{\Omega}^{\mathsf{debias}}\left(\bm{M}^{\star}-\bm{X}\bm{Y}^{\top}\right)+\mathcal{P}_{\Omega}\left(\bm{E}\right)\right] \\
 & \overset{(\text{iii})}{=}\tfrac{1}{\lambda} \mathcal{P}_{T^{\perp}}\left[\bm{U}\left(p\bm{\Sigma}+\lambda\bm{I}_{r}\right)\bm{V}^{\top}+\lambda\bm{W}_{2}\right] \\
 & \overset{(\text{iv})}{=} \bm{W}_{2} ,
\end{align*}
where (i), (ii) and (iv) arise from the facts that $\bm{U}\bm{V}^{\top}\in T$, $\bm{X}\bm{Y}^{\top}\in T$
and $\bm{U}(p\bm{\Sigma}+\lambda\bm{I}_{r})\bm{V}^{\top}\in T$, respectively,
and (iii) relies on the identity (\ref{eq:rank-r-decomposition-new}).

It then suffices to control $\|\bm{W}_{2}\|$. To this end, apply
Weyl's inequality to (\ref{eq:rank-r-decomposition-new}) to obtain
that: for $r+1\leq i\leq n$, the $i$th largest singular value of
$\bm{U}(p\bm{\Sigma}+\lambda\bm{I}_{r})\bm{V}^{\top}+\lambda\bm{W}_{2}$
obeys
\begin{align*}
\sigma_{i}\left(\bm{U}\left(p\bm{\Sigma}+\lambda\bm{I}_{r}\right)\bm{V}^{\top}+\lambda\bm{W}_{2}\right) & \leq p\sigma_{i}\left(\bm{M}^{\star}\right)+\left\Vert \mathcal{P}_{\Omega}^{\mathsf{debias}}\left(\bm{M}^{\star}-\bm{X}\bm{Y}^{\top}\right)+\mathcal{P}_{\Omega}\left(\bm{E}\right)\right\Vert \\
 & \leq\left\Vert \mathcal{P}_{\Omega}^{\mathsf{debias}}\left(\bm{M}^{\star}-\bm{X}\bm{Y}^{\top}\right)\right\Vert +\left\Vert \mathcal{P}_{\Omega}\left(\bm{E}\right)\right\Vert \\
 & <\lambda,
\end{align*}
where the second inequality comes from the fact that $\bm{M}^{\star}$
has rank $r$ (so that $\sigma_{i}(\bm{M}^{\star})=0$ for $r+1\leq i\leq n$)
as well as the triangle inequality, and the last inequality follows
from the assumptions of the lemma. Furthermore, it is seen that $\bm{U}(p\bm{\Sigma}+\lambda\bm{I}_{r})\bm{V}^{\top}$
has rank $r$ and all of its singular values are at least $\lambda$.
These facts taken collectively demonstrate that
\[
\|\bm{W}\|=\left\Vert \bm{W}_{2}\right\Vert =\tfrac{1}{\lambda}\max_{r<i\leq n}\sigma_{i}\left(\bm{U}\left(p\bm{\Sigma}+\lambda\bm{I}_{r}\right)\bm{V}^{\top}+\lambda\bm{W}_{2}\right)<1.
\]
This together with Lemma \ref{lemma:exact-primal-dual} completes
the proof.

\subsection{Proof of Lemma \ref{lemma:approx-link-cvx-ncvx-simple} \label{sec:Proof-of-Lemma-approx-link-cvx-ncvx-simple}}

We begin by collecting a few simple properties resulting from our
assumptions. By definition, the gradient of $f(\cdot,\cdot)$ in (\ref{eq:nonconvex_mc_noisy})
is given by
\[
\nabla f\left(\bm{X},\bm{Y}\right)=\frac{1}{p}\left[\begin{array}{c}
\mathcal{P}_{\Omega}\left(\bm{X}\bm{Y}^{\top}-\bm{M}\right)\bm{Y}+\lambda\bm{X}\\
\left[\mathcal{P}_{\Omega}\left(\bm{X}\bm{Y}^{\top}-\bm{M}\right)\right]^{\top}\bm{X}+\lambda\bm{Y}
\end{array}\right],
\]
which together with the small-gradient assumption $\|\nabla f(\bm{X},\bm{Y})\|_{\mathrm{F}}\leq c\lambda\sqrt{c_{\text{inj}}\,p\sigma_{\min}/\kappa^{2}}/p$
implies that\begin{subequations}\label{subeq:assumption-approx-stationary}
\begin{align}
\left\Vert \mathcal{P}_{\Omega}\left(\bm{X}\bm{Y}^{\top}-\bm{M}\right)\bm{Y}+\lambda\bm{X}\right\Vert _{\mathrm{F}} & \leq p\left\Vert \nabla f\left(\bm{X},\bm{Y}\right)\right\Vert _{\mathrm{F}}\leq c\lambda\sqrt{c_{\text{inj}}\,p\sigma_{\min}/\kappa^{2}};\label{eq:assumption-approx-stationary-X}\\
\big\|\big(\mathcal{P}_{\Omega}(\bm{X}\bm{Y}^{\top}-\bm{M})\big)^{\top}\bm{X}+\lambda\bm{Y}\big\|_{\mathrm{F}} & \leq p\left\Vert \nabla f\left(\bm{X},\bm{Y}\right)\right\Vert _{\mathrm{F}}\leq c\lambda\sqrt{c_{\text{inj}}\,p\sigma_{\min}/\kappa^{2}}.\label{eq:assumption-approx-stationary-Y}
\end{align}
\end{subequations}Throughout the proof, we let the SVD of $\bm{X}\bm{Y}^{\top}$
be $\bm{X}\bm{Y}^{\top}=\bm{U}\bm{\Sigma}\bm{V}^{\top}$, and denote
by $T$ the tangent space of $\bm{X}\bm{Y}^{\top}$ and by $T^{\perp}$
its orthogonal complement. Additionally, our assumption regarding
the singular values of $\bm{X}$ and $\bm{Y}$ implies that
\begin{align}
\sigma_{\min}/2\leq\sigma_{\min}\left(\bm{\Sigma}\right) & \leq\sigma_{\max}\left(\bm{\Sigma}\right)\leq2\sigma_{\max}.\label{eq:claim-1}
\end{align}
This can be easily seen from the following two inequalities
\begin{align*}
\sigma_{\max}\left(\bm{\Sigma}\right) & =\left\Vert \bm{X}\bm{Y}^{\top}\right\Vert \leq\left\Vert \bm{X}\right\Vert \left\Vert \bm{Y}\right\Vert \leq2\sigma_{\max};\\
\sigma_{\min}\left(\bm{\Sigma}\right) & =\sigma_{\min}\left(\bm{X}\bm{Y}^{\top}\right)\geq\sigma_{\min}\left(\bm{X}\right)\sigma_{\min}\left(\bm{Y}\right)\geq\sigma_{\min}/2.
\end{align*}
Before proceeding, we record a claim that will prove useful in the
subsequent analysis.

\begin{claim}\label{claim:approx-stat-cvx} Under the notations
and assumptions of Lemma \ref{lemma:approx-link-cvx-ncvx-simple},
one has
\begin{equation}
\mathcal{P}_{\Omega}\left(\bm{X}\bm{Y}^{\top}-\bm{M}\right)=-\lambda\bm{U}\bm{V}^{\top}+\bm{R},\label{eq:lem:approx-dual-cond1}
\end{equation}
where $\bR$ is some residual matrix satisfying
\begin{equation}
\Fnorm{\cP_{T}(\bR)}\leq72\kappa\frac{p}{\sqrt{\sigma_{\min}}}\left\Vert \nabla f\left(\bm{X},\bm{Y}\right)\right\Vert _{\mathrm{F}}\quad\mathrm{and}\quad\norm{\cP_{T^{\perp}}(\bR)}<\lambda/2.\label{eq:lem:approx-dual-cond2}
\end{equation}
\end{claim}

With Claim \ref{claim:approx-stat-cvx} in place, we are ready to
prove Lemma \ref{lemma:approx-link-cvx-ncvx-simple}. Let $\bm{Z}_{\mathsf{cvx}}$
be any minimizer of (\ref{eq:convex-LS}) and denote $\bDelta\triangleq\bm{Z}_{\mathsf{cvx}}-\bX\bY^{\top}$.
The proof can be divided into the following steps.
\begin{itemize}
\item First, show that the difference $\bm{\Delta}$ primarily lies in the
tangent space of $\bm{X}\bm{Y}^{\top}$; see (\ref{eq:lem_approx_nuclear_F}).
\item Next, utilize this property to connect $\Fnorm{\cP_{\Omega}(\bDelta)}^{2}$
with the size of the gradient $\nabla f(\bm{X},\bm{Y})$; see (\ref{eq:lem_approx_main1}).
\item In the end, obtain a lower bound on $\Fnorm{\cP_{\Omega}(\bDelta)}^{2}$
in terms of $\|\bm{\Delta}\|_{\mathrm{F}}$ using the injectivity
property; see (\ref{eq:lem_approx_main2}).
\end{itemize}
The desired upper bound on $\|\bm{\Delta}\|_{\mathrm{F}}$ advertised
in the lemma then follows by combining these results. In what follows,
we shall carry out these steps one by one.
\begin{enumerate}
\item The optimality of $\bm{Z}_{\mathsf{cvx}}=\bX\bY^{\top}+\bDelta$ reveals
that
\[
\tfrac{1}{2}\Fnorm{\cP_{\Omega}\left(\bX\bY^{\top}+\bDelta-\bM\right)}^{2}+\lambda\nuclearnorm{\bX\bY^{\top}+\bDelta}\leq\tfrac{1}{2}\Fnorm{\cP_{\Omega}\left(\bX\bY^{\top}-\bM\right)}^{2}+\lambda\big\|\bX\bY^{\top}\big\|_{*}.
\]
A little algebra allows us to rearrange terms as follows
\begin{equation}
\tfrac{1}{2}\Fnorm{\cP_{\Omega}(\bDelta)}^{2}\leq-\inner{\cP_{\Omega}\left(\bX\bY^{\top}-\bM\right)}{\bDelta}+\lambda\big\|\bX\bY^{\top}\big\|_{*}-\lambda\big\|\bX\bY^{\top}+\bm{\Delta}\big\|_{*}.\label{eq:master-inequality}
\end{equation}
In addition, it follows from the convexity of $\|\cdot\|_{*}$ that
\begin{equation}
\left\Vert \bm{X}\bm{Y}^{\top}+\bm{\Delta}\right\Vert _{*}\geq\left\Vert \bm{X}\bm{Y}^{\top}\right\Vert _{*}+\left\langle \bm{U}\bm{V}^{\top}+\bm{W},\bm{\Delta}\right\rangle \label{eq:nuclear-norm-convex}
\end{equation}
for any $\bm{W}\in T^{\perp}$ obeying $\|\bm{W}\|\leq1$, where $\bm{U}\bm{V}^{\top}+\bm{W}$
serves as a subgradient of $\|\cdot\|_{*}$ at $\bm{X}\bm{Y}^{\top}$.
In what follows, we shall pick $\bm{W}$ such that $\inner{\bW}{\bDelta}=\|\cP_{T^{\perp}}(\bDelta)\|_{*}$.
Combining this with (\ref{eq:master-inequality}) and (\ref{eq:nuclear-norm-convex}),
we reach
\begin{equation}
\begin{aligned}\tfrac{1}{2}\left\Vert \mathcal{P}_{\Omega}\left(\bm{\Delta}\right)\right\Vert _{\mathrm{F}}^{2} & \leq-\inner{\cP_{\Omega}\left(\bX\bY^{\top}-\bm{M}\right)}{\bDelta}-\lambda\inner{\bU\bV^{\top}}{\bDelta}-\lambda\inner{\bW}{\bDelta}\\
 & =-\inner{\cP_{\Omega}\left(\bX\bY^{\top}-\bm{M}\right)}{\bDelta}-\lambda\inner{\bU\bV^{\top}}{\bDelta}-\lambda\nuclearnorm{\cP_{T^{\perp}}\left(\bDelta\right)}.
\end{aligned}
\label{eq:lem_approx_dual_lower_bound_1}
\end{equation}
This together with the decomposition \eqref{eq:lem:approx-dual-cond1}
leads to
\begin{align}
0\leq\tfrac{1}{2}\left\Vert \mathcal{P}_{\Omega}\left(\bm{\Delta}\right)\right\Vert _{\mathrm{F}}^{2} & \leq-\inner{\bR}{\bDelta}-\lambda\nuclearnorm{\cP_{T^{\perp}}\left(\bDelta\right)}\nonumber \\
 & =-\inner{\cP_{T}(\bR)}{\bDelta}-\inner{\cP_{T^{\perp}}(\bR)}{\bDelta}-\lambda\nuclearnorm{\cP_{T^{\perp}}\left(\bDelta\right)},\label{eq:master-inequality-2}
\end{align}
and therefore
\begin{align}
\inner{\cP_{T}(\bR)}{\bDelta}+\inner{\cP_{T^{\perp}}(\bR)}{\bDelta}+\lambda\nuclearnorm{\cP_{T^{\perp}}\left(\bDelta\right)} & \leq0.\label{eq:master-inequality-2-1}
\end{align}
In addition, elementary inequalities give
\begin{align*}
 & -\Fnorm{\cP_{T}(\bR)}\Fnorm{\cP_{T}(\bDelta)}-\norm{\cP_{T^{\perp}}(\bR)}\nuclearnorm{\cP_{T^{\perp}}(\bDelta)}+\lambda\nuclearnorm{\cP_{T^{\perp}}(\bDelta)}\\
 & \qquad\qquad\leq\inner{\cP_{T}(\bR)}{\bDelta}+\inner{\cP_{T^{\perp}}(\bR)}{\bDelta}+\lambda\nuclearnorm{\cP_{T^{\perp}}\left(\bDelta\right)}\leq0.
\end{align*}
From the condition \eqref{eq:lem:approx-dual-cond2} we have $\norm{\cP_{T^{\perp}}(\bR)}\leq\lambda/2$,
and hence the above inequality gives
\begin{equation}
\Fnorm{\cP_{T}(\bR)}\Fnorm{\cP_{T}(\bDelta)}\geq-\norm{\cP_{T^{\perp}}(\bR)}\nuclearnorm{\cP_{T^{\perp}}(\bDelta)}+\lambda\nuclearnorm{\cP_{T^{\perp}}(\bDelta)}\geq\tfrac{\lambda}{2}\nuclearnorm{\cP_{T\perp}(\bDelta)},\label{eq:useful-ineq-3}
\end{equation}
which together with the condition \eqref{eq:lem:approx-dual-cond2}
on $\Fnorm{\cP_{T}(\bR)}$ and the small gradient assumption \eqref{eq:small-gradient-f}
yields
\begin{equation}
\nuclearnorm{\cP_{T^{\perp}}(\bDelta)}\leq144\kappa\frac{p}{\lambda\sqrt{\sigma_{\min}}}\left\Vert \nabla f\left(\bm{X},\bm{Y}\right)\right\Vert _{\mathrm{F}}\Fnorm{\cP_{T}(\bDelta)}\leq144c\sqrt{c_{\text{inj}}p}\Fnorm{\cP_{T}(\bDelta)}.\label{eq:lem_approx_nuclear_F}
\end{equation}
This essentially means that $\bDelta$ lies primarily in the tangent
space of $\bX\bY^{\top}$ for $c$ sufficiently small. As an immediate
consequence,
\begin{equation}
\Fnorm{\cP_{T^{\perp}}(\bDelta)}\leq\nuclearnorm{\cP_{T^{\perp}}(\bDelta)}\leq144c\sqrt{c_{\text{inj}}p}\Fnorm{\cP_{T}(\bDelta)}\leq\Fnorm{\cP_{T}(\bDelta)},\label{eq:lem_approx_dual_nuclear-norm-domination}
\end{equation}
as long as $c$ is sufficiently small. Note that we also use the elementary
fact that $c_{\text{inj}}\leq1/p$ (otherwise we will have the contradictory
inequality $p^{-1}\|\mathcal{P}_{\Omega}(\bm{H})\|_{\mathrm{F}}^{2}\geq c_{\text{inj}}\|\bm{H}\|_{\mathrm{F}}^{2}>p^{-1}\|\bm{H}\|_{\mathrm{F}}^{2}$).
\item Continue the upper bound in (\ref{eq:master-inequality-2}) to obtain
\begin{align*}
\tfrac{1}{2}\Fnorm{\cP_{\Omega}(\bDelta)}^{2} & \leq-\inner{\cP_{T}(\bR)}{\bDelta}-\inner{\cP_{T^{\perp}}(\bR)}{\bDelta}-\lambda\nuclearnorm{\cP_{T^{\perp}}\left(\bDelta\right)}\\
 & \leq\Fnorm{\cP_{T}(\bR)}\Fnorm{\cP_{T}(\bDelta)}-\tfrac{\lambda}{2}\nuclearnorm{\cP_{T^{\perp}}(\bDelta)}.
\end{align*}
Here, the last line uses the fact that $-\inner{\cP_{T^{\perp}}(\bR)}{\bDelta}\leq\|\cP_{T^{\perp}}(\bR)\|\cdot\|\mathcal{P}_{T^{\perp}}(\bDelta)\|_{*}\leq\frac{\lambda}{2}\nuclearnorm{\cP_{T^{\perp}}(\bDelta)}$,
which follows from (\ref{eq:lem:approx-dual-cond2}). Therefore, using
the condition (\ref{eq:lem:approx-dual-cond2}) we reach
\begin{align}
\frac{1}{2}\left\Vert \mathcal{P}_{\Omega}\left(\bm{\Delta}\right)\right\Vert _{\mathrm{F}}^{2} & \leq\Fnorm{\cP_{T}(\bR)}\Fnorm{\cP_{T}(\bDelta)}\leq72\kappa\frac{p}{\sqrt{\sigma_{\min}}}\left\Vert \nabla f\left(\bm{X},\bm{Y}\right)\right\Vert _{\mathrm{F}}\left\Vert \bm{\Delta}\right\Vert _{\mathrm{F}}.\label{eq:lem_approx_main1}
\end{align}
\item We are left with lower bounding $\Fnorm{\cP_{\Omega}(\bDelta)}^{2}$.
Using the decomposition $\bDelta=\cP_{T}(\bDelta)+\cP_{T^{\perp}}(\bDelta)$,
we obtain
\begin{align*}
\tfrac{1}{\sqrt{p}}\Fnorm{\cP_{\Omega}(\bDelta)} & =\tfrac{1}{\sqrt{p}}\Fnorm{\cP_{\Omega}\cP_{T}(\bDelta)+\cP_{\Omega}\cP_{T^{\perp}}(\bDelta)}\geq\tfrac{1}{\sqrt{p}}\Fnorm{\cP_{\Omega}\cP_{T}(\bDelta)}-\tfrac{1}{\sqrt{p}}\Fnorm{\cP_{\Omega}\cP_{T^{\perp}}(\bDelta)}\\
 & \geq\sqrt{c_{\mathrm{inj}}}\Fnorm{\cP_{T}(\bDelta)}-\tfrac{1}{\sqrt{p}}\Fnorm{\cP_{T^{\perp}}(\bDelta)},
\end{align*}
where the last inequality follows from the injectivity assumption
(\ref{eq:assumption-injectivity}). In addition, \eqref{eq:lem_approx_nuclear_F}
implies
\[
\tfrac{1}{\sqrt{p}}\Fnorm{\cP_{T^{\perp}}(\bDelta)}\leq\tfrac{1}{\sqrt{p}}\nuclearnorm{\cP_{T^{\perp}}(\bDelta)}\leq\tfrac{1}{\sqrt{p}}144c\sqrt{c_{\text{inj}}p}\Fnorm{\cP_{T}(\bDelta)}\leq\tfrac{\sqrt{c_{\mathrm{inj}}}}{2}\Fnorm{\cP_{T}(\bDelta)}
\]
as long as $c$ is sufficiently small. As a result,
\[
\tfrac{1}{\sqrt{p}}\Fnorm{\cP_{\Omega}(\bDelta)}\geq\tfrac{\sqrt{c_{\mathrm{inj}}}}{2}\Fnorm{\cP_{T}(\bDelta)}.
\]
In addition, by \eqref{eq:lem_approx_dual_nuclear-norm-domination}
we have
\[
\Fnorm{\bDelta}\leq\Fnorm{\cP_{T}(\bDelta)}+\Fnorm{\cP_{T^{\perp}}(\bDelta)}\leq2\Fnorm{\cP_{T}(\bDelta)},
\]
and therefore
\begin{equation}
\tfrac{1}{\sqrt{p}}\Fnorm{\cP_{\Omega}(\bDelta)}\geq\tfrac{\sqrt{c_{\mathrm{inj}}}}{2}\Fnorm{\cP_{T}(\bDelta)}\geq\tfrac{\sqrt{c_{\mathrm{inj}}}}{4}\Fnorm{\bDelta}.\label{eq:lem_approx_main2}
\end{equation}
\end{enumerate}
Taking \eqref{eq:lem_approx_main1} and \eqref{eq:lem_approx_main2}
collectively yields
\[
\tfrac{c_{\mathrm{inj}}}{32}\Fnorm{\bDelta}^{2}\leq\tfrac{1}{2p}\Fnorm{\cP_{\Omega}(\bDelta)}^{2}\leq72\kappa\tfrac{1}{\sqrt{\sigma_{\min}}}\left\Vert \nabla f\left(\bm{X},\bm{Y}\right)\right\Vert _{\mathrm{F}}\left\Vert \bm{\Delta}\right\Vert _{\mathrm{F}},
\]
thus indicating that
\[
\left\Vert \bm{\Delta}\right\Vert _{\mathrm{F}}\lesssim\frac{\kappa}{c_{\mathrm{inj}}}\frac{1}{\sqrt{\sigma_{\min}}}\left\Vert \nabla f\left(\bm{X},\bm{Y}\right)\right\Vert _{\mathrm{F}}.
\]

\subsubsection{Proof of Claim \ref{claim:approx-stat-cvx}}

Before proceeding to the proof of Claim \ref{claim:approx-stat-cvx},
we state a useful fact; the proof is deferred to Appendix \ref{subsec:Proof-of-Claim-balancing}.

\begin{claim}\label{claim:balancing}Instate the notations and assumptions
in Lemma \ref{lemma:approx-link-cvx-ncvx-simple}. Let $\bm{U}\bm{\Sigma}\bm{V}^{\top}$
be the SVD of $\bm{X}\bm{Y}^{\top}$. There exists an invertible matrix
$\bm{Q}\in\mathbb{R}^{r\times r}$ such that $\bm{X}=\bm{U}\bm{\Sigma}^{1/2}\bm{Q}$,
$\bm{Y}=\bm{V}\bm{\Sigma}^{1/2}\bm{Q}^{-\top}$ and
\begin{equation}
\big\|\bm{\Sigma}_{\bm{Q}}-\bm{\Sigma}_{\bm{Q}}^{-1}\big\|_{\mathrm{F}}\leq8\sqrt{\kappa}\frac{p}{\lambda\sqrt{\sigma_{\min}}}\left\Vert \nabla f\left(\bm{X},\bm{Y}\right)\right\Vert _{\mathrm{F}}\leq8c\sqrt{c_{\mathrm{inj}}p/\kappa},\label{eq:claim-2}
\end{equation}
where $\bm{U}_{\bm{Q}}\bm{\Sigma}_{\bm{Q}}\bm{V}_{\bm{Q}}^{\top}$
is the SVD of $\bm{Q}$. \end{claim}

In light of the assumptions (\ref{subeq:assumption-approx-stationary}),
one has
\begin{equation}
\cP_{\Omega}\left(\bX\bY^{\top}-\bm{M}\right)\bY=-\lambda\bX+\bB_{1}\qquad\text{and}\qquad\big[\cP_{\Omega}\left(\bX\bY^{\top}-\bm{M}\right)\big]^{\top}\bX=-\lambda\bY+\bB_{2}\label{eq:cor_suff_3}
\end{equation}
for some $\bB_{1}\in\R^{n\times r}$ and $\bB_{2}\in\R^{n\times r}$,
where
$\max\left\{ \left\Vert \bm{B}_{1}\right\Vert _{\mathrm{F}},\left\Vert \bm{B}_{2}\right\Vert _{\mathrm{F}}\right\} \leq p\left\Vert \nabla f\left(\bm{X},\bm{Y}\right)\right\Vert _{\mathrm{F}}$.
Recall that
\begin{equation}
\cP_{\Omega}\left(\bX\bY^{\top}-\bm{M}\right)=-\lambda\bU\bV^{\top}+\bR.\label{eq:cor_suff_4}
\end{equation}
In the sequel, we shall prove the upper bounds on both $\|\mathcal{P}_{T}(\bm{R})\|_{\mathrm{F}}$
and $\|\mathcal{P}_{T^{\perp}}(\bm{R})\|$ separately.
\begin{enumerate}
\item From the definition of $\mathcal{P}_{T}(\cdot)$ (see (\ref{eq:defn-PT})),
we have
\begin{align}
\Fnorm{\cP_{T}(\bR)} & =\Fnorm{\bU\bU^{\top}\bR(\bm{I}-\bm{V}\bm{V}^{\top})+\bR\bV\bV^{\top}}\nonumber \\
 & \leq\Fnorm{\bU^{\top}\bR(\bm{I}-\bm{V}\bm{V}^{\top})}+\Fnorm{\bR\bV}\nonumber \\
 & \leq\Fnorm{\bU^{\top}\bR}+\Fnorm{\bR\bV}.\label{eq:PTR-F}
\end{align}
In addition, invoke Claim \ref{claim:balancing} to obtain
\begin{equation}
\bm{X}=\bm{U}\bm{\Sigma}^{1/2}\bm{Q}\qquad\text{and}\qquad\bm{Y}=\bm{V}\bm{\Sigma}^{1/2}\bm{Q}^{-\top}\label{eq:X-Y-Q-representation}
\end{equation}
for some invertible matrix $\bm{Q}\in\mathbb{R}^{r\times r}$, whose
SVD $\bm{U}_{\bm{Q}}\bm{\Sigma}_{\bm{Q}}\bm{V}_{\bm{Q}}^{\top}$ obeys
(\ref{eq:claim-2}). Combine \eqref{eq:cor_suff_3} and \eqref{eq:cor_suff_4}
to see
\[
-\lambda\bm{U}\bm{V}^{\top}\bm{Y}+\bm{R}\bm{Y}=-\lambda\bX+\bB_{1},
\]
which together with (\ref{eq:X-Y-Q-representation}) yields
\[
\bm{R}\bm{V}=\lambda\bm{U}\bm{\Sigma}^{1/2}\left(\bm{I}_{r}-\bm{Q}\bm{Q}^{\top}\right)\bm{\Sigma}^{-1/2}+\bm{B}_{1}\bm{Q}^{\top}\bm{\Sigma}^{-1/2}.
\]
Apply the triangle inequality to get
\begin{align}
\left\Vert \bm{R}\bm{V}\right\Vert _{\mathrm{F}} & \leq\|\lambda\bm{U}\bm{\Sigma}^{1/2}\left(\bm{I}_{r}-\bm{Q}\bm{Q}^{\top}\right)\bm{\Sigma}^{-1/2}\|_{\mathrm{F}}+\|\bm{B}_{1}\bm{Q}^{\top}\bm{\Sigma}^{-1/2}\|_{\mathrm{F}}\nonumber \\
 & \leq\lambda\big\|\bSigma^{1/2}\big\|\big\|\bSigma^{-1/2}\big\|\Fnorm{\bQ\bQ^{\top}-\bI_{r}}+\norm{\bQ}\big\|\bSigma^{-1/2}\big\|\Fnorm{\bB_{1}}.\label{eq:prev-eq}
\end{align}
In order to further upper bound (\ref{eq:prev-eq}), we first recognize
that (\ref{eq:claim-1}) implies
\[
\big\|\bm{\Sigma}^{1/2}\big\|\leq\sqrt{2\sigma_{\max}},\qquad\text{and}\qquad\big\|\bm{\Sigma}^{-1/2}\big\|=1/\sqrt{\sigma_{\min}\left(\bm{\Sigma}\right)}\leq\sqrt{2/\sigma_{\min}}.
\]
Second, Claim \ref{claim:balancing} yields
\[
\big\|\bm{\Sigma}_{\bm{Q}}-\bm{\Sigma}_{\bm{Q}}^{-1}\big\|_{\mathrm{F}}\leq8\sqrt{\kappa}\frac{p}{\lambda\sqrt{\sigma_{\min}}}\left\Vert \nabla f\left(\bm{X},\bm{Y}\right)\right\Vert _{\mathrm{F}}\leq8c\sqrt{c_{\text{inj}}p/\kappa}\ll1,
\]
with the proviso that $c$ is sufficiently small. Here we have used
the facts that $c_{\text{inj}}\leq1/p$ and that $\kappa\geq1$. This
in turn implies that $\left\Vert \bm{Q}\right\Vert =\big\|\bm{\Sigma}_{\bm{Q}}\big\|\leq2$. Putting the above bounds together yields
\begin{align*}
\left\Vert \bm{R}\bm{V}\right\Vert _{\mathrm{F}} & \leq\lambda\sqrt{2\sigma_{\max}}\sqrt{\frac{2}{\sigma_{\min}}}\left\Vert \bm{\Sigma}_{\bm{Q}}^{2}-\bm{I}_{r}\right\Vert _{\mathrm{F}}+2\sqrt{\frac{2}{\sigma_{\min}}}p\left\Vert \nabla f\left(\bm{X},\bm{Y}\right)\right\Vert _{\mathrm{F}}\\
 & \leq\lambda\sqrt{2\sigma_{\max}}\sqrt{\frac{2}{\sigma_{\min}}}\left\Vert \bm{\Sigma}_{\bm{Q}}\right\Vert \big\|\bm{\Sigma}_{\bm{Q}}-\bm{\Sigma}_{\bm{Q}}^{-1}\big\|_{\mathrm{F}}+2\sqrt{\frac{2}{\sigma_{\min}}}p\left\Vert \nabla f\left(\bm{X},\bm{Y}\right)\right\Vert _{\mathrm{F}}\\
 & \leq2\lambda\sqrt{2\sigma_{\max}}\sqrt{\frac{2}{\sigma_{\min}}}8\sqrt{\kappa}\frac{p}{\lambda\sqrt{\sigma_{\min}}}\left\Vert \nabla f\left(\bm{X},\bm{Y}\right)\right\Vert _{\mathrm{F}}+2\sqrt{\frac{2}{\sigma_{\min}}}p\left\Vert \nabla f\left(\bm{X},\bm{Y}\right)\right\Vert _{\mathrm{F}}\\
 & \leq36\kappa\frac{p}{\sqrt{\sigma_{\min}}}\left\Vert \nabla f\left(\bm{X},\bm{Y}\right)\right\Vert _{\mathrm{F}}.
\end{align*}
Similarly we can show that $\Fnorm{\bU^{\top}\bR}\leq36\kappa p\|\nabla f(\bm{X},\bm{Y})\|_{\mathrm{F}}/\sqrt{\sigma_{\min}}$.
These bounds together with (\ref{eq:PTR-F}) result in
\begin{equation}
\Fnorm{\cP_{T}(\bR)}\leq72\kappa\frac{p}{\sqrt{\sigma_{\min}}}\left\Vert \nabla f\left(\bm{X},\bm{Y}\right)\right\Vert _{\mathrm{F}}.
\end{equation}

\item We now move on to bounding $\norm{\cP_{T^{\perp}}(\bR)}$. In view of the definition of $\mathcal{P}_{\Omega}^{\mathsf{debias}}(\cdot)$ in \eqref{eq:defn-Pomega-tilde}, we can rearrange
\eqref{eq:cor_suff_3} to derive
\begin{align*}
\left[p\bM^{\star}+\cP_{\Omega}(\bE)-\mathcal{P}_{\Omega}^{\mathsf{debias}}\left(\bX\bY^{\top}-\bM^{\star}\right)\right]\bY & =p\bX\bY^{\top}\bY+\lambda\bX-\bB_{1},\\
\left[p\bM^{\star}+\cP_{\Omega}(\bE)-\mathcal{P}_{\Omega}^{\mathsf{debias}}\left(\bX\bY^{\top}-\bM^{\star}\right)\right]^{\top}\bX & =p\bY\bX^{\top}\bX+\lambda\bY-\bB_{2}.
\end{align*}
In view of the representation $\bX=\bU\bSigma^{1/2}\bQ$ and $\bY=\bV\bSigma^{1/2}\bQ^{-\top}$,
the above identities are equivalent to
\begin{align*}
\left[p\bM^{\star}+\cP_{\Omega}(\bE)-\mathcal{P}_{\Omega}^{\mathsf{debias}}\left(\bX\bY^{\top}-\bM^{\star}\right)\right]\bV & =p\bU\bSigma+\lambda\bU\bSigma^{1/2}\bQ\bQ^{\top}\bSigma^{-1/2}-\bB_{1}\bQ^{\top}\bSigma^{-1/2},\\
\left[p\bM^{\star}+\cP_{\Omega}(\bE)-\mathcal{P}_{\Omega}^{\mathsf{debias}}\left(\bX\bY^{\top}-\bM^{\star}\right)\right]^{\top}\bU & =p\bV\bSigma+\lambda\bV\bSigma^{1/2}\bQ^{-\top}\bQ^{-1}\bSigma^{-1/2}-\bB_{2}\bQ^{-1}\bSigma^{-1/2}.
\end{align*}
Letting
\begin{equation}
p\bM^{\star}+\cP_{\Omega}(\bE)-\mathcal{P}_{\Omega}^{\mathsf{debias}}\left(\bX\bY^{\top}-\bM^{\star}\right)=p\bU\bSigma\bV^{\top}+\lambda\bU\bSigma^{1/2}\bQ\bQ^{\top}\bSigma^{-1/2}\bV^{\top}+\tilde{\bR}\label{eq:cor_suff_6}
\end{equation}
for some residual matrix $\tilde{\bR}\in\R^{n\times n}$, we have
\begin{align}
 \mathcal{P}_{T^{\perp}}\left(\bm{R}\right)  & \overset{(\text{i})}{=} \mathcal{P}_{T^{\perp}}\left[\mathcal{P}_{\Omega}\left(\bm{X}\bm{Y}^{\top}-\bm{M}^{\star}-\bm{E}\right)\right] \nonumber \\
 & \overset{(\text{ii})}{=} \mathcal{P}_{T^{\perp}}\left[p\left(\bm{X}\bm{Y}^{\top}-\bm{M}^{\star}\right)+\mathcal{P}_{\Omega}^{\mathsf{debias}}\left(\bm{X}\bm{Y}^{\top}-\bm{M}^{\star}\right)-\mathcal{P}_{\Omega}\left(\bm{E}\right)\right] \nonumber \\
 & \overset{(\text{iii})}{=} \mathcal{P}_{T^{\perp}}\left[p\bM^{\star}+\cP_{\Omega}(\bE)-\mathcal{P}_{\Omega}^{\mathsf{debias}}\left(\bX\bY^{\top}-\bM^{\star}\right)\right] \nonumber \\
 & \overset{(\text{iv})}{=}\mathcal{P}_{T^{\perp}}\big(\tilde{\bR}\big),\label{eq:equiv-R-R-tilde}
\end{align}
where (i) follows from the definition of $\bm{R}$ and the fact that
$\bm{U}\bm{V}^{\top}\in T$, (ii) uses the definition of $\mathcal{P}_{\Omega}^{\mathsf{debias}}(\cdot)$,
(iii) relies on the fact that $\bm{X}\bm{Y}^{\top}\in T$, and (iv)
applies \eqref{eq:cor_suff_6} and the facts that $\bm{U}\bm{\Sigma}\bm{V}^{\top}\in T$
and that $\bU\bSigma^{1/2}\bQ\bQ^{\top}\bSigma^{-1/2}\bV^{\top}\in T$.
Therefore, it suffices to bound $\|\mathcal{P}_{T^{\perp}}(\tilde{\bm{R}})\|$.
Rewrite \eqref{eq:cor_suff_6} as
\begin{equation}
p\bM^{\star}+\cP_{\Omega}(\bE)-\mathcal{P}_{\Omega}^{\mathsf{debias}}\left(\bX\bY^{\top}-\bM^{\star}\right)-\cP_{T}\big(\tilde{\bR}\big)
	=
\bU(p\bSigma+\lambda\bSigma^{1/2}\bQ\bQ^{\top}\bSigma^{-1/2})\bV^{\top}+\cP_{T^{\perp}}\big(\tilde{\bR}\big).\label{eq:cor_suff_8}
\end{equation}
Suppose for the moment that
\begin{equation}
\big\|\cP_{T}\big(\tilde{\bR}\big)\big\|\leq\lambda/4.\label{eq:P-T-R}
\end{equation}
This together with the assumptions that $\|\mathcal{P}_{\Omega}^{\mathsf{debias}}\left(\bX\bY^{\top}-\bM^{\star}\right)\|<\lambda/8$
and $\|\mathcal{P}_{\Omega}(\bm{E})\|<\lambda/8$ reveals that
\begin{align}
	\big\|\cP_{\Omega}(\bE)-\mathcal{P}_{\Omega}^{\mathsf{debias}}(\bX\bY^{\top}-\bM^{\star})-\cP_{T}(\tilde{\bR})\big\|<\lambda/2. \label{eq:UB-51}
\end{align}
By Weyl's inequality and the relations  \eqref{eq:cor_suff_8} and \eqref{eq:UB-51}, one has
\begin{align}
\sigma_{i}\left[\bU\big(p\bSigma+\lambda\bSigma^{1/2}\bQ\bQ^{\top}\bSigma^{-1/2}\big)\bV^{\top}+\cP_{T^{\perp}}\big(\tilde{\bR}\big)\right]
	%&= \sigma_{i}\left[ p\bM^{\star}+\cP_{\Omega}(\bE)-\mathcal{P}_{\Omega}^{\mathsf{debias}}\left(\bX\bY^{\top}-\bM^{\star}\right)-\cP_{T}\big(\tilde{\bR}\big) \right]  \nonumber\\
	& \leq  \sigma_{i}\left(p\bm{M}^{\star}\right)+ \big\|\cP_{\Omega}(\bE)-\mathcal{P}_{\Omega}^{\mathsf{debias}}(\bX\bY^{\top}-\bM^{\star})-\cP_{T}(\tilde{\bR})\big\| \nonumber\\
	&< p \sigma_{i}\left(\bm{M}^{\star}\right)+\lambda/2 = \lambda/2 \label{eq:singular-value-upper-bound}
\end{align}
for any $r+1\leq i\leq n$, where $\sigma_{i}(\bm{A})$ denotes the $i$th
largest singular value of a matrix $\bm{A}$. Here, we have used the
fact that $\bM^{\star}$ has rank $r$ and hence $\sigma_i(\bm{M}^{\star})=0$ for any $i>r$. In addition, it is seen that
\begin{align*}
\big\|\bm{\Sigma}^{1/2}\bm{Q}\bm{Q}^{\top}\bm{\Sigma}^{-1/2}-\bm{I}_{r}\big\| & =\big\|\bm{\Sigma}^{1/2}(\bm{Q}\bm{Q}^{\top}-\bm{I}_{r})\bm{\Sigma}^{-1/2}\big\|\\
 & \leq\big\|\bm{\Sigma}^{1/2}\big\|\big\|\bm{\Sigma}^{-1/2}\big\|\left\Vert \bm{Q}\bm{Q}^{\top}-\bm{I}_{r}\right\Vert _{\mathrm{F}}.
\end{align*}
Note that in (\ref{eq:prev-eq}), we have obtained 
\[
\big\|\bm{\Sigma}^{1/2}\big\|\big\|\bm{\Sigma}^{-1/2}\big\|\left\Vert \bm{Q}\bm{Q}^{\top}-\bm{I}_{r}\right\Vert _{\mathrm{F}}\leq2\sqrt{2\sigma_{\max}}\sqrt{{2}/{\sigma_{\min}}}8c\sqrt{c_{\text{inj}}p/\kappa}\leq {1}/{10}
\]
as long as $c$ is sufficiently small, and hence
$\big\|\bm{\Sigma}^{1/2}\bm{Q}\bm{Q}^{\top}\bm{\Sigma}^{-1/2}-\bm{I}_{r}\big\| 
\leq {1}/{10}$. 
Therefore, for any $1\leq i\leq r$ we know that 
\begin{align*}
\sigma_{i}\left[\bU\big(p\bSigma+\lambda\bSigma^{1/2}\bQ\bQ^{\top}\bSigma^{-1/2}\big)\bV^{\top}\right] 
	& \geq\sigma_{r}\left[ \bU \big( p\bSigma+\lambda\bm{I}_{r}+\lambda\big(\bSigma^{1/2}\bQ\bQ^{\top}\bSigma^{-1/2}-\bm{I}_{r}\big) \big)\bV^{\top} \right]\\
 & \geq\sigma_{r}\left(p\bSigma+\lambda\bm{I}_{r}\right)-\lambda\big\|\bm{\Sigma}^{1/2}\bm{Q}\bm{Q}^{\top}\bm{\Sigma}^{-1/2}-\bm{I}_{r}\big\|\\
 & \geq \lambda -\lambda\big\|\bm{\Sigma}^{1/2}\bm{Q}\bm{Q}^{\top}\bm{\Sigma}^{-1/2}-\bm{I}_{r}\big\|\\
 & \geq\lambda-\lambda/10>\lambda/2,
\end{align*}
where the second inequality results from Weyl's
inequality. This combined with (\ref{eq:equiv-R-R-tilde}) and
(\ref{eq:singular-value-upper-bound}) yields
\[
\norm{\cP_{T^{\perp}}(\bR)}=\big\|\cP_{T^{\perp}}\big(\tilde{\bR}\big)\big\|<\lambda/2;
\]
this happens because at least $n-r$ singular values of
$\bU\left(p\bSigma+\lambda\bSigma^{1/2}\bQ\bQ^{\top}\bSigma^{-1/2}\right)\bV^{\top}+\cP_{T^{\perp}}\big(\tilde{\bR}\big)$
are no larger than $\lambda/2$ and they cannot correspond to directions
simultaneously in the column space spanned by $\bm{U}$ and the row
space spanned by $\bm{V}^{\top}$.
\end{enumerate}
The proof is then complete by verifying (\ref{eq:P-T-R}). To this
end, observe that
\[
\tilde{\bR}\bV=-\bB_{1}\bQ^{\top}\bSigma^{-1/2},\quad\tilde{\bR}^{\top}\bU=\lambda\bV\bSigma^{1/2}\bQ^{-\top}\bQ^{-1}\bSigma^{-1/2}-\lambda\bV\bSigma^{-1/2}\bQ\bQ^{\top}\bSigma^{1/2}-\bB_{2}\bQ^{-1}\bSigma^{-1/2}.
\]
Then following similar technique used to bound $\norm{\cP_{T}(\bR)}$,
we have
\begin{equation}
\big\|\cP_{T}(\tilde{\bR})\big\|\leq\big\|\cP_{T}\big(\tilde{\bR}\big)\big\|_{\mathrm{F}}\leq\big\|\bU^{\top}\tilde{\bR}\big\|_{\mathrm{F}}+\big\|\tilde{\bR}\bV\big\|_{\mathrm{F}}\lesssim c\sqrt{c_{\text{inj}}p}\lambda<\lambda/4\label{eq:cor_suff_9}
\end{equation}
as long as $c$ is small enough.

\subsubsection{Proof of Claim \ref{claim:balancing}\label{subsec:Proof-of-Claim-balancing}}

Let
\begin{equation}
\mathcal{P}_{\Omega}\left(\bm{X}\bm{Y}^{\top}-\bm{M}\right)\bm{Y}+\lambda\bm{X}=\bm{B}_{1}\qquad\text{and}\qquad\left[\mathcal{P}_{\Omega}\left(\bm{X}\bm{Y}^{\top}-\bm{M}\right)\right]^{\top}\bm{X}+\lambda\bm{Y}=\bm{B}_{2}\label{eq:identities-B1B2}
\end{equation}
for some $\bm{B}_{1},\bm{B}_{2}\in\mathbb{R}^{n\times r}$. Clearly,
it is seen from the assumption (\ref{subeq:assumption-approx-stationary})
that
\begin{equation}
\max\{\|\bm{B}_{1}\|_{\mathrm{F}},\|\bm{B}_{2}\|_{\mathrm{F}}\}\leq p\|\nabla f(\bm{X},\bm{Y})\|_{\mathrm{F}}.\label{eq:B1B2-UB}
\end{equation}
In addition, the identities (\ref{eq:identities-B1B2}) allow us to
obtain
\begin{align}
\left\Vert \bm{X}^{\top}\bm{X}-\bm{Y}^{\top}\bm{Y}\right\Vert _{\mathrm{F}} & =\tfrac{1}{\lambda}\bigl\Vert\bm{X}^{\top}(\bm{B}_{1}-\mathcal{P}_{\Omega}\left(\bm{X}\bm{Y}^{\top}-\bm{M}\right)\bm{Y})-\big(\bm{B}_{2}-\left[\mathcal{P}_{\Omega}\left(\bm{X}\bm{Y}^{\top}-\bm{M}\right)\right]^{\top}\bm{X}\big)^{\top}\bm{Y}\bigr\Vert_{\mathrm{F}}\nonumber \\
 & =\tfrac{1}{\lambda}\left\Vert \bm{X}^{\top}\bm{B}_{1}-\bm{B}_{2}^{\top}\bm{Y}\right\Vert _{\mathrm{F}}\nonumber \\
 & \leq\tfrac{1}{\lambda}\left\Vert \bm{X}\right\Vert \left\Vert \bm{B}_{1}\right\Vert _{\mathrm{F}}+\tfrac{1}{\lambda}\left\Vert \bm{B}_{2}\right\Vert _{\mathrm{F}}\left\Vert \bm{Y}\right\Vert \nonumber \\
 & \leq2\tfrac{p}{\lambda}\sqrt{2\sigma_{\max}}\left\Vert \nabla f\left(\bm{X},\bm{Y}\right)\right\Vert _{\mathrm{F}}.\label{eq:XY-F-bound}
\end{align}
Here, the last line makes use of (\ref{eq:B1B2-UB}) and the assumption
that $\left\Vert \bm{X}\right\Vert ,\left\Vert \bm{Y}\right\Vert \leq\sqrt{2\sigma_{\max}}$.
In view of Lemma \ref{lemma:balance_determine}, one can find an invertible
$\bm{Q}$ such that $\bm{X}=\bm{U}\bm{\Sigma}^{1/2}\bm{Q}$, $\bm{Y}=\bm{V}\bm{\Sigma}^{1/2}\bm{Q}^{-\top}$
and
\begin{align*}
\big\|\bm{\Sigma}_{\bm{Q}}-\bm{\Sigma}_{\bm{Q}}^{-1}\big\|_{\mathrm{F}} & \leq\frac{1}{\sigma_{\min}\left(\bm{\Sigma}\right)}\left\Vert \bm{X}^{\top}\bm{X}-\bm{Y}^{\top}\bm{Y}\right\Vert _{\mathrm{F}}\\
 & \overset{(\text{i})}{\leq}\frac{2}{\sigma_{\min}}\cdot2\frac{p}{\lambda}\sqrt{2\sigma_{\max}}\left\Vert \nabla f\left(\bm{X},\bm{Y}\right)\right\Vert _{\mathrm{F}}\\
 & \leq8\sqrt{\kappa}\frac{p}{\lambda\sqrt{\sigma_{\min}}}\left\Vert \nabla f\left(\bm{X},\bm{Y}\right)\right\Vert _{\mathrm{F}}\\
 & \overset{(\text{ii})}{\leq}8c\sqrt{c_{\text{inj}}p/\kappa},
\end{align*}
where $\bm{\Sigma}_{\bm{Q}}$ is a diagonal matrix consisting of all
singular values of $\bm{Q}$. Here, (i) follows from (\ref{eq:claim-1})
as well as the bound (\ref{eq:XY-F-bound}), and the last inequality
(ii) uses the assumption (\ref{eq:small-gradient-f}). This completes
the proof.

\subsection{Proof of Lemma \ref{lemma:injectivity-main} \label{subsec:Proof-of-Lemma-injectivity-main}}

Lemma \ref{lemma:injectivity-main} consists of two parts, which we
restate into the following two lemmas, namely Lemmas \ref{lemma:injectivity}-\ref{lemma:P-tilde}.

First of all, Lemma \ref{lemma:injectivity} demonstrates that as
long as $(\bm{X},\bm{Y})$ is sufficiently close to $(\bm{X}^{\star},\bm{Y}^{\star})$,
the operator $\mathcal{P}_{\Omega}(\cdot)$ restricted to the tangent
space $T$ of $\bm{X}\bm{Y}^{\top}$ is injective. The proof is deferred
to Appendix \ref{subsec:Proof-of-Lemma-injectivity}.

\begin{lemma}\label{lemma:injectivity}Suppose that the sample complexity
obeys $n^{2}p\geq C\mu rn\log n$ for some sufficiently large constant
$C>0$. Then with probability exceeding $1-O(n^{-10})$,
\[
\frac{1}{p}\left\Vert \mathcal{P}_{\Omega}\left(\bm{H}\right)\right\Vert _{\mathrm{F}}^{2}\geq\frac{1}{32\kappa}\left\Vert \bm{H}\right\Vert _{\mathrm{F}}^{2},\qquad\forall\bm{H}\in T
\]
holds simultaneously for all $(\bm{X},\bm{Y})$ obeying
\begin{equation}
\max\big\{\left\Vert \bm{X}-\bm{X}^{\star}\right\Vert _{2,\infty},\left\Vert \bm{Y}-\bm{Y}^{\star}\right\Vert _{2,\infty}\big\}\leq\frac{c}{\kappa\sqrt{n}}\left\Vert \bm{X}^{\star}\right\Vert .\label{eq:new-inf-condition}
\end{equation}
Here, $c>0$ is some sufficiently small constant, and $T$ denotes
the tangent space of $\bm{X}\bm{Y}^{\top}$.\end{lemma}

\begin{remark}In the prior literature, the injectivity of $\mathcal{P}_{\Omega}(\cdot)$
has been mostly studied when restricted to a \emph{fixed} tangent
space independent of $\Omega$ (see \cite{ExactMC09,Gross2011recovering}).
In comparison, this lemma demonstrates that the injectivity property
holds uniformly over a large set of tangent spaces. This allows one
to handle tangent spaces that are statistically dependent on $\Omega$.\end{remark}

\begin{remark}Note that the condition (\ref{eq:new-inf-condition})
on $(\bm{X},\bm{Y})$ is weaker than (\ref{subeq:condition-inf})
under the assumptions of Lemma \ref{lemma:injectivity-main}. To see
this, if (\ref{subeq:condition-inf}) holds, then one necessarily
has
\begin{align*}
\left\Vert \bm{X}-\bm{X}^{\star}\right\Vert _{2,\infty} & \leq C_{\infty}\kappa\left(\frac{\sigma}{\sigma_{\min}}\sqrt{\frac{n\log n}{p}}+\frac{\lambda}{p\,\sigma_{\min}}\right)\max\left\{ \left\Vert \bm{X}^{\star}\right\Vert _{2,\infty},\left\Vert \bm{Y}^{\star}\right\Vert _{2,\infty}\right\} \\
 & \overset{(\text{i})}{\lesssim}C_{\infty}\kappa\frac{\sigma}{\sigma_{\min}}\sqrt{\frac{n\log n}{p}}\max\left\{ \left\Vert \bm{X}^{\star}\right\Vert _{2,\infty},\left\Vert \bm{Y}^{\star}\right\Vert _{2,\infty}\right\} \\
 & \overset{(\text{ii})}{\leq}C_{\infty}\kappa\frac{\sigma}{\sigma_{\min}}\sqrt{\frac{n\log n}{p}}\sqrt{\frac{\mu r}{n}}\left\Vert \bm{X}^{\star}\right\Vert \\
 & \overset{(\text{iii})}{\leq}\frac{c}{\kappa\sqrt{n}}\left\Vert \bm{X}^{\star}\right\Vert .
\end{align*}
Here, (i) follows from the choice $\lambda\asymp\sigma\sqrt{np}$;
(ii) relies on the incoherence assumption (\ref{eq:X-Y-incoherence});
and (iii) holds true under the noise condition $\frac{\sigma}{\sigma_{\min}}\sqrt{\frac{n}{p}}\ll\frac{1}{\sqrt{\kappa^{4}\mu r\log n}}$.
A similar bound holds for $\left\Vert \bm{Y}-\bm{Y}^{\star}\right\Vert _{2,\infty} $. \end{remark}

The next lemma shows that for all $(\bm{X},\bm{Y})$ close to $(\bm{X}^{\star},\bm{Y}^{\star})$,
$\mathcal{P}_{\Omega}(\bm{X}\bm{Y}^{\top}-\bm{M}^{\star})$ is uniformly
close to its expectation $p(\bm{X}\bm{Y}^{\top}-\bm{M}^{\star})$.
The proof can be found in Appendix \ref{subsec:Proof-of-Lemma-P-debias}.

\begin{lemma}\label{lemma:P-tilde}Suppose that $n^{2}p\gg\kappa^{4}\mu^{2}r^{2}n\log^{2}n$
and $\sigma\sqrt{n(\log n)/p}\ll\sigma_{\min}/\kappa$. With probability
exceeding $1-O(n^{-10})$, one has
\[
\left\Vert \mathcal{P}_{\Omega}\left(\bm{X}\bm{Y}^{\top}-\bm{M}^{\star}\right)-p\left(\bm{X}\bm{Y}^{\top}-\bm{M}^{\star}\right)\right\Vert <\lambda/8
\]
simultaneously for any $(\bm{X},\bm{Y})$ obeying (\ref{subeq:condition-inf}),
provided that $\lambda=C_{\lambda}\sigma\sqrt{np}$ for some constant
$C_{\lambda}>0$. \end{lemma}

\subsubsection{Proof of Lemma \ref{lemma:injectivity} \label{subsec:Proof-of-Lemma-injectivity}}

By definition, any $\bm{H}\in T$ can be expressed as
\begin{equation}
\bm{H}=\bm{X}\bm{A}^{\top}+\bm{B}\bm{Y}^{\top}\label{eq:H-A-B}
\end{equation}
for some $\bm{A},\bm{B}\in\mathbb{R}^{n\times r}$. Given that this
is an underdetermined linear system of equations, there might be numerous
$(\bm{A},\bm{B})$'s compatible with (\ref{eq:H-A-B}). We take a
specific choice as follows
\begin{align}
\left(\bm{A},\bm{B}\right) & :=\arg\min_{(\tilde{\bm{A}},\tilde{\bm{B}})}\text{ }\text{ }0.5\big\|\tilde{\bm{A}}\big\|_{\mathrm{F}}^{2}+0.5\big\|\tilde{\bm{B}}\big\|_{\mathrm{F}}^{2}\label{eq:constrained-opt-V-1}\\
 & \quad\text{subject to}\quad\bm{H}=\bm{X}\tilde{\bm{A}}^{\top}+\tilde{\bm{B}}\bm{Y}^{\top}.\nonumber
\end{align}
which satisfies a property that plays an important role in the subsequent
analysis:
\begin{equation}
\bm{X}^{\top}\bm{B}=\bm{A}^{\top}\bm{Y}.\label{eq:tangent_key}
\end{equation}
To see this, consider the Lagrangian
\[
\mathcal{L}(\tilde{\bm{A}},\tilde{\bm{B}},\bm{\Lambda}):=0.5\big\|\tilde{\bm{A}}\big\|_{\mathrm{F}}^{2}+0.5\big\|\tilde{\bm{B}}\big\|_{\mathrm{F}}^{2}+\langle\bm{\Lambda},\bm{X}\tilde{\bm{A}}^{\top}+\tilde{\bm{B}}\bm{Y}^{\top}-\bm{H}\rangle.
\]
Taking the derivatives w.r.t.~$\tilde{\bm{A}}$ and $\tilde{\bm{B}}$
and setting them to zero yield
\[
\bm{A}=-\bm{\Lambda}^{\top}\bm{X}\qquad\text{and}\qquad\bm{B}=-\bm{\Lambda}\bm{Y}
\]
for some Lagrangian multiplier matrix $\bm{\Lambda}\in\mathbb{R}^{n\times n}$.
The claim (\ref{eq:tangent_key}) then follows immediately.

The remaining proof consists of two steps.
\begin{itemize}
\item First, we would like to show that
\begin{equation}
\left\Vert \bm{H}\right\Vert _{\mathrm{F}}^{2}\leq8\sigma_{\max}\big(\left\Vert \bm{A}\right\Vert _{\mathrm{F}}^{2}+\left\Vert \bm{B}\right\Vert _{\mathrm{F}}^{2}\big).\label{eq:tangent_upper}
\end{equation}
\item Second, we prove that
\begin{equation}
\frac{1}{2p}\left\Vert \mathcal{P}_{\Omega}\left(\bm{H}\right)\right\Vert _{\mathrm{F}}^{2}=\frac{1}{2p}\left\Vert \mathcal{P}_{\Omega}\left(\bm{X}\bm{A}^{\top}+\bm{B}\bm{Y}^{\top}\right)\right\Vert _{\mathrm{F}}^{2}\geq\frac{\sigma_{\min}}{8}\big(\left\Vert \bm{A}\right\Vert _{\mathrm{F}}^{2}+\left\Vert \bm{B}\right\Vert _{\mathrm{F}}^{2}\big).\label{eq:tangent_lower}
\end{equation}
\end{itemize}
Taking (\ref{eq:tangent_upper}) and (\ref{eq:tangent_lower}) together
immediately yields the claimed bounds in the lemma. In what follows,
we shall establish these two bounds separately.
\begin{enumerate}
\item Regarding the upper bound (\ref{eq:tangent_upper}), it follows from
elementary inequalities that
\begin{align}
\left\Vert \bm{H}\right\Vert _{\mathrm{F}}^{2} & =\left\Vert \bm{X}\bm{A}^{\top}+\bm{B}\bm{Y}^{\top}\right\Vert _{\mathrm{F}}^{2}\leq2\big(\left\Vert \bm{X}\bm{A}^{\top}\right\Vert _{\mathrm{F}}^{2}+\left\Vert \bm{B}\bm{Y}^{\top}\right\Vert _{\mathrm{F}}^{2}\big)\nonumber \\
 & \leq2\big(\left\Vert \bm{X}\right\Vert ^{2}\left\Vert \bm{A}\right\Vert _{\mathrm{F}}^{2}+\left\Vert \bm{Y}\right\Vert ^{2}\left\Vert \bm{B}\right\Vert _{\mathrm{F}}^{2}\big)\nonumber \\
 & \leq2\max\left\{ \left\Vert \bm{X}\right\Vert ^{2},\left\Vert \bm{Y}\right\Vert ^{2}\right\} \big(\left\Vert \bm{A}\right\Vert _{\mathrm{F}}^{2}+\left\Vert \bm{B}\right\Vert _{\mathrm{F}}^{2}\big).\label{eq:tangent_upper_1}
\end{align}
It then suffices to control $\max\{\|\bm{X}\|,\|\bm{Y}\|\}$. In view
of the assumption (\ref{eq:new-inf-condition}), one has
\begin{equation}
\left\Vert \bm{X}-\bm{X}^{\star}\right\Vert \leq\left\Vert \bm{X}-\bm{X}^{\star}\right\Vert _{\mathrm{F}}\leq\sqrt{n}\left\Vert \bm{X}-\bm{X}^{\star}\right\Vert _{2,\infty}\leq\frac{c}{\kappa}\left\Vert \bm{X}^{\star}\right\Vert \leq\left\Vert \bm{X}^{\star}\right\Vert ,\label{eq:X-spectral-perturbation}
\end{equation}
as long as $c<1$. This together with the triangle inequality reveals
that
\[
\left\Vert \bm{X}\right\Vert \leq\left\Vert \bm{X}^{\star}\right\Vert +\left\Vert \bm{X}-\bm{X}^{\star}\right\Vert \leq2\left\Vert \bm{X}^{\star}\right\Vert \leq2\sqrt{\sigma_{\max}}.
\]
Similarly, one has $\left\Vert \bm{Y}\right\Vert \leq2\sqrt{\sigma_{\max}}$.
Substitution into (\ref{eq:tangent_upper_1}) yields the desired upper
bound (\ref{eq:tangent_upper}).
\item We now move on to the lower bound (\ref{eq:tangent_lower}). To this
end, one first decomposes
\[
\frac{1}{2p}\left\Vert \mathcal{P}_{\Omega}\left(\bm{X}\bm{A}^{\top}+\bm{B}\bm{Y}^{\top}\right)\right\Vert _{\mathrm{F}}^{2}=\underbrace{\frac{1}{2p}\left\Vert \mathcal{P}_{\Omega}\left(\bm{X}\bm{A}^{\top}+\bm{B}\bm{Y}^{\top}\right)\right\Vert _{\mathrm{F}}^{2}-\frac{1}{2}\left\Vert \bm{X}\bm{A}^{\top}+\bm{B}\bm{Y}^{\top}\right\Vert _{\mathrm{F}}^{2}}_{:=\alpha_{1}}+\underbrace{\frac{1}{2}\left\Vert \bm{X}\bm{A}^{\top}+\bm{B}\bm{Y}^{\top}\right\Vert _{\mathrm{F}}^{2}}_{:=\alpha_{2}}.
\]
The basic idea is to demonstrate that $\left(\text{1}\right)$ $\alpha_{2}$
is bounded from below, and $\left(\text{2}\right)$ $\alpha_{1}$
is sufficiently small compared to $\alpha_{2}$.
\begin{enumerate}
\item We start by controlling $\alpha_{2}$, towards which we can expand
\begin{align*}
\alpha_{2} & =\frac{1}{2}\left(\left\Vert \bm{X}\bm{A}^{\top}\right\Vert _{\mathrm{F}}^{2}+\left\Vert \bm{B}\bm{Y}^{\top}\right\Vert _{\mathrm{F}}^{2}\right)+\mathrm{Tr}\left(\bm{X}^{\top}\bm{B}\bm{Y}^{\top}\bm{A}\right).
\end{align*}
The property $\bm{X}^{\top}\bm{B}=\bm{A}^{\top}\bm{Y}$ (see (\ref{eq:tangent_key}))
implies that
\[
\mathrm{Tr}\left(\bm{X}^{\top}\bm{B}\bm{Y}^{\top}\bm{A}\right)=\left\Vert \bm{X}^{\top}\bm{B}\right\Vert _{\mathrm{F}}^{2}\geq0\qquad\Longrightarrow\qquad\alpha_{2}\geq\frac{1}{2}\left(\left\Vert \bm{X}\bm{A}^{\top}\right\Vert _{\mathrm{F}}^{2}+\left\Vert \bm{B}\bm{Y}^{\top}\right\Vert _{\mathrm{F}}^{2}\right).
\]
Write $\bm{\Delta}_{\bm{X}}=\bm{X}-\bm{X}^{\star}$ and $\bm{\Delta}_{\bm{Y}}=\bm{Y}-\bm{Y}^{\star}$.
We have
\begin{align*}
\left\Vert \bm{X}\bm{A}^{\top}\right\Vert _{\mathrm{F}}^{2} & =\left\Vert \left(\bm{X}^{\star}+\bm{\Delta}_{\bm{X}}\right)\bm{A}^{\top}\right\Vert _{\mathrm{F}}^{2}=\left\Vert \bm{X}^{\star}\bm{A}^{\top}\right\Vert _{\mathrm{F}}^{2}+\left\Vert \bm{\Delta}_{\bm{X}}\bm{A}^{\top}\right\Vert _{\mathrm{F}}^{2}+2\left\langle \bm{X}^{\star}\bm{A}^{\top},\bm{\Delta}_{\bm{X}}\bm{A}^{\top}\right\rangle \\
 & \geq\left\Vert \bm{X}^{\star}\bm{A}^{\top}\right\Vert _{\mathrm{F}}^{2}-2\left\Vert \bm{X}^{\star}\bm{A}^{\top}\right\Vert _{\mathrm{F}}\left\Vert \bm{\Delta}_{\bm{X}}\bm{A}^{\top}\right\Vert _{\mathrm{F}}\\
 & \geq\left\Vert \bm{X}^{\star}\bm{A}^{\top}\right\Vert _{\mathrm{F}}^{2}-2\left\Vert \bm{X}^{\star}\right\Vert \left\Vert \bm{\Delta}_{\bm{X}}\right\Vert \left\Vert \bm{A}\right\Vert _{\mathrm{F}}^{2},
\end{align*}
where the second line arises from the Cauchy-Schwarz inequality. Recalling
from (\ref{eq:X-spectral-perturbation}) that $\|\bm{\Delta}_{\bm{X}}\|\leq c\|\bm{X}^{\star}\|/\kappa$,
we arrive at
\[
\left\Vert \bm{X}\bm{A}^{\top}\right\Vert _{\mathrm{F}}^{2}\geq\left\Vert \bm{X}^{\star}\bm{A}^{\top}\right\Vert _{\mathrm{F}}^{2}-2c\sigma_{\min}\left\Vert \bm{A}\right\Vert _{\mathrm{F}}^{2}\geq\left\Vert \bm{X}^{\star}\bm{A}^{\top}\right\Vert _{\mathrm{F}}^{2}-\sigma_{\min}\left\Vert \bm{A}\right\Vert _{\mathrm{F}}^{2}/100,
\]
provided that $c\leq1/200$. A similar bound holds for $\|\bm{B}\bm{Y}^{\top}\|_{\mathrm{F}}^{2}$,
thus leading to
\[
\alpha_{2}\geq\frac{1}{2}\left(\left\Vert \bm{X}^{\star}\bm{A}^{\top}\right\Vert _{\mathrm{F}}^{2}+\left\Vert \bm{B}\bm{Y}^{\star\top}\right\Vert _{\mathrm{F}}^{2}\right)-\frac{1}{100}\sigma_{\min}\left(\left\Vert \bm{A}\right\Vert _{\mathrm{F}}^{2}+\left\Vert \bm{B}\right\Vert _{\mathrm{F}}^{2}\right).
\]
\item Next, we control $\alpha_{1}$. First, it is seen that
\begin{align*}
\bm{X}\bm{A}^{\top}+\bm{B}\bm{Y}^{\top} & =\left(\bm{X}^{\star}+\bm{\Delta}_{\bm{X}}\right)\bm{A}^{\top}+\bm{B}\left(\bm{Y}^{\star}+\bm{\Delta}_{\bm{Y}}\right)^{\top}\\
 & =\bm{X}^{\star}\bm{A}^{\top}+\bm{B}\bm{Y}^{\star\top}+\bm{\Delta}_{\bm{X}}\bm{A}^{\top}+\bm{B}\bm{\Delta}_{\bm{Y}}^{\top}.
\end{align*}
As a result, we can expand $\alpha_{1}$ as
\begin{align*}
\alpha_{1} & =\frac{1}{2p}\left\Vert \mathcal{P}_{\Omega}\left(\bm{X}^{\star}\bm{A}^{\top}+\bm{B}\bm{Y}^{\star\top}+\bm{\Delta}_{\bm{X}}\bm{A}^{\top}+\bm{B}\bm{\Delta}_{\bm{Y}}^{\top}\right)\right\Vert _{\mathrm{F}}^{2}-\frac{1}{2}\left\Vert \bm{X}^{\star}\bm{A}^{\top}+\bm{B}\bm{Y}^{\star\top}+\bm{\Delta}_{\bm{X}}\bm{A}^{\top}+\bm{B}\bm{\Delta}_{\bm{Y}}^{\top}\right\Vert _{\mathrm{F}}^{2}\\
 & =\underbrace{\frac{1}{2p}\left\Vert \mathcal{P}_{\Omega}\left(\bm{X}^{\star}\bm{A}^{\top}+\bm{B}\bm{Y}^{\star\top}\right)\right\Vert _{\mathrm{F}}^{2}-\frac{1}{2}\left\Vert \bm{X}^{\star}\bm{A}^{\top}+\bm{B}\bm{Y}^{\star\top}\right\Vert _{\mathrm{F}}^{2}}_{:=\gamma_{1}}\\
 & \quad+\underbrace{\frac{1}{2p}\left\Vert \mathcal{P}_{\Omega}\left(\bm{\Delta}_{\bm{X}}\bm{A}^{\top}\right)\right\Vert _{\mathrm{F}}^{2}-\frac{1}{2}\left\Vert \bm{\Delta}_{\bm{X}}\bm{A}^{\top}\right\Vert _{\mathrm{F}}^{2}}_{:=\gamma_{2}}+\underbrace{\frac{1}{2p}\left\Vert \mathcal{P}_{\Omega}\left(\bm{B}\bm{\Delta}_{\bm{Y}}^{\top}\right)\right\Vert _{\mathrm{F}}^{2}-\frac{1}{2}\left\Vert \bm{B}\bm{\Delta}_{\bm{Y}}^{\top}\right\Vert _{\mathrm{F}}^{2}}_{:=\gamma_{3}}\\
 & \quad+\underbrace{\frac{1}{p}\left\langle \mathcal{P}_{\Omega}\left(\bm{\Delta}_{\bm{X}}\bm{A}^{\top}\right),\mathcal{P}_{\Omega}\left(\bm{B}\bm{\Delta}_{\bm{Y}}^{\top}\right)\right\rangle -\left\langle \bm{\Delta}_{\bm{X}}\bm{A}^{\top},\bm{B}\bm{\Delta}_{\bm{Y}}^{\top}\right\rangle }_{:=\gamma_{4}}\\
 & \quad+\underbrace{\frac{1}{p}\left\langle \mathcal{P}_{\Omega}\left(\bm{X}^{\star}\bm{A}^{\top}+\bm{B}\bm{Y}^{\star\top}\right),\mathcal{P}_{\Omega}\left(\bm{\Delta}_{\bm{X}}\bm{A}^{\top}+\bm{B}\bm{\Delta}_{\bm{Y}}^{\top}\right)\right\rangle -\left\langle \bm{X}^{\star}\bm{A}^{\top}+\bm{B}\bm{Y}^{\star\top},\bm{\Delta}_{\bm{X}}\bm{A}^{\top}+\bm{B}\bm{\Delta}_{\bm{Y}}^{\top}\right\rangle }_{:=\gamma_{5}}.
\end{align*}

\begin{enumerate}
\item Regarding $\gamma_{1}$, it follows from the bounds in \cite[Section 4.2]{ExactMC09}
that
\[
\left|\gamma_{1}\right|\leq\frac{1}{64}\left\Vert \bm{X}^{\star}\bm{A}^{\top}+\bm{B}\bm{Y}^{\star\top}\right\Vert _{\mathrm{F}}^{2}\leq\frac{1}{32}\left(\left\Vert \bm{X}^{\star}\bm{A}^{\top}\right\Vert _{\mathrm{F}}^{2}+\left\Vert \bm{B}\bm{Y}^{\star\top}\right\Vert _{\mathrm{F}}^{2}\right),
\]
as long as $np\gg\mu r\log n$.
\item Invoke Lemma \ref{lemma:zheng-and-lafferty} to show that
\begin{align*}
\left|\gamma_{2}\right| & \leq\frac{3n}{2}\left\Vert \bm{\Delta}_{\bm{X}}\right\Vert _{2,\infty}^{2}\left\Vert \bm{A}\right\Vert _{\mathrm{F}}^{2}\leq\frac{3c^{2}}{2\kappa}\sigma_{\min}\left\Vert \bm{A}\right\Vert _{\mathrm{F}}^{2}\leq\frac{1}{100}\sigma_{\min}\left\Vert \bm{A}\right\Vert _{\mathrm{F}}^{2},\\
\left|\gamma_{3}\right| & \leq\frac{3n}{2}\left\Vert \bm{\Delta}_{\bm{Y}}\right\Vert _{2,\infty}^{2}\left\Vert \bm{B}\right\Vert _{\mathrm{F}}^{2}\leq\frac{3c^{2}}{2\kappa}\sigma_{\min}\left\Vert \bm{B}\right\Vert _{\mathrm{F}}^{2}\leq\frac{1}{100}\sigma_{\min}\left\Vert \bm{B}\right\Vert _{\mathrm{F}}^{2},
\end{align*}
as long as $n^{2}p\gg n\log n$ and $c>0$ is sufficiently small.
Here we have utilized the assumption that $\max\{\|\bm{\Delta}_{\bm{X}}\|_{2,\infty},\|\bm{\Delta}_{\bm{Y}}\|_{2,\infty}\}\leq c\|\bm{X}^{\star}\|/(\kappa\sqrt{n})$.
\item The term $\gamma_{4}$ can be controlled via Lemma \ref{lemma:chen-and-ji}:
\begin{align*}
\left|\gamma_{4}\right| & \leq\left\Vert \frac{1}{p}\mathcal{P}_{\Omega}\left(\bm{1}\bm{1}^{\top}\right)-\bm{1}\bm{1}^{\top}\right\Vert \left\Vert \bm{\Delta}_{\bm{X}}\right\Vert _{2,\infty}\left\Vert \bm{A}\right\Vert _{\mathrm{F}}\left\Vert \bm{\Delta}_{\bm{Y}}\right\Vert _{2,\infty}\left\Vert \bm{B}\right\Vert _{\mathrm{F}}\\
 & \lesssim\sqrt{\frac{n}{p}}\left\Vert \bm{\Delta}_{\bm{X}}\right\Vert _{2,\infty}\left\Vert \bm{A}\right\Vert _{\mathrm{F}}\left\Vert \bm{\Delta}_{\bm{Y}}\right\Vert _{2,\infty}\left\Vert \bm{B}\right\Vert _{\mathrm{F}},
\end{align*}
where the second line uses the bound $\|p^{-1}\mathcal{P}_{\Omega}\left(\bm{1}\bm{1}^{\top}\right)-\bm{1}\bm{1}^{\top}\|\lesssim\sqrt{n/p}$
guaranteed by \cite[Lemma 3.2]{KesMonSew2010}. Continue the upper
bound to get
\[
\left|\gamma_{4}\right|\overset{(\text{i})}{\lesssim}n\frac{c^{2}}{\kappa^{2}n}\sigma_{\max}\left\Vert \bm{A}\right\Vert _{\mathrm{F}}\left\Vert \bm{B}\right\Vert _{\mathrm{F}}\overset{(\text{ii})}{\leq}\frac{c^{2}}{2\kappa}\sigma_{\min}\left(\left\Vert \bm{A}\right\Vert _{\mathrm{F}}^{2}+\left\Vert \bm{B}\right\Vert _{\mathrm{F}}^{2}\right)\overset{(\text{iii})}{\leq}\frac{1}{100}\sigma_{\min}\left(\left\Vert \bm{A}\right\Vert _{\mathrm{F}}^{2}+\left\Vert \bm{B}\right\Vert _{\mathrm{F}}^{2}\right).
\]
Here the first relation (i) arises from the assumption that $np\gg1$.
The second inequality (ii) applies the elementary inequality $ab\leq(a^{2}+b^{2})/2$
and the last one (iii) holds with the proviso that $c>0$ is small
enough.
\item The last term $\gamma_{5}$ can be further decomposed into the sum
of four terms. For brevity, we take one out as an example, namely
the term
\[
\frac{1}{p}\left\langle \mathcal{P}_{\Omega}\left(\bm{X}^{\star}\bm{A}^{\top}\right),\mathcal{P}_{\Omega}\left(\bm{\Delta}_{\bm{X}}\bm{A}^{\top}\right)\right\rangle -\left\langle \bm{X}^{\star}\bm{A}^{\top},\bm{\Delta}_{\bm{X}}\bm{A}^{\top}\right\rangle .
\]
Apply the triangle inequality to obtain
\begin{align*}
 & \left|\frac{1}{p}\left\langle \mathcal{P}_{\Omega}\left(\bm{X}^{\star}\bm{A}^{\top}\right),\mathcal{P}_{\Omega}\left(\bm{\Delta}_{\bm{X}}\bm{A}^{\top}\right)\right\rangle -\left\langle \bm{X}^{\star}\bm{A}^{\top},\bm{\Delta}_{\bm{X}}\bm{A}^{\top}\right\rangle \right|\\
 & \quad\leq\left|\frac{1}{p}\left\langle \mathcal{P}_{\Omega}\left(\bm{X}^{\star}\bm{A}^{\top}\right),\mathcal{P}_{\Omega}\left(\bm{\Delta}_{\bm{X}}\bm{A}^{\top}\right)\right\rangle \right|+\left|\left\langle \bm{X}^{\star}\bm{A}^{\top},\bm{\Delta}_{\bm{X}}\bm{A}^{\top}\right\rangle \right|\\
 & \quad\leq\frac{1}{\sqrt{p}}\left\Vert \mathcal{P}_{\Omega}\left(\bm{X}^{\star}\bm{A}^{\top}\right)\right\Vert _{\mathrm{F}}\frac{1}{\sqrt{p}}\left\Vert \mathcal{P}_{\Omega}\left(\bm{\Delta}_{\bm{X}}\bm{A}^{\top}\right)\right\Vert _{\mathrm{F}}+\left\Vert \bm{X}^{\star}\bm{A}^{\top}\right\Vert _{\mathrm{F}}\left\Vert \bm{\Delta}_{\bm{X}}\bm{A}^{\top}\right\Vert _{\mathrm{F}}.
\end{align*}
In light of \cite[Section 4.2]{ExactMC09} and \cite[Lemma 9]{zheng2016convergence},
we have
\begin{align*}
\frac{1}{\sqrt{p}}\left\Vert \mathcal{P}_{\Omega}\left(\bm{X}^{\star}\bm{A}^{\top}\right)\right\Vert _{\mathrm{F}} & \leq1.1\left\Vert \bm{X}^{\star}\bm{A}^{\top}\right\Vert _{\mathrm{F}};\\
\frac{1}{\sqrt{p}}\left\Vert \mathcal{P}_{\Omega}\left(\bm{\Delta}_{\bm{X}}\bm{A}^{\top}\right)\right\Vert _{\mathrm{F}} & \leq\sqrt{2n}\left\Vert \bm{\Delta}_{\bm{X}}\right\Vert _{2,\infty}\left\Vert \bm{A}\right\Vert _{\mathrm{F}}.
\end{align*}
Taking the above three bounds collectively yields
\begin{align*}
 & \left|\frac{1}{p}\left\langle \mathcal{P}_{\Omega}\left(\bm{X}^{\star}\bm{A}^{\top}\right),\mathcal{P}_{\Omega}\left(\bm{\Delta}_{\bm{X}}\bm{A}^{\top}\right)\right\rangle -\left\langle \bm{X}^{\star}\bm{A}^{\top},\bm{\Delta}_{\bm{X}}\bm{A}^{\top}\right\rangle \right|\\
 & \quad\leq5\left\Vert \bm{X}^{\star}\bm{A}^{\top}\right\Vert _{\mathrm{F}}\sqrt{n}\left\Vert \bm{\Delta}_{\bm{X}}\right\Vert _{2,\infty}\left\Vert \bm{A}\right\Vert _{\mathrm{F}}+\left\Vert \bm{X}^{\star}\bm{A}^{\top}\right\Vert _{\mathrm{F}}\left\Vert \bm{\Delta}_{\bm{X}}\bm{A}^{\top}\right\Vert _{\mathrm{F}}\\
 & \quad\leq5\sqrt{n}\left\Vert \bm{\Delta}_{\bm{X}}\right\Vert _{2,\infty}\left\Vert \bm{X}^{\star}\right\Vert \left\Vert \bm{A}\right\Vert _{\mathrm{F}}^{2}+\sqrt{n}\left\Vert \bm{\Delta}_{\bm{X}}\right\Vert _{2,\infty}\left\Vert \bm{X}^{\star}\right\Vert \left\Vert \bm{A}\right\Vert _{\mathrm{F}}^{2}\\
 & \quad=6\sqrt{n}\left\Vert \bm{\Delta}_{\bm{X}}\right\Vert _{2,\infty}\left\Vert \bm{X}^{\star}\right\Vert \left\Vert \bm{A}\right\Vert _{\mathrm{F}}^{2}.
\end{align*}
Using the assumption that $\|\bm{\Delta}_{\bm{X}}\|_{2,\infty}\leq c\|\bm{X}^{\star}\|/(\kappa\sqrt{n})$,
one has
\[
\left|\frac{1}{p}\left\langle \mathcal{P}_{\Omega}\left(\bm{X}^{\star}\bm{A}^{\top}\right),\mathcal{P}_{\Omega}\left(\bm{\Delta}_{\bm{X}}\bm{A}^{\top}\right)\right\rangle -\left\langle \bm{X}^{\star}\bm{A}^{\top},\bm{\Delta}_{\bm{X}}\bm{A}^{\top}\right\rangle \right|\lesssim\sqrt{n}\frac{c}{\kappa\sqrt{n}}\sigma_{\max}\left\Vert \bm{A}\right\Vert _{\mathrm{F}}^{2}\leq\frac{1}{100}\sigma_{\min}\left\Vert \bm{A}\right\Vert _{\mathrm{F}}^{2}
\]
for $c>0$ small enough. The same argument applies to the remaining
three terms, resulting in
\[
\left|\gamma_{5}\right|\leq\frac{1}{50}\sigma_{\min}\left(\left\Vert \bm{A}\right\Vert _{\mathrm{F}}^{2}+\left\Vert \bm{B}\right\Vert _{\mathrm{F}}^{2}\right).
\]
\item Combining the previous bounds on $\gamma_{1}$ through $\gamma_{5}$,
we arrive at
\begin{align*}
\left|\alpha_{1}\right| & \leq\left|\gamma_{1}\right|+\left|\gamma_{2}\right|+\left|\gamma_{3}\right|+\left|\gamma_{4}\right|+\left|\gamma_{5}\right|\\
 & \leq\frac{1}{32}\left(\left\Vert \bm{X}^{\star}\bm{A}^{\top}\right\Vert _{\mathrm{F}}^{2}+\left\Vert \bm{B}\bm{Y}^{\star\top}\right\Vert _{\mathrm{F}}^{2}\right)+\frac{1}{25}\sigma_{\min}\left(\left\Vert \bm{A}\right\Vert _{\mathrm{F}}^{2}+\left\Vert \bm{B}\right\Vert _{\mathrm{F}}^{2}\right).
\end{align*}
\end{enumerate}
\item Taking the preceding bounds on $\alpha_{1}$ and $\alpha_{2}$ collectively
yields
\begin{align*}
 & \frac{1}{2p}\left\Vert \mathcal{P}_{\Omega}\left(\bm{X}\bm{A}^{\top}+\bm{B}\bm{Y}^{\top}\right)\right\Vert _{\mathrm{F}}^{2}\geq\alpha_{2}-\left|\alpha_{1}\right|\\
 & \quad\geq\frac{15}{32}\left(\left\Vert \bm{X}^{\star}\bm{A}^{\top}\right\Vert _{\mathrm{F}}^{2}+\left\Vert \bm{B}\bm{Y}^{\star\top}\right\Vert _{\mathrm{F}}^{2}\right)-\frac{1}{5}\sigma_{\min}\left(\left\Vert \bm{A}\right\Vert _{\mathrm{F}}^{2}+\left\Vert \bm{B}\right\Vert _{\mathrm{F}}^{2}\right)\\
 & \quad\geq\frac{15}{32}\sigma_{\min}\left(\left\Vert \bm{A}\right\Vert _{\mathrm{F}}^{2}+\left\Vert \bm{B}\right\Vert _{\mathrm{F}}^{2}\right)-\frac{1}{5}\sigma_{\min}\left(\left\Vert \bm{A}\right\Vert _{\mathrm{F}}^{2}+\left\Vert \bm{B}\right\Vert _{\mathrm{F}}^{2}\right)\\
 & \quad\geq\frac{1}{8}\sigma_{\min}\left(\left\Vert \bm{A}\right\Vert _{\mathrm{F}}^{2}+\left\Vert \bm{B}\right\Vert _{\mathrm{F}}^{2}\right).
\end{align*}
\end{enumerate}
\end{enumerate}
The proof is then complete.

\subsubsection{Proof of Lemma \ref{lemma:P-tilde} \label{subsec:Proof-of-Lemma-P-debias}}

To start with, we have
\[
\bm{X}\bm{Y}^{\top}-\bm{M}^{\star}=\left(\bm{X}-\bm{X}^{\star}\right)\bm{Y}^{\top}+\bm{X}^{\star}\left(\bm{Y}-\bm{Y}^{\star}\right)^{\top},
\]
which together with the triangle inequality implies
\begin{align*}
\left\Vert \mathcal{P}_{\Omega}^{\mathsf{debias}}\left(\bm{X}\bm{Y}^{\top}-\bm{M}^{\star}\right)\right\Vert  & \leq\left\Vert \mathcal{P}_{\Omega}^{\mathsf{debias}}\left[\left(\bm{X}-\bm{X}^{\star}\right)\bm{Y}^{\top}\right]\right\Vert +\big\|\mathcal{P}_{\Omega}^{\mathsf{debias}}\big[\bm{X}^{\star}\left(\bm{Y}-\bm{Y}^{\star}\right)^{\top}\big]\big\|.
\end{align*}
Apply \cite[Lemma 4.5]{chen2017memory} to obtain
\begin{align*}
\left\Vert \mathcal{P}_{\Omega}^{\mathsf{debias}}\left[\left(\bm{X}-\bm{X}^{\star}\right)\bm{Y}^{\top}\right]\right\Vert  & \leq\left\Vert \mathcal{P}_{\Omega}^{\mathsf{debias}}\left(\bm{1}\bm{1}^{\top}\right)\right\Vert \left\Vert \bm{X}-\bm{X}^{\star}\right\Vert _{2,\infty}\left\Vert \bm{Y}\right\Vert _{2,\infty}\\
 & \lesssim\sqrt{np}\left\Vert \bm{X}-\bm{X}^{\star}\right\Vert _{2,\infty}\left\Vert \bm{Y}\right\Vert _{2,\infty},
\end{align*}
where the second line is due to $\|\mathcal{P}_{\Omega}^{\mathsf{debias}}(\bm{1}\bm{1}^{\top})\|\lesssim\sqrt{np}$
(cf.~\cite[Lemma 3.2]{KesMonSew2010}). Similarly,
\begin{align*}
\big\|\mathcal{P}_{\Omega}^{\mathsf{debias}}\big[\bm{X}^{\star}\left(\bm{Y}-\bm{Y}^{\star}\right)^{\top}\big]\big\| & \lesssim\sqrt{np}\left\Vert \bm{Y}-\bm{Y}^{\star}\right\Vert _{2,\infty}\left\Vert \bm{X}^{\star}\right\Vert _{2,\infty}.
\end{align*}
In addition, the assumption (\ref{subeq:condition-inf}) yields
\begin{align*}
\left\Vert \bm{Y}\right\Vert _{2,\infty} & \leq\left\Vert \bm{Y}-\bm{Y}^{\star}\right\Vert _{2,\infty}+\left\Vert \bm{Y}^{\star}\right\Vert _{2,\infty}\\
 & \leq C_{\infty}\kappa\left(\frac{\sigma}{\sigma_{\min}}\sqrt{\frac{n\log n}{p}}+\frac{\lambda}{p\,\sigma_{\min}}\right)\left\Vert \bm{Y}^{\star}\right\Vert _{2,\infty}+\left\Vert \bm{Y}^{\star}\right\Vert _{2,\infty}\\
 & \leq2\left\Vert \bm{Y}^{\star}\right\Vert _{2,\infty},
\end{align*}
as long as $\frac{\sigma}{\sigma_{\min}}\sqrt{\frac{n\log n}{p}}\ll1/\kappa$
(recall that $\lambda=C_{\lambda}\sigma\sqrt{np}$ for some constant
$C_{\lambda}>0$). As a consequence, one obtains
\begin{align}
\left\Vert \mathcal{P}_{\Omega}^{\mathsf{debias}}\left(\bm{X}\bm{Y}^{\top}-\bm{M}^{\star}\right)\right\Vert  & \lesssim\sqrt{np}\kappa\left(\frac{\sigma}{\sigma_{\min}}\sqrt{\frac{n\log n}{p}}+\frac{\lambda}{p\,\sigma_{\min}}\right)\left\Vert \bm{X}^{\star}\right\Vert _{2,\infty}\left\Vert \bm{Y}^{\star}\right\Vert _{2,\infty}\nonumber \\
 & \leq\sqrt{np}\kappa\left(\frac{\sigma}{\sigma_{\min}}\sqrt{\frac{n\log n}{p}}+\frac{\lambda}{p\,\sigma_{\min}}\right)\frac{\mu r\sigma_{\max}}{n},\label{eq:UB-appendix-10}
\end{align}
where the last inequality follows from the upper bound $\max\{\|\bm{X}^{\star}\|_{2,\infty},\|\bm{Y}^{\star}\|_{2,\infty}\}\leq\sqrt{\mu r\sigma_{\max}/n}$
(cf.~(\ref{eq:X-Y-incoherence})). Rearrange the right-hand side
of (\ref{eq:UB-appendix-10}) to reach
\begin{align*}
\left\Vert \mathcal{P}_{\Omega}^{\mathsf{debias}}\left(\bm{X}\bm{Y}^{\top}-\bm{M}^{\star}\right)\right\Vert  & \lesssim\sigma\sqrt{np}\cdot\sqrt{\frac{\kappa^{4}\mu^{2}r^{2}\log n}{np}}+\lambda\sqrt{\frac{\kappa^{4}\mu^{2}r^{2}}{np}}<\lambda/8,
\end{align*}
where the last line holds because of the assumption $n^{2}p\gg\kappa^{4}\mu^{2}r^{2}n\log n$
as well as the choice of $\lambda$.

\section{Analysis of the nonconvex gradient descent algorithm \label{sec:Proof-of-Lemma-nonconvex-GD}}

Lemma~\ref{lemma:nonconvex-GD} shares similar spirit as~\cite[Theorem~2]{ma2017implicit} and~\cite[Lemma~3.5]{chen2019nonconvex}
with one difference: the nonconvex loss function~(\ref{eq:nonconvex_mc_noisy})
has an additional term $\|\bm{X}\|_{\mathrm{F}}^{2}+\|\bm{Y}\|_{\mathrm{F}}^{2}$
to balance the scale of $\bm{X}$ and $\bm{Y}$. To simplify the presentation,
we find it convenient to introduce a few notations. Denote 
\begin{equation}
\bm{F}^{t}\triangleq\left[\begin{array}{c}
\bm{X}^{t}\\
\bm{Y}^{t}
\end{array}\right]\in\mathbb{R}^{2n\times r}\qquad\text{and}\qquad\bm{F}^{\star}\triangleq\left[\begin{array}{c}
\bm{X}^{\star}\\
\bm{Y}^{\star}
\end{array}\right]\in\mathbb{R}^{2n\times r}.\label{eq:defn-Ft}
\end{equation}
It is easily seen from~(\ref{eq:defn-rotation-H}) that 
\begin{equation}
\bm{H}^{t}=\arg\min_{\bm{R}\in\mathcal{O}^{r\times r}}\left\Vert \bm{F}^{t}\bm{R}-\bm{F}^{\star}\right\Vert _{\mathrm{F}}.\label{eq:defn-H-appendix}
\end{equation}
\begin{algorithm}
\caption{Construction of the $l$th leave-one-out sequence.}

\label{alg:gd-mc-LOO}\begin{algorithmic}

\STATE \textbf{{Initialization}}: $\bm{X}^{0,(l)}=\bm{X}^{\star}$;
$\bm{Y}^{0,(l)}=\bm{Y}^{\star}$; Set $\bm{F}^{0,(l)}\triangleq\left[\begin{array}{c}
\bm{X}^{0,(l)}\\
\bm{Y}^{0,(l)}
\end{array}\right]$.

\STATE \textbf{{Gradient updates}}: \textbf{for }$t=0,1,\ldots,t_{0}-1$
\textbf{do}

\STATE \vspace{-1em}
 \begin{subequations} \label{subeq:GD-rules-LOO} 
\[
\bm{F}^{t+1,(l)}\triangleq\left[\begin{array}{c}
\bm{X}^{t+1,(l)}\\
\bm{Y}^{t+1,(l)}
\end{array}\right]=\left[\begin{array}{c}
\bm{X}^{t,(l)}-\eta\nabla_{\bm{X}}f^{(l)}(\bm{X}^{t,(l)},\bm{Y}^{t,(l)})\\
\bm{Y}^{t,(l)}-\eta\nabla_{\bm{Y}}f^{(l)}(\bm{X}^{t,(l)},\bm{Y}^{t,(l)})
\end{array}\right],
\]
\end{subequations}

where $\eta>0$ is the step size.

\end{algorithmic} 
\end{algorithm}

Similar to~\cite{ma2017implicit,chen2019nonconvex}, we resort to
the leave-one-out sequences to control the $\ell_{2}/\ell_{\infty}$
error. Specifically, for each $1\leq l\leq n$ (corresponding to row indices), we construct $\{\bm{F}^{t,(l)}\}_{t\geq0}$
to be the gradient descent iterates (see Algorithm~\ref{alg:gd-mc-LOO})
w.r.t.~the following auxiliary loss function 
\begin{equation}
f^{\left(l\right)}\left(\bm{X},\bm{Y}\right)=\frac{1}{2p}\left\Vert \mathcal{P}_{\Omega_{-l,\cdot}}\left(\bm{X}\bm{Y}^{\top}-\bm{M}\right)\right\Vert _{\mathrm{F}}^{2}+\frac{1}{2}\left\Vert \mathcal{P}_{l,\cdot}\left(\bm{X}\bm{Y}^{\top}-\bm{M}^{\star}\right)\right\Vert _{\mathrm{F}}^{2}+\frac{\lambda}{2p}\left\Vert \bm{X}\right\Vert _{\mathrm{F}}^{2}+\frac{\lambda}{2p}\left\Vert \bm{Y}\right\Vert _{\mathrm{F}}^{2}.\label{eq:defn-fl}
\end{equation}
Here $\mathcal{P}_{\Omega_{-l,\cdot}}(\cdot)$ (resp.~$\mathcal{P}_{l,\cdot}(\cdot)$)
denotes the orthogonal projection onto the space of matrices which
are supported on the index set $\Omega_{-l,\cdot}=\{(i,j)\in\Omega|i\neq l\}$
(resp.~$\{(i,j)|i=l\}$). Mathematically, we have for any matrix
$\bm{B}\in\mathbb{R}^{n\times n}$ 
\begin{equation}
\left[\mathcal{P}_{\Omega_{-l,\cdot}}\left(\bm{B}\right)\right]_{ij}=\begin{cases}
B_{ij}, & \text{if }\left(i,j\right)\in\Omega\text{ and }i\neq l,\\
0, & \text{otherwise}
\end{cases}\quad\text{and}\quad\left[\mathcal{P}_{l,\cdot}\left(\bm{B}\right)\right]_{ij}=\begin{cases}
B_{ij}, & \text{if }i=l,\\
0, & \text{otherwise.}
\end{cases}\label{eq:defn-Pl}
\end{equation}
Similarly, for each $n+1\leq l\leq2n$ (with $l-n$ corresponding to the column index), we define $\{\bm{F}^{t,(l)}\}_{t\geq0}$
to be the GD iterates (see Algorithm~\ref{alg:gd-mc-LOO}) operating
on 
\[
f^{\left(l\right)}\left(\bm{X},\bm{Y}\right)=\frac{1}{2p}\left\Vert \mathcal{P}_{\Omega_{\cdot,-(l-n)}}\left(\bm{X}\bm{Y}^{\top}-\bm{M}\right)\right\Vert _{\mathrm{F}}^{2}+\frac{1}{2}\left\Vert \mathcal{P}_{\cdot,(l-n)}\left(\bm{X}\bm{Y}^{\top}-\bm{M}^{\star}\right)\right\Vert _{\mathrm{F}}^{2}+\frac{\lambda}{2p}\left\Vert \bm{X}\right\Vert _{\mathrm{F}}^{2}+\frac{\lambda}{2p}\left\Vert \bm{Y}\right\Vert _{\mathrm{F}}^{2},
\]
where $\mathcal{P}_{\Omega_{\cdot,-(l-n)}}(\cdot)$ and $\mathcal{P}_{\cdot,(l-n)}(\cdot)$
are defined as 
\[
\left[\mathcal{P}_{\Omega_{\cdot,-(l-n)}}\left(\bm{B}\right)\right]_{ij}=\begin{cases}
B_{ij}, & \text{if }\left(i,j\right)\in\Omega\text{ and }j\neq l-n,\\
0, & \text{otherwise}
\end{cases}\quad\text{and}\quad\left[\mathcal{P}_{\cdot,(l-n)}\left(\bm{B}\right)\right]_{ij}=\begin{cases}
B_{ij}, & \text{if }j=l-n,\\
0, & \text{otherwise},
\end{cases}
\]
for any matrix $\bm{B}\in\mathbb{R}^{n\times n}$. The key ideas are:~(1)~the iterates are not perturbed by much when one drops a small
number of samples (and hence $\bm{F}^{t}$ and $\bm{F}^{t,(l)}$ remain
sufficiently close); (2)~the auxiliary iterates $\bm{F}^{t,(l)}$
are independent of the samples directly related to the $l$th row
of $\bm{M}$, which in turn allows to exploit certain statistical
independence to control the $l$th row of $\bm{F}^{t,(l)}$ (and hence
$\bm{F}^{t}$). See~\cite[Section~5]{ma2017implicit} for a detailed
explanation. Last but not least, the step size is set to be $\eta$,
and we take $\bm{F}^{0,(l)}=\bm{F}^{\star}$ for all $1\leq l\leq2n$
(the same initialization as in Algorithm~\ref{alg:gd-mc-primal}).

With the help of the leave-one-out sequences, we are ready to establish
Lemma~\ref{lemma:nonconvex-GD} in an inductive manner. Concretely
we aim at proving that \begin{subequations}\label{subeq:nonconvex-induction-hypotheses}
\begin{align}
\left\Vert \bm{F}^{t}\bm{H}^{t}-\bm{F}^{\star}\right\Vert _{\mathrm{F}} & \leq C_{\mathrm{F}}\left(\frac{\sigma}{\sigma_{\min}}\sqrt{\frac{n}{p}}+\frac{\lambda}{p\,\sigma_{\min}}\right)\left\Vert \bm{X}^{\star}\right\Vert _{\mathrm{F}},\label{eq:induction-fro}\\
\left\Vert \bm{F}^{t}\bm{H}^{t}-\bm{F}^{\star}\right\Vert  & \leq C_{\mathrm{op}}\left(\frac{\sigma}{\sigma_{\min}}\sqrt{\frac{n}{p}}+\frac{\lambda}{p\,\sigma_{\min}}\right)\left\Vert \bm{X}^{\star}\right\Vert, \label{eq:induction-op}\\
\max_{1\leq l\leq2n}\big\|\bm{F}^{t}\bm{H}^{t}-\bm{F}^{t,(l)}\bm{R}^{t,(l)}\big\|_{\mathrm{F}} & \leq C_{3}\left(\frac{\sigma}{\sigma_{\min}}\sqrt{\frac{n\log n}{p}}+\frac{\lambda}{p\,\sigma_{\min}}\right)\left\Vert \bm{F}^{\star}\right\Vert _{2,\infty},\label{eq:induction-loo-dist}\\
\max_{1\leq l\leq2n}\big\|\big(\bm{F}^{t,(l)}\bm{H}^{t,(l)}-\bm{F}^{\star}\big)_{l,\cdot}\big\|_{2} & \leq C_{4}\kappa\left(\frac{\sigma}{\sigma_{\min}}\sqrt{\frac{n\log n}{p}}+\frac{\lambda}{p\,\sigma_{\min}}\right)\left\Vert \bm{F}^{\star}\right\Vert _{2,\infty},\label{eq:induction-loo-l}\\
\left\Vert \bm{F}^{t}\bm{H}^{t}-\bm{F}^{\star}\right\Vert _{\mathrm{2,\infty}} & \leq C_{\infty}\kappa\left(\frac{\sigma}{\sigma_{\min}}\sqrt{\frac{n\log n}{p}}+\frac{\lambda}{p\,\sigma_{\min}}\right)\left\Vert \bm{F}^{\star}\right\Vert _{2,\infty},\label{eq:induction-2-infty}\\
\left\Vert \bm{X}^{t\top}\bm{X}^{t}-\bm{Y}^{t\top}\bm{Y}^{t}\right\Vert _{\mathrm{F}} & \leq C_{\mathrm{B}}\kappa\eta\left(\frac{\sigma}{\sigma_{\min}}\sqrt{\frac{n}{p}}+\frac{\lambda}{p\,\sigma_{\min}}\right)\sqrt{r}\sigma_{\max}^{2}\label{eq:induction-balancing}
\end{align}
\end{subequations}hold for all $0\leq t\leq t_{0}=n^{18}$ and for
some constants $C_{\mathrm{F}},C_{\mathrm{op}},C_{3},C_{4},C_{\infty},C_{\mathrm{B}}>0$,
provided that $\eta\asymp1/(n\kappa^{3}\sigma_{\max})$. In addition, we also
intend to establish that 
\begin{equation}
f\left(\bm{X}^{t},\bm{Y}^{t}\right)\leq f\left(\bm{X}^{t-1},\bm{Y}^{t-1}\right)-\frac{\eta}{2}\left\Vert \nabla f\left(\bm{X}^{t-1},\bm{Y}^{t-1}\right)\right\Vert _{\mathrm{F}}^{2}\label{eq:induction-function-value}
\end{equation}
holds for all $1\leq t\leq t_{0}=n^{18}$. Here, $\bm{H}^{t,(l)}$
and $\bm{R}^{t,(l)}$ are rotation matrices defined as \begin{subequations}\label{subeq:defn-rotation-new}
\begin{align}
\bm{H}^{t,(l)} & \triangleq\arg\min_{\bm{R}\in\mathcal{O}^{r\times r}}\big\|\bm{F}^{t,(l)}\bm{R}-\bm{F}^{\star}\big\|_{\mathrm{F}};\label{eq:defn-H-t-l}\\
\bm{R}^{t,(l)} & \triangleq\arg\min_{\bm{R}\in\mathcal{O}^{r\times r}}\big\|\bm{F}^{t,(l)}\bm{R}-\bm{F}^{t}\bm{H}^{t}\big\|_{\mathrm{F}}.\label{eq:defn-R-t-l}
\end{align}
\end{subequations}Note that the induction hypotheses~(\ref{eq:induction-fro}),
(\ref{eq:induction-op}) and~(\ref{eq:induction-2-infty}) readily
imply the statements~(\ref{eq:nonconvex-fro-norm}),~(\ref{eq:nonconvex-spectral-norm})
and~(\ref{eq:nonconvex-2-infty-norm}) in Lemma~\ref{lemma:nonconvex-GD},
respectively, whereas the last bound on the size of the gradient~(\ref{eq:nonconvex-small-gradient})
follows from~(\ref{eq:induction-function-value}). We summarize the
last connection in the following lemma, whose proof is in Appendix
\ref{subsec:Proof-of-Lemma-small-gradient}.

\begin{lemma}[\textbf{Small gradient}~(\ref{eq:nonconvex-small-gradient})]\label{lemma:small-gradient-smooth-function}Set $\lambda=C_{\lambda}\sigma\sqrt{np}$
for some large constant $C_{\lambda}>0$. Suppose that the sample size obeys $n^2p\gg\kappa\mu r n \log^2 n$ and that the noise satisfies
$\frac{\sigma}{\sigma_{\min}}\sqrt{\frac{n}{p}}\ll\frac{1}{\sqrt{\kappa^{4}\mu r\log n}}$.
If the induction hypotheses~(\ref{subeq:nonconvex-induction-hypotheses}) hold for all $0
\leq t \leq t_{0}$ and that~(\ref{eq:induction-function-value}) holds for all $1\leq t\leq t_{0}$,
then 
\[
\min_{0\leq t<t_{0}}\left\Vert \nabla f\left(\bm{X}^{t},\bm{Y}^{t}\right)\right\Vert _{\mathrm{F}}\leq\frac{1}{n^{5}}\frac{\lambda}{p}\sqrt{\sigma_{\min}},
\]
as long as $\eta\asymp1/(n\kappa^{3}\sigma_{\max})$.\end{lemma}

The rest of this section is devoted to proving the hypotheses~(\ref{subeq:nonconvex-induction-hypotheses})
and~(\ref{eq:induction-function-value}) via induction. We start with
the base case, i.e.~$t=0$. All the induction hypotheses~(\ref{subeq:nonconvex-induction-hypotheses})
are easily verified by noting that 
\[
\bm{F}^{0}=\bm{F}^{0,(l)}=\bm{F}^{\star},\qquad\text{for all }1\leq l\leq2n.
\]
We now proceed to the induction step, which are demonstrated via the
following lemmas. All the proofs are in subsequent subsections.

\begin{lemma}[\textbf{Frobenius norm error}~\eqref{eq:induction-fro}]
\label{lemma:fro-contraction}Set $\lambda=C_{\lambda}\sigma\sqrt{np}$
for some large constant $C_{\lambda}>0$. Suppose that the sample
size obeys $n^{2}p\gg\kappa\mu rn\log^{2}n$ and the noise satisfies
$\frac{\sigma}{\sigma_{\min}}\sqrt{\frac{n}{p}}\ll\frac{1}{\sqrt{\kappa^{4}\mu r\log n}}$.
If the iterates satisfy~(\ref{subeq:nonconvex-induction-hypotheses})
at the $t$th iteration, then with probability at least $1-O(n^{-100})$,
\begin{align*}
\left\Vert \bm{F}^{t+1}\bm{H}^{t+1}-\bm{F}^{\star}\right\Vert _{\mathrm{F}} & \leq C_{\mathrm{F}}\left(\frac{\sigma}{\sigma_{\min}}\sqrt{\frac{n}{p}}+\frac{\lambda}{p\,\sigma_{\min}}\right)\left\Vert \bm{X}^{\star}\right\Vert _{\mathrm{F}},
\end{align*}
holds as long as $0<\eta\ll1/(\kappa^{5/2}\sigma_{\max})$ and $C_{\mathrm{F}}>0$
is large enough. \end{lemma}

\begin{lemma}[\textbf{Spectral norm error}~\eqref{eq:induction-op}]\label{lemma:operator-contraction}Set
$\lambda=C_{\lambda}\sigma\sqrt{np}$ for some large constant $C_{\lambda}>0$.
Suppose the sample size obeys $n^{2}p\gg\kappa^{4}\mu^{2}r^{2}n\log^2 n$
and the noise satisfies $\frac{\sigma}{\sigma_{\min}}\sqrt{\frac{n}{p}}\ll\frac{1}{\sqrt{\kappa^{4}\log n}}$.
If the iterates satisfy~(\ref{subeq:nonconvex-induction-hypotheses})
at the $t$th iteration, then with probability at least $1-O(n^{-100})$,
\[
\left\Vert \bm{F}^{t+1}\bm{H}^{t+1}-\bm{F}^{\star}\right\Vert \leq C_{\mathrm{op}}\left(\frac{\sigma}{\sigma_{\min}}\sqrt{\frac{n}{p}}+\frac{\lambda}{p\,\sigma_{\min}}\right)\left\Vert \bm{X}^{\star}\right\Vert 
\]
holds with the proviso that $0<\eta\ll1/(\kappa^{3}\sigma_{\max}\sqrt{r})$
and that $C_{\mathrm{op}}\gg1$. \end{lemma}

\begin{lemma}[\textbf{Leave-one-out perturbation}~\eqref{eq:induction-loo-dist}]\label{lemma:loo-dist-contraction}Set
$\lambda=C_{\lambda}\sigma\sqrt{np}$ for some large constant $C_{\lambda}>0$.
Suppose that the sample size satisfies $n^{2}p\gg\kappa^{4}\mu^{2}r^{2}n\log^{3}n$
and that the noise satisfies $\frac{\sigma}{\sigma_{\min}}\sqrt{\frac{n}{p}}\ll\frac{1}{\sqrt{\kappa^{4}\mu r\log n}}$.
If the iterates satisfy~(\ref{subeq:nonconvex-induction-hypotheses})
at the $t$th iteration, then with probability at least $1-O(n^{-99})$,
\[
\max_{1\leq l\leq2n}\big\|\bm{F}^{t+1}\bm{H}^{t+1}-\bm{F}^{t+1,(l)}\bm{R}^{t+1,(l)}\big\|_{\mathrm{F}}\leq C_{3}\left(\frac{\sigma}{\sigma_{\min}}\sqrt{\frac{n\log n}{p}}+\frac{\lambda}{p\,\sigma_{\min}}\right)\left\Vert \bm{F}^{\star}\right\Vert _{2,\infty}
\]
holds, provided that $0<\eta\ll1/(\kappa^{2}\sigma_{\max}n)$ and
that $C_{3}>0$ is some sufficiently large constant. \end{lemma}

\begin{lemma}[\textbf{$\ell_{2}/\ell_{\infty}$ norm error of leave-one-out
sequences}~\eqref{eq:induction-loo-l}]\label{lemma:loo-l-contraction}Set
$\lambda=C_{\lambda}\sigma\sqrt{np}$ for some large constant $C_{\lambda}>0$.
Suppose that the sample size obeys $n^{2}p\gg\kappa^{2}\mu^{2}r^{2}n\log^{3}n$
and that the noise satisfies $\frac{\sigma}{\sigma_{\min}}\sqrt{\frac{n}{p}}\ll\frac{1}{\sqrt{\kappa^{2}\log n}}$.
If the iterates satisfy~(\ref{subeq:nonconvex-induction-hypotheses})
at the $t$th iteration, then with probability at least $1-O(n^{-99})$,
\[
\max_{1\leq l\leq2n}\big\|\big(\bm{F}^{t+1,(l)}\bm{H}^{t+1,(l)}-\bm{F}^{\star}\big)_{l,\cdot}\big\|_{2}\leq C_{4}\kappa\left(\frac{\sigma}{\sigma_{\min}}\sqrt{\frac{n\log n}{p}}+\frac{\lambda}{p\,\sigma_{\min}}\right)\left\Vert \bm{F}^{\star}\right\Vert _{2,\infty}
\]
holds, provided that $0<\eta\ll1/(\kappa^{2}\sqrt{r}\sigma_{\max})$,
$C_{\mathrm{op}}\gg1$ and $C_{4}\gg C_{\mathrm{op}}$. \end{lemma}

\begin{lemma}[\textbf{$\ell_{2}/\ell_{\infty}$ norm error}~\eqref{eq:induction-2-infty}]\label{lemma:2-infty-contraction}Set
$\lambda=C_{\lambda}\sigma\sqrt{np}$ for some large constant $C_{\lambda}>0$.
Suppose that $n\geq\mu r$ and that the noise satisfies $\frac{\sigma}{\sigma_{\min}}\sqrt{\frac{n}{p}}\ll\frac{1}{\sqrt{\kappa^{2}\log n}}$.
If the iterates satisfy~(\ref{subeq:nonconvex-induction-hypotheses})
at the $t$th iteration, then with probability at least $1-O(n^{-99})$,
\begin{align*}
\big\|\bm{F}^{t+1}\bm{H}^{t+1}-\bm{F}^{\star}\big\|_{\mathrm{2,\infty}} & \leq C_{\infty}\kappa\left(\frac{\sigma}{\sigma_{\min}}\sqrt{\frac{n\log n}{p}}+\frac{\lambda}{p\,\sigma_{\min}}\right)\left\Vert \bm{F}^{\star}\right\Vert _{2,\infty},
\end{align*}
holds provided that $C_{\infty}\geq5C_{3}+C_{4}$. \end{lemma}

\begin{lemma}[\textbf{Approximate balancedness}~\eqref{eq:induction-balancing}]\label{lemma:balancing}Set
$\lambda=C_{\lambda}\sigma\sqrt{np}$ for some large constant $C_{\lambda}>0$.
Suppose that the sample size satisfies $n^{2}p\gg\kappa^{2}\mu^{2}r^{2}n\log n$
and that the noise satisfies $\frac{\sigma}{\sigma_{\min}}\sqrt{\frac{n}{p}}\ll\frac{1}{\sqrt{\kappa^{2}\log n}}$.
If the iterates satisfy~(\ref{subeq:nonconvex-induction-hypotheses})
at the $t$th iteration, then with probability at least $1-O(n^{-100})$,
\begin{align*}
\left\Vert \bm{X}^{t+1\top}\bm{X}^{t+1}-\bm{Y}^{t+1\top}\bm{Y}^{t+1}\right\Vert _{\mathrm{F}} & \leq C_{\mathrm{B}}\kappa\eta\left(\frac{\sigma}{\sigma_{\min}}\sqrt{\frac{n}{p}}+\frac{\lambda}{p\,\sigma_{\min}}\right)\sqrt{r}\sigma_{\max}^{2},\\
\max_{1\leq l\leq2n}\left\Vert \bm{X}^{t+1,(l)\top}\bm{X}^{t+1,(l)}-\bm{Y}^{t+1,(l)\top}\bm{Y}^{t+1,(l)}\right\Vert _{\mathrm{F}} & \leq C_{\mathrm{B}}\kappa\eta\left(\frac{\sigma}{\sigma_{\min}}\sqrt{\frac{n}{p}}+\frac{\lambda}{p\,\sigma_{\min}}\right)\sqrt{r}\sigma_{\max}^{2},
\end{align*}
holds for some sufficiently large constant $C_{\mathrm{B}}\gg C_{\mathrm{op}}^{2}$,
provided that $0<\eta<1/\sigma_{\min}$. \end{lemma}

\begin{lemma}[\textbf{Decreasing of function values }\eqref{eq:induction-function-value}]\label{lemma:function-value-decreasing}Set
$\lambda=C_{\lambda}\sigma\sqrt{np}$ for some large constant $C_{\lambda}>0$.
Suppose that the noise satisfies $\frac{\sigma}{\sigma_{\min}}\sqrt{\frac{n}{p}}\ll1/\sqrt{r}$.
If the iterates satisfy~(\ref{subeq:nonconvex-induction-hypotheses})
at the $t$th iteration, then with probability at least $1-O(n^{-99})$,
\[
f\left(\bm{X}^{t+1},\bm{Y}^{t+1}\right)\leq f\left(\bm{X}^{t},\bm{Y}^{t}\right)-\frac{\eta}{2}\left\Vert \nabla f\left(\bm{X}^{t},\bm{Y}^{t}\right)\right\Vert _{\mathrm{F}}^{2},
\]
as long as $\eta\ll1/(\kappa n\sigma_{\max})$. \end{lemma}

\subsection{Preliminaries and notations}

Before proceeding to the proofs, we collect a few useful facts and
notations. To begin with, for any matrix $\bm{A}$, we denote by $\bm{A}_{l,\cdot}$
(resp. $\bm{A}_{\cdot,l}$) the $l$th row (reps. column) of $\bm{A}$.

Define an augmented loss function $f_{\mathsf{aug}}(\bm{X},\bm{Y})$
to be

\begin{equation}
f_{\mathsf{aug}}\left(\bm{X},\bm{Y}\right)\triangleq\frac{1}{2p}\left\Vert \mathcal{P}_{\Omega}\left(\bm{X}\bm{Y}^{\top}-\bm{M}\right)\right\Vert _{\mathrm{F}}^{2}+\frac{\lambda}{2p}\left\Vert \bm{X}\right\Vert _{\mathrm{F}}^{2}+\frac{\lambda}{2p}\left\Vert \bm{Y}\right\Vert _{\mathrm{F}}^{2}+\frac{1}{8}\left\Vert \bm{X}^{\top}\bm{X}-\bm{Y}^{\top}\bm{Y}\right\Vert _{\mathrm{F}}^{2}.\label{eq:defn-faug}
\end{equation}
As the name suggests, this new function augments the original loss
function~(cf.~(\ref{eq:nonconvex_mc_noisy})) with an additional
term $\|\bm{X}^{\top}\bm{X}-\bm{Y}^{\top}\bm{Y}\|_{\mathrm{F}}^{2}/8$,
which is commonly used in the literature of asymmetric low-rank matrix
factorization to balance the scale of $\bm{X}$ and $\bm{Y}$~\cite{tu2016low,yi2016fast,chen2019nonconvex}.
We emphasize that, in contrast to aforementioned works, here our gradient
descent algorithm~(cf.~Algorithm~\ref{alg:gd-mc-primal}) operates
on $f(\cdot,\cdot)$ instead of $f_{\mathsf{aug}}(\cdot,\cdot)$.
The introduction of $f_{\mathsf{aug}}(\cdot,\cdot)$ is mainly to
simplify the proof.

It is easily seen that the gradients of $f_{\mathsf{aug}}(\cdot,\cdot)$
are given by\begin{subequations}\label{subeq:defn-nabla-faug} 
\begin{align}
\nabla_{\bX}f_{\mathsf{aug}}(\bX,\bY) & =\frac{1}{p}\cP_{\Omega}\left(\bX\bY^{\top}-\bm{M}\right)\bY+\frac{\lambda}{p}\bX+\frac{1}{2}\bX\left(\bX^{\top}\bX-\bY^{\top}\bY\right);\\
\nabla_{\bY}f_{\mathsf{aug}}(\bX,\bY) & =\frac{1}{p}\cP_{\Omega}\left(\bX\bY^{\top}-\bm{M}\right)^{\top}\bX+\frac{\lambda}{p}\bY+\frac{1}{2}\bY\left(\bY^{\top}\bY-\bX^{\top}\bX\right).
\end{align}
\end{subequations}Correspondingly, define the difference between
gradients of $\nabla f(\bm{X},\bm{Y})$ and $\nabla f_{\mathsf{aug}}(\bm{X},\bm{Y})$
as follows \begin{subequations} \label{subeq:defn-nabla-fdiff} 
\begin{align}
\nabla_{\bX}f_{\mathsf{diff}}(\bX,\bY) & =-\bX\left(\bX^{\top}\bX-\bY^{\top}\bY\right)/2;\\
\nabla_{\bY}f_{\mathsf{diff}}(\bX,\bY) & =-\bY\left(\bY^{\top}\bY-\bX^{\top}\bX\right)/2,
\end{align}
\end{subequations}such that \begin{subequations}\label{subeq:gradient-decomposition}
\begin{align}
\nabla_{\bm{X}}f\left(\bm{X},\bm{Y}\right) & =\nabla_{\bm{X}}f_{\mathsf{aug}}\left(\bm{X},\bm{Y}\right)+\nabla_{\bm{X}}f_{\mathsf{diff}}\left(\bm{X},\bm{Y}\right);\\
\nabla_{\bm{Y}}f\left(\bm{X},\bm{Y}\right) & =\nabla_{\bm{Y}}f_{\mathsf{aug}}\left(\bm{X},\bm{Y}\right)+\nabla_{\bm{Y}}f_{\mathsf{diff}}\left(\bm{X},\bm{Y}\right).
\end{align}
\end{subequations}

Regarding $\bm{F}^{\star}$, simple algebra reveals that \begin{subequations}\label{subeq:F-property}
\begin{align}
\sigma_{1}\left(\bm{F}^{\star}\right) & =\left\Vert \bm{F}^{\star}\right\Vert =\sqrt{2\sigma_{\max}},\qquad\sigma_{r}\left(\bm{F}^{\star}\right)=\sqrt{2\sigma_{\min}},\label{eq:F-singular-value}\\
\left\Vert \bm{F}^{\star}\right\Vert _{2,\infty} & =\max\big\{\left\Vert \bm{X}^{\star}\right\Vert _{2,\infty},\left\Vert \bm{Y}^{\star}\right\Vert _{2,\infty}\big\}\leq\sqrt{\mu r\sigma_{\max}/n},\label{eq:F-incoherence}
\end{align}
\end{subequations}where the last one follows from the incoherence
assumption~(\ref{eq:X-Y-incoherence}).

We start with a lemma that characterizes the local geometry of the nonconvex loss function, whose proof is given in Appendix~\ref{subsec:Proof-of-Lemma-local-geometry}.

\begin{lemma}\label{lemma:hessian}Set $\lambda=C_{\lambda}\sigma\sqrt{np}$
for some constant $C_{\lambda}>0$. Suppose that the sample size obeys
$n^{2}p\geq C\kappa\mu rn\log^2 n$ for some sufficiently large constant
$C>0$ and that the noise satisfies $\frac{\sigma}{\sigma_{\min}}\sqrt{\frac{n}{p}}\ll1$.
Recall the function $f_{\mathsf{aug}}(\cdot, \cdot)$ defined in~(\ref{eq:defn-faug}).
Then with probability at least $1-O(n^{-10})$, 
\begin{align*}
\mathsf{vec}\left(\bm{\Delta}\right)^{\top}\nabla^{2}f_{\mathsf{aug}}\left(\bm{X},\bm{Y}\right)\mathsf{vec}\left(\bm{\Delta}\right) & \geq\tfrac{1}{10}\sigma_{\min}\left\Vert \bm{\Delta}\right\Vert _{\mathrm{F}}^{2},\\
\max\left\{ \left\Vert \nabla^{2}f_{\mathsf{aug}}\left(\bm{X},\bm{Y}\right)\right\Vert ,\left\Vert \nabla^{2}f\left(\bm{X},\bm{Y}\right)\right\Vert \right\}  & \leq10\sigma_{\max}
\end{align*}
hold uniformly over all $\bm{X},\bm{Y}\in\mathbb{R}^{n\times r}$
obeying
\[
\left\Vert \left[\begin{array}{c}
\bm{X}-\bm{X}^{\star}\\
\bm{Y}-\bm{Y}^{\star}
\end{array}\right]\right\Vert _{2,\infty}\leq\frac{1}{1000\kappa\sqrt{n}}\left\Vert \bm{X}^{\star}\right\Vert 
\]
and all $\bm{\Delta}=\left[\begin{array}{c}
\bm{\Delta}_{\bm{X}}\\
\bm{\Delta}_{\bm{Y}}
\end{array}\right]\in\mathbb{R}^{2n\times r}$ lying in the set 
\[
\left\{ \left.\left[\begin{array}{c}
\bm{X}_{1}\\
\bm{Y}_{1}
\end{array}\right]\hat{\bm{H}}-\left[\begin{array}{c}
\bm{X}_{2}\\
\bm{Y}_{2}
\end{array}\right]\,\right|\,\,\left\Vert \left[\begin{array}{c}
\bm{X}_{2}-\bm{X}^{\star}\\
\bm{Y}_{2}-\bm{Y}^{\star}
\end{array}\right]\right\Vert \leq\frac{1}{500\kappa}\left\Vert \bm{X}^{\star}\right\Vert ,\hat{\bm{H}}\triangleq\arg\min_{\bm{R}\in\mathcal{O}^{r\times r}}\left\Vert \left[\begin{array}{c}
\bm{X}_{1}\\
\bm{Y}_{1}
\end{array}\right]\bm{R}-\left[\begin{array}{c}
\bm{X}_{2}\\
\bm{Y}_{2}
\end{array}\right]\right\Vert _{\mathrm{F}}\right\} .
\]
\end{lemma}

Last but not least, a few immediate consequences of~(\ref{subeq:nonconvex-induction-hypotheses})
are gathered in the following lemma, whose proof is given in Appendix
\ref{subsec:Proof-of-Lemma-immediate-consequence}.

\begin{lemma}\label{lemma:immediate-consequence}We have the following
four sets of consequences of the induction hypotheses~(\ref{subeq:induction_GD}). 
\begin{enumerate}
\item Suppose that the sample size obeys $n\gg\mu r\log n$. If the $t$th
iterates obey~(\ref{subeq:nonconvex-induction-hypotheses}), then
one has \begin{subequations}\label{subeq:immediate-F-t-l-R-t-l}
\begin{align}
\left\Vert \bm{F}^{t,(l)}\bm{R}^{t,(l)}-\bm{F}^{\star}\right\Vert _{2,\infty} & \leq\left(C_{\infty}\kappa+C_{3}\right)\left(\frac{\sigma}{\sigma_{\min}}\sqrt{\frac{n\log n}{p}}+\frac{\lambda}{p\,\sigma_{\min}}\right)\left\Vert \bm{F}^{\star}\right\Vert _{2,\infty},\label{eq:immediate-F-t-l-R-t-l-2-infty}\\
\left\Vert \bm{F}^{t,(l)}\bm{R}^{t,(l)}-\bm{F}^{\star}\right\Vert  & \leq2C_{\mathrm{op}}\left(\frac{\sigma}{\sigma_{\min}}\sqrt{\frac{n}{p}}+\frac{\lambda}{p\,\sigma_{\min}}\right)\left\Vert \bm{X}^{\star}\right\Vert .\label{eq:immediate-F-t-l-R-t-l-2-op}
\end{align}
\end{subequations} 
\item Suppose that the noise satisfies $\frac{\sigma}{\sigma_{\min}}\sqrt{\frac{n}{p}}\ll\frac{1}{\sqrt{\kappa^{2}\log n}}$.
If the $t$th iterates obey~(\ref{subeq:nonconvex-induction-hypotheses}),
then one has \begin{subequations}\label{subeq:immediate-consequence-F-t}
\begin{align}
\left\Vert \bm{F}^{t}\bm{H}^{t}-\bm{F}^{\star}\right\Vert \leq\left\Vert \bm{X}^{\star}\right\Vert ,\quad\left\Vert \bm{F}^{t}\bm{H}^{t}-\bm{F}^{\star}\right\Vert _{\mathrm{F}} & \leq\left\Vert \bm{X}^{\star}\right\Vert _{\mathrm{F}},\quad\left\Vert \bm{F}^{t}\bm{H}^{t}-\bm{F}^{\star}\right\Vert _{\mathrm{2,\infty}}\leq\left\Vert \bm{F}^{\star}\right\Vert _{2,\infty},\label{eq:immediate-consequence-F-t-diff}\\
\left\Vert \bm{F}^{t}\right\Vert \leq2\left\Vert \bm{X}^{\star}\right\Vert ,\quad\left\Vert \bm{F}^{t}\right\Vert _{\mathrm{F}} & \leq2\left\Vert \bm{X}^{\star}\right\Vert _{\mathrm{F}},\quad\left\Vert \bm{F}^{t}\right\Vert _{\mathrm{2,\infty}}\leq2\left\Vert \bm{F}^{\star}\right\Vert _{2,\infty}.\label{eq:immediate-consequence-F-t-itself}
\end{align}
\end{subequations} 
\item Suppose that $n\gg\kappa^{2}\mu r\log n$ and that $\frac{\sigma}{\sigma_{\min}}\sqrt{\frac{n}{p}}\ll\frac{1}{\sqrt{\kappa^{2}\log n}}$.
If the $t$th iterates obey~(\ref{subeq:nonconvex-induction-hypotheses}),
then we have 
\begin{align*}
\big\|\bm{F}^{t}\bm{H}^{t}-\bm{F}^{t,(l)}\bm{H}^{t,(l)}\big\|_{\mathrm{F}} & \leq5\kappa\big\|\bm{F}^{t}\bm{H}^{t}-\bm{F}^{t,(l)}\bm{R}^{t,(l)}\big\|_{\mathrm{F}}.
\end{align*}
\item Suppose that $n\geq\kappa\mu$ and that $\frac{\sigma}{\sigma_{\min}}\sqrt{\frac{n}{p}}\ll\frac{1}{\sqrt{\kappa^{2}\log n}}$.
If the $t$th iterates obey~(\ref{subeq:nonconvex-induction-hypotheses}),
then~(\ref{subeq:immediate-consequence-F-t}) also holds for $\bm{F}^{t,(l)}\bm{H}^{t,(l)}$.
In addition, one has 
\[
\sigma_{\min}/2\leq\sigma_{\min}\left((\bm{Y}^{t,(l)}\bm{H}^{t,(l)})^{\top}\bm{Y}^{t,(l)}\bm{H}^{t,(l)}\right)\leq\sigma_{\max}\left((\bm{Y}^{t,(l)}\bm{H}^{t,(l)})^{\top}\bm{Y}^{t,(l)}\bm{H}^{t,(l)}\right)\leq2\sigma_{\max}.
\]
\end{enumerate}
\end{lemma}

\subsection{Proof of Lemma~\ref{lemma:small-gradient-smooth-function} \label{subsec:Proof-of-Lemma-small-gradient}}

Summing~(\ref{eq:induction-function-value}) from $t=1$ to $t=t_{0}$
leads to a telescopic sum 
\begin{align*}
f\left(\bm{X}^{t_{0}},\bm{Y}^{t_{0}}\right) & \leq f\left(\bm{X}^{0},\bm{Y}^{0}\right)-\frac{\eta}{2}\sum_{t=0}^{t_{0}-1}\left\Vert \nabla f\left(\bm{X}^{t},\bm{Y}^{t}\right)\right\Vert _{\mathrm{F}}^{2}.
\end{align*}
This further implies that 
\begin{equation}
\min_{0\leq t<t_{0}}\left\Vert \nabla f\left(\bm{X}^{t},\bm{Y}^{t}\right)\right\Vert _{\mathrm{F}}\leq \left\{ \frac{1}{t_0} \sum_{t=0}^{t_{0}-1}\left\Vert \nabla f\left(\bm{X}^{t},\bm{Y}^{t}\right)\right\Vert _{\mathrm{F}}^{2}\right\}^{1/2} \leq \left\{ \frac{2}{\eta t_{0}}\left[f\left(\bm{X}^{\star},\bm{Y}^{\star}\right)-f\left(\bm{X}^{t_{0}},\bm{Y}^{t_{0}}\right)\right]\right\} ^{1/2},\label{eq:smooth-gradient}
\end{equation}
where we have used the assumption that $(\bm{X}^{0},\bm{Y}^{0})=(\bm{X}^{\star},\bm{Y}^{\star})$.

It remains to control $f(\bm{X}^{\star},\bm{Y}^{\star})-f(\bm{X}^{t_{0}},\bm{Y}^{t_{0}})$.
Towards this end, we can use the fact that $f(\bm{X},\bm{Y})=f(\bm{X}\bm{R},\bm{Y}\bm{R})$
for any $\bm{R}\in\mathcal{O}^{r\times r}$ to obtain 
\[
f\left(\bm{F}^{t_{0}}\right)=f\left(\bm{F}^{t_{0}}\bm{H}^{t_{0}}\right)=f\left(\bm{F}^{\star}\right)+\left\langle \nabla f\left(\bm{F}^{\star}\right),\bm{F}^{t_{0}}\bm{H}^{t_{0}}-\bm{F}^{\star}\right\rangle +\frac{1}{2}\mathsf{vec}\left(\bm{F}^{t_{0}}\bm{H}^{t_{0}}-\bm{F}^{\star}\right)^{\top}\nabla^{2}f\big(\tilde{\bm{F}}\big)\mathsf{vec}\left(\bm{F}^{t_{0}}\bm{H}^{t_{0}}-\bm{F}^{\star}\right),
\]
where $\tilde{\bm{F}}$ lies in the line segment connecting $\bm{F}^{t_{0}}\bm{H}^{t_{0}}$
and $\bm{F}^{\star}$. Apply the triangle inequality to see 
\begin{align*}
f\left(\bm{F}^{\star}\right)-f\left(\bm{F}^{t_{0}}\right) & \leq\left\Vert \nabla f\left(\bm{F}^{\star}\right)\right\Vert _{\mathrm{F}}\left\Vert \bm{F}^{t_{0}}\bm{H}^{t_{0}}-\bm{F}^{\star}\right\Vert _{\mathrm{F}}-\frac{1}{2}\mathsf{vec}\left(\bm{F}^{t_{0}}\bm{H}^{t_{0}}-\bm{F}^{\star}\right)^{\top}\nabla^{2}f\big(\tilde{\bm{F}}\big)\mathsf{vec}\left(\bm{F}^{t_{0}}\bm{H}^{t_{0}}-\bm{F}^{\star}\right)\\
 & \leq\left\Vert \nabla f\left(\bm{F}^{\star}\right)\right\Vert _{\mathrm{F}}\left\Vert \bm{F}^{t_{0}}\bm{H}^{t_{0}}-\bm{F}^{\star}\right\Vert _{\mathrm{F}}+5\sigma_{\max}\left\Vert \bm{F}^{t_{0}}\bm{H}^{t_{0}}-\bm{F}^{\star}\right\Vert _{\mathrm{F}}^{2}.
\end{align*}
Here the second line follows from the fact that $\|\nabla^{2}f(\tilde{\bm{F}})\|\leq10\sigma_{\max}$. To see this, use~\eqref{eq:induction-2-infty} to obtain that
\begin{align}
\big\| \tilde{\bm{F}}-\bm{F}^{\star}\big\| _{2,\infty} & \leq\left\Vert \bm{F}^{t_0}\bm{H}^{t_0}-\bm{F}^{\star}\right\Vert _{2,\infty}\leq C_{\infty}\kappa\left(\frac{\sigma}{\sigma_{\min}}\sqrt{\frac{n\log n}{p}}+\frac{\lambda}{p\,\sigma_{\min}}\right)\left\Vert \bm{F}^{\star}\right\Vert _{2,\infty} \nonumber\\
 & \leq C_{\infty}\kappa\left(\frac{\sigma}{\sigma_{\min}}\sqrt{\frac{n\log n}{p}}+\frac{\lambda}{p\,\sigma_{\min}}\right)\sqrt{\frac{\mu r}{n}}\sqrt{\sigma_{\max}} \nonumber\\
 & \leq\frac{1}{2000\kappa\sqrt{n}}\sqrt{\sigma_{\max}}\label{eq:2-infty-line-segment},
\end{align}
where the second line arises from the incoherence assumption~(\ref{eq:F-incoherence})
and the last inequality holds as long as $\lambda \asymp \sigma\sqrt{np}$ and $\frac{\sigma}{\sigma_{\min}}\sqrt{\frac{n}{p}}\ll\frac{1}{\sqrt{\kappa^{4}\mu r\log n}}$. Apply Lemma~\ref{lemma:hessian} to conclude that $\|\nabla^{2}f(\tilde{\bm{F}})\|\leq10\sigma_{\max}$. Recognize that 
\begin{align}
\left\Vert \nabla f\left(\bm{F}^{\star}\right)\right\Vert _{\mathrm{F}} & \leq\left\Vert \nabla_{\bm{X}}f\left(\bm{F}^{\star}\right)\right\Vert _{\mathrm{F}}+\left\Vert \nabla_{\bm{Y}}f\left(\bm{F}^{\star}\right)\right\Vert _{\mathrm{F}}\nonumber \\
 & \leq\frac{1}{p}\left\Vert \mathcal{P}_{\Omega}\left(\bm{E}\right)\bm{Y}^{\star}\right\Vert _{\mathrm{F}}+\frac{\lambda}{p}\left\Vert \bm{X}^{\star}\right\Vert _{\mathrm{F}}+\frac{1}{p}\big\|\mathcal{P}_{\Omega}\left(\bm{E}\right)^{\top}\bm{X}^{\star}\big\|_{\mathrm{F}}+\frac{\lambda}{p}\left\Vert \bm{Y}^{\star}\right\Vert _{\mathrm{F}}\nonumber \\
 & \leq\left(\frac{1}{p}\left\Vert \mathcal{P}_{\Omega}\left(\bm{E}\right)\right\Vert +\frac{\lambda}{p}\right)\left(\left\Vert \bm{X}^{\star}\right\Vert _{\mathrm{F}}+\left\Vert \bm{Y}^{\star}\right\Vert _{\mathrm{F}}\right),\label{eq:grad-f-Fstar-1}
\end{align}
where we have used the fact that $\nabla_{\bm{X}}f(\bm{F}^{\star})=\frac{1}{p}\mathcal{P}_{\Omega}(\bm{X}^{\star}\bm{Y}^{\star\top}-\bm{M}^{\star}-\bm{E})\bm{Y}^{\star}+\frac{\lambda}{p}\bm{X}^{\star}=-\frac{1}{p}\mathcal{P}_{\Omega}(\bm{E})\bm{Y}^{\star}+\frac{\lambda}{p}\bm{X}^{\star}$ (similar expression holds true for $\nabla_{\bm{Y}}f(\bm{F}^{\star})$). 
This together with Lemma~\ref{lemma:noise-bound} and the assumption
that $\lambda\asymp\sigma\sqrt{np}$ yields 
\begin{equation}
\left\Vert \nabla f\left(\bm{F}^{\star}\right)\right\Vert _{\mathrm{F}}\lesssim\left(\sigma\sqrt{\frac{n}{p}}+\frac{\lambda}{p}\right)\sqrt{r\sigma_{\max}}\asymp\frac{\lambda}{p}\sqrt{r\sigma_{\max}}.\label{eq:grad-f-Fstar}
\end{equation}
The above bounds together with the induction hypothesis~(\ref{eq:induction-fro})
for $t=t_{0}$ give 
\begin{align*}
f\left(\bm{F}^{\star}\right)-f\left(\bm{F}^{t_{0}}\right) & \lesssim\frac{\lambda}{p}\sqrt{r\sigma_{\max}}\left(\frac{\sigma}{\sigma_{\min}}\sqrt{\frac{n}{p}}+\frac{\lambda}{p\,\sigma_{\min}}\right)\left\Vert \bm{X}^{\star}\right\Vert _{\mathrm{F}}+\sigma_{\max}\left(\frac{\sigma}{\sigma_{\min}}\sqrt{\frac{n}{p}}+\frac{\lambda}{p\,\sigma_{\min}}\right)^{2}\left\Vert \bm{X}^{\star}\right\Vert _{\mathrm{F}}^{2}\\
 & \lesssim r\kappa^{2}\left(\frac{\lambda}{p}\right)^{2},
\end{align*}
where the last relation arises from $\sigma\sqrt{np}\asymp\lambda$. Substitution
into~(\ref{eq:smooth-gradient}) results in 
\begin{align*}
\min_{0\leq t<t_{0}}\left\Vert \nabla f\left(\bm{X}^{t},\bm{Y}^{t}\right)\right\Vert _{\mathrm{F}} & \lesssim\sqrt{\frac{1}{\eta t_{0}}r\kappa^{2}\left(\frac{\lambda}{p}\right)^{2}}\leq\frac{1}{n^{5}}\frac{\lambda}{p}\sqrt{\sigma_{\min}},
\end{align*}
provided that $\eta\asymp1/(n\kappa^{3}\sigma_{\max})$, $t_{0}=n^{18}$
and that $n\geq\kappa$, which is a consequence of our sample complexity $n\geq np \gg \kappa \mu r \log^2 n$.

\subsection{Proof of Lemma~\ref{lemma:fro-contraction}}

From the definitions of $\bH^{t+1}$~(cf.$\,$(\ref{eq:defn-H-appendix})),
$\nabla f_{\mathsf{aug}}$~(cf.$\,$(\ref{subeq:defn-nabla-faug}))
and $\nabla f_{\mathsf{diff}}$~(cf.$\,$(\ref{subeq:defn-nabla-fdiff})),
we have 
\[
\begin{aligned} & \Fnorm{\bF^{t+1}\bH^{t+1}-\bF^{\star}}\leq\Fnorm{\bF^{t+1}\bH^{t}-\bF^{\star}}=\Fnorm{\left[\bF^{t}-\eta\nabla f\left(\bF^{t}\right)\right]\bH^{t}-\bF^{\star}}\\
 & \quad\overset{(\text{i})}{=}\Fnorm{\bF^{t}\bH^{t}-\eta\nabla f\left(\bF^{t}\bH^{t}\right)-\bF^{\star}}\\
 & \quad\overset{(\text{ii})}{\leq}\underbrace{\Fnorm{\bF^{t}\bH^{t}-\eta\nabla f_{\mathsf{aug}}\left(\bF^{t}\bH^{t}\right)-\left[\bF^{\star}-\eta\nabla f_{\mathsf{aug}}\left(\bF^{\star}\right)\right]}}_{:=\alpha_{1}}+\underbrace{\eta\Fnorm{\nabla f_{\mathsf{diff}}\left(\bF^{t}\bH^{t}\right)}}_{:=\alpha_{2}}+\underbrace{\eta\left\Vert \nabla f_{\mathsf{aug}}\left(\bm{F}^{\star}\right)\right\Vert _{\mathrm{F}}}_{:=\alpha_{3}}.
\end{aligned}
\]
Here (i) uses the fact that $\nabla f(\bm{F}\bm{R})=\nabla f(\bm{F})\bm{R}$
for all $\bm{R}\in\mathcal{O}^{r\times r}$; the last relation (ii)
uses the decomposition~(\ref{subeq:gradient-decomposition}) and the
triangle inequality. In the following, we bound $\alpha_{1},\alpha_{2}$
and $\alpha_{3}$ in the reverse order. 
\begin{enumerate}
\item First, regarding $\alpha_{3}$, since $\bm{X}^{\star\top}\bm{X}^{\star}=\bm{Y}^{\star\top}\bm{Y}^{\star}$,
one has $\eta\|\nabla f(\bm{F}^{\star})\|_{\mathrm{F}}=\eta\|\nabla f_{\mathsf{aug}}(\bm{F}^{\star})\|_{\mathrm{F}}$.
Repeating our arguments for~(\ref{eq:grad-f-Fstar-1}) and~(\ref{eq:grad-f-Fstar})
gives 
\begin{align*}
\alpha_{3} & =\eta\left\Vert \nabla f\left(\bm{F}^{\star}\right)\right\Vert _{\mathrm{F}}\leq4\eta\frac{\lambda}{p}\left\Vert \bm{X}^{\star}\right\Vert _{\mathrm{F}}
\end{align*}
as long as $\lambda\asymp\sigma\sqrt{np}$. Here the last inequality also relies on the fact that $\|\bm{X}^{\star}\|_{\mathrm{F}}=\|\bm{Y}^{\star}\|_{\mathrm{F}}$. 
\item We now move on to $\alpha_{2}$, for which one has 
\begin{align*}
\alpha_{2} & \leq\frac{\eta}{2}\left(\left\Vert \bm{X}^{t}\left(\bX^{t\top}\bX^{t}-\bY^{t\top}\bY^{t}\right)\bm{H}^{t}\right\Vert _{\mathrm{F}}+\left\Vert \bm{Y}^{t}\left(\bY^{t\top}\bY^{t}-\bX^{t\top}\bX^{t}\right)\bm{H}^{t}\right\Vert _{\mathrm{F}}\right)\\
 & \leq\frac{\eta}{2}\left(\left\Vert \bm{X}^{t}\right\Vert +\left\Vert \bm{Y}^{t}\right\Vert \right)\left\Vert \bX^{t\top}\bX^{t}-\bY^{t\top}\bY^{t}\right\Vert _{\mathrm{F}}.
\end{align*}
Utilize the fact that $\max\{\|\bm{X}^{t}\|,\|\bm{Y}^{t}\|\}\leq\|\bm{F}^{\star}\|\leq2\|\bm{X}^{\star}\|$
(see Lemma~\ref{lemma:immediate-consequence}) and the induction hypothesis~(\ref{eq:induction-balancing}) to obtain 
\begin{align*}
\alpha_{2} & \leq2\eta\sqrt{\sigma_{\max}}\cdot C_{\mathrm{B}}\kappa\eta\left(\frac{\sigma}{\sigma_{\min}}\sqrt{\frac{n}{p}}+\frac{\lambda}{p\,\sigma_{\min}}\right)\sqrt{r}\sigma_{\max}^{2}\\
 & \leq\big(2C_{\mathrm{B}}\kappa^{5/2}\eta\sigma_{\max}\big)\sigma_{\min}\eta\left(\frac{\sigma}{\sigma_{\min}}\sqrt{\frac{n}{p}}+\frac{\lambda}{p\,\sigma_{\min}}\right)\left\Vert \bm{X}^{\star}\right\Vert _{\mathrm{F}}\\
 & \leq\sigma_{\min}\eta\left(\frac{\sigma}{\sigma_{\min}}\sqrt{\frac{n}{p}}+\frac{\lambda}{p\,\sigma_{\min}}\right)\left\Vert \bm{X}^{\star}\right\Vert _{\mathrm{F}},
\end{align*}
where the second inequality uses $\|\bm{X}^{\star}\|_{\mathrm{F}}\geq\sqrt{r\sigma_{\min}}$
and the last one holds as long as $2C_{\mathrm{B}}\kappa^{5/2}\sigma_{\max}\eta\leq1$. 
\item In the end, for $\alpha_{1}$, the fundamental theorem of calculus
\cite[Chapter XIII, Theorem 4.2]{lang1993real} reveals that 
\begin{align}
& \mathsf{vec}\left[\bF^{t}\bH^{t}-\eta\nabla f_{\mathsf{aug}}\left(\bF^{t}\bH^{t}\right)-\left[\bF^{\star}-\eta\nabla f_{\mathsf{aug}}\left(\bF^{\star}\right)\right]\right]\nonumber\\
 & \quad=\mathsf{vec}\left[\bF^{t}\bH^{t}-\bF^{\star}\right]-\eta\cdot\mathsf{vec}\left[\nabla f_{\mathsf{aug}}\left(\bF^{t}\bH^{t}\right)-\nabla f_{\mathsf{aug}}\left(\bF^{\star}\right)\right]\nonumber\\
 & \quad=\bigg(\bI_{2nr}-\eta\underbrace{\int_{0}^{1}\nabla^{2}f_{\mathsf{aug}}\left(\bF(\tau)\right)\mathrm{d}\tau}_{:=\bA}\bigg)\mathsf{vec}\left(\bF^{t}\bH^{t}-\bF^{\star}\right)\label{eq:gradient-MVT-MC},
\end{align}
where we denote $\bm{F}(\tau)\triangleq\bm{F}^{\star}+\tau(\bm{F}^{t}{\bm{H}}^{t}-\bm{F}^{\star})$
for all $0\leq\tau\leq1$. Taking the squared Euclidean norm of both
sides of the equality~(\ref{eq:gradient-MVT-MC}) leads to 
\begin{align}
\alpha_{1}^{2} & =\mathsf{vec}\left(\bm{F}^{t}{\bm{H}}^{t}-\bm{F}^{\star}\right)^{\top}\left(\bm{I}_{2nr}-\eta\bm{A}\right)^{2}\mathsf{vec}\left(\bm{F}^{t}{\bm{H}}^{t}-\bm{F}^{\star}\right)\nonumber \\
 & =\mathsf{vec}\left(\bm{F}^{t}{\bm{H}}^{t}-\bm{F}^{\star}\right)^{\top}\left(\bm{I}_{2nr}-2\eta\bm{A}+\eta^{2}\bm{A}^{2}\right)\mathsf{vec}\left(\bm{F}^{t}{\bm{H}}^{t}-\bm{F}^{\star}\right)\nonumber \\
 & \leq\Fnorm{\bm{F}^{t}{\bm{H}}^{t}-\bm{F}^{\star}}^{2}+\eta^{2}\norm{\bm{A}}^{2}\Fnorm{\bm{F}^{t}{\bm{H}}^{t}-\bm{F}^{\star}}^{2}-2\eta\mathsf{vec}\left(\bm{F}^{t}{\bm{H}}^{t}-\bm{F}^{\star}\right)^{\top}\bm{A}\text{ }\mathsf{vec}\left(\bm{F}^{t}{\bm{H}}^{t}-\bm{F}^{\star}\right),\label{eq:v^A^2v-MC}
\end{align}
where~(\ref{eq:v^A^2v-MC}) results from the fact that 
\[
\mathsf{vec}\left(\bm{F}^{t}{\bm{H}}^{t}-\bm{F}^{\star}\right)^{\top}\bm{A}^{2}\,\mathsf{vec}\big(\bm{F}^{t}{\bm{H}}^{t}-\bm{F}^{\star}\big)\leq\left\Vert \bm{A}\right\Vert ^{2}\left\Vert \bm{F}^{t}{\bm{H}}^{t}-\bm{F}^{\star}\right\Vert _{\mathrm{F}}^{2}.
\]
Applying the same argument as in~(\ref{eq:2-infty-line-segment}), one gets for all $0\leq \tau \leq 1$, $\|\bm{F}(\tau) -\bm{F}^{\star}\|_{2,\infty}\leq \frac{1}{2000\kappa \sqrt{n}}\|\bm{X}^\star\|$. Invoke Lemma~\ref{lemma:hessian} with $\bX=\bm{X}^{\star}+\tau(\bm{X}^{t}{\bm{H}}^{t}-\bm{X}^{\star})$,
$\bY=\bm{Y}^{\star}+\tau(\bm{Y}^{t}{\bm{H}}^{t}-\bm{Y}^{\star})$,
$(\bm{X}_{1},\bm{Y}_{1})=(\bm{X}^{t},\bm{Y}^{t})$ and $(\bm{X}_{2},\bm{Y}_{2})=(\bm{X}^{\star},\bm{Y}^{\star})$
to obtain $\Vert\bA\Vert\leq10\sigma_{\max}$ and 
\[
\mathrm{vec}\left(\bm{F}^{t}{\bm{H}}^{t}-\bm{F}^{\star}\right)^{\top}\bm{A}\text{ }\mathrm{vec}\left(\bm{F}^{t}{\bm{H}}^{t}-\bm{F}^{\star}\right)\geq\frac{1}{10}\sigma_{\min}\Fnorm{\bF^{t}\bH^{t}-\bF^{\star}}^{2}.
\]
Putting these two bounds back to~(\ref{eq:v^A^2v-MC}) yields 
\[
\begin{aligned}\alpha_{1}^{2}\leq\left(1+100\eta^{2}\sigma_{\max}^{2}-\frac{1}{5}\eta\sigma_{\min}\right)\Fnorm{\bF^{t}\bH^{t}-\bF^{\star}}^{2}\leq\left(1-\frac{\sigma_{\min}}{10}\eta\right)\Fnorm{\bF^{t}\bH^{t}-\bF^{\star}}^{2}.\end{aligned}
\]
Here the last relation holds as long as $0\leq\eta\leq1/(1000\kappa\sigma_{\max})$.
As a result, we have 
\begin{align*}
\alpha_{1} & \leq\left(1-\frac{\sigma_{\min}}{20}\eta\right)\Fnorm{\bF^{t}\bH^{t}-\bF^{\star}}.
\end{align*}
\end{enumerate}
Combine the above bounds on $\alpha_{1},\alpha_{2}$ and $\alpha_{3}$
to conclude that 
\begin{align*}
& \Fnorm{\bF^{t+1}\bH^{t+1}-\bF^{\star}}  \leq\left(1-\frac{\sigma_{\min}}{20}\eta\right)\Fnorm{\bF^{t}\bH^{t}-\bF^{\star}}+4\eta\frac{\lambda}{p}\left\Vert \bm{X}^{\star}\right\Vert _{\mathrm{F}}+\eta\sigma_{\min}\left(\frac{\sigma}{\sigma_{\min}}\sqrt{\frac{n}{p}}+\frac{\lambda}{p\,\sigma_{\min}}\right)\left\Vert \bm{X}^{\star}\right\Vert _{\mathrm{F}}\\
 & \leq\left(1-\frac{\sigma_{\min}}{20}\eta\right)C_{\mathrm{F}}\left(\frac{\sigma}{\sigma_{\min}}\sqrt{\frac{n}{p}}+\frac{\lambda}{p\,\sigma_{\min}}\right)\left\Vert \bm{X}^{\star}\right\Vert _{\mathrm{F}}+4\eta\sigma_{\min}\frac{\lambda}{p\,\sigma_{\min}}\left\Vert \bm{X}^{\star}\right\Vert _{\mathrm{F}} +\eta\sigma_{\min}\left(\frac{\sigma}{\sigma_{\min}}\sqrt{\frac{n}{p}}+\frac{\lambda}{p\,\sigma_{\min}}\right)\left\Vert \bm{X}^{\star}\right\Vert _{\mathrm{F}}\\
 & \leq C_{\mathrm{F}}\left(\frac{\sigma}{\sigma_{\min}}\sqrt{\frac{n}{p}}+\frac{\lambda}{p\,\sigma_{\min}}\right)\left\Vert \bm{X}^{\star}\right\Vert _{\mathrm{F}},
\end{align*}
provided that $C_{\mathrm{F}}>0$ is large enough.

\subsection{Proof of Lemma~\ref{lemma:operator-contraction}}

To facilitate analysis, we define an auxiliary point $\tilde{\bF}^{t+1}\triangleq\left[\begin{array}{c}
\tilde{\bm{X}}^{t+1}\\
\tilde{\bm{Y}}^{t+1}
\end{array}\right]$ as
\begin{subequations} 
\begin{align}
\tilde{\bX}^{t+1} & =\bX^{t}\bH^{t}-\eta\left[\frac{1}{p}\cP_{\Omega}\left(\bX^{t}\bY^{t\top}-\bM^{\star}-\bE\right)\bY^{\star}+\frac{\lambda}{p}\bX^{\star}+\frac{1}{2}\bX^{\star}\bH^{t\top}\left(\bX^{t\top}\bX^{t}-\bY^{t\top}\bY^{t}\right){\bH}^{t}\right];\\
\tilde{\bY}^{t+1} & =\bY^{t}\bH^{t}-\eta\left[\frac{1}{p}\cP_{\Omega}\left(\bX^{t}\bY^{t\top}-\bM^{\star}-\bE\right)^{\top}\bX^{\star}+\frac{\lambda}{p}\bY^{\star}+\frac{1}{2}\bY^{\star}\bH^{t\top}\left(\bY^{t\top}\bY^{t}-\bX^{t\top}\bX^{t}\right)\bH^{t}\right].
\end{align}
\end{subequations} Then the triangle inequality tells us that 
\begin{equation}
\norm{\bm{F}^{t+1}\bH^{t+1}-\bF^{\star}}\leq\underbrace{\big\|\bm{F}^{t+1}\bH^{t+1}-\tilde{\bF}^{t+1}\big\|}_{:=\alpha_{1}}+\underbrace{\big\|\tilde{\bF}^{t+1}-\bF^{\star}\big\|}_{:=\alpha_{2}}.
\end{equation}
In what follows, we shall control $\alpha_{1}$ and $\alpha_{2}$
separately. 
\begin{enumerate}
\item We start with $\alpha_{2}$. By the triangle inequality again we have
\begin{align*}
\alpha_{2} & \leq\underbrace{\norm{\left[\begin{matrix}\bX^{t}\bH^{t}-\eta\left[\left(\bX^{t}\bY^{t\top}-\bM^{\star}\right)\bY^{\star}+\frac{1}{2}\bX^{\star}\bH^{t\top}\left(\bX^{t\top}\bX^{t}-\bY^{t\top}\bY^{t}\right){\bH}^{t}\right]-\bX^{\star}\\
\bY^{t}\bH^{t}-\eta\left[\left(\bX^{t}\bY^{t\top}-\bM^{\star}\right)^{\top}\bX^{\star}+\frac{1}{2}\bY^{\star}\bH^{t\top}\left(\bY^{t\top}\bY^{t}-\bX^{t\top}\bX^{t}\right)\bH^{t}\right]-\bY^{\star}
\end{matrix}\right]}}_{:=\beta_{1}}\\
 & \quad+\underbrace{\frac{\eta}{p}\norm{\left[\begin{matrix}\cP_{\Omega}(\bE)\bY^{\star}\\
\cP_{\Omega}(\bE)^{\top}\bX^{\star}
\end{matrix}\right]}+\eta\frac{\lambda}{p}\norm{\left[\begin{matrix}\bX^{\star}\\
\bY^{\star}
\end{matrix}\right]}}_{:=\beta_{2}}+\underbrace{\eta\norm{\left[\begin{matrix}\frac{1}{p}\cP_{\Omega}\left(\bX^{t}\bY^{t\top}-\bM^{\star}\right)\bY^{\star}-\left(\bX^{t}\bY^{t\top}-\bM^{\star}\right)\bY^{\star}\\
\frac{1}{p}\left[\cP_{\Omega}\left(\bX^{t}\bY^{t\top}-\bM^{\star}\right)\right]^{\top}\bX^{\star}-\left(\bX^{t}\bY^{t\top}-\bM^{\star}\right)^{\top}\bX^{\star}
\end{matrix}\right]}}_{:=\beta_{3}}
\end{align*}
Denote $\bm{\Delta}^{t}\triangleq\bm{F}^{t}\bm{H}^{t}-\bm{F}^{\star}=\left[\begin{array}{c}
\bm{\Delta}_{\bm{X}}^{t}\\
\bm{\Delta}_{\bm{Y}}^{t}
\end{array}\right]$. The term $\beta_{1}$ is the same as the term $\alpha_{2}$ in~\cite[Section 4.2]{chen2019nonconvex}.
Therefore we can adopt the bound therein to obtain 
\[
\beta_{1}\leq(1-\eta\sigma_{\min})\norm{\bDelta^{t}}+4\eta\norm{\bDelta^{t}}^{2}\norm{\bX^{\star}}.
\]
Moving to $\beta_{2}$, one has 
\[
\beta_{2}=\eta\norm{\left[\begin{matrix}\frac{1}{p}\cP_{\Omega}(\bE) & \bm{0}\\
\bm{0} & \frac{1}{p}\cP_{\Omega}(\bE)^{\top}
\end{matrix}\right]\left[\begin{matrix}\bY^{\star}\\
\bX^{\star}
\end{matrix}\right]}+\eta\frac{\lambda}{p}\norm{\bF^{\star}}\leq\frac{\eta}{p}\norm{\cP_{\Omega}(\bE)}\norm{\bF^{\star}}+\eta\frac{\lambda}{p}\norm{\bF^{\star}}\leq C\eta\left(\sigma\sqrt{\frac{n}{p}}+\frac{\lambda}{p}\right)\norm{\bX^{\star}}
\]
for some constant $C>0$. Here the last inequality arises from Lemma
\ref{lemma:noise-bound} and the fact that $\|\bm{F}^{\star}\|=\sqrt{2}\|\bm{X}^{\star}\|$~(cf.~(\ref{eq:F-singular-value})). We are now left with the term
$\beta_{3}$, which is exactly the term $\alpha_{1}$ in~\cite[Section 4.2]{chen2019nonconvex}.
Reusing their results, we have 
\[
\begin{aligned}\beta_{3} & \leq\frac{2\eta}{p}\left\Vert \bm{X}^{\star}\right\Vert \left\Vert \mathcal{P}_{\Omega}\left(\bm{1}\bm{1}^{\top}\right)-p\bm{1}\bm{1}^{\top}\right\Vert (\twotoinftynorm{\bDelta_{\bX}^{t}}\twotoinftynorm{\bDelta_{\bY}^{t}}+\twotoinftynorm{\bDelta_{\bX}^{t}}\twotoinftynorm{\bY^{\star}}+\twotoinftynorm{\bX^{\star}}\twotoinftynorm{\bDelta_{\bY}^{t}})\\
 & \lesssim\eta\sqrt{\frac{n}{p}}\twotoinftynorm{\bDelta^{t}}\twotoinftynorm{\bF^{\star}}\norm{\bX^{\star}}.
\end{aligned}
\]
The last line follows from the facts that $\|\mathcal{P}_{\Omega}(\bm{1}\bm{1}^{\top})-p\bm{1}\bm{1}^{\top}\|\lesssim\sqrt{np}$
(see~\cite[Lemma 3.2]{KesMonSew2010}) and that $\max\{\|\bm{\Delta}_{\bm{X}}^{t}\|_{2,\infty},\|\bm{\Delta}_{\bm{Y}}^{t}\|_{2,\infty}\}\leq\|\bm{F}^{\star}\|_{2,\infty}$,
provided that $\frac{\sigma}{\sigma_{\min}}\sqrt{\frac{n}{p}}\ll\frac{1}{\sqrt{\kappa^{2}\log n}}$
(see Lemma~\ref{lemma:immediate-consequence}). Combining the above
three bounds gives 
\begin{align}
\alpha_{2} & \leq\left(1-\eta\sigma_{\min}\right)\norm{\bDelta^{t}}+4\eta\norm{\bDelta^{t}}^{2}\norm{\bX^{\star}}+\tilde{C}\eta\left(\sigma\sqrt{\frac{n}{p}}+\frac{\lambda}{p}\right)\norm{\bX^{\star}}+\tilde{C}\eta\sqrt{\frac{n}{p}}\twotoinftynorm{\bDelta^{t}}\twotoinftynorm{\bF^{\star}}\norm{\bX^{\star}}\nonumber \\
 & \leq\left(1-\frac{\eta}{2}\sigma_{\min}\right)\norm{\bDelta^{t}}+\tilde{C}\eta\left(\sigma\sqrt{\frac{n}{p}}+\frac{\lambda}{p}\right)\norm{\bX^{\star}}+\tilde{C}\eta\sqrt{\frac{n}{p}}\twotoinftynorm{\bDelta^{t}}\twotoinftynorm{\bF^{\star}}\norm{\bX^{\star}}\label{eq:operator-alpha-2-upper-bound}
\end{align}
for some sufficiently large constant $\tilde{C}>0$. Here the second
inequality arises from the condition 
\[
4\left\Vert \bm{\Delta}^{t}\right\Vert \left\Vert \bm{X}^{\star}\right\Vert \leq\sigma_{\min}/2,
\]
which would hold if $\frac{\sigma}{\sigma_{\min}}\sqrt{\frac{n}{p}}\ll\frac{1}{\kappa}$.
An immediate consequence of~(\ref{eq:operator-alpha-2-upper-bound})
is that 
\begin{equation}
\alpha_{2}=\big\|\tilde{\bF}^{t+1}-\bF^{\star}\big\|\leq(\sqrt{2}\kappa)^{-1}\norm{\bX^{\star}}.\label{eq:lem_op_1}
\end{equation}
To see this, apply the induction hypotheses~(\ref{eq:induction-op})
and~(\ref{eq:induction-2-infty}) to get 
\begin{align}
\alpha_{2} & \leq\left(1-\frac{\eta\sigma_{\min}}{2}\right)C_{\mathrm{op}}\left(\frac{\sigma}{\sigma_{\min}}\sqrt{\frac{n}{p}}+\frac{\lambda}{p\,\sigma_{\min}}\right)\left\Vert \bm{X}^{\star}\right\Vert +\tilde{C}\eta\left(\sigma\sqrt{\frac{n}{p}}+\frac{\lambda}{p}\right)\norm{\bX^{\star}}\nonumber \\
 & \quad\quad+\tilde{C}\eta\sqrt{\frac{n}{p}}C_{\infty}\kappa\left(\frac{\sigma}{\sigma_{\min}}\sqrt{\frac{n\log n}{p}}+\frac{\lambda}{p\,\sigma_{\min}}\right)\left\Vert \bm{F}^{\star}\right\Vert _{2,\infty}^{2}\left\Vert \bm{X}^{\star}\right\Vert \nonumber \\
 & \overset{(\text{i})}{\leq}C_{\mathrm{op}}\left(\frac{\sigma}{\sigma_{\min}}\sqrt{\frac{n}{p}}+\frac{\lambda}{p\,\sigma_{\min}}\right)\left\Vert \bm{X}^{\star}\right\Vert \label{eq:same-before}\\
 & \overset{(\text{ii})}{\leq}(\sqrt{2}\kappa)^{-1}\norm{\bX^{\star}}.\nonumber 
\end{align}
Here (i) holds under the assumptions that $C_{\mathrm{op}}\gg\tilde{C}+\tilde{C}C_{\infty}\kappa^{2}\sqrt{\frac{\mu^{2}r^{2}\log n}{np}}$
and that $\left\Vert \bm{F}^{\star}\right\Vert _{2,\infty}\leq\sqrt{\frac{\mu r}{n}}\|\bm{X}^{\star}\|$ (cf.~(\ref{eq:F-incoherence})); (ii) arises since $\frac{\sigma}{\sigma_{\min}}\sqrt{\frac{n}{p}}\ll1/\kappa$
and $\lambda\asymp\sigma\sqrt{np}$. Under the sample complexity ${n^{2}p\gg\kappa^{4}\mu^{2}r^{2}n\log n}$,
the first condition can be simplified to $C_{\mathrm{op}}\gg2\tilde{C}\gg1$. 
\item Next we bound $\alpha_{1}$, towards which we first observe that 
\[
\alpha_{1}=\big\|\bF^{t+1}\bH^{t+1}-\tilde{\bF}^{t+1}\big\|=\big\|\bF^{t+1}\bH^{t}\bH^{t\top}\bH^{t+1}-\tilde{\bF}^{t+1}\big\|.
\]
It is straightforward to verify that $\bH^{t\top}\bH^{t+1}$ is the
best rotation matrix to align $\bF^{t+1}\bH^{t}$ and $\bm{F}^{\star}$
(in the sense of~(\ref{eq:defn-H-appendix})). Regarding $\tilde{\bm{F}}^{t+1}$,
we obtain the following claim, which demonstrates that it is already
aligned with $\bm{F}^{\star}$, i.e.~$\bm{I}_{r}$ is the best rotation
matrix to align $\tilde{\bm{F}}^{t+1}$ and $\bm{F}^{\star}$. \begin{claim}\label{claim:already-aligned}Suppose
(\ref{eq:same-before}) holds true, one has 
\[
\bm{I}_{r}=\arg\min_{\bm{R}\in\mathcal{O}^{r\times r}}\big\|\tilde{\bm{F}}^{t+1}\bm{R}-\bm{F}^{\star}\big\|_{\mathrm{F}}.
\]
\end{claim}Now we intend to apply Lemma~\ref{lemma:rotation-perturbation}
with 
\[
\bm{F}_{0}=\bF^{\star},\quad\bm{F}_{1}=\tilde{\bF}^{t+1},\quad\bm{F}_{2}=\bF^{t+1}\bH^{t},
\]
for which we need to check the two conditions therein. First, in view
of~(\ref{eq:lem_op_1}), one has 
\[
\norm{\bm{F}_{1}-\bm{F}_{0}}\norm{\bm{F}_{0}}=\big\|\tilde{\bF}^{t+1}-\bF^{\star}\big\|\norm{\bF^{\star}}\leq\frac{1}{\sqrt{2}\kappa}\norm{\bX^{\star}}\norm{\bF^{\star}}=\sigma_{\min}=\frac{1}{2}\sigma_{r}^{2}(\bm{F}_{0}).
\]
Second, making use of the gradient update rules~(\ref{subeq:GD-rules})
and the decomposition~(\ref{subeq:gradient-decomposition}), we obtain
\begin{align*}
\big\|\bm{F}^{t+1}\bm{H}^{t}-\tilde{\bm{F}}^{t+1}\big\| & =\norm{\left(\bm{F}^{t}-\eta\nabla f_{\mathsf{aug}}\left(\bm{F}^{t}\right)-\eta\nabla f_{\mathsf{diff}}\left(\bm{F}^{t}\right)\right)\bH^{t}-\tilde{\bm{F}}^{t+1}}\\
 & \leq\underbrace{\norm{\left(\bm{F}^{t}-\eta\nabla f_{\mathsf{aug}}\left(\bm{F}^{t}\right)\right)\bm{H}^{t}-\tilde{\bm{F}}^{t+1}}}_{:=\theta_{1}}+\underbrace{\eta\left\Vert \nabla f_{\mathsf{diff}}\left(\bm{F}^{t}\right)\right\Vert }_{:=\theta_{2}}.
\end{align*}
The term $\theta_{2}$ has been controlled as $\alpha_{2}$ in the
proof of Lemma~\ref{lemma:fro-contraction}, where we obtained 
\begin{align*}
\theta_{2} & \leq2C_{\mathrm{B}}\kappa\eta^{2}\left(\frac{\sigma}{\sigma_{\min}}\sqrt{\frac{n}{p}}+\frac{\lambda}{p\,\sigma_{\min}}\right)\sqrt{r}\sigma_{\max}^{2}\left\Vert \bm{X}^{\star}\right\Vert .
\end{align*}
We now move on to $\theta_{1}$, for which we have 
\begin{align*}
\theta_{1} & \leq\left\Vert \left(\bm{F}^{t}-\eta\left\{ \nabla f_{\mathsf{aug}}\left(\bm{F}^{t}\right)-\eta\left[\begin{array}{c}
\frac{1}{p}\mathcal{P}_{\Omega}\left(\bm{E}\right)\bm{Y}^{t}\\
\frac{1}{p}\mathcal{P}_{\Omega}\left(\bm{E}\right)^{\top}\bm{X}^{t}
\end{array}\right]-\eta\frac{\lambda}{p}\left[\begin{array}{c}
\bm{X}^{t}\\
\bm{Y}^{t}
\end{array}\right]\right\} \right)\bm{H}^{t}\right.\\
 & \underset{:=\xi_{1}}{\quad\underbrace{\qquad\qquad\qquad\qquad\left.-\tilde{\bm{F}}^{t+1}-\eta\left[\begin{array}{c}
\frac{1}{p}\mathcal{P}_{\Omega}\left(\bm{E}\right)\bm{Y}^{\star}\\
\frac{1}{p}\mathcal{P}_{\Omega}\left(\bm{E}\right)^{\top}\bm{X}^{\star}
\end{array}\right]-\eta\frac{\lambda}{p}\left[\begin{array}{c}
\bm{X}^{\star}\\
\bm{Y}^{\star}
\end{array}\right]\right\Vert \qquad}}\\
 & \quad+\underbrace{\eta\left\Vert \left[\begin{array}{c}
\frac{1}{p}\mathcal{P}_{\Omega}\left(\bm{E}\right)\bm{Y}^{t}\\
\frac{1}{p}\mathcal{P}_{\Omega}\left(\bm{E}\right)^{\top}\bm{X}^{t}
\end{array}\right]\bm{H}^{t}+\frac{\lambda}{p}\left[\begin{array}{c}
\bm{X}^{t}\\
\bm{Y}^{t}
\end{array}\right]\bm{H}^{t}-\left[\begin{array}{c}
\frac{1}{p}\mathcal{P}_{\Omega}\left(\bm{E}\right)\bm{Y}^{\star}\\
\frac{1}{p}\mathcal{P}_{\Omega}\left(\bm{E}\right)^{\top}\bm{X}^{\star}
\end{array}\right]-\frac{\lambda}{p}\left[\begin{array}{c}
\bm{X}^{\star}\\
\bm{Y}^{\star}
\end{array}\right]\right\Vert }_{:=\xi_{2}}.
\end{align*}

Combining~\cite[Equation~(4.13)]{chen2019nonconvex} and~\cite[Lemma 3.2]{KesMonSew2010}
yields 
\begin{align*}
\xi_{1} & \lesssim\eta\sqrt{\frac{n}{p}}\left(\twotoinftynorm{\bDelta_{\bX}^{t}}\twotoinftynorm{\bY^{\star}}+\twotoinftynorm{\bDelta_{\bY}^{t}}\twotoinftynorm{\bX^{\star}}+\twotoinftynorm{\bDelta_{\bX}^{t}}\twotoinftynorm{\bDelta_{\bY}^{t}}\right)\norm{\bDelta^{t}}\\
 & \quad+\eta\left(\norm{\bDelta_{\bX}^{t}}\left\Vert \bm{Y}^{\star}\right\Vert +\norm{\bDelta_{\bY}^{t}}\left\Vert \bm{X}^{\star}\right\Vert +\norm{\bDelta_{\bX}^{t}}\norm{\bDelta_{\bY}^{t}}+2\left\Vert \bm{X}^{\star}\right\Vert \norm{\bDelta_{\bX}^{t}}+2\norm{\bY^{\star}}\norm{\bDelta_{\bY}^{t}}+\norm{\bDelta_{\bX}^{t}}^{2}+\norm{\bDelta_{\bY}^{t}}^{2}\right)\norm{\bDelta_{t}}\\
 & \lesssim\eta\sqrt{\frac{n}{p}}\left\Vert \bm{\Delta}^{t}\right\Vert _{2,\infty}\left\Vert \bm{F}^{\star}\right\Vert _{2,\infty}\left\Vert \bm{\Delta}^{t}\right\Vert +\eta\left\Vert \bm{\Delta}^{t}\right\Vert ^{2}\left\Vert \bm{X}^{\star}\right\Vert \\
 & \leq\frac{1}{15\kappa}\frac{\sigma_{\min}}{4}\eta\left\Vert \bm{\Delta}^{t}\right\Vert .
\end{align*}
Here the penultimate inequality arises from the facts that $\max\{\|\bm{\Delta}_{\bm{X}}^{t}\|_{2,\infty},\|\bm{\Delta}_{\bm{X}}^{t}\|_{2,\infty}\}\leq\|\bm{\Delta}^{t}\|_{2,\infty}\leq\|\bm{F}^{\star}\|_{2,\infty}$
and similarly $\max\{\|\bm{\Delta}_{\bm{X}}^{t}\|,\|\bm{\Delta}_{\bm{X}}^{t}\|\}\leq\|\bm{\Delta}^{t}\|\leq\|\bm{X}^{\star}\|$;
see Lemma~\ref{lemma:immediate-consequence}. In addition, the last
line holds because of the induction hypotheses~(\ref{eq:induction-op})
and~(\ref{eq:induction-2-infty}), provided that 
\[
C_{\infty}\kappa\frac{\sigma}{\sigma_{\min}}\sqrt{\frac{n}{p}}\sqrt{\frac{\mu^{2}r^{2}\log n}{np}}\ll\frac{1}{\kappa^{2}}\qquad\text{and}\qquad C_{\mathrm{op}}\frac{\sigma}{\sigma_{\min}}\sqrt{\frac{n}{p}}\ll\frac{1}{\kappa^{2}}.
\]
Again, the first condition would be guaranteed by the sample size
condition $n^{2}p\gg\kappa^{4}\mu^{2}r^{2}n\log n$ and the noise
condition $\frac{\sigma}{\sigma_{\min}}\sqrt{\frac{n}{p}}\ll1/\kappa$.
Next, the term $\xi_{2}$ can be easily controlled as follows 
\begin{align*}
\xi_{2} & \leq\eta\norm{\left[\begin{matrix}\frac{1}{p}\cP_{\Omega}(\bE)\left(\bY^{t}\bH^{t}-\bY^{\star}\right)\\
\frac{1}{p}\cP_{\Omega}(\bE)^{\top}\left(\bX^{t}\bH^{t}-\bX^{\star}\right)
\end{matrix}\right]}+\eta\frac{\lambda}{p}\norm{\bF^{t}\bH^{t}-\bF^{\star}}\\
 & \leq\tilde{C}\eta\left(\sigma\sqrt{\frac{n}{p}}+\frac{\lambda}{p}\right)\left\Vert \bm{\Delta}^{t}\right\Vert ,
\end{align*}
where the last line follows from the same argument for bounding $\beta_{2}$
above. Taking the bounds on $\theta_{1}$ and $\theta_{2}$ collectively
yields 
\begin{align}
\big\|\bm{F}^{t+1}\bm{H}^{t}-\tilde{\bm{F}}^{t+1}\big\| & \leq\frac{1}{15\kappa}\frac{\sigma_{\min}}{4}\eta\left\Vert \bm{\Delta}^{t}\right\Vert +\tilde{C}\eta\left(\sigma\sqrt{\frac{n}{p}}+\frac{\lambda}{p}\right)\left\Vert \bm{\Delta}^{t}\right\Vert +2C_{\mathrm{B}}\kappa\eta^{2}\left(\frac{\sigma}{\sigma_{\min}}\sqrt{\frac{n}{p}}+\frac{\lambda}{p\,\sigma_{\min}}\right)\sqrt{r}\sigma_{\max}^{2}\left\Vert \bm{X}^{\star}\right\Vert \nonumber \\
 & \leq\frac{1}{5\kappa}\frac{\sigma_{\min}}{4}\eta\norm{\bDelta^{t}}+2C_{\mathrm{B}}\kappa\eta^{2}\left(\frac{\sigma}{\sigma_{\min}}\sqrt{\frac{n}{p}}+\frac{\lambda}{p\,\sigma_{\min}}\right)\sqrt{r}\sigma_{\max}^{2}\left\Vert \bm{X}^{\star}\right\Vert .\label{eq:operator-alpha1-some-term}
\end{align}
The final inequality is true as long as $\lambda\asymp\sigma\sqrt{np}$
and $\frac{\sigma}{\sigma_{\min}}\sqrt{\frac{n}{p}}\ll\frac{1}{\kappa}$.
An immediate consequence of~(\ref{eq:operator-alpha1-some-term})
is that 
\begin{equation}
\big\|\tilde{\bF}^{t+1}-\bF^{t+1}\bH^{t}\big\|\leq(2\sqrt{2}\kappa)^{-1}\norm{\bX^{\star}},\label{eq:lem_op_2-1}
\end{equation}
as long as $\eta\ll1/(C_{\mathrm{B}}\kappa^{2}\sigma_{\max}\sqrt{r})$, $\frac{\sigma}{\sigma_{\min}}\sqrt{\frac{n}{p}}\ll\frac{1}{\kappa}$ 
and $\lambda\asymp\sigma\sqrt{np}$. As a result, one obtains 
\begin{align*}
\norm{\bm{F}_{1}-\bm{F}_{2}}\norm{\bm{F}_{0}} & =\big\|\tilde{\bF}^{t+1}-\bF^{t+1}\bH^{t}\big\|\norm{\bF^{\star}}\leq(2\sqrt{2}\kappa)^{-1}\norm{\bX^{\star}}\norm{\bF^{\star}}=\sigma_{\min}/2=\sigma_{\min}^{2}(\bm{F}_{0})/4.
\end{align*}
Armed with these two conditions, we can invoke Lemma~\ref{lemma:rotation-perturbation}
to obtain 
\begin{align*}
\alpha_{1} & =\big\|\tilde{\bm{F}}^{t+1}-\bm{F}^{t+1}\bm{H}^{t+1}\big\|\leq5\kappa\big\|\tilde{\bm{F}}^{t+1}-\bm{F}^{t+1}\bm{H}^{t}\big\|\\
 & \leq\frac{1}{4}\sigma_{\min}\eta\left\Vert \bm{\Delta}^{t}\right\Vert +10C_{\mathrm{B}}\kappa^{2}\eta^{2}\left(\frac{\sigma}{\sigma_{\min}}\sqrt{\frac{n}{p}}+\frac{\lambda}{p\,\sigma_{\min}}\right)\sqrt{r}\sigma_{\max}^{2}\left\Vert \bm{X}^{\star}\right\Vert \\
 & \leq\frac{1}{4}\sigma_{\min}\eta\left\Vert \bm{\Delta}^{t}\right\Vert +\eta\left(\sigma\sqrt{\frac{n}{p}}+\frac{\lambda}{p}\right)\norm{\bX^{\star}},
\end{align*}
provided that $\eta\ll1/(C_{\mathrm{B}}\kappa^{3}\sigma_{\max}\sqrt{r})$. 
\end{enumerate}
Combine the bounds on $\alpha_{1}$ and $\alpha_{2}$ to reach 
\begin{align*}
 & \left\Vert \bF^{t+1}\bH^{t+1}-\bF^{\star}\right\Vert \\
 & \quad\leq\left(1-\frac{\eta}{2}\sigma_{\min}\right)\norm{\bDelta^{t}}+\left(\tilde{C}+1\right)\eta\left(\sigma\sqrt{\frac{n}{p}}+\frac{\lambda}{p}\right)\norm{\bX^{\star}}+\tilde{C}\eta\sqrt{\frac{n}{p}}\twotoinftynorm{\bDelta^{t}}\twotoinftynorm{\bF^{\star}}\norm{\bX^{\star}}+\frac{\sigma_{\min}}{4}\eta\left\Vert \bm{\Delta}^{t}\right\Vert \\
 & \quad\leq\left(1-\frac{\eta}{4}\sigma_{\min}\right)C_{\mathrm{op}}\left(\frac{\sigma}{\sigma_{\min}}\sqrt{\frac{n}{p}}+\frac{\lambda}{p\,\sigma_{\min}}\right)\left\Vert \bm{X}^{\star}\right\Vert +\left(\tilde{C}+1\right)\eta\left(\sigma\sqrt{\frac{n}{p}}+\frac{\lambda}{p}\right)\norm{\bX^{\star}}\\
 & \quad\quad\quad+\tilde{C}\eta\sqrt{\frac{n}{p}}C_{\infty}\kappa\left(\frac{\sigma}{\sigma_{\min}}\sqrt{\frac{n\log n}{p}}+\frac{\lambda}{p\,\sigma_{\min}}\right)\left\Vert \bm{F}^{\star}\right\Vert _{2,\infty}^{2}\left\Vert \bm{X}^{\star}\right\Vert \\
 & \quad\leq C_{\mathrm{op}}\left(\frac{\sigma}{\sigma_{\min}}\sqrt{\frac{n}{p}}+\frac{\lambda}{p\,\sigma_{\min}}\right)\left\Vert \bm{X}^{\star}\right\Vert ,
\end{align*}
with the proviso that $C_{\mathrm{op}}\gg1$ and $n^{2}p\gg\kappa^{4}\mu^{2}r^{2}n\log n$.
Here the last line follows from the same argument as in bounding~(\ref{eq:same-before}).
This completes the proof.

\begin{proof}[Proof of Claim~\ref{claim:already-aligned}]In view
of~\cite[Lemma 35]{ma2017implicit}, it suffices to show that $\bF^{\star\top}\tilde{\bF}^{t+1}$
is symmetric and positive semidefinite. Recognizing that $\bF^{\star\top}\bF^{t}\bH^{t}=(\bX^{\star\top}\bX^{t}+\bY^{\star\top}\bY^{t})\bH^{t}$
is symmetric (see~\cite[Lemma 35]{ma2017implicit}), it is straightforward
to verify that $\bF^{\star\top}\tilde{\bF}^{t+1}$ is also symmetric
(which we omit here for brevity). In addition, by~\eqref{eq:lem_op_1}
we have 
\[
\big\|\bF^{\star\top}\tilde{\bF}^{t+1}-\bF^{\star\top}\bF^{\star}\big\|\leq\norm{\bF^{\star}}\big\|\tilde{\bF}^{t+1}-\bF^{\star}\big\|=\alpha_{2}\norm{\bF^{\star}}\leq\frac{1}{\sqrt{2}\kappa}\norm{\bX^{\star}}\norm{\bF^{\star}}=\sigma_{\min}.
\]
Since $\bF^{\star\top}\bF^{\star}=\bX^{\star\top}\bX^{\star}+\bY^{\star\top}\bY^{\star}=2\bSigma^{\star}$,
Weyl's inequality gives 
\[
\lambda_{\min}(\bF^{\star\top}\tilde{\bF}^{t+1})\geq2\sigma_{\min}-\big\|\bF^{\star\top}\tilde{\bF}^{t+1}-\bF^{\star\top}\bF^{\star}\big\|\geq\sigma_{\min}>0,
\]
where $\lambda_{\min}(\bm{A})$ stands for the minimum eigenvalue
of a matrix $\bm{A}$. To conclude, $\bF^{\star\top}\tilde{\bF}^{t+1}$
is both symmetric and positive semidefinite, thus establishing the
claim.\end{proof}

\subsection{Proof of Lemma~\ref{lemma:loo-dist-contraction}}

Without loss of generality, we consider the case when $1\leq l\leq n$;
the case with $n+1\leq l\leq2n$ can be derived in a similar way.
From the definition of $\bR^{t+1,(l)}$~(cf.~(\ref{eq:defn-R-t-l})),
we have 
\begin{align*}
\big\|\bF^{t+1}\bH^{t+1}-\bF^{t+1,(l)}\bR^{t+1,(l)}\big\|_{\mathrm{F}} & \leq\big\|\bF^{t+1}\bH^{t}-\bF^{t+1,(l)}\bR^{t,(l)}\big\|_{\mathrm{F}}.
\end{align*}
The gradient update rules~(\ref{subeq:GD-rules}) and~(\ref{subeq:GD-rules-LOO})
give 
\begin{align*}
 & \bF^{t+1}\bH^{t}-\bF^{t+1,(l)}\bR^{t,(l)}\\
 & \quad=\left[\bF^{t}-\eta\nabla f\left(\bF^{t}\right)\right]\bH^{t}-\left[\bF^{t,(l)}-\eta\nabla f^{(l)}\big(\bF^{t,(l)}\big)\right]\bR^{t,(l)}\\
 & \quad=\bF^{t}\bH^{t}-\eta\nabla f\left(\bF^{t}\bH^{t}\right)-\left[\bF^{t,(l)}\bR^{t,(l)}-\eta\nabla f^{(l)}\big(\bF^{t,(l)}\bR^{t,(l)}\big)\right]\\
 & \quad=\underbrace{\big(\bF^{t}\bH^{t}-\bF^{t,(l)}\bR^{t,(l)}\big)-\eta\left[\nabla f_{\mathsf{aug}}\left(\bF^{t}\bH^{t}\right)-\nabla f_{\mathsf{aug}}\big(\bF^{t,(l)}\bR^{t,(l)}\big)\right]}_{:=\bA_{1}}-\underbrace{\eta\left[\nabla f_{\mathsf{diff}}\left(\bF^{t}\bH^{t}\right)-\nabla f_{\mathsf{diff}}\big(\bF^{t,(l)}\bR^{t,(l)}\big)\right]}_{:=\bA_{2}}\\
 & \quad\quad+\underbrace{\eta\left[\nabla f^{(l)}\big(\bF^{t,(l)}\bR^{t,(l)}\big)-\nabla f\big(\bF^{t,(l)}\bR^{t,(l)}\big)\right]}_{:=\bA_{3}},
\end{align*}
where we have used the facts that $\nabla f(\bF)\bR=\nabla f(\bF\bR)$
and $\nabla f^{(l)}(\bF)\bR=\nabla f^{(l)}(\bF\bR)$ for any orthonormal
matrix $\bR\in\cO^{r\times r}$.

In what follows, we shall bound $\bm{A}_{1},\bm{A}_{2}$ and $\bm{A}_{3}$
sequentially. 
\begin{enumerate}
\item The first term $\bA_{1}$ is similar to $\alpha_{1}$ in the proof
of Lemma~\ref{lemma:fro-contraction}. Going through the same derivations
therein, we obtain 
\begin{equation}
\Fnorm{\bA_{1}}\leq\left(1-\frac{\sigma_{\min}}{20}\eta\right)\big\|\bF^{t}\bH^{t}-\bF^{t,(l)}\bR^{t,(l)}\big\|_{\mathrm{F}},
\end{equation}
provided that $\frac{\sigma}{\sigma_{\min}}\sqrt{\frac{n}{p}}\ll\frac{1}{\sqrt{\kappa^{4}\mu r\log n}}$
and that $0\leq\eta\leq1/(1000\kappa\sigma_{\max})$. 
\item Next, we turn attention to $\bA_{2}$, which clearly obeys 
\[
\left\Vert \bm{A}_{2}\right\Vert _{\mathrm{F}}\leq\eta\left\Vert \nabla f_{\mathsf{diff}}\big(\bF^{t}\bH^{t}\big)\right\Vert _{\mathrm{F}}+\eta\big\|\nabla f_{\mathsf{diff}}\big(\bF^{t,(l)}\bR^{t,(l)}\big)\big\|_{\mathrm{F}}.
\]
Recall from the term $\alpha_{2}$ in the proof of Lemma~\ref{lemma:fro-contraction}
that 
\[
\eta\left\Vert \nabla f_{\mathsf{diff}}\left(\bF^{t}\bH^{t}\right)\right\Vert _{\mathrm{F}}\leq2C_{\mathrm{B}}\kappa\eta^{2}\left(\frac{\sigma}{\sigma_{\min}}\sqrt{\frac{n}{p}}+\frac{\lambda}{p\,\sigma_{\min}}\right)\sqrt{r}\sigma_{\max}^{2}\norm{\bX^{\star}}.
\]
Applying Lemma~\ref{lemma:balancing} and going through the same derivation
as in bounding $\alpha_{2}$ in the proof of Lemma~\ref{lemma:fro-contraction},
one gets 
\[
\eta\big\|\nabla f_{\mathsf{diff}}\big(\bF^{t,(l)}\bR^{t,(l)}\big)\big\|_{\mathrm{F}}\leq2C_{\mathrm{B}}\kappa\eta^{2}\left(\frac{\sigma}{\sigma_{\min}}\sqrt{\frac{n}{p}}+\frac{\lambda}{p\,\sigma_{\min}}\right)\sqrt{r}\sigma_{\max}^{2}\norm{\bX^{\star}}.
\]
Combine the above three inequalities to obtain 
\begin{align*}
\left\Vert \bm{A}_{2}\right\Vert _{\mathrm{F}} & \leq4C_{\mathrm{B}}\kappa\eta^{2}\left(\frac{\sigma}{\sigma_{\min}}\sqrt{\frac{n}{p}}+\frac{\lambda}{p\,\sigma_{\min}}\right)\sqrt{r}\sigma_{\max}^{2}\norm{\bX^{\star}}\\
 & \leq4\sqrt{n}C_{\mathrm{B}}\kappa\eta^{2}\left(\frac{\sigma}{\sigma_{\min}}\sqrt{\frac{n}{p}}+\frac{\lambda}{p\,\sigma_{\min}}\right)\sqrt{r}\sigma_{\max}^{2}\left\Vert \bm{X}^{\star}\right\Vert _{2,\infty}\\
 & \leq\eta\left(\sigma\sqrt{\frac{n}{p}}+\frac{\lambda}{p}\right)\left\Vert \bm{F}^{\star}\right\Vert _{2,\infty}.
\end{align*}
Here the second inequality arises from the elementary inequality $\|\bm{X}^{\star}\|\leq\sqrt{n}\|\bm{X}^{\star}\|_{2,\infty}$,
whereas the last one holds true because of $\|\bm{X}^{\star}\|_{2,\infty}\leq\|\bm{F}^{\star}\|_{2,\infty}$
and the condition that $\eta\ll\frac{1}{n\kappa^{2}\sigma_{\max}}$.
 
\item We are now left with $\bA_{3}$. To this end, we first observe that
\begin{align*}
\bA_{3} & =\eta\left[\small\begin{matrix}\underbrace{\left[\cP_{l,\cdot}\left(\bX^{t,(l)}\bY^{t,(l)\top}-\bM^{\star}\right)-p^{-1}\cP_{\Omega_{l,\cdot}}\left(\bX^{t,(l)}\bY^{t,(l)\top}-\bM^{\star}\right)\right]\bY^{t,(l)}\bR^{t,(l)}}_{:=\bB_{1}}+\underbrace{p^{-1}\cP_{\Omega_{l,\cdot}}\left(\bE\right)\bY^{t,(l)}\bR^{t,(l)}}_{:=\bC_{1}}\\
\underbrace{\left[\cP_{l,\cdot}\left(\bX^{t,(l)}\bY^{t,(l)\top}-\bM^{\star}\right)-p^{-1}\cP_{\Omega_{l,\cdot}}\left(\bX^{t,(l)}\bY^{t,(l)\top}-\bM^{\star}\right)\right]^{\top}\bX^{t,(l)}\bR^{t,(l)}}_{:=\bB_{2}}+\underbrace{p^{-1}\cP_{\Omega_{l,\cdot}}\left(\bE\right)^{\top}\bX^{t,(l)}\bR^{t,(l)}}_{:=\bm{C}_{2}}
\end{matrix}\right].
\end{align*}
The following claims allow one to bound $\bm{B}_{1},\bm{B}_{2}$ and $\bm{C}_{1},\bm{C}_{2}$; the proofs are deferred to the end of this
subsection. 
\begin{claim} \label{claim:B1} Suppose that $\frac{\sigma}{\sigma_{\min}}\sqrt{\frac{n}{p}}\ll\frac{1}{\sqrt{\kappa^{2}\log n}}$
and that $np\gg\log^2 n$. With probability at least $1-O(n^{-100})$,
\begin{equation}
\Fnorm{\bB_{1}}\lesssim\sqrt{\frac{\mu^{2}r^{2}\log n}{np}}\big\|\bm{F}^{t,\left(l\right)}\bm{R}^{t,\left(l\right)}-\bm{F}^{\star}\big\|_{2,\infty}\sigma_{\max}.
\end{equation}
\end{claim} \begin{claim} \label{claim:B2} Suppose that $\frac{\sigma}{\sigma_{\min}}\sqrt{\frac{n}{p}}\ll\frac{1}{\sqrt{\kappa^{2}\log n}}$
and that $np\gg\log n$. With probability at least $1-O(n^{-100})$,
\begin{equation}
\Fnorm{\bB_{2}}\lesssim\sqrt{\frac{\mu^{2}r^{2}\log n}{np}}\big\|\bm{F}^{t,\left(l\right)}\bm{R}^{t,\left(l\right)}-\bm{F}^{\star}\big\|_{2,\infty}\sigma_{\max}.
\end{equation}
\end{claim} 

\begin{claim} \label{claim:C1C2} Suppose that $\frac{\sigma}{\sigma_{\min}}\sqrt{\frac{n}{p}}\ll\frac{1}{\sqrt{\kappa^{2}\log n}}$
and that $np\gg\log^{3}n$. With probability at least $1-O(n^{-100})$,
\begin{equation}
\max\left\{ \left\Vert \bm{C}_{1}\right\Vert _{\mathrm{F}},\left\Vert \bm{C}_{2}\right\Vert _{\mathrm{F}}\right\} \lesssim\sigma\sqrt{\frac{n\log n}{p}}\twotoinftynorm{\bm{F}^{\star}}.
\end{equation}
\end{claim}

With these claims in place, one can readily obtain that
\begin{align*}
\left\Vert \bm{A}_{3}\right\Vert _{\mathrm{F}} & \leq\eta\left(\left\Vert \bm{B}_{1}\right\Vert _{\mathrm{F}}+\left\Vert \bm{B}_{2}\right\Vert _{\mathrm{F}}+\left\Vert \bm{C}_{1}\right\Vert _{\mathrm{F}}+\left\Vert \bm{C}_{2}\right\Vert _{\mathrm{F}}\right)\\
 & \lesssim\eta\sigma\sqrt{\frac{n\log n}{p}}\twotoinftynorm{\bm{F}^{\star}}+\eta\sqrt{\frac{\mu^{2}r^{2}\log n}{np}}\big\|\bm{F}^{t,\left(l\right)}\bm{R}^{t,\left(l\right)}-\bm{F}^{\star}\big\|_{2,\infty}\sigma_{\max}\\
 & \leq\eta\sigma\sqrt{\frac{n\log n}{p}}\twotoinftynorm{\bm{F}^{\star}}+\eta\sqrt{\frac{\mu^{2}r^{2}\log n}{np}}\left(C_{\infty}\kappa+C_{3}\right)\left(\frac{\sigma}{\sigma_{\min}}\sqrt{\frac{n\log n}{p}}+\frac{\lambda}{p\,\sigma_{\min}}\right)\left\Vert \bm{F}^{\star}\right\Vert _{2,\infty}\sigma_{\max},
\end{align*}
where the last line follows from the induction hypotheses~(\ref{eq:induction-loo-dist})
and~(\ref{eq:induction-2-infty}). 
\end{enumerate}
This together with the bounds on $\bm{A}_{1}$ and $\bm{A}_{2}$ gives:
for some constant $\tilde{C}>0$, 
\begin{align*}
 & \big\|\bF^{t+1}\bH^{t+1}-\bF^{t+1,(l)}\bR^{t+1,(l)}\big\|_{\mathrm{F}}\leq\Fnorm{\bA_{1}}+\Fnorm{\bA_{2}}+\left\Vert \bm{A}_{3}\right\Vert _{\mathrm{F}}\\
 & \quad\leq\left(1-\frac{\sigma_{\min}}{20}\eta\right)\big\|\bF^{t}\bH^{t}-\bF^{t,(l)}\bR^{t,(l)}\big\|_{\mathrm{F}}+\eta\left(\sigma\sqrt{\frac{n}{p}}+\frac{\lambda}{p}\right)\left\Vert \bm{F}^{\star}\right\Vert _{2,\infty}\\
 & \quad\quad+\tilde{C}\eta\sigma\sqrt{\frac{n\log n}{p}}\twotoinftynorm{\bm{F}^{\star}}+\tilde{C}\eta\sqrt{\frac{\mu^{2}r^{2}\log n}{np}}\left(C_{\infty}\kappa+C_{3}\right)\left(\frac{\sigma}{\sigma_{\min}}\sqrt{\frac{n\log n}{p}}+\frac{\lambda}{p\,\sigma_{\min}}\right)\left\Vert \bm{F}^{\star}\right\Vert _{2,\infty}\sigma_{\max}\\
 & \quad\overset{\text{(i)}}{\leq}\left(1-\frac{\sigma_{\min}}{20}\eta\right)C_{3}\left(\frac{\sigma}{\sigma_{\min}}\sqrt{\frac{n\log n}{p}}+\frac{\lambda}{p\,\sigma_{\min}}\right)\left\Vert \bm{F}^{\star}\right\Vert _{2,\infty}+\eta\left(\sigma\sqrt{\frac{n}{p}}+\frac{\lambda}{p}\right)\left\Vert \bm{F}^{\star}\right\Vert _{2,\infty}\\
 & \quad\quad+\tilde{C}\eta\sigma\sqrt{\frac{n\log n}{p}}\twotoinftynorm{\bF^{\star}}+\tilde{C}\eta\sqrt{\frac{\mu^{2}r^{2}\log n}{np}}\left(C_{\infty}\kappa+C_{3}\right)\left(\frac{\sigma}{\sigma_{\min}}\sqrt{\frac{n\log n}{p}}+\frac{\lambda}{p\,\sigma_{\min}}\right)\left\Vert \bm{F}^{\star}\right\Vert _{2,\infty}\sigma_{\max}\\
 & \quad\overset{\text{(ii)}}{\leq}C_{3}\left(\frac{\sigma}{\sigma_{\min}}\sqrt{\frac{n\log n}{p}}+\frac{\lambda}{p\,\sigma_{\min}}\right)\left\Vert \bm{F}^{\star}\right\Vert _{2,\infty}
\end{align*}
as claimed. Here, (i) invokes the induction hypothesis~\eqref{eq:induction-loo-dist},
whereas (ii) holds as long as $C_{3}$ is large enough and the sample
size satisfies $n^{2}p\gg\kappa^{4}\mu^{2}r^{2}n\log n$.

\begin{proof}[Proof of Claim~\ref{claim:B1}] For notational simplicity,
we denote 
\begin{equation}
\bm{C}\triangleq\bm{X}^{t,(l)}\bm{Y}^{t,(l)\top}-\bm{M}^{\star}=\bm{X}^{t,(l)}\bm{Y}^{t,(l)\top}-\bm{X}^{\star}\bm{Y}^{\star\top}.\label{eq:def_C}
\end{equation}
Since the Frobenius norm is unitarily invariant, we have 
\[
\left\Vert \bm{B}_{1}\right\Vert _{\mathrm{F}}=\Big\|\underset{:=\bm{W}}{\underbrace{\left[p^{-1}\mathcal{P}_{\Omega_{l,\cdot}}\left(\bm{C}\right)-\mathcal{P}_{l,\cdot}\left(\bm{C}\right)\right]}}\bm{Y}^{t,(l)}\Big\|_{\mathrm{F}}.
\]
All nonzero entries of the matrix $\bm{W}$ reside in its $l$th
row and therefore 
\begin{align*}
p\left\Vert \bm{B}_{1}\right\Vert _{\mathrm{F}} & =\left\Vert \sum\nolimits _{j=1}^{n}(\delta_{lj}-p)C_{lj}\bm{Y}_{j,\cdot}^{t,\left(l\right)}\right\Vert _{2},
\end{align*}
where $\delta_{lj}\triangleq\ind_{\left\{ (l,j)\in\Omega\right\} }$.
Notice that conditional on $\bm{X}^{t,(l)}$ and $\bm{Y}^{t,(l)}$,
the right-hand side is composed of a sum of independent random vectors,
where the randomness comes from $\{\delta_{lj}\}_{1\leq j\leq n}$.
It then follows that 
\begin{align*}
L & \triangleq\max_{1\leq j\leq n}\left\Vert \left(\delta_{l,j}-p\right)C_{l,j}\bm{Y}_{j,\cdot}^{t,\left(l\right)}\right\Vert _{2}\leq\left\Vert \bm{C}\right\Vert _{\infty}\big\|\bm{Y}^{t,\left(l\right)}\big\|_{2,\infty}\overset{\left(\text{i}\right)}{\leq}2\left\Vert \bm{C}\right\Vert _{\infty}\left\Vert \bm{Y}^{\star}\right\Vert _{2,\infty},\\
V & \triangleq\Big\|\sum_{j=1}^{n}\EE\big[\left(\delta_{l,j}-p\right)^{2}\big]C_{l,j}^{2}\bm{Y}_{j,\cdot}^{t,\left(l\right)}\bm{Y}_{j,\cdot}^{t,\left(l\right)\top}\Big\|\leq p\|\bm{C}\|_{\infty}^{2}\Big\|\sum_{j=1}^{n}\bm{Y}_{j,\cdot}^{t,\left(l\right)}\bm{Y}_{j,\cdot}^{t,\left(l\right)\top}\Big\|\\
 & =p\left\Vert \bm{C}\right\Vert _{\infty}^{2}\big\|\bm{Y}^{t,\left(l\right)}\big\|_{\mathrm{F}}^{2}\overset{\left(\text{ii}\right)}{\leq}4p\left\Vert \bm{C}\right\Vert _{\infty}^{2}\left\Vert \bm{Y}^{\star}\right\Vert _{\mathrm{F}}^{2}.
\end{align*}
Here, both (i) and (ii) arise from Lemma~\ref{lemma:immediate-consequence},
as long as $\frac{\sigma}{\sigma_{\min}}\sqrt{\frac{n}{p}}\ll\frac{1}{\sqrt{\kappa^{2}\log n}}$.
The matrix Bernstein inequality~\cite[Theorem 6.1.1]{Tropp:2015:IMC:2802188.2802189}
reveals that 
\begin{align*}
\Big\|\sum_{j=1}^{n}\left(\delta_{l,j}-p\right)C_{l,j}\bm{Y}_{j,\cdot}^{t,\left(l\right)}\Big\|_{2} & \lesssim\sqrt{V\log n}+L\log n\lesssim\sqrt{p\left\Vert \bm{C}\right\Vert _{\infty}^{2}\left\Vert \bm{Y}^{\star}\right\Vert _{\mathrm{F}}^{2}\log n}+\left\Vert \bm{C}\right\Vert _{\infty}\left\Vert \bm{Y}^{\star}\right\Vert _{2,\infty}\log n
\end{align*}
with probability exceeding $1-O(n^{-100})$. As a result, we arrive
at 
\begin{align}
p\left\Vert \bm{B}_{1}\right\Vert _{\mathrm{F}}\lesssim\sqrt{p\log n}\left\Vert \bm{C}\right\Vert _{\infty}\left\Vert \bm{Y}^{\star}\right\Vert _{\mathrm{F}}+\sqrt{np}\left\Vert \bm{C}\right\Vert _{\infty}\left\Vert \bm{Y}^{\star}\right\Vert _{2,\infty}\label{eq:nu2-bound-MC}
\end{align}
as soon as $np\gg\log^2 n$.

To finish up, we make the observation that 
\begin{align}
\left\Vert \bm{C}\right\Vert _{\infty} & =\big\|\bm{X}^{t,(l)}\bm{R}^{t,\left(l\right)}\big(\bm{Y}^{t,(l)}\bm{R}^{t,\left(l\right)}\big)^{\top}-\bm{X}^{\star}\bm{Y}^{\star\top}\big\|_{\infty}\nonumber \\
 & \leq\left\Vert \left(\bm{X}^{t,(l)}\bm{R}^{t,\left(l\right)}-\bm{X}^{\star}\right)\left(\bm{Y}^{t,(l)}\bm{R}^{t,\left(l\right)}\right)^{\top}\right\Vert _{\infty}+\left\Vert \bm{X}^{\star}\left(\bm{Y}^{t,(l)}\bm{R}^{t,\left(l\right)}-\bm{Y}^{\star}\right)^{\top}\right\Vert _{\infty}\nonumber \\
 & \leq\left\Vert \bm{X}^{t,(l)}\bm{R}^{t,\left(l\right)}-\bm{X}^{\star}\right\Vert _{2,\infty}\left\Vert \bm{Y}^{t,(l)}\bm{R}^{t,\left(l\right)}\right\Vert _{2,\infty}+\left\Vert \bm{X}^{\star}\right\Vert _{2,\infty}\left\Vert \bm{Y}^{t,(l)}\bm{R}^{t,\left(l\right)}-\bm{Y}^{\star}\right\Vert _{2,\infty}\nonumber \\
 & \leq3\big\|\bm{F}^{t,(l)}\bm{R}^{t,\left(l\right)}-\bm{F}^{\star}\big\|_{2,\infty}\left\Vert \bm{F}^{\star}\right\Vert _{2,\infty},\label{eq:Xt-M-inf-MC}
\end{align}
where the last line arises from Lemma~\ref{lemma:immediate-consequence}.
This combined with~\eqref{eq:nu2-bound-MC} gives 
\begin{align*}
\left\Vert \bm{B}_{1}\right\Vert _{\mathrm{F}} & \lesssim\sqrt{\frac{\log n}{p}}\left\Vert \bm{C}\right\Vert _{\infty}\left\Vert \bm{Y}^{\star}\right\Vert _{\mathrm{F}}+\sqrt{\frac{n}{p}}\left\Vert \bm{C}\right\Vert _{\infty}\left\Vert \bm{Y}^{\star}\right\Vert _{2,\infty}\\
 & \overset{(\text{i})}{\lesssim}\sqrt{\frac{\log n}{p}}\big\|\bm{F}^{t,(l)}\bm{R}^{t,\left(l\right)}-\bm{F}^{\star}\big\|_{2,\infty}\left\Vert \bm{F}^{\star}\right\Vert _{2,\infty}\left\Vert \bm{Y}^{\star}\right\Vert _{\mathrm{F}}+\sqrt{\frac{n}{p}}\big\|\bm{F}^{t,(l)}\bm{R}^{t,\left(l\right)}-\bm{F}^{\star}\big\|_{2,\infty}\left\Vert \bm{F}^{\star}\right\Vert _{2,\infty}^{2}\\
 & \overset{(\text{ii})}{\lesssim}\sqrt{\frac{\log n}{p}}\big\|\bm{F}^{t,(l)}\bm{R}^{t,\left(l\right)}-\bm{F}^{\star}\big\|_{2,\infty}\sqrt{\frac{\mu r^{2}}{n}}\sigma_{\max}+\sqrt{\frac{n}{p}}\big\|\bm{F}^{t,(l)}\bm{R}^{t,\left(l\right)}-\bm{F}^{\star}\big\|_{2,\infty}\frac{\mu r}{n}\sigma_{\max}\\
 & \lesssim\sqrt{\frac{\mu^{2}r^{2}\log n}{np}}\big\|\bm{F}^{t,(l)}\bm{R}^{t,\left(l\right)}-\bm{F}^{\star}\big\|_{2,\infty}\sigma_{\max},
\end{align*}
where (i) comes from~\eqref{eq:Xt-M-inf-MC}, and (ii) makes use of
the incoherence condition $\|\bm{F}^{\star}\|_{2,\infty}\leq\sqrt{\mu r\sigma_{\max}/n}$
and the fact that $\|\bm{Y}^{\star}\|_{\mathrm{F}}\leq\sqrt{r\sigma_{\max}}$.
\end{proof}

\begin{proof}[Proof of Claim~\ref{claim:B2}] Instate the notation
in proof of Claim~\ref{claim:B1}. By the unitary invariance of Frobenius
norm and the fact that all nonzero entries of the matrix $\bm{W}$
reside in its $l$th row, we have 
\[
p\left\Vert \bm{B}_{2}\right\Vert _{\mathrm{F}}=\Fnorm{p\bW^{\top}\bX^{t,(l)}}=\Bigg\|\underbrace{\footnotesize\left[\begin{matrix}\left(\delta_{l1}-p\right)C_{l1}\\
\vdots\\
\left(\delta_{ln}-p\right)C_{ln}
\end{matrix}\right]}_{:=\bb}\bX_{l,\cdot}^{t,(l)}\Bigg\|_{\mathrm{F}}=\twonorm{\bb}\big\|\bX_{l,\cdot}^{t,(l)}\big\|_{2}.
\]
We can write $\bb$ as 
\[
\bb=\sum\nolimits _{j=1}^{n}\underbrace{\be_{j}\left(\delta_{lj}-p\right)C_{lj}}_{:=\bu_{j}}=\sum\nolimits _{j=1}^{n}\bu_{j}.
\]
Note that for all $j$, one has 
\begin{align*}
L & \triangleq\max_{1\leq j\leq n}\twonorm{\bu_{j}}\leq\inftynorm{\bC},\\
V & \triangleq\left\Vert \sum\nolimits _{j=1}^{n}\EE\big[\left(\delta_{lj}-p\right)^{2}\big]C_{lj}^{2}\be_{j}^{\top}\be_{j}\right\Vert \leq p\|\bm{C}\|_{\infty}^{2}\left\Vert \sum\nolimits _{j=1}^{n}\be_{j}^{\top}\be_{j}\right\Vert =np\left\Vert \bm{C}\right\Vert _{\infty}^{2}.
\end{align*}
Then the matrix Bernstein inequality~\cite[Theorem 6.1.1]{Tropp:2015:IMC:2802188.2802189}
reveals that 
\begin{align*}
\left\Vert \bb\right\Vert _{2} & \lesssim\sqrt{V\log n}+L\log n\lesssim\sqrt{np\log n}\inftynorm{\bC}+\inftynorm{\bC}\log n\\
 & \lesssim\sqrt{np\log n}\inftynorm{\bC}\\
 & \lesssim\sqrt{np\log n}\big\|\bm{F}^{t,(l)}\bm{R}^{t,\left(l\right)}-\bm{F}^{\star}\big\|_{2,\infty}\left\Vert \bm{F}^{\star}\right\Vert _{2,\infty}
\end{align*}
with probability exceeding $1-O\left(n^{-100}\right)$ as long as
$np\gg\log n$. Here the last relation uses~(\ref{eq:Xt-M-inf-MC}).
Observe that $\|\bm{X}^{t,(l)}\|_{2,\infty}\leq2\|\bm{F}^{\star}\|_{2,\infty}$
as long as $\frac{\sigma}{\sigma_{\min}}\sqrt{\frac{n}{p}}\ll\frac{1}{\sqrt{\kappa^{2}\log n}}$;
see Lemma~\ref{lemma:immediate-consequence}. Making use of the incoherence
condition~(\ref{eq:F-singular-value}) to get
\[
\Fnorm{\bB_{2}}\lesssim\sqrt{\frac{n\log n}{p}}\big\|\bm{F}^{t,(l)}\bm{R}^{t,\left(l\right)}-\bm{F}^{\star}\big\|_{2,\infty}\left\Vert \bm{F}^{\star}\right\Vert _{2,\infty}^{2}\lesssim\sqrt{\frac{\mu^{2}r^{2}\log n}{np}}\big\|\bm{F}^{t,(l)}\bm{R}^{t,\left(l\right)}-\bm{F}^{\star}\big\|_{2,\infty}\sigma_{\max}.
\]
We can then conclude the proof. \end{proof}

\begin{proof}[Proof of Claim~\ref{claim:C1C2}] By the unitary
invariance of the Frobenius norm, one has 
\[
\Fnorm{\bC_{1}}=p^{-1}\big\|\cP_{\Omega_{l,\cdot}}\left(\bE\right)\bY^{t,(l)}\big\|_{\mathrm{F}}.
\]
Since the entries of $\cP_{\Omega_{l,\cdot}}(\bE)$ are all zero except
those on the $l$th row, we have 
\[
p\Fnorm{\bC_{1}}=\Big\|\sum\nolimits _{j=1}^{n}\underbrace{\delta_{lj}E_{lj}\bY_{j,\cdot}^{t,(l)}}_{:=\bu_{j}}\Big\|_{2},
\]
where we denote $\delta_{lj}\triangleq\ind_{(l,j)\in\Omega}$. Since
$\bY^{t,(l)}$ is independent of $\{\delta_{lj}\}_{1\leq j\leq n}$
and $\{E_{lj}\}_{1\leq j\leq n}$, the vectors $\{\bu_{j}\}_{1\leq j\leq n}$
are independent conditioning on $\bm{Y}^{t,(l)}$. Therefore, from
now on we shall condition on a fixed $\bm{Y}^{t,(l)}$. It is easy
to verify that 
\[
\big\|\|\bm{u}_{j}\|_{2}\big\|_{\psi_{1}}\leq\big\|\bm{Y}^{t,\left(l\right)}\big\|_{2,\infty}\left\Vert \delta_{lj}E_{lj}\right\Vert _{\psi_{1}}\lesssim\sigma\big\|\bm{Y}^{t,\left(l\right)}\big\|_{2,\infty},
\]
where $\|\cdot\|_{\psi_{1}}$ denotes the sub-exponential norm~\cite[Section 6]{MR2906869}.
Further, one can calculate 
\begin{align*}
V & :=\left\Vert \EE\left[\sum\nolimits _{j=1}^{n}\left(\delta_{lj}E_{lj}\right)^{2}\bm{Y}_{j,\cdot}^{t,\left(l\right)}\bm{Y}_{j,\cdot}^{t,\left(l\right)\top}\right]\right\Vert \lesssim p\sigma^{2}\left\Vert \EE\left[\sum\nolimits _{j=1}^{n}\bm{Y}_{j,\cdot}^{t,\left(l\right)}\bm{Y}_{j,\cdot}^{t,\left(l\right)\top}\right]\right\Vert =p\sigma^{2}\big\|\bm{Y}^{t,\left(l\right)}\big\|_{\mathrm{F}}^{2}.
\end{align*}
Invoke the matrix Bernstein inequality~\cite[Proposition 2]{MR2906869}
to discover that with probability at least $1-O\left(n^{-100}\right)$,
\begin{align*}
\Big\|\sum\nolimits _{j=1}^{n}\bm{u}_{j}\Big\|_{2} & \lesssim\sqrt{V\log n}+\Big\|\|\bm{u}_{j}\|_2\Big\|_{\psi_{1}}\log^{2}n\\
 & \lesssim\sqrt{p\sigma^{2}\left\Vert \bm{Y}^{t,\left(l\right)}\right\Vert _{\mathrm{F}}^{2}\log n}+\sigma\big\|\bm{Y}^{t,\left(l\right)}\big\|_{2,\infty}\log^{2}n\\
 & \lesssim\sigma\sqrt{np\log n}\big\|\bm{Y}^{t,\left(l\right)}\big\|_{2,\infty}+\sigma\big\|\bm{Y}^{t,\left(l\right)}\big\|_{2,\infty}\log^{2}n\\
 & \lesssim\sigma\sqrt{np\log n}\big\|\bm{Y}^{t,\left(l\right)}\big\|_{2,\infty},
\end{align*}
where the third inequality follows from $\| \bm{Y}^{t,\left(l\right)}\|_{\mathrm{F}}^{2}\leq n\| \bm{Y}^{t,\left(l\right)}\| _{2,\infty}^{2}$,
and the last inequality holds if $np\gg\log^{3}n$. We then complete
the proof by observing that $\|\bm{Y}^{t,\left(l\right)}\|_{2,\infty}\leq2\|\bm{F}^{\star}\|_{2,\infty}$
as long as $\frac{\sigma}{\sigma_{\min}}\sqrt{\frac{n}{p}}\ll\frac{1}{\sqrt{\kappa^{2}\log n}}$;
see Lemma~\ref{lemma:immediate-consequence}. The bound on $\bm{C}_{2}$
follows from similar arguments to that used to bound $\bm{B}_{2}$. \end{proof}

\subsection{Proof of Lemma~\ref{lemma:loo-l-contraction}}

Without loss of generality, we assume $1\leq l\leq n$. One can then
decompose $(\bF^{t+1,(l)}\bH^{t+1,(l)}-\bF^{\star})_{l,\cdot}$ as
\begin{align*}
 & (\bF^{t+1,(l)}\bH^{t+1,(l)}-\bF^{\star})_{l,\cdot}=\bX_{l,\cdot}^{t+1,(l)}\bH^{t+1,(l)}-\bX_{l,\cdot}^{\star}\\
 & =\left\{ \bX_{l,\cdot}^{t,(l)}-\eta\big[(\bX^{t,(l)}\bY^{t,(l)\top}-\bM^{\star})_{l,\cdot}\bY^{t,(l)}+\tfrac{\lambda}{p}\bX_{l,\cdot}^{t,(l)}\big]\right\} \bH^{t+1,(l)}-\bX_{l,\cdot}^{\star}\\
 & =\bX_{l,\cdot}^{t,(l)}\bH^{t+1,(l)}-\bX_{l,\cdot}^{\star}-\eta\big[\big(\bX^{t,(l)}\bY^{t,(l)\top}-\bM^{\star}\big)_{l,\cdot}\bY^{t,(l)}+\tfrac{\lambda}{p}\bX_{l,\cdot}^{t,(l)}\big]\bH^{t+1,(l)}\\
 & =\underbrace{\bX_{l,\cdot}^{t,(l)}\bH^{t,(l)}-\bX_{l,\cdot}^{\star}-\eta\big[\big(\bX^{t,(l)}\bY^{t,(l)\top}-\bM^{\star}\big)_{l,\cdot}\bY^{t,(l)}+\tfrac{\lambda}{p}\bX_{l,\cdot}^{t,(l)}\big]\bH^{t,(l)}}_{:=\bm{a}_{1}}\\
 & \quad+\underbrace{\left\{ \bX_{l,\cdot}^{t,(l)}\bH^{t,(l)}-\eta\big[\big(\bX^{t,(l)}\bY^{t,(l)\top}-\bM^{\star}\big)_{l,\cdot}\bY^{t,(l)}+\tfrac{\lambda}{p}\bX_{l,\cdot}^{t,(l)}\big]\bH^{t,(l)}\right\} \left[\big(\bH^{t,(l)}\big)^{-1}\bH^{t+1,(l)}-\bI_{r}\right]}_{:=\bm{a}_{2}}.
\end{align*}
Note that here $\bm{a}_{1}$ and $\bm{a}_{2}$ are $r$-dimensional
row vectors. In the sequel, let us control $\|\bm{a}_{1}\|_{2}$ and
$\|\bm{a}_{2}\|_{2}$ separately. 
\begin{enumerate}
\item We begin with $\ba_{1}$. For notational convenience, define $\bDelta^{t,(l)}\triangleq\left[\footnotesize\begin{array}{c}
\bm{\Delta}_{\bm{X}}^{t,(l)}\\
\bm{\Delta}_{\bm{Y}}^{t,(l)}
\end{array}\right]$, where $\bDelta_{\bX}^{t,(l)}\triangleq\bX^{t,(l)}\bH^{t,(l)}-\bX^{\star}$
and $\bDelta_{\bY}^{t,(l)}\triangleq\bY^{t,(l)}\bH^{t,(l)}-\bY^{\star}$.
Then $\bm{a}_{1}$ can be rewritten as 
\begin{align*}
\ba_{1} & =(\bDelta_{\bX}^{t,(l)})_{l,\cdot}-\eta\left[(\bm{\Delta}_{\bm{X}}^{t,(l)})_{l,\cdot}(\bm{Y}^{t,(l)}\bm{H}^{t,(l)})^{\top}+\bm{X}_{l,\cdot}^{\star}\bm{\Delta}_{\bm{Y}}^{t,(l)\top}\right]\bm{Y}^{t,(l)}\bm{H}^{t,(l)}-\eta\frac{\lambda}{p}\bX_{l,\cdot}^{t,(l)}\bH^{t,(l)}\\
 & =(\bDelta_{\bX}^{t,(l)})_{l,\cdot}\left[\bm{I}_{r}-\eta(\bm{Y}^{t,(l)}\bm{H}^{t,(l)})^{\top}\bm{Y}^{t,(l)}\bm{H}^{t,(l)}\right]-\eta\bm{X}_{l,\cdot}^{\star}\bm{\Delta}_{\bm{Y}}^{t,(l)\top}\bm{Y}^{t,(l)}\bm{H}^{t,(l)}-\eta\frac{\lambda}{p}\bX_{l,\cdot}^{t,(l)}\bH^{t,(l)},
\end{align*}
which together with the triangle inequality yields 
\begin{align*}
\twonorm{\ba_{1}} & \leq\big\|\bm{I}_{r}-\eta(\bm{Y}^{t,(l)}\bm{H}^{t,(l)})^{\top}\bm{Y}^{t,(l)}\bm{H}^{t,(l)}\big\|\cdot\big\|\big(\bDelta_{\bX}^{t,(l)}\big)_{l,\cdot}\big\|_{2}\\
 & \qquad+\eta\|\bm{X}_{l,\cdot}^{\star}\|_{2}\cdot\big\|\bm{\Delta}_{\bm{Y}}^{t,(l)}\big\|\cdot\big\|\bm{Y}^{t,(l)}\bm{H}^{t,(l)}\big\|+\eta\frac{\lambda}{p}\|\bX^{t,(l)}\bm{H}^{t,(l)}\|_{2,\infty}.
\end{align*}
In view of Lemma~\ref{lemma:immediate-consequence}, we have 
\begin{align}
\sigma_{\min}/2\leq\sigma_{\min}\left[(\bm{Y}^{t,(l)}\bm{H}^{t,(l)})^{\top}\bm{Y}^{t,(l)}\bm{H}^{t,(l)}\right] & \leq\sigma_{\max}\left[(\bm{Y}^{t,(l)}\bm{H}^{t,(l)})^{\top}\bm{Y}^{t,(l)}\bm{H}^{t,(l)}\right]\leq2\sigma_{\max},\nonumber \\
\|\bm{\Delta}_{\bm{Y}}^{t,(l)}\|\leq\|\bDelta^{t,(l)}\| & \leq2C_{\mathrm{op}}\left(\frac{\sigma}{\sigma_{\min}}\sqrt{\frac{n}{p}}+\frac{\lambda}{p\,\sigma_{\min}}\right)\left\Vert \bm{X}^{\star}\right\Vert ,\label{eq:Delta-t-l-op-norm}\\
\|\bm{Y}^{t,(l)}\bm{H}^{t,(l)}\|\leq\|\bm{F}^{t,(l)}\| & \leq2\|\bm{X}^{\star}\|\qquad\mathrm{and}\nonumber \\
\|\bX^{t,(l)}\bm{H}^{t,(l)}\|_{2,\infty}\leq\|\bm{F}^{t,(l)}\|_{2,\infty} & \leq2\|\bm{F}^{\star}\|_{2,\infty},\nonumber 
\end{align}
provided that the sample size obeys $n\gg\kappa\mu$ and that the
noise satisfies $\frac{\sigma}{\sigma_{\min}}\sqrt{\frac{n}{p}}\ll\frac{1}{\sqrt{\kappa^{2}\log n}}$.
These allow us to further upper bound $\|\bm{a}_{1}\|_{2}$ by 
\begin{align*}
\twonorm{\ba_{1}} & \leq\left(1-\frac{\eta\sigma_{\min}}{2}\right)\big\|(\bDelta_{\bX}^{t,(l)})_{l,\cdot}\big\|_{2}+4\eta C_{\mathrm{op}}\left(\frac{\sigma}{\sigma_{\min}}\sqrt{\frac{n}{p}}+\frac{\lambda}{p\,\sigma_{\min}}\right)\sigma_{\max}\|\bm{F}^{\star}\|_{2,\infty}+2\eta\frac{\lambda}{p}\|\bm{F}^{\star}\|_{2,\infty},
\end{align*}
as long as $\eta\leq1/(2\sigma_{\max})$. As an immediate consequence,
\begin{align}
\|\bm{a}_{1}\|_{2} & \leq C_{4}\kappa\left(\frac{\sigma}{\sigma_{\min}}\sqrt{\frac{n\log n}{p}}+\frac{\lambda}{p\,\sigma_{\min}}\right)\left\Vert \bm{F}^{\star}\right\Vert _{2,\infty}+4\eta C_{\mathrm{op}}\left(\frac{\sigma}{\sigma_{\min}}\sqrt{\frac{n}{p}}+\frac{\lambda}{p\,\sigma_{\min}}\right)\sigma_{\max}\|\bm{F}^{\star}\|_{2,\infty}+\frac{2\eta\lambda}{p}\|\bm{F}^{\star}\|_{2,\infty}\nonumber \\
 & \leq\|\bm{F}^{\star}\|_{2,\infty},\label{eq:a1-upper-bound}
\end{align}
where the first inequality follows from the induction hypothesis~(\ref{eq:induction-loo-l})
and the last one holds as long as $\frac{\sigma}{\sigma_{\min}}\sqrt{\frac{n}{p}}\ll\frac{1}{\sqrt{\kappa^{2}\log n}}$
and $\eta\ll1/\sigma_{\max}$. 
\item Next, we turn attention to $\|\ba_{2}\|_{2}$, which satisfies 
\[
\twonorm{\ba_{2}}=\twonorm{\left(\ba_{1}+\bX_{l,\cdot}^{\star}\right)\big[(\bH^{t,(l)})^{-1}\bH^{t+1,(l)}-\bI_{r}\big]}\leq\norm{(\bH^{t,(l)})^{-1}\bH^{t+1,(l)}-\bI_{r}}\twonorm{\ba_{1}+\bX_{l,\cdot}^{\star}}.
\]
From~(\ref{eq:a1-upper-bound}), it is easily seen that 
\[
\twonorm{\ba_{1}+\bX_{l,\cdot}^{\star}}\leq\twonorm{\ba_{1}}+\twotoinftynorm{\bF^{\star}}\leq2\twotoinftynorm{\bF^{\star}}.
\]
Regarding the term $\|(\bH^{t,(l)})^{-1}\bH^{t+1,(l)}-\bm{I}_{r}\|$,
we find the following claim useful. \begin{claim} \label{claim:loo-matrix}
With probability at least $1-O(n^{-100})$, we have 
\[
\norm{(\bH^{t,(l)})^{-1}\bH^{t+1,(l)}-\bm{I}_{r}}\lesssim\eta\kappa C_{\mathrm{op}}^{2}\left(\frac{\sigma}{\sigma_{\min}}\sqrt{\frac{n}{p}}+\frac{\lambda}{p\,\sigma_{\min}}\right)^{2}\sigma_{\max}+\eta^{2}C_{\mathrm{B}}\kappa^{2}\left(\frac{\sigma}{\sigma_{\min}}\sqrt{\frac{n}{p}}+\frac{\lambda}{p\,\sigma_{\min}}\right)\sqrt{r}\sigma_{\max}^{2},
\]
provided that $C_{\mathrm{op}} \gg 1$.
\end{claim} 
\end{enumerate}
Finally, taking the bounds on $\|\bm{a}_{1}\|_{2}$ and $\|\bm{a}_{2}\|_{2}$
collectively yields that: for some absolute constant $\tilde{C}>0$,
\begin{align*}
 & \big\|\big(\bF^{t+1,(l)}\bH^{t+1,(l)}-\bF^{\star}\big)_{l,\cdot}\big\|_{2}\leq\twonorm{\ba_{1}}+\twonorm{\ba_{2}}\\
 & \leq\left(1-\frac{\eta}{2}\sigma_{\min}\right)C_{4}\kappa\left(\frac{\sigma}{\sigma_{\min}}\sqrt{\frac{n\log n}{p}}+\frac{\lambda}{p\,\sigma_{\min}}\right)\left\Vert \bm{F}^{\star}\right\Vert _{2,\infty}+4C_{\mathrm{op}}\eta\left(\frac{\sigma}{\sigma_{\min}}\sqrt{\frac{n}{p}}+\frac{\lambda}{p\,\sigma_{\min}}\right)\sigma_{\max}\twotoinftynorm{\bF^{\star}}+\frac{2\eta\lambda}{p}\twotoinftynorm{\bF^{\star}}\\
 & \quad+\tilde{C}\eta\kappa C_{\mathrm{op}}^{2}\left(\frac{\sigma}{\sigma_{\min}}\sqrt{\frac{n}{p}}+\frac{\lambda}{p\,\sigma_{\min}}\right)^{2}\sigma_{\max}\twotoinftynorm{\bF^{\star}}+\tilde{C}\eta^{2}C_{\mathrm{B}}\kappa^{2}\left(\frac{\sigma}{\sigma_{\min}}\sqrt{\frac{n}{p}}+\frac{\lambda}{p\,\sigma_{\min}}\right)\sqrt{r}\sigma_{\max}^{2}\twotoinftynorm{\bF^{\star}}\\
 & \leq C_{4}\kappa\left(\frac{\sigma}{\sigma_{\min}}\sqrt{\frac{n\log n}{p}}+\frac{\lambda}{p\,\sigma_{\min}}\right)\left\Vert \bm{F}^{\star}\right\Vert _{2,\infty},
\end{align*}
provided that $C_{4}\gg C_{\mathrm{op}}$, $\frac{\sigma}{\sigma_{\min}}\sqrt{\frac{n}{p}}\ll1/\kappa$
and $\eta\ll1/(\kappa^{2}\sqrt{r}\sigma_{\max})$. This finishes the
proof of the lemma. It remains to establish Claim~\ref{claim:loo-matrix}.

\begin{proof}[Proof of Claim~\ref{claim:loo-matrix}] To facilitate
analysis, we introduce an auxiliary point $\tilde{\bm{F}}^{t+1}\triangleq\left[\begin{array}{c}
\tilde{\bX}^{t+1,(l)}\\
\tilde{\bY}^{t+1,(l)}
\end{array}\right]$ where 
\begin{align*}
\tilde{\bX}^{t+1,(l)} & =\bX^{t,(l)}\bH^{t,(l)}-\eta\left[\tfrac{1}{p}\cP_{\Omega_{-l,\cdot}}\left(\bX^{t,(l)}\bY^{t,(l)\top}-\bM^{\star}-\bE\right)+\cP_{l,\cdot}\left(\bX^{t,(l)}\bY^{t,(l)\top}-\bM^{\star}\right)\right]\bY^{\star}\\
 & \quad-\eta\frac{\lambda}{p}\bX^{\star}-\frac{\eta}{2}\bX^{\star}\bH^{t,(l)\top}\left(\bX^{t,(l)\top}\bX^{t,(l)}-\bY^{t,(l)\top}\bY^{t,(l)}\right)\bH^{t,(l)},\\
\tilde{\bY}^{t+1,(l)} & =\bY^{t,(l)}\bH^{t,(l)}-\eta\left[\tfrac{1}{p}\cP_{\Omega_{-l,\cdot}}\left(\bX^{t,(l)}\bY^{t,(l)\top}-\bM^{\star}-\bE\right)+\cP_{l,\cdot}\left(\bX^{t,(l)}\bY^{t,(l)\top}-\bM^{\star}\right)\right]^{\top}\bX^{\star}\\
 & \quad-\eta\frac{\lambda}{p}\bY^{\star}-\frac{\eta}{2}\bY^{\star}\bH^{t,(l)\top}\left(\bY^{t,(l)\top}\bY^{t,(l)}-\bX^{t,(l)\top}\bX^{t,(l)}\right)\bH^{t,(l)}.
\end{align*}
We first claim that $\bI_{r}$ is the best rotation matrix to align
$\tilde{\bF}^{t+1,(l)}$ and $\bF^{\star}$; its proof is similar
to that of Claim~\ref{claim:already-aligned} and is hence omitted
for brevity. \begin{claim}\label{claim:identity-rotation}One has
\[
\bm{I}_{r}=\arg\min_{\bm{R}\in\mathcal{O}^{r}}\big\|\tilde{\bm{F}}^{t+1,(l)}\bm{R}-\bm{F}^{\star}\big\|_{\mathrm{F}}\qquad\mathrm{and}\qquad\sigma_{\min}\big(\tilde{\bF}^{t+1,(l)\top}\bF^{\star}\big)\geq\sigma_{\min}/2.
\]
\end{claim}With this claim at hand, we intend to invoke Lemma~\ref{lemma:Ma36}
with 
\[
\bm{S}=\tilde{\bF}^{t+1,(l)\top}\bF^{\star},\quad\bm{K}=\big(\bF^{t+1,(l)}\bH^{t,(l)}-\tilde{\bF}^{t+1,(l)}\big)^{\top}\bF^{\star}
\]
to get 
\begin{align}
\big\|(\bH^{t,(l)})^{-1}\bH^{t+1,(l)}-\bm{I}_{r}\big\| & =\|\mathsf{sgn}(\bm{S}+\bm{K})-\mathsf{sgn}(\bm{S})\|\leq\frac{1}{\sigma_{\min}(\bm{S})}\|\bm{K}\|\nonumber \\
 & =\frac{1}{\sigma_{\min}(\tilde{\bF}^{t+1,(l)\top}\bF^{\star})}\big\|\big(\bF^{t+1,(l)}\bH^{t,(l)}-\tilde{\bF}^{t+1,(l)}\big)^{\top}\bF^{\star}\big\|.\nonumber \\
 & \leq\frac{2}{\sigma_{\min}}\big\|\bF^{t+1,(l)}\bH^{t,(l)}-\tilde{\bF}^{t+1,(l)}\big\|\|\bm{F}^{\star}\|,\label{eq:rotation-diff}
\end{align}
where the last line uses Claim~\ref{claim:identity-rotation}. Here
$\mathsf{sgn}(\bm{A})=\bm{U}\bm{V}^{\top}$ for any matrix $\bm{A}$
with SVD $\bm{U}\bm{\Sigma}\bm{V}^{\top}$. It then boils down to
controlling $\|\bF^{t+1,(l)}\bH^{t,(l)}-\tilde{\bF}^{t+1,(l)}\|$,
for which we have 
\begin{align*}
\bF^{t+1,(l)}\bH^{t,(l)}-\tilde{\bF}^{t+1,(l)} & =\eta\left[\begin{matrix}\bB & \bm{0}\\
\bm{0} & \bB^{\top}
\end{matrix}\right]\left[\begin{matrix}\bDelta_{\bY}^{t,(l)}\\
\bDelta_{\bX}^{t,(l)}
\end{matrix}\right]+\frac{\eta}{2}\left[\begin{array}{c}
\bm{X}^{\star}\\
-\bm{Y}^{\star}
\end{array}\right]\bH^{t,(l)\top}\bm{C}\bH^{t,(l)}-\eta\frac{\lambda}{p}\bDelta^{t,(l)},
\end{align*}
where we denote 
\begin{align*}
\bB & \triangleq-p^{-1}\cP_{\Omega_{-l,\cdot}}\big(\bX^{t,(l)}\bY^{t,(l)\top}-\bM^{\star}-\bE\big)-\cP_{l,\cdot}\big(\bX^{t,(l)}\bY^{t,(l)\top}-\bM^{\star}\big);\\
\bC & \triangleq\bX^{t,(l)\top}\bX^{t,(l)}-\bY^{t,(l)\top}\bY^{t,(l)}.
\end{align*}
This enables us to obtain 
\begin{equation}
\big\|\bF^{t+1,(l)}\bH^{t,(l)}-\tilde{\bF}^{t+1,(l)}\big\|\leq\eta\|\bm{B}\|\big\|\bm{\Delta}^{t,(l)}\big\|+\frac{\eta}{2}\|\bm{F}^{\star}\|\|\bm{C}\|_{\mathrm{F}}+\frac{\eta\lambda}{p}\|\bm{\Delta}^{t,(l)}\|.\label{eq:upper-bound-111}
\end{equation}
In view of Lemma~\ref{lemma:balancing}, one has 
\begin{equation}
\|\bm{C}\|_{\mathrm{F}}\leq C_{\mathrm{B}}\kappa\eta\left(\frac{\sigma}{\sigma_{\min}}\sqrt{\frac{n}{p}}+\frac{\lambda}{p\,\sigma_{\min}}\right)\sqrt{r}\sigma_{\max}^{2}.\label{eq:upper-bound-222}
\end{equation}
We are left with bounding $\|\bm{B}\|$. Decompose $\bB$ into 
\begin{align*}
\bB=\underbrace{-\tfrac{1}{p}\cP_{\Omega}\left(\bX^{t,(l)}\bY^{t,(l)\top}-\bM^{\star}\right)}_{:=\bB_{1}}+\underbrace{\tfrac{1}{p}\cP_{\Omega_{l,\cdot}}\left(\bX^{t,(l)}\bY^{t,(l)\top}-\bM^{\star}\right)-\cP_{l,\cdot}\left(\bX^{t,(l)}\bY^{t,(l)\top}-\bM^{\star}\right)}_{:=\bB_{2}}+\underbrace{\tfrac{1}{p}\cP_{\Omega_{-l,\cdot}}\left(\bE\right)}_{:=\bB_{3}}.
\end{align*}
To control $\bB_{1}$, following the same argument in Lemma~\ref{lemma:P-tilde}, we see that 
\begin{align*}
\norm{\bB_{1}} & \leq\big\| p^{-1}\cP_{\Omega}(\bX^{t,(l)}\bY^{t,(l)\top}-\bM^{\star})-(\bX^{t,(l)}\bY^{t,(l)\top}-\bM^{\star})\big\|+\big\|\bX^{t,(l)}\bY^{t,(l)\top}-\bM^{\star}\big\|\\
 & \lesssim\sqrt{n/p}\,\big\|\bF^{t,(l)}\bR^{t,(l)}-\bF^{\star}\big\|_{2,\infty}\twotoinftynorm{\bF^{\star}}+\big\|\bF^{t,(l)}\bR^{t,(l)}-\bF^{\star}\big\|\norm{\bF^{\star}}.
\end{align*}
We now move on to $\|\bB_{2}\|$, which is equal to $\|\bb\|_{2}/p$
defined in the proof of Claim~\ref{claim:B2}, namely, 
\[
\norm{\bB_{2}}=\twonorm{\bb}/p\lesssim\sqrt{n\log n/p}\,\big\|\bF^{t,(l)}\bR^{t,(l)}-\bF^{\star}\big\|_{2,\infty}\twotoinftynorm{\bF^{\star}}.
\]
The last term $\bB_{3}$ can be easily bound via Lemma~\ref{lemma:noise-bound},
that is, 
\[
\norm{\bB_{3}}\leq p^{-1}\norm{\cP_{\Omega}\left(\bE\right)}\lesssim\sigma\sqrt{n/p}.
\]
Combining the above three bounds with Lemma~\ref{lemma:immediate-consequence},
we arrive at 
\begin{align}
\norm{\bB} & \leq\norm{\bB_{1}}+\norm{\bB_{2}}+\norm{\bB_{3}}\nonumber \\
 & \lesssim\sqrt{\frac{n\log n}{p}}\big\|\bF^{t,(l)}\bR^{t,(l)}-\bF^{\star}\big\|_{2,\infty}\twotoinftynorm{\bF^{\star}}+\big\|\bF^{t,(l)}\bR^{t,(l)}-\bF^{\star}\big\|\norm{\bF^{\star}}+\sigma\sqrt{\frac{n}{p}}\nonumber \\
 & \lesssim\sqrt{\frac{n\log n}{p}}\left(C_{\infty}\kappa+C_{3}\right)\left(\frac{\sigma}{\sigma_{\min}}\sqrt{\frac{n\log n}{p}}+\frac{\lambda}{p\,\sigma_{\min}}\right)\frac{\mu r}{n}\sigma_{\max}+2C_{\mathrm{op}}\left(\frac{\sigma}{\sigma_{\min}}\sqrt{\frac{n}{p}}+\frac{\lambda}{p\,\sigma_{\min}}\right)\sigma_{\max}+\sigma\sqrt{\frac{n}{p}}\nonumber \\
 & \lesssim C_{\mathrm{op}}\left(\frac{\sigma}{\sigma_{\min}}\sqrt{\frac{n}{p}}+\frac{\lambda}{p\,\sigma_{\min}}\right)\sigma_{\max},\label{eq:upper-bound-333}
\end{align}
provided that $n^{2}p\gg\kappa^{2}\mu^{2}r^{2}n\log^2 n$ and $C_{\mathrm{op}}>0$
is large enough. Taking~(\ref{eq:upper-bound-111}),~(\ref{eq:upper-bound-222})
and~(\ref{eq:upper-bound-333}) collectively, we arrive at 
\begin{align*}
 & \big\|\bF^{t+1,(l)}\bH^{t,(l)}-\tilde{\bF}^{t+1,(l)}\big\|\\
 & \quad\lesssim\eta C_{\mathrm{op}}^{2}\left(\frac{\sigma}{\sigma_{\min}}\sqrt{\frac{n}{p}}+\frac{\lambda}{p\,\sigma_{\min}}\right)^{2}\sigma_{\max}\norm{\bX^{\star}}+\eta^{2}C_{\mathrm{B}}\kappa\left(\frac{\sigma}{\sigma_{\min}}\sqrt{\frac{n}{p}}+\frac{\lambda}{p\,\sigma_{\min}}\right)\sqrt{r}\sigma_{\max}^{2}\norm{\bX^{\star}}\\
 & \quad\quad\quad+\eta\frac{\lambda}{p}C_{\mathrm{op}}\left(\frac{\sigma}{\sigma_{\min}}\sqrt{\frac{n}{p}}+\frac{\lambda}{p\,\sigma_{\min}}\right)\left\Vert \bm{X}^{\star}\right\Vert \\
 & \quad\lesssim\eta C_{\mathrm{op}}^{2}\left(\frac{\sigma}{\sigma_{\min}}\sqrt{\frac{n}{p}}+\frac{\lambda}{p\,\sigma_{\min}}\right)^{2}\sigma_{\max}\norm{\bX^{\star}}+\eta^{2}C_{\mathrm{B}}\kappa\left(\frac{\sigma}{\sigma_{\min}}\sqrt{\frac{n}{p}}+\frac{\lambda}{p\,\sigma_{\min}}\right)\sqrt{r}\sigma_{\max}^{2}\norm{\bX^{\star}},
\end{align*}
provided that $C_{\mathrm{op}}$ is large enough. Here the last relation
uses~(\ref{eq:Delta-t-l-op-norm}). Substitution into~(\ref{eq:rotation-diff})
yields 
\begin{align*}
\|(\bH^{t,(l)})^{-1}\bH^{t+1,(l)}-\bm{I}_{r}\| & \lesssim\eta\kappa C_{\mathrm{op}}^{2}\left(\frac{\sigma}{\sigma_{\min}}\sqrt{\frac{n}{p}}+\frac{\lambda}{p\,\sigma_{\min}}\right)^{2}\sigma_{\max}+\eta^{2}C_{\mathrm{B}}\kappa^{2}\left(\frac{\sigma}{\sigma_{\min}}\sqrt{\frac{n}{p}}+\frac{\lambda}{p\,\sigma_{\min}}\right)\sqrt{r}\sigma_{\max}^{2},
\end{align*}
which concludes the proof. \end{proof}

\subsection{Proof of Lemma~\ref{lemma:2-infty-contraction}}

Fix any $1\leq l\leq2n$. Apply the triangle inequality to see that
\begin{align}
 & \big\|\big(\bm{F}^{t+1}\bm{H}^{t+1}-\bm{F}^{\star}\big)_{l,\cdot}\big\|_{2}\leq\big\|\big(\bm{F}^{t+1}\bm{H}^{t+1}-\bm{F}^{t+1,(l)}\bm{H}^{t+1,(l)}\big)_{l,\cdot}\big\|_{2}+\big\|\big(\bm{F}^{t+1,(l)}\bm{H}^{t+1,(l)}-\bm{F}^{\star}\big)_{l,\cdot}\big\|_{2}\nonumber \\
 & \qquad\leq\big\|\bm{F}^{t+1}\bm{H}^{t+1}-\bm{F}^{t+1,(l)}\bm{H}^{t+1,(l)}\big\|_{\mathrm{F}}+C_{4}\kappa\left(\frac{\sigma}{\sigma_{\min}}\sqrt{\frac{n\log n}{p}}+\frac{\lambda}{p\,\sigma_{\min}}\right)\left\Vert \bm{F}^{\star}\right\Vert _{2,\infty},\label{eq:2-infty-triangle}
\end{align}
where the second line follows from Lemma~\ref{lemma:loo-l-contraction}.
Apply Lemma~\ref{lemma:immediate-consequence} to the $(t+1)$th iterates
to see that 
\begin{align}
\big\|\bm{F}^{t+1}\bm{H}^{t+1}-\bm{F}^{t+1,(l)}\bm{H}^{t+1,(l)}\big\|_{\mathrm{F}} & \leq5\kappa\big\|\bm{F}^{t+1}\bm{H}^{t+1}-\bm{F}^{t+1,(l)}\bm{R}^{t+1,(l)}\big\|_{\mathrm{F}}\nonumber \\
 & \leq5\kappa C_{3}\left(\frac{\sigma}{\sigma_{\min}}\sqrt{\frac{n\log n}{p}}+\frac{\lambda}{p\,\sigma_{\min}}\right)\left\Vert \bm{F}^{\star}\right\Vert _{2,\infty}.\label{eq:rotation-upper-bound}
\end{align}
Here the second line follows from Lemma~\ref{lemma:loo-dist-contraction}.
Combine~(\ref{eq:2-infty-triangle}) and~(\ref{eq:rotation-upper-bound})
to reach 
\begin{align*}
\big\|\big(\bm{F}^{t+1}\bm{H}^{t+1}-\bm{F}^{\star}\big)_{l,\cdot}\big\|_{2} & \leq5\kappa C_{3}\left(\frac{\sigma}{\sigma_{\min}}\sqrt{\frac{n\log n}{p}}+\frac{\lambda}{p\,\sigma_{\min}}\right)\left\Vert \bm{F}^{\star}\right\Vert _{2,\infty}+C_{4}\kappa\left(\frac{\sigma}{\sigma_{\min}}\sqrt{\frac{n\log n}{p}}+\frac{\lambda}{p\,\sigma_{\min}}\right)\left\Vert \bm{F}^{\star}\right\Vert _{2,\infty}\\
 & \leq C_{\infty}\kappa\left(\frac{\sigma}{\sigma_{\min}}\sqrt{\frac{n\log n}{p}}+\frac{\lambda}{p\,\sigma_{\min}}\right)\left\Vert \bm{F}^{\star}\right\Vert _{2,\infty}
\end{align*}
as long as $C_{\infty}\geq5C_{3}+C_{4}$. The proof is then complete
since this holds for all $1\leq l\leq2n$.

\subsection{Proof of Lemma~\ref{lemma:balancing}}

To simplify the notation hereafter, we denote 
\[
\bm{A}^{t}\triangleq\bm{X}^{t\top}\bm{X}^{t}-\bm{Y}^{t\top}\bm{Y}^{t}\qquad\text{and}\qquad\bm{A}^{t+1}\triangleq\bm{X}^{t+1\top}\bm{X}^{t+1}-\bm{Y}^{t+1\top}\bm{Y}^{t+1}.
\]
In view of the gradient descent update rules~(\ref{subeq:GD-rules}),
we have 
\begin{align*}
\bm{X}^{t+1\top}\bm{X}^{t+1} & =\bm{X}^{t\top}\bm{X}^{t}-\eta\left[\bX^{t\top}\nabla_{\bX}f(\bX^{t},\bY^{t})+\nabla_{\bX}f(\bX^{t},\bY^{t})^{\top}\bX^{t}\right]+\eta^{2}\nabla_{\bX}f(\bX^{t},\bY^{t})^{\top}\nabla_{\bX}f(\bX^{t},\bY^{t}),\\
\bm{Y}^{t+1\top}\bm{Y}^{t+1} & =\bm{Y}^{t\top}\bm{Y}^{t}-\eta\left[\bY^{t\top}\nabla_{\bY}f(\bX^{t},\bY^{t})+\nabla_{\bY}f(\bX^{t},\bY^{t})^{\top}\bY^{t}\right]+\eta^{2}\nabla_{\bY}f(\bX^{t},\bY^{t})^{\top}\nabla_{\bY}f(\bX^{t},\bY^{t}).
\end{align*}
This gives rise to the following identity 
\begin{equation}
\bm{A}^{t+1}=\bm{A}^{t}-\eta\bB^{t}+\eta^{2}\bC^{t},\label{eq:balancing-recursion}
\end{equation}
where we denote 
\begin{align*}
\bB^{t} & \triangleq\bX^{t\top}\nabla_{\bX}f(\bX^{t},\bY^{t})+\nabla_{\bX}f(\bX^{t},\bY^{t})^{\top}\bX^{t}-\bY^{t\top}\nabla_{\bY}f(\bX^{t},\bY^{t})-\nabla_{\bY}f(\bX^{t},\bY^{t})^{\top}\bY^{t},\\
\bC^{t} & \triangleq\nabla_{\bX}f(\bX^{t},\bY^{t})^{\top}\nabla_{\bX}f(\bX^{t},\bY^{t})-\nabla_{\bY}f(\bX^{t},\bY^{t})^{\top}\nabla_{\bY}f(\bX^{t},\bY^{t}).
\end{align*}
Denoting 
\begin{equation}
\bD^{t}\triangleq p^{-1}\cP_{\Omega}(\bX^{t}\bY^{t\top}-\bM),\label{eq:defn-Dt}
\end{equation}
we have 
\[
\nabla_{\bm{X}}f\left(\bm{X}^{t},\bm{Y}^{t}\right)=\bm{D}^{t}\bm{Y}^{t}+\tfrac{\lambda}{p}\bm{X}^{t}\qquad\text{and}\qquad\nabla_{\bm{Y}}f\left(\bm{X}^{t},\bm{Y}^{t}\right)=\bm{D}^{t\top}\bm{X}^{t}+\tfrac{\lambda}{p}\bm{Y}^{t}.
\]
With these in mind, a little calculation reveals that 
\begin{align*}
\bm{B}^{t} & =\bX^{t\top}\bD^{t}\bY^{t}+\bY^{t\top}\bD^{t\top}\bX^{t}+\tfrac{2\lambda}{p}\bm{X}^{t\top}\bm{X}^{t}-\bY^{t\top}\bD^{t\top}\bX^{t}-\bX^{t\top}\bD^{t}\bY^{t}-\tfrac{2\lambda}{p}\bm{Y}^{t\top}\bm{Y}^{t}=\tfrac{2\lambda}{p}\bA^{t}
\end{align*}
as well as 
\begin{align*}
\bC^{t} & =\left(\bm{D}^{t}\bm{Y}^{t}+\tfrac{\lambda}{p}\bm{X}^{t}\right)^{\top}\left(\bm{D}^{t}\bm{Y}^{t}+\tfrac{\lambda}{p}\bm{X}^{t}\right)-\left(\bm{D}^{t\top}\bm{X}^{t}+\tfrac{\lambda}{p}\bm{Y}^{t}\right)^{\top}\left(\bm{D}^{t\top}\bm{X}^{t}+\tfrac{\lambda}{p}\bm{Y}^{t}\right)\\
 & =\bY^{t\top}\bD^{t\top}\bD^{t}\bY^{t}+\tfrac{\lambda}{p}\bY^{t\top}\bD^{t\top}\bX^{t}+\tfrac{\lambda}{p}\bX^{t\top}\bm{D}^{t}\bY^{t}+\left(\tfrac{\lambda}{p}\right)^{2}\bX^{t\top}\bX^{t}\\
 & \quad-\bX^{t\top}\bm{D}^{t}\bD^{t\top}\bX^{t}-\tfrac{\lambda}{p}\bX^{t\top}\bm{D}^{t}\bY^{t}-\tfrac{\lambda}{p}\bY^{t\top}\bD^{t\top}\bX^{t}-\left(\tfrac{\lambda}{p}\right)^{2}\bY^{t\top}\bY^{t}\\
 & =\left(\bY^{t\top}\bD^{t\top}\bD^{t}\bY^{t}-\bX^{t\top}\bD^{t}\bD^{t\top}\bX^{t}\right)+\left(\tfrac{\lambda}{p}\right)^{2}\bA^{t}.
\end{align*}
Substituting the identities for $\bm{B}^{t}$ and $\bm{C}^{t}$ into
(\ref{eq:balancing-recursion}) yields 
\begin{align*}
\bm{A}^{t+1} & =\bm{A}^{t}-2\eta\tfrac{\lambda}{p}\bA^{t}+\eta^{2}\left(\bY^{t\top}\bD^{t\top}\bD^{t}\bY^{t}-\bX^{t\top}\bD^{t}\bD^{t\top}\bX^{t}\right)+\eta^{2}\left(\tfrac{\lambda}{p}\right)^{2}\bA^{t}\\
 & =\left(1-\lambda\eta/p\right)^{2}\bm{A}^{t}+\eta^{2}\left(\bY^{t\top}\bD^{t\top}\bD^{t}\bY^{t}-\bX^{t\top}\bD^{t}\bD^{t\top}\bX^{t}\right),
\end{align*}
which together with the triangle inequality gives 
\begin{align*}
\Fnorm{\bA^{t+1}} & \leq\left(1-\lambda\eta/p\right)^{2}\Fnorm{\bA^{t}}+\eta^{2}\Fnorm{\bY^{t\top}\bD^{t\top}\bD^{t}\bY^{t}-\bX^{t\top}\bD^{t}\bD^{t\top}\bX^{t}}\\
 & \leq\left(1-\lambda\eta/p\right)\Fnorm{\bA^{t}}+\eta^{2}\Fnorm{\bY^{t\top}\bD^{t\top}\bD^{t}\bY^{t}-\bX^{t\top}\bD^{t}\bD^{t\top}\bX^{t}},
\end{align*}
as long as $\lambda\eta/p<1$ --- a condition that is guaranteed
by our assumptions on $\lambda$ and $\eta$. It then boils down to
controlling $\Vert\bY^{t\top}\bD^{t\top}\bD^{t}\bY^{t}-\bX^{t\top}\bD^{t}\bD^{t\top}\bX^{t}\Vert_{\mathrm{F}}$,
which is supplied in the following claim.

\begin{claim}\label{claim:balancing-residual}Suppose that the sample
complexity satisfies $n^{2}p\gg\kappa^{2}\mu^{2}r^{2}n\log n$ and
that the noise satisfies $\frac{\sigma}{\sigma_{\min}}\sqrt{\frac{n}{p}}\ll\frac{1}{\sqrt{\kappa^{2}\log n}}$,
then one has 
\[
\left\Vert \bY^{t\top}\bD^{t\top}\bD^{t}\bY^{t}-\bX^{t\top}\bD^{t}\bD^{t\top}\bX^{t}\right\Vert _{\mathrm{F}}\lesssim C_{\mathrm{op}}^{2}\left(\frac{\sigma}{\sigma_{\min}}\sqrt{\frac{n}{p}}+\frac{\lambda}{p\,\sigma_{\min}}\right)^{2}\sqrt{r}\sigma_{\max}^{3}.
\]
\end{claim}

Taking the above bounds together, we arrive at for some constant $\tilde{C}>0$,
\begin{align*}
\Fnorm{\bA^{t+1}} & \leq\left(1-\frac{\lambda}{p}\eta\right)\Fnorm{\bA^{t}}+\eta^{2}\tilde{C}C_{\mathrm{op}}^{2}\left(\frac{\sigma}{\sigma_{\min}}\sqrt{\frac{n}{p}}+\frac{\lambda}{p\,\sigma_{\min}}\right)^{2}\sqrt{r}\sigma_{\max}^{3}\\
 & \leq\left(1-\frac{\lambda}{p}\eta\right)C_{\mathrm{B}}\kappa\eta\left(\frac{\sigma}{\sigma_{\min}}\sqrt{\frac{n}{p}}+\frac{\lambda}{p\,\sigma_{\min}}\right)\sqrt{r}\sigma_{\max}^{2}+\eta^{2}\tilde{C}C_{\mathrm{op}}^{2}\left(\frac{\sigma}{\sigma_{\min}}\sqrt{\frac{n}{p}}+\frac{\lambda}{p\,\sigma_{\min}}\right)^{2}\sqrt{r}\sigma_{\max}^{3}\\
 & \leq C_{\mathrm{B}}\kappa\eta\left(\frac{\sigma}{\sigma_{\min}}\sqrt{\frac{n}{p}}+\frac{\lambda}{p\,\sigma_{\min}}\right)\sqrt{r}\sigma_{\max}^{2},
\end{align*}
as long as $\lambda\geq\sigma\sqrt{np}$ and $C_{\mathrm{B}}\gg C_{\mathrm{op}}^{2}$.

\begin{proof}[Proof of Claim~\ref{claim:balancing-residual}]The
triangle inequality yields 
\begin{align}
\left\Vert \bY^{t\top}\bD^{t\top}\bD^{t}\bY^{t}-\bX^{t\top}\bD^{t}\bD^{t\top}\bX^{t}\right\Vert _{\mathrm{F}} & \leq\left\Vert \bY^{t\top}\bD^{t\top}\bD^{t}\bY^{t}\right\Vert _{\mathrm{F}}+\left\Vert \bX^{t\top}\bD^{t}\bD^{t\top}\bX^{t}\right\Vert _{\mathrm{F}}\nonumber \\
 & \leq\left\Vert \bm{Y}^{t}\right\Vert \left\Vert \bm{D}^{t}\right\Vert ^{2}\left\Vert \bm{Y}^{t}\right\Vert _{\mathrm{F}}+\left\Vert \bm{X}^{t}\right\Vert \left\Vert \bm{D}^{t}\right\Vert ^{2}\left\Vert \bm{X}^{t}\right\Vert _{\mathrm{F}}.\label{eq:upper-bound-D}
\end{align}
It is easy to see from Lemma~\ref{lemma:immediate-consequence} that
\[
\left\Vert \bm{Y}^{t}\right\Vert \leq2\left\Vert \bm{Y}^{\star}\right\Vert ,\quad\left\Vert \bm{Y}^{t}\right\Vert _{\mathrm{F}}\leq2\left\Vert \bm{Y}^{\star}\right\Vert _{\mathrm{F}},\quad\left\Vert \bm{X}^{t}\right\Vert \leq2\left\Vert \bm{X}^{\star}\right\Vert \quad\text{and}\quad\left\Vert \bm{X}^{t}\right\Vert _{\mathrm{F}}\leq2\left\Vert \bm{X}^{\star}\right\Vert _{\mathrm{F}}
\]
provided that $\frac{\sigma}{\sigma_{\min}}\sqrt{\frac{n}{p}}\ll\frac{1}{\sqrt{\kappa^{2}\log n}}$.
These allow us to further upper bound~(\ref{eq:upper-bound-D}) as
\begin{align}
\left\Vert \bY^{t\top}\bD^{t\top}\bD^{t}\bY^{t}-\bX^{t\top}\bD^{t}\bD^{t\top}\bX^{t}\right\Vert _{\mathrm{F}} & \leq4\left\Vert \bm{D}^{t}\right\Vert ^{2}\left\Vert \bm{Y}^{\star}\right\Vert \left\Vert \bm{Y}^{\star}\right\Vert _{\mathrm{F}}+4\left\Vert \bm{D}^{t}\right\Vert ^{2}\left\Vert \bm{X}^{\star}\right\Vert \left\Vert \bm{X}^{\star}\right\Vert _{\mathrm{F}}\nonumber \\
 & \leq8\left\Vert \bm{D}^{t}\right\Vert ^{2}\sqrt{r}\sigma_{\max}.\label{eq:D-residual-upper-bound}
\end{align}
It remains to bound $\|\bm{D}^{t}\|$. To this end, recall from~(\ref{eq:defn-Dt})
that 
\begin{align*}
\norm{\bD^{t}} & \leq p^{-1}\norm{\cP_{\Omega}(\bE)}+p^{-1}\norm{\mathcal{P}_{\Omega}^{\mathsf{debias}}\left(\bX^{t}\bY^{t\top}-\bM^{\star}\right)}+\norm{\bX^{t}\bY^{t\top}-\bM^{\star}}.
\end{align*}
In the sequel we shall bound these three terms sequentially. First,
Lemma~\ref{lemma:noise-bound} tells us that $\frac{1}{p}\norm{\cP_{\Omega}(\bE)}\lesssim\sigma\sqrt{\frac{n}{p}}.$
Next, repeating the arguments in the proof of Lemma~\ref{lemma:P-tilde}
gives 
\begin{align*}
\left\Vert \mathcal{P}_{\Omega}^{\mathsf{debias}}\left(\bm{X}^{t}\bm{Y}^{t\top}-\bm{M}^{\star}\right)\right\Vert  & =\left\Vert \mathcal{P}_{\Omega}^{\mathsf{debias}}\left[\bm{X}^{t}\bm{H}^{t}\left(\bm{Y}^{t}\bm{H}^{t}\right)^{\top}-\bm{M}^{\star}\right]\right\Vert \\
 & \lesssim\sqrt{np}\left(\left\Vert \bm{X}^{t}\bm{H}^{t}-\bm{X}^{\star}\right\Vert _{2,\infty}\left\Vert \bm{Y}^{\star}\right\Vert _{2,\infty}+\left\Vert \bm{Y}^{t}\bm{H}^{t}-\bm{Y}^{\star}\right\Vert _{2,\infty}\left\Vert \bm{X}^{\star}\right\Vert _{2,\infty}\right),
\end{align*}
which together with the induction hypothesis~(\ref{eq:induction-2-infty})
yields 
\begin{align*}
\frac{1}{p}\left\Vert \mathcal{P}_{\Omega}^{\mathsf{debias}}\left(\bm{X}^{t}\bm{Y}^{t\top}-\bm{M}^{\star}\right)\right\Vert  & \lesssim\sqrt{\frac{n}{p}}C_{\infty}\kappa\left(\frac{\sigma}{\sigma_{\min}}\sqrt{\frac{n\log n}{p}}+\frac{\lambda}{p\,\sigma_{\min}}\right)\left\Vert \bm{F}^{\star}\right\Vert _{2,\infty}\left\Vert \bm{F}^{\star}\right\Vert _{2,\infty}\\
 & \lesssim C_{\infty}\kappa\left(\frac{\sigma}{\sigma_{\min}}\sqrt{\frac{n\log n}{p}}+\frac{\lambda}{p\,\sigma_{\min}}\right)\sqrt{\frac{\mu^{2}r^{2}}{np}}\sigma_{\max}.
\end{align*}
Here the last relation uses the incoherence assumption $\|\bm{F}^{\star}\|_{2,\infty}\leq\sqrt{\mu r\sigma_{\max}/n}$
(cf.~(\ref{eq:F-singular-value})). Regarding $\|\bm{X}^{t}\bm{Y}^{t\top}-\bm{M}^{\star}\|$,
the triangle inequality reveals that 
\begin{align*}
\left\Vert \bm{X}^{t}\bm{Y}^{t\top}-\bm{M}^{\star}\right\Vert  & =\left\Vert \bm{X}^{t}\bm{H}^{t}\left(\bm{Y}^{t}\bm{H}^{t}\right)^{\top}-\bm{M}^{\star}\right\Vert \\
 & \leq\left\Vert \bm{X}^{t}\bm{H}^{t}\left(\bm{Y}^{t}\bm{H}^{t}\right)^{\top}-\bm{X}^{t}\bm{H}^{t}\bm{Y}^{\star\top}\right\Vert +\left\Vert \bm{X}^{t}\bm{H}^{t}\bm{Y}^{\star\top}-\bm{X}^{\star}\bm{Y}^{\star\top}\right\Vert \\
 & \leq\left\Vert \bm{X}^{t}\bm{H}^{t}\right\Vert \left\Vert \bm{Y}^{t}\bm{H}^{t}-\bm{Y}^{\star}\right\Vert +\left\Vert \bm{X}^{t}\bm{H}^{t}-\bm{X}^{\star}\right\Vert \left\Vert \bm{Y}^{\star}\right\Vert .
\end{align*}
Combine the induction hypothesis~(\ref{eq:induction-op}) and the
fact that $\|\bm{X}^{t}\bm{H}^{t}\|=\|\bm{X}^{t}\|\leq2\|\bm{X}^{\star}\|$
to reach 
\[
\left\Vert \bm{X}^{t}\bm{Y}^{t\top}-\bm{M}^{\star}\right\Vert \lesssim C_{\mathrm{op}}\left(\frac{\sigma}{\sigma_{\min}}\sqrt{\frac{n}{p}}+\frac{\lambda}{p\,\sigma_{\min}}\right)\sigma_{\max}.
\]
Putting together the previous three bounds, we arrive at 
\begin{align}
\left\Vert \bm{D}^{t}\right\Vert  & \lesssim\sigma\sqrt{\frac{n}{p}}+C_{\infty}\kappa\left(\frac{\sigma}{\sigma_{\min}}\sqrt{\frac{n\log n}{p}}+\frac{\lambda}{p\,\sigma_{\min}}\right)\sqrt{\frac{\mu^{2}r^{2}}{np}}\sigma_{\max}+C_{\mathrm{op}}\left(\frac{\sigma}{\sigma_{\min}}\sqrt{\frac{n}{p}}+\frac{\lambda}{p\,\sigma_{\min}}\right)\sigma_{\max}\nonumber \\
 & \lesssim C_{\mathrm{op}}\left(\frac{\sigma}{\sigma_{\min}}\sqrt{\frac{n}{p}}+\frac{\lambda}{p\,\sigma_{\min}}\right)\sigma_{\max}\label{eq:D-upper-bound}
\end{align}
since $np\gg\kappa^{2}\mu^{2}r^{2}\log n$. Putting~(\ref{eq:D-upper-bound})
back to~(\ref{eq:D-residual-upper-bound}) leads to the claimed upper
bound.

The upper bound on the leave-one-out sequences can be derived similarly.
For brevity, we omit it. \end{proof}

\subsection{Proof of Lemma~\ref{lemma:function-value-decreasing}}

In light of the facts that $f(\bm{F}\bm{R})=f(\bm{F})$ and $\nabla f(\bm{F}\bm{R})=\nabla f(\bm{F})\bm{R}$
for any $\bm{R}\in\mathcal{O}^{r\times r}$, one has 
\begin{align*}
f\left(\bm{F}^{t+1}\right) & =f\left(\bm{F}^{t+1}\bm{H}^{t}\right)=f\left(\left[\bm{F}^{t}-\eta\nabla f\left(\bm{F}^{t}\right)\right]\bm{H}^{t}\right)\\
 & =f\left(\bm{F}^{t}\bm{H}^{t}-\eta\nabla f\left(\bm{F}^{t}\bm{H}^{t}\right)\right)\\
 & =f\left(\bm{F}^{t}\bm{H}^{t}\right)-\eta\left\langle \nabla f\left(\bm{F}^{t}\bm{H}^{t}\right),\nabla f\left(\bm{F}^{t}\bm{H}^{t}\right)\right\rangle +\frac{\eta^{2}}{2}\mathsf{vec}\left(\nabla f\left(\bm{F}^{t}\bm{H}^{t}\right)\right)^{\top}\nabla^{2}f\big(\tilde{\bm{F}}\big)\mathsf{vec}\left(\nabla f\big(\bm{F}^{t}\bm{H}^{t}\big)\right)
\end{align*}
for some $\tilde{\bm{F}}$ which lies between $\bm{F}^{t}\bm{H}^{t}$
and $\bm{F}^{t}\bm{H}^{t}-\eta\nabla f(\bm{F}^{t}\bm{H}^{t})$. Suppose
for the moment that \begin{subequations}\label{subeq:function-value-condition}
\begin{align}
\|\bm{F}^{t}\bm{H}^{t}-\bm{F}^{\star}\|_{2,\infty} & \leq\frac{1}{2000\kappa\sqrt{n}}\|\bm{X}^{\star}\|,\label{eq:function-value-cond-1}\\
\|\bm{F}^{t}\bm{H}^{t}-\eta\nabla f(\bm{F}^{t}\bm{H}^{t})-\bm{F}^{\star}\|_{2,\infty} & \leq\frac{1}{1000\kappa\sqrt{n}}\|\bm{X}^{\star}\|.\label{eq:function-value-cond-2}
\end{align}
\end{subequations}One can invoke Lemma~\ref{lemma:hessian} to obtain
$\|\nabla^{2}f(\tilde{\bm{F}})\|\leq10\sigma_{\max}$ and hence
\begin{align*}
f\left(\bm{F}^{t+1}\right) & \leq f\left(\bm{F}^{t}\bm{H}^{t}\right)-\eta\left\Vert \nabla f\left(\bm{F}^{t}\bm{H}^{t}\right)\right\Vert _{\mathrm{F}}^{2}+5\eta^{2}\sigma_{\max}\left\Vert \nabla f\left(\bm{F}^{t}\bm{H}^{t}\right)\right\Vert _{\mathrm{F}}^{2}\\
 & =f\left(\bm{F}^{t}\right)-\eta\left\Vert \nabla f\left(\bm{F}^{t}\right)\right\Vert _{\mathrm{F}}^{2}+5\eta^{2}\sigma_{\max}\left\Vert \nabla f\left(\bm{F}^{t}\right)\right\Vert _{\mathrm{F}}^{2}\\
 & \leq f\left(\bm{F}^{t}\right)-\tfrac{\eta}{2}\left\Vert \nabla f\left(\bm{F}^{t}\right)\right\Vert _{\mathrm{F}}^{2}.
\end{align*}
Here the equality uses again the facts that $f(\bm{F}\bm{R})=f(\bm{F})$
and $\nabla f(\bm{F}\bm{R})=\nabla f(\bm{F})\bm{R}$ for any $\bm{R}\in\mathcal{O}^{r\times r}$
and the last inequality holds as long as $\eta\leq\frac{1}{10\sigma_{\max}}$.
We are left with proving the aforementioned conditions~(\ref{subeq:function-value-condition}).
The first condition has been established in the proof of Lemma~\ref{lemma:small-gradient-smooth-function}
and hence we concentrate on the second one, namely~(\ref{eq:function-value-cond-2}).
Apply the triangle inequality and the fundamental theorem of calculus
\cite[Chapter XIII, Theorem 4.2]{lang1993real} to obtain 
\begin{align*}
 & \|\bm{F}^{t}\bm{H}^{t}-\eta\nabla f(\bm{F}^{t}\bm{H}^{t})-\bm{F}^{\star}\|_{2,\infty} \leq\|\bm{F}^{t}\bm{H}^{t}-\bm{F}^{\star}\|_{2,\infty}+\eta\left\Vert \nabla f(\bm{F}^{t}\bm{H}^{t})-\nabla f(\bm{F}^{\star})\right\Vert _{\mathrm{F}}+\eta\left\Vert \nabla f(\bm{F}^{\star})\right\Vert _{\mathrm{F}}\\
	& \qquad \leq\|\bm{F}^{t}\bm{H}^{t}-\bm{F}^{\star}\|_{2,\infty}+\eta\left\Vert \int_{0}^{1}\nabla^{2}f\left(\bm{F}\left(\tau\right)\right)\mathrm{d}\tau\,\mathsf{vec}\left(\bm{F}^{t}\bm{H}^{t}-\bm{F}^{\star}\right)\right\Vert _{\mathrm{2}}+\eta\left\Vert \nabla f(\bm{F}^{\star})\right\Vert _{\mathrm{F}},
\end{align*}
where $\bm{F}(\tau)\triangleq\bm{F}^{\star}+\tau(\bm{F}^{t}\bm{H}^{t}-\bm{F}^{\star})$
for $0\leq\tau\leq1$. Following similar arguments in the proof of
Lemma~\ref{lemma:fro-contraction} and the proof of Lemma~\ref{lemma:small-gradient-smooth-function},
one obtains 
\begin{align*}
	& \eta\left\Vert \int_{0}^{1}\nabla^{2}f\left(\bm{F}\left(\tau\right)\right)\mathrm{d}\tau\,\mathsf{vec}\left(\bm{F}^{t}\bm{H}^{t}-\bm{F}^{\star}\right)\right\Vert _{\mathrm{F}}+\eta\left\Vert \nabla f(\bm{F}^{\star})\right\Vert _{\mathrm{F}}\\
 & \quad\leq\eta\cdot 10\sigma_{\max}\left\Vert \bm{F}^{t}\bm{H}^{t}-\bm{F}^{\star}\right\Vert_{\mathrm{F}} +\eta\frac{\lambda}{p}\sqrt{r\sigma_{\max}}\\
 & \quad\lesssim\eta\sigma_{\max}\left(\frac{\sigma}{\sigma_{\min}}\sqrt{\frac{n}{p}}+\frac{\lambda}{p\,\sigma_{\min}}\right)\sqrt{r}\|\bm{X}^{\star}\|+\eta\sigma_{\min}\frac{\lambda}{p\,\sigma_{\min}}\sqrt{r}\|\bm{X}^{\star}\|\leq\frac{1}{2000\kappa\sqrt{n}}\|\bm{X}^{\star}\|.
\end{align*}
Here the middle inequality uses the induction hypothesis~(\ref{eq:induction-op})
and the last relation holds true provided that $\lambda\asymp\sigma\sqrt{np}$,
$\frac{\sigma}{\sigma_{\min}}\sqrt{\frac{n}{p}}\ll1/\sqrt{r}$ and that $\eta\ll1/(\kappa n\sigma_{\max})$.
This proves the second condition and also the whole lemma.

\subsection{Proof of Lemma~\ref{lemma:hessian}}\label{subsec:Proof-of-Lemma-local-geometry}

 We start by defining a new loss function
\[
f_{\mathsf{clean}}\left(\bm{X},\bm{Y}\right)\triangleq\tfrac{1}{2p}\left\Vert \mathcal{P}_{\Omega}\left(\bm{X}\bm{Y}^{\top}-\bm{M}^{\star}\right)\right\Vert _{\mathrm{F}}^{2}+\tfrac{1}{8}\left\Vert \bm{X}^{\top}\bm{X}-\bm{Y}^{\top}\bm{Y}\right\Vert _{\mathrm{F}}^{2};
\]
compared with $f_{\mathsf{aug}}(\cdot, \cdot)$, this new function $f_{\mathsf{clean}}(\cdot, \cdot)$
sets $\lambda=0$ and excludes the noise $\bm{E}$ from consideration. It is straightforward to check that for any $\bm{\Delta}\in\mathbb{R}^{2n\times r}$,
\[
\mathsf{vec}\left(\bm{\Delta}\right)^{\top}\nabla^{2}f_{\mathsf{aug}}\left(\bm{X},\bm{Y}\right)\mathsf{vec}\left(\bm{\Delta}\right)=\mathsf{vec}\left(\bm{\Delta}\right)^{\top}\nabla^{2}f_{\mathsf{clean}}\left(\bm{X},\bm{Y}\right)\mathsf{vec}\left(\bm{\Delta}\right)-\tfrac{2}{p}\left\langle \mathcal{P}_{\Omega}\left(\bm{E}\right),\bm{\Delta}_{\bm{X}}\bm{\Delta}_{\bm{Y}}^{\top}\right\rangle +\tfrac{\lambda}{p}\left\Vert \bm{\Delta}\right\Vert _{\mathrm{F}}^{2}.
\]
It has been proven in~\cite[Lemma 3.2]{chen2019nonconvex} that under
the assumptions stated in Lemma~\ref{lemma:hessian}, one has 
\begin{equation}
\mathsf{vec}\left(\bm{\Delta}\right)^{\top}\nabla^{2}f_{\mathsf{clean}}\left(\bm{X},\bm{Y}\right)\mathsf{vec}\left(\bm{\Delta}\right)\geq\tfrac{1}{5}\sigma_{\min}\left\Vert \bm{\Delta}\right\Vert _{\mathrm{F}}^{2}\qquad\text{and}\qquad\left\Vert \nabla^{2}f_{\mathsf{clean}}\left(\bm{X},\bm{Y}\right)\right\Vert \leq5\sigma_{\max}.\label{eq:xiaodong-hessian}
\end{equation}
It then boils down to controlling $-\frac{2}{p}\left\langle \mathcal{P}_{\Omega}\left(\bm{E}\right),\bm{\Delta}_{\bm{X}}\bm{\Delta}_{\bm{Y}}^{\top}\right\rangle +\frac{\lambda}{p}\left\Vert \bm{\Delta}\right\Vert _{\mathrm{F}}^{2}.$
To this end, one has 
\begin{align}
\left|\tfrac{1}{p}\left\langle \mathcal{P}_{\Omega}\left(\bm{E}\right),\bm{\Delta}_{\bm{X}}\bm{\Delta}_{\bm{Y}}^{\top}\right\rangle \right| & \leq\left\Vert \tfrac{1}{p}\mathcal{P}_{\Omega}\left(\bm{E}\right)\right\Vert \left\Vert \bm{\Delta}_{\bm{X}}\bm{\Delta}_{\bm{Y}}^{\top}\right\Vert _{*}\lesssim\sigma\sqrt{\tfrac{n}{p}}\left\Vert \bm{\Delta}\right\Vert _{\mathrm{F}}^{2},\label{eq:tech-lemma-prev}
\end{align}
where the last relation holds due to Lemma~\ref{lemma:noise-bound}
and the elementary fact about the nuclear norm~(\ref{eq:nuclear-fro-relation}),
i.e.~
\[
2\left\Vert \bm{\Delta}_{\bm{X}}\bm{\Delta}_{\bm{Y}}^{\top}\right\Vert _{*}\leq\left\Vert \bm{\Delta}_{\bm{X}}\right\Vert _{\mathrm{F}}^{2}+\left\Vert \bm{\Delta}_{\bm{Y}}\right\Vert _{\mathrm{F}}^{2}=\left\Vert \bm{\Delta}\right\Vert _{\mathrm{F}}^{2}.
\]
Regarding the term $\lambda\|\bm{\Delta}\|_{\mathrm{F}}^{2}/p$, it
is easy to see from the assumption $\lambda\asymp\sigma\sqrt{np}$
that $\frac{\lambda}{p}\left\Vert \bm{\Delta}\right\Vert _{\mathrm{F}}^{2}\asymp\sigma\sqrt{\frac{n}{p}}\left\Vert \bm{\Delta}\right\Vert _{\mathrm{F}}^{2}$.
Combine the above two bounds and use the triangle inequality to reach
\begin{equation}
\left|-\tfrac{2}{p}\left\langle \mathcal{P}_{\Omega}\left(\bm{E}\right),\bm{\Delta}_{\bm{X}}\bm{\Delta}_{\bm{Y}}^{\top}\right\rangle +\tfrac{\lambda}{p}\left\Vert \bm{\Delta}\right\Vert _{\mathrm{F}}^{2}\right|\lesssim\sigma\sqrt{\tfrac{n}{p}}\left\Vert \bm{\Delta}\right\Vert _{\mathrm{F}}^{2}\leq\tfrac{1}{10}\sigma_{\min}\left\Vert \bm{\Delta}\right\Vert _{\mathrm{F}}^{2},\label{eq:hessian-perturbation}
\end{equation}
with the proviso that $\frac{\sigma}{\sigma_{\min}}\sqrt{\frac{n}{p}}\ll1$.
Taking~(\ref{eq:xiaodong-hessian}) and~(\ref{eq:hessian-perturbation})
together immediately establishes the claims on $\nabla^2 f_{\mathsf{aug}}(\cdot, \cdot)$.

Moving on to $\nabla^{2}f(\bX,\bY)$, one has 
{
\begin{align*}
 & \mathsf{vec}(\bDelta)^\top\nabla^{2}f\left(\bX,\bY\right)\mathsf{vec}(\bDelta)=\tfrac{1}{p}\big\| \cP_{\Omega}\big(\bX\bDelta_{\bY}^{\top}+\bDelta_{\bX}\bY^{\top}\big) \big\|_{\mathrm{F}}^{2}+\tfrac{2}{p}\big\langle \cP_{\Omega}\big(\bX\bY^{\top}-\bM^{\star}-\bE\big),\bDelta_{\bX}\bDelta_{\bY}^{\top}\big\rangle +\tfrac{\lambda}{p}\|\bm{\Delta}\|_{\mathrm{F}}^{2} \\
\quad & =\underbrace{\big\|{\bX\bDelta_{\bY}^{\top}+\bDelta_{\bX}\bY^{\top}\big\|_{\mathrm{F}}}^{2}}_{:=\alpha_{1}}+\underbrace{2\big\langle \bX\bY^{\top}-\bM^{\star},\bDelta_{\bX}\bDelta_{\bY}^{\top}\big\rangle }_{:=\alpha_{2}}+\underbrace{\big(-\tfrac{2}{p}\big\langle \cP_{\Omega}(\bE),\bDelta_{\bX}\bDelta_{\bY}^{\top}\big\rangle+\tfrac{\lambda}{p}\|\bm{\Delta}\|_{\mathrm{F}}^{2} \big)}_{:=\alpha_{3}}\\
 & ~+\underbrace{\tfrac{1}{p}\big\| \cP_{\Omega}\big(\bX\bDelta_{\bY}^{\top}+\bDelta_{\bX}\bY^{\top}\big) \big\|_{\mathrm{F}}^{2}+\tfrac{2}{p}\big\langle \cP_{\Omega}\big(\bX\bY^{\top}-\bM^{\star}\big),\bDelta_{\bX}\bDelta_{\bY}^{\top}\big\rangle -\big\| \bX\bDelta_{\bY}^{\top}+\bDelta_{\bX}\bY^{\top} \big\|_{\mathrm{F}}^{2}-2\big\langle \bX\bY^{\top}-\bM^{\star},\bDelta_{\bX}\bDelta_{\bY}^{\top}\big\rangle }_{:=\alpha_{4}}.
\end{align*}
}
The term $\alpha_{4}$ can be bounded by~\cite[Equation A.4]{chen2019nonconvex}
\begin{align*}
\abs{\alpha_{4}}\leq\tfrac{1}{5}\sigma_{\min}\big(\Fnorm{\bDelta_{\bX}}^{2}+\Fnorm{\bDelta_{\bY}}^{2}\big)+\tfrac{1}{5}\big(\Fnorm{\bDelta_{\bX}\bY^{\star\top}}^{2}+\Fnorm{\bX^{\star}\bDelta_{\bY}^{\top}}^{2}\big)\leq\tfrac{2}{5}\sigma_{\max}\Fnorm{\bDelta}^{2}.
\end{align*}
The term $\alpha_{3}$ has been bounded in~(\ref{eq:hessian-perturbation})
where $|\alpha_{3}|\leq\sigma_{\max}\|\bm{\Delta}\|_{\mathrm{F}}^{2}$
provided that $\frac{\sigma}{\sigma_{\min}}\sqrt{\frac{n}{p}}\ll1$.
The term $\alpha_{2}$ can be written as 
\[
\abs{\alpha_{2}}\leq2\norm{\bX\bY^{\top}-\bM^{\star}}\nuclearnorm{\bDelta_{\bX}\bDelta_{\bY}^{\top}}\leq\left(\norm{\bX-\bX^{\star}}\norm{\bY}+\norm{\bX^{\star}}\norm{\bY-\bY^{\star}}\right)\Fnorm{\bDelta}^{2}.
\]
Since 
\[
\norm{\left[\begin{matrix}\bX-\bX^{\star}\\
\bY-\bY^{\star}
\end{matrix}\right]}\leq\Fnorm{\left[\begin{matrix}\bX-\bX^{\star}\\
\bY-\bY^{\star}
\end{matrix}\right]}\leq\sqrt{2n}\twotoinftynorm{\left[\begin{matrix}\bX-\bX^{\star}\\
\bY-\bY^{\star}
\end{matrix}\right]}\leq\frac{1}{500\kappa}\norm{\bX^{\star}},
\]
we immediately have 
\[
\abs{\alpha_{2}}\leq\frac{3}{500\kappa}\sigma_{\max}\Fnorm{\bDelta}^{2}\leq\frac{1}{2}\sigma_{\max}\Fnorm{\bDelta}^{2}.
\]
The term $\alpha_{1}$ can be bounded by 
\[
\alpha_{1}\leq2\left(\Fnorm{\bX^{\star}\bDelta_{\bY}^{\top}}^{2}+\Fnorm{\bDelta_{\bX}\bY^{\star\top}}^{2}\right)\leq2\left(\norm{\bX^{\star}}^{2}\Fnorm{\bDelta_{\bY}}^{2}+\norm{\bY^{\star}}^{2}\Fnorm{\bDelta_{\bX}}^{2}\right)=2\sigma_{\max}\Fnorm{\bDelta}^{2}.
\]
Combining all these bounds yields 
\[
\mathsf{vec}(\bDelta)^\top\nabla^{2}f\left(\bX,\bY\right)\mathsf{vec}(\bDelta)\leq10\sigma_{\max}\Fnorm{\bDelta}^{2}.
\]

\subsection{Proof of Lemma~\ref{lemma:immediate-consequence}} \label{subsec:Proof-of-Lemma-immediate-consequence}

The first set of consequences~(\ref{subeq:immediate-F-t-l-R-t-l})
follows straightforwardly from the triangle inequality. For instance,
combine the induction hypotheses~(\ref{eq:induction-loo-dist}) and
(\ref{eq:induction-2-infty}) to obtain 
\begin{align*}
\big\Vert \bm{F}^{t,(l)}\bm{R}^{t,(l)}-\bm{F}^{\star}\big\Vert _{2,\infty} & \leq\big\Vert \bm{F}^{t,(l)}\bm{R}^{t,(l)}-\bm{F}^{t}\bm{H}^{t}\big\Vert _{2,\infty}+\big\Vert \bm{F}^{t}\bm{H}^{t}-\bm{F}^{\star}\big\Vert _{2,\infty}\\
 & \leq\left(C_{\infty}\kappa+C_{3}\right)\left(\frac{\sigma}{\sigma_{\min}}\sqrt{\frac{n\log n}{p}}+\frac{\lambda}{p\,\sigma_{\min}}\right)\left\Vert \bm{F}^{\star}\right\Vert _{2,\infty}.
\end{align*}
Similar bounds can be obtained for $\|\bm{F}^{t,(l)}\bm{R}^{t,(l)}-\bm{F}^{\star}\|$
provided that $n\gg\mu r\log n$.

We continue to establish the second set of consequences namely~(\ref{subeq:immediate-consequence-F-t}).
Since $\|\cdot\|$ is unitarily invariant, one can apply the triangle
inequality to get 
\begin{align*}
\left\Vert \bm{F}^{t}\right\Vert  & =\left\Vert \bm{F}^{t}\bm{H}^{t}\right\Vert \leq\left\Vert \bm{F}^{t}\bm{H}^{t}-\bm{F}^{\star}\right\Vert +\left\Vert \bm{F}^{\star}\right\Vert \\
 & \overset{(\text{i})}{\leq}C_{\mathrm{op}}\left(\frac{\sigma}{\sigma_{\min}}\sqrt{\frac{n}{p}}+\frac{\lambda}{p\,\sigma_{\min}}\right)\left\Vert \bm{X}^{\star}\right\Vert +\sqrt{2}\left\Vert \bm{X}^{\star}\right\Vert \overset{(\text{ii})}{\leq}2\left\Vert \bm{X}^{\star}\right\Vert .
\end{align*}
Here (i) uses the induction hypothesis~(\ref{eq:induction-op}) and
the fact that $\|\bm{F}^{\star}\|=\sqrt{2}\|\bm{X}^{\star}\|$, and
(ii) holds as long as $\frac{\sigma}{\sigma_{\min}}\sqrt{\frac{n}{p}}\ll1$.
Similarly one can obtain $\|\bm{F}^{t}\|_{\mathrm{F}}\leq2\|\bm{X}^{\star}\|_{\mathrm{F}}$
provided that $\frac{\sigma}{\sigma_{\min}}\sqrt{\frac{n}{p}}\ll1$
and $\|\bm{F}^{t}\|_{2,\infty}\leq2\|\bm{F}^{\star}\|_{2,\infty}$
as long as $\frac{\sigma}{\sigma_{\min}}\sqrt{\frac{n}{p}}\ll1/(\sqrt{\kappa^{2}\log n})$.
Notice that along the way, we have also proven that 
\[
\left\Vert \bm{F}^{t}\bm{H}^{t}-\bm{F}^{\star}\right\Vert \leq\left\Vert \bm{X}^{\star}\right\Vert ,\quad\left\Vert \bm{F}^{t}\bm{H}^{t}-\bm{F}^{\star}\right\Vert _{\mathrm{F}}\leq\left\Vert \bm{X}^{\star}\right\Vert _{\mathrm{F}}\quad\mathrm{and}\quad\left\Vert \bm{F}^{t}\bm{H}^{t}-\bm{F}^{\star}\right\Vert _{\mathrm{2,\infty}}\leq\left\Vert \bm{F}^{\star}\right\Vert _{2,\infty}.
\]

We now move on to $\|\bm{F}^{t}\bm{H}^{t}-\bm{F}^{t,(l)}\bm{H}^{t,(l)}\|_{\mathrm{F}}$,
for which we intend to apply Lemma~\ref{lemma:rotation-perturbation}
to connect it with $\|\bm{F}^{t}\bm{H}^{t}-\bm{F}^{t,(l)}\bm{R}^{t,(l)}\|_{\mathrm{F}}$.
First, in view of the induction hypothesis~(\ref{eq:induction-op}),
one has 
\begin{align*}
\big\|\bm{F}^{t}\bm{H}^{t}-\bm{F}^{\star}\big\|\left\Vert \bm{F}^{\star}\right\Vert  & \leq C_{\mathrm{op}}\left(\frac{\sigma}{\sigma_{\min}}\sqrt{\frac{n}{p}}+\frac{\lambda}{p\,\sigma_{\min}}\right)\left\Vert \bm{X}^{\star}\right\Vert \left\Vert \bm{F}^{\star}\right\Vert \\
 & =\sqrt{2}C_{\mathrm{op}}\left(\frac{\sigma}{\sigma_{\min}}\sqrt{\frac{n}{p}}+\frac{\lambda}{p\,\sigma_{\min}}\right)\sigma_{\max}\\
 & \leq\sigma_{r}^{2}\left(\bm{F}^{\star}\right)/2,
\end{align*}
where the equality arises since $\|\bm{F}^{\star}\|=\sqrt{2\sigma_{\max}}$
(see~(\ref{eq:F-singular-value})) and $\|\bm{X}^{\star}\|=\sqrt{\sigma_{\max}}$,
and the last line holds as long as $\frac{\sigma}{\sigma_{\min}}\sqrt{\frac{n}{p}}\ll1/\kappa$.
In addition, it follows from the induction hypothesis~(\ref{eq:induction-loo-dist})
that 
\begin{align*}
\big\|\bm{F}^{t}\bm{H}^{t}-\bm{F}^{t,(l)}\bm{R}^{t,(l)}\big\|\left\Vert \bm{F}^{\star}\right\Vert  & \leq\big\|\bm{F}^{t}\bm{H}^{t}-\bm{F}^{t,(l)}\bm{R}^{t,(l)}\big\|_{\mathrm{F}}\left\Vert \bm{F}^{\star}\right\Vert \\
 & \leq C_{3}\left(\frac{\sigma}{\sigma_{\min}}\sqrt{\frac{n\log n}{p}}+\frac{\lambda}{p\,\sigma_{\min}}\right)\left\Vert \bm{F}^{\star}\right\Vert _{2,\infty}\left\Vert \bm{F}^{\star}\right\Vert \\
 & \leq\sqrt{2}C_{3}\left(\frac{\sigma}{\sigma_{\min}}\sqrt{\frac{n\log n}{p}}+\frac{\lambda}{p\,\sigma_{\min}}\right)\sqrt{\frac{\mu r}{n}}\sigma_{\max}\\
 & \leq\sigma_{r}^{2}\left(\bm{F}^{\star}\right)/4,
\end{align*}
where the penultimate inequality arises from the facts that $\|\bm{F}^{\star}\|_{2,\infty}\leq\sqrt{\mu r\sigma_{\max}/n}$
and that $\|\bm{F}^{\star}\|=\sqrt{2\sigma_{\max}}$~(cf.~(\ref{subeq:F-property})),
and the last relation holds as long as $(\frac{\sigma}{\sigma_{\min}}\sqrt{\frac{n\log n}{p}}+\frac{\lambda}{p\,\sigma_{\min}})\sqrt{\frac{\mu r}{n}}\ll1/\kappa$.
Invoke Lemma~\ref{lemma:rotation-perturbation} with 
\[
\bm{F}_{0}=\bm{F}^{\star},\qquad\bm{F}_{1}=\bm{F}^{t}\bm{H}^{t}\qquad\text{and}\qquad\bm{F}_{2}=\bm{F}^{t,(l)}\bm{R}^{t,(l)}
\]
to arrive at 
\begin{align*}
\big\|\bm{F}^{t}\bm{H}^{t}-\bm{F}^{t,(l)}\bm{H}^{t,(l)}\big\|_{\mathrm{F}} & \leq5\frac{\sigma_{1}^{2}\left(\bm{F}^{\star}\right)}{\sigma_{r}^{2}\left(\bm{F}^{\star}\right)}\big\|\bm{F}^{t}\bm{H}^{t}-\bm{F}^{t,(l)}\bm{R}^{t,(l)}\big\|_{\mathrm{F}}
	=5\kappa\big\|\bm{F}^{t}\bm{H}^{t}-\bm{F}^{t,(l)}\bm{R}^{t,(l)}\big\|_{\mathrm{F}}.
\end{align*}

The last set of consequences can be derived following similar arguments
to that for establishing the first set. For brevity, we omit the proof.

\subsection{Proof of the inequalities~(\ref{subeq:XY-quality}) \label{subsec:Proof-of-XY-quality}}

We single out the proof of $\|\bm{X}^{t}\bm{Y}^{t\top}-\bm{M}^{\star}\|_{\infty}$,
whereas the proofs of $\|\bm{X}^{t}\bm{Y}^{t\top}-\bm{M}^{\star}\|_{\mathrm{F}}$
and $\|\bm{X}^{t}\bm{Y}^{t\top}-\bm{M}^{\star}\|$ follow from the
same argument. Recognize the following decomposition

\[
\bm{X}^{t}\bm{Y}^{t\top}-\bm{M}^{\star}=\left(\bm{X}^{t}\bm{H}^{t}-\bm{X}^{\star}\right)\left(\bm{Y}^{t}\bm{H}^{t}\right)^{\top}+\bm{X}^{\star}\left(\bm{Y}^{t}\bm{H}^{t}-\bm{Y}^{\star}\right)^{\top},
\]
which together with the triangle inequality gives 
\begin{align*}
\left\Vert \bm{X}^{t}\bm{Y}^{t\top}-\bm{M}^{\star}\right\Vert _{\infty} & \leq\left\Vert \left(\bm{X}^{t}\bm{H}^{t}-\bm{X}^{\star}\right)\left(\bm{Y}^{t}\bm{H}^{t}\right)^{\top}\right\Vert _{\infty}+\left\Vert \bm{X}^{\star}\left(\bm{Y}^{t}\bm{H}^{t}-\bm{Y}^{\star}\right)^{\top}\right\Vert _{\infty}\\
 & \leq\left\Vert \bm{X}^{t}\bm{H}^{t}-\bm{X}^{\star}\right\Vert _{2,\infty}\left\Vert \bm{Y}^{t}\bm{H}^{t}\right\Vert _{2,\infty}+\left\Vert \bm{X}^{\star}\right\Vert _{2,\infty}\left\Vert \bm{Y}^{t}\bm{H}^{t}-\bm{Y}^{\star}\right\Vert _{2,\infty}.
\end{align*}
In view of Lemma~\ref{lemma:immediate-consequence}, one has $\|\bm{Y}^{t}\bm{H}^{t}\|_{2,\infty}\leq2\|\bm{F}^{\star}\|_{2,\infty}$
as long as the noise obeys $\frac{\sigma}{\sigma_{\min}}\sqrt{\frac{n}{p}}\ll\frac{1}{\sqrt{\kappa^{2}\log n}}$.
This further implies that 
\begin{align*}
\left\Vert \bm{X}^{t}\bm{Y}^{t\top}-\bm{M}^{\star}\right\Vert _{\infty} & \leq3C_{\infty}\kappa\left(\frac{\sigma}{\sigma_{\min}}\sqrt{\frac{n\log n}{p}}+\frac{\lambda}{p\,\sigma_{\min}}\right)\left\Vert \bm{F}^{\star}\right\Vert _{2,\infty}\left\Vert \bm{F}^{\star}\right\Vert _{2,\infty}\\
 & \leq3C_{\infty}\sqrt{\kappa^{3}\mu r}\left(\frac{\sigma}{\sigma_{\min}}\sqrt{\frac{n\log n}{p}}+\frac{\lambda}{p\,\sigma_{\min}}\right)\left\Vert \bm{M}^{\star}\right\Vert _{\infty},
\end{align*}
where the last relation is $\|\bm{F}^{\star}\|_{2,\infty}\|\bm{F}^{\star}\|_{2,\infty}\leq\sqrt{\kappa\mu r}\|\bm{M}^{\star}\|_{\infty}$.
To see this, one has for any $1\leq i\leq n$, 
\[
n\left\Vert \bm{M}^{\star}\right\Vert _{\infty}^{2}\geq\sum_{j=1}^{n}(M_{ij}^{\star})^{2}=\bm{X}_{i,\cdot}^{\star}\bm{Y}^{\star\top}\bm{Y}^{\star}\bm{X}_{i,\cdot}^{\star\top}\geq\big\|\bm{X}_{i,\cdot}^{\star}\big\|_{2}^{2}\lambda_{\min}\left(\bm{Y}^{\star\top}\bm{Y}^{\star}\right)=\sigma_{\min}\left\Vert \bm{X}_{i,\cdot}^{\star}\right\Vert _{2}^{2}.
\]
Here $\lambda_{\min}(\cdot)$ denotes the minimum eigenvalue. Since
the inequality holds for all $1\leq i\leq n$, we arrive at 
\[
\left\Vert \bm{X}^{\star}\right\Vert _{2,\infty}\leq\sqrt{\frac{n}{\sigma_{\min}}}\left\Vert \bm{M}^{\star}\right\Vert _{\infty}.
\]
Similarly one can obtain $\|\bm{Y}^{\star}\|_{2,\infty}\leq\sqrt{n/\sigma_{\min}}\|\bm{M}^{\star}\|_{\infty}$,
which further implies $\|\bm{F}^{\star}\|_{2,\infty}=\max\{\|\bm{X}^{\star}\|_{2,\infty},\|\bm{Y}^{\star}\|_{2,\infty}\}\leq\sqrt{n/\sigma_{\min}}\|\bm{M}^{\star}\|_{\infty}$.
As a result, we arrive at 
\[
\left\Vert \bm{F}^{\star}\right\Vert _{2,\infty}\left\Vert \bm{F}^{\star}\right\Vert _{2,\infty}\leq\sqrt{\frac{n}{\sigma_{\min}}}\left\Vert \bm{M}^{\star}\right\Vert _{\infty}\cdot\sqrt{\frac{\mu r}{n}}\sqrt{\sigma_{\max}}\leq\sqrt{\kappa\mu r}\left\Vert \bm{M}^{\star}\right\Vert _{\infty}.
\]
Here we used the incoherence assumption~(\ref{eq:F-singular-value}). 

\section{Technical lemmas}

%\subsection{Matrix concentration bounds}
 
\begin{lemma}\label{lemma:zheng-and-lafferty}Suppose $n^{2}p\geq Cn\log n$
for some sufficiently large constant $C>0$. Then with probability
exceeding $1-O(n^{-10})$, 
\[
\left|p^{-1}\left\Vert \mathcal{P}_{\Omega}\left(\bm{A}\bm{B}^{\top}\right)\right\Vert _{\mathrm{F}}^{2}-\left\Vert \bm{A}\bm{B}^{\top}\right\Vert _{\mathrm{F}}^{2}\right|\leq3n\min\left\{ \left\Vert \bm{A}\right\Vert _{2,\infty}^{2}\left\Vert \bm{B}\right\Vert _{\mathrm{F}}^{2},\left\Vert \bm{B}\right\Vert _{2,\infty}^{2}\left\Vert \bm{A}\right\Vert _{\mathrm{F}}^{2}\right\} 
\]
holds uniformly for all matrices $\bm{A},\bm{B}\in\mathbb{R}^{n\times r}$.
\end{lemma}\begin{proof}In view of \cite[Lemma 9]{zheng2016convergence},
one has 
\[
p^{-1}\left\Vert \mathcal{P}_{\Omega}\left(\bm{A}\bm{B}^{\top}\right)\right\Vert _{\mathrm{F}}^{2}\leq2n\min\left\{ \left\Vert \bm{A}\right\Vert _{2,\infty}^{2}\left\Vert \bm{B}\right\Vert _{\mathrm{F}}^{2},\left\Vert \bm{B}\right\Vert _{2,\infty}^{2}\left\Vert \bm{A}\right\Vert _{\mathrm{F}}^{2}\right\} 
\]
with high probability. In addition, simple algebra reveals that 
\[
\left\Vert \bm{A}\bm{B}^{\top}\right\Vert _{\mathrm{F}}^{2}\leq\left\Vert \bm{A}\right\Vert _{\mathrm{F}}^{2}\left\Vert \bm{B}\right\Vert _{\mathrm{F}}^{2}\leq n\left\Vert \bm{A}\right\Vert _{2,\infty}^{2}\left\Vert \bm{B}\right\Vert _{\mathrm{F}}^{2}
\]
and, similarly, $\|\bm{A}\bm{B}^{\top}\|_{\mathrm{F}}^{2}\leq n\|\bm{A}\|_{\mathrm{F}}^{2}\|\bm{B}\|_{2,\infty}^{2}$.
Combining the previous bounds with the triangle inequality establishes
the claim. \end{proof}

%\subsection{Matrix perturbation bounds}

%In this section, we collect a few useful lemmas regarding perturbation bounds of matrices.

\begin{lemma} \label{lemma:balance_determine}Let $\bm{U}\bm{\Sigma}\bm{V}^{\top}$
be the SVD of a rank-$r$ matrix $\bm{X}\bm{Y}^{\top}$ with $\bm{X},\bm{Y}\in\bR^{n\times r}$.
Then there exists an invertible matrix $\bm{Q}\in\mathbb{R}^{r\times r}$
such that $\bm{X}=\bm{U}\bm{\Sigma}^{1/2}\bm{Q}$ and $\bm{Y}=\bm{V}\bm{\Sigma}^{1/2}\bm{Q}^{-\top}$.
In addition, one has 
\begin{equation}
\big\|\bSigma_{\bQ}-\bSigma_{\bQ}^{-1}\big\|_{\mathrm{F}}\leq\frac{1}{\sigma_{\min}\left(\bm{\Sigma}\right)}\Fnorm{\bX^{\top}\bX-\bY^{\top}\bY},\label{eq:balance-perturbation}
\end{equation}
where $\bU_{\bQ}\bSigma_{\bQ}\bV_{\bQ}^{\top}$ is the SVD of $\bm{Q}$.
In particular, if $\bm{X}$ and $\bm{Y}$ have balanced scale, i.e.~$\bX^{\top}\bX-\bY^{\top}\bY=\bm{0},$
then $\bm{Q}$ must be a rotation matrix.\end{lemma} \begin{proof}The
existence of $\bm{Q}$ is trivial by setting 
\[
\bm{Q}=\bm{\Sigma}^{-1/2}\bm{U}^{\top}\bm{X}.
\]
To see this, one has 
\[
\bm{U}\bm{\Sigma}^{1/2}\bm{Q}=\bm{U}\bm{\Sigma}^{1/2}\bm{\Sigma}^{-1/2}\bm{U}^{\top}\bm{X}=\bm{U}\bm{U}^{\top}\bm{X}=\bm{X},
\]
where the last equality follows from the fact that the columns of
$\bm{U}$ are the left singular vectors of $\bm{X}$. The relation
$\bm{Y}=\bm{V}\bm{\Sigma}^{1/2}\bm{Q}^{-\top}$ can also be verified
by the identity 
\[
\bm{X}\bm{Y}^{\top}=\bm{U}\bm{\Sigma}^{1/2}\bm{Q}\bm{Y}^{\top}=\bm{U}\bm{\Sigma}\bm{V}^{\top}.
\]

We now move on to proving the perturbation bound (\ref{eq:balance-perturbation}).
In view of the SVD of $\bm{Q}$, i.e. $\bm{Q}=\bU_{\bQ}\bSigma_{\bQ}\bV_{\bQ}^{\top}$,
one can obtain 
\[
\begin{aligned}\bX^{\top}\bX-\bY^{\top}\bY & =\bQ^{\top}\bm{\Sigma}\bQ-\bQ^{-1}\bm{\Sigma}\bQ^{-\top}\\
 & =\bV_{\bQ}\bSigma_{\bQ}\bU_{\bQ}^{\top}\bm{\Sigma}\bU_{\bQ}\bSigma_{\bQ}\bV_{\bQ}^{\top}-\bV_{\bQ}\bSigma_{\bQ}^{-1}\bU_{\bQ}^{\top}\bm{\Sigma}\bU_{\bQ}\bSigma_{\bQ}^{-1}\bV_{\bQ}^{\top}.
\end{aligned}
\]
Denote $\bB:=\bU_{\bQ}^{\top}\bm{\Sigma}\bU_{\bQ}\succ0$. Then we
have 
\[
\Fnorm{\bX^{\top}\bX-\bY^{\top}\bY}^{2}=\Fnorm{\bV_{\bQ}\bSigma_{\bQ}\bB\bSigma_{\bQ}\bV_{\bQ}^{\top}-\bV_{\bQ}\bSigma_{\bQ}^{-1}\bB\bSigma_{\bQ}^{-1}\bV_{\bQ}^{\top}}^{2}=\Fnorm{\bSigma_{\bQ}\bB\bSigma_{\bQ}-\bSigma_{\bQ}^{-1}\bB\bSigma_{\bQ}^{-1}}^{2}.
\]
Let $\bC=\bSigma_{\bQ}\bB^{1/2}$ and $\bD=\bSigma_{\bQ}^{-1}\bB^{1/2}$,
and denote $\bDelta=\bC-\bD$. One then has 
\[
\begin{aligned}\Fnorm{\bX^{\top}\bX-\bY^{\top}\bY}^{2} & =\Fnorm{\bC\bC^{\top}-\bD\bD^{\top}}^{2}=\Fnorm{\bC\bDelta^{\top}+\bDelta\bC^{\top}-\bDelta\bDelta^{\top}}^{2}\\
 & =\trace{2\bC^{\top}\bC\bDelta^{\top}\bDelta+\bDelta\bDelta^{\top}\bDelta\bDelta^{\top}+2\bC^{\top}\bDelta\bC^{\top}\bDelta-4\bC^{\top}\bDelta\bDelta^{\top}\bDelta}\\
 & =\text{Tr}\Big[\big(\bDelta^{\top}\bDelta-\sqrt{2}\bC^{\top}\bDelta\big)^{2}+(4-2\sqrt{2})\bC^{\top}(\bC-\bDelta)\bDelta^{\top}\bDelta+(2\sqrt{2}-1)\bC^{\top}\bC\bDelta^{\top}\bDelta\Big].
\end{aligned}
\]
Note that $\bm{C}^{\top}\bm{D}=\bm{B}$ and that 
$\bC^{\top}\bDelta=\bm{C}^{\top}\bm{C}-\bm{C}^{\top}\bm{D}=\bC^{\top}\bC-\bB$
is symmetric. One can continue the bound as 
\begin{align*}
\Fnorm{\bX^{\top}\bX-\bY^{\top}\bY}^{2} & =\Fnorm{\bDelta^{\top}\bDelta-\sqrt{2}\bC^{\top}\bDelta}^{2}+(4-2\sqrt{2})\trace{\bm{B}\bDelta\bDelta^{\top}}+(2\sqrt{2}-1)\Fnorm{\bC\bDelta^{\top}}^{2}\\
 & \geq\trace{\bm{B}\bDelta\bDelta^{\top}},
\end{align*}
where the inequality follows since $4-2\sqrt{2}\geq1$. Write $\bm{B}=\bm{B}^{1/2}\cdot\bm{B}^{1/2}$
to see 
\begin{align*}
\Fnorm{\bX^{\top}\bX-\bY^{\top}\bY}^{2} & \geq\mathrm{Tr}\big(\bm{B}^{1/2}\bDelta\bDelta^{\top}\bm{B}^{1/2}\big)=\big\|\bm{B}^{1/2}\bm{\Delta}\big\|_{\mathrm{F}}^{2}\\
 & =\big\|\bm{B}^{1/2}\big(\bSigma_{\bQ}-\bSigma_{\bQ}^{-1}\big)\bB^{1/2}\big\|_{\mathrm{F}}^{2}\\
 & \geq\sigma_{\min}^{2}\left(\bm{B}\right)\Fnorm{\bSigma_{\bQ}-\bSigma_{\bQ}^{-1}}^{2}.
\end{align*}
Recognizing that $\sigma_{\min}(\bm{B})=\sigma_{\min}(\bm{\Sigma})$
finishes the proof of (\ref{eq:balance-perturbation}).

Combining $\bm{X}^{\top}\bm{X}=\bm{Y}^{\top}\bm{Y}$ and (\ref{eq:balance-perturbation})
yields $\big\|\bSigma_{\bQ}-\bSigma_{\bQ}^{-1}\big\|_{\mathrm{F}}=0$,
which implies $\bm{\Sigma}_{\bm{Q}}=\bm{I}$. Under this circumstance,
$\bm{Q}=\bm{U}_{\bm{Q}}\bm{\Sigma}_{\bm{Q}}\bm{V}_{\bm{Q}}^{\top}=\bm{U}_{\bm{Q}}\bm{V}_{\bm{Q}}^{\top}$
is a rotation matrix. The proof is then complete.\end{proof}

\begin{lemma}\label{lemma:chen-and-ji}For all $\bm{A},\bm{B},\bm{C},\bm{D}\in\mathbb{R}^{n\times r}$,
one has 
\[
\left|\left\langle \mathcal{P}_{\Omega}\left(\bm{A}\bm{C}^{\top}\right),\mathcal{P}_{\Omega}\left(\bm{B}\bm{D}^{\top}\right)\right\rangle -p\left\langle \bm{A}\bm{C}^{\top},\bm{B}\bm{D}^{\top}\right\rangle \right|\leq\left\Vert \mathcal{P}_{\Omega}\left(\bm{1}\bm{1}^{\top}\right)-p\bm{1}\bm{1}^{\top}\right\Vert \left\Vert \bm{A}\right\Vert _{2,\infty}\left\Vert \bm{B}\right\Vert _{\mathrm{F}}\left\Vert \bm{C}\right\Vert _{2,\infty}\left\Vert \bm{D}\right\Vert _{\mathrm{F}}.
\]
\end{lemma}\begin{proof}This is a simple consequence of \cite[Lemma 4.4]{chen2017memory},
where they have shown 
\begin{align*}
 & \left|\left\langle \mathcal{P}_{\Omega}\left(\bm{A}\bm{C}^{\top}\right),\mathcal{P}_{\Omega}\left(\bm{B}\bm{D}^{\top}\right)\right\rangle -p\left\langle \bm{A}\bm{C}^{\top},\bm{B}\bm{D}^{\top}\right\rangle \right|\\
 & \quad\leq\left\Vert \mathcal{P}_{\Omega}\left(\bm{1}\bm{1}^{\top}\right)-p\bm{1}\bm{1}^{\top}\right\Vert \sqrt{\sum\nolimits _{k=1}^{n}\left\Vert \bm{A}_{k,\cdot}\right\Vert _{2}^{2}\left\Vert \bm{B}_{k,\cdot}\right\Vert _{2}^{2}}\sqrt{\sum\nolimits _{k=1}^{n}\left\Vert \bm{C}_{k,\cdot}\right\Vert _{2}^{2}\left\Vert \bm{D}_{k,\cdot}\right\Vert _{2}^{2}}.
\end{align*}
Recognize that 
\[
\sum\nolimits _{k=1}^{n}\left\Vert \bm{A}_{k,\cdot}\right\Vert _{2}^{2}\left\Vert \bm{B}_{k,\cdot}\right\Vert _{2}^{2}\leq\left\Vert \bm{A}\right\Vert _{2,\infty}^{2}\sum\nolimits _{k=1}^{n}\left\Vert \bm{B}_{k,\cdot}\right\Vert _{2}^{2}=\left\Vert \bm{A}\right\Vert _{2,\infty}^{2}\left\Vert \bm{B}\right\Vert _{\mathrm{F}}^{2}
\]
and, similarly, $\sum_{k}\|\bm{C}_{k,\cdot}\|_{2}^{2}\|\bm{D}_{k,\cdot}\|_{2}^{2}\leq\|\bm{C}\|_{2,\infty}^{2}\|\bm{D}\|_{\mathrm{F}}^{2}$.
Putting these together concludes the proof. \end{proof}

\begin{lemma}\label{lemma:rotation-perturbation} Suppose $\bm{F}_{1},\bm{F}_{2},\bm{F}_{0}\in\mathbb{R}^{2n\times r}$
are three matrices such that 
\[
\left\Vert \bm{F}_{1}-\bm{F}_{0}\right\Vert \left\Vert \bm{F}_{0}\right\Vert \leq\sigma_{r}^{2}\left(\bm{F}_{0}\right)/2\qquad\text{and}\qquad\left\Vert \bm{F}_{1}-\bm{F}_{2}\right\Vert \left\Vert \bm{F}_{0}\right\Vert \leq\sigma_{r}^{2}\left(\bm{F}_{0}\right)/4,
\]
where $\sigma_{i}(\bm{A})$ stands for the $i$th largest singular
value of $\bm{A}$. Denote 
\[
\bm{R}_{1}\triangleq\arg\min_{\bm{R}\in\mathcal{O}^{r\times r}}\left\Vert \bm{F}_{1}\bm{R}-\bm{F}_{0}\right\Vert _{\mathrm{F}}\qquad\text{and}\qquad\bm{R}_{2}\triangleq\arg\min_{\bm{R}\in\mathcal{O}^{r\times r}}\left\Vert \bm{F}_{2}\bm{R}-\bm{F}_{0}\right\Vert _{\mathrm{F}}.
\]
Then the following two inequalities hold true: 
\[
\left\Vert \bm{F}_{1}\bm{R}_{1}-\bm{F}_{2}\bm{R}_{2}\right\Vert \leq5\frac{\sigma_{1}^{2}\left(\bm{F}_{0}\right)}{\sigma_{r}^{2}\left(\bm{F}_{0}\right)}\left\Vert \bm{F}_{1}-\bm{F}_{2}\right\Vert \qquad\text{and}\qquad\left\Vert \bm{F}_{1}\bm{R}_{1}-\bm{F}_{2}\bm{R}_{2}\right\Vert _{\mathrm{F}}\leq5\frac{\sigma_{1}^{2}\left(\bm{F}_{0}\right)}{\sigma_{r}^{2}\left(\bm{F}_{0}\right)}\left\Vert \bm{F}_{1}-\bm{F}_{2}\right\Vert _{\mathrm{F}}.
\]
\end{lemma}\begin{proof}This is the same as \cite[Lemma 37]{ma2017implicit}.
\end{proof}

\begin{lemma}\label{lemma:Ma36}Let $\bm{S}\in\mathbb{R}^{r\times r}$
be a nonsingular matrix. Then for any matrix $\bm{K}\in\mathbb{R}^{r\times r}$
with $\|\bm{K}\|\leq\sigma_{\min}(\bm{S})$, one has 
\[
\|\mathsf{sgn}(\bm{S}+\bm{K})-\mathsf{sgn}(\bm{S})\|\leq\frac{2}{\sigma_{r-1}(\bm{S})+\sigma_{r}(\bm{S})}\|\bm{K}\|,
\]
where $\mathsf{sgn}(\cdot)$ denotes the matrix sign function, i.e. $\mathsf{sgn}(\bm{A})=\bm{U}\bm{V}^{\top}$ for a matrix $\bm{A}$ with SVD $\bm{U}\bm{\Sigma}\bm{V}^{\top}$. \end{lemma}\begin{proof}This
is the same as \cite[Lemma 36]{ma2017implicit}. \end{proof}

\end{document}